\documentclass[letterpaper]{article}
\usepackage[margin=1in,dvips]{geometry}
\usepackage{graphicx,psfrag,amsmath,amsthm,amssymb}
\usepackage{natbib}
\usepackage{url,color}
\usepackage{algorithm,algorithmic}

\newcommand{\A}{\ensuremath{\mathbf{A}}}
\newcommand{\B}{\ensuremath{\mathbf{B}}}

\newcommand{\I}{\ensuremath{\mathbf{I}}}

\newcommand{\Q}{\ensuremath{\mathbf{Q}}}
\newcommand{\RR}{\ensuremath{\mathbf{R}}}

\newcommand{\V}{\ensuremath{\mathbf{V}}}
\newcommand{\W}{\ensuremath{\mathbf{W}}}
\newcommand{\X}{\ensuremath{\mathbf{X}}}
\newcommand{\Y}{\ensuremath{\mathbf{Y}}}
\newcommand{\Z}{\ensuremath{\mathbf{Z}}}

\renewcommand{\b}{\ensuremath{\mathbf{b}}}

\newcommand{\e}{\ensuremath{\mathbf{e}}}
\newcommand{\f}{\ensuremath{\mathbf{f}}}

\newcommand{\h}{\ensuremath{\mathbf{h}}}

\newcommand{\q}{\ensuremath{\mathbf{q}}}

\newcommand{\w}{\ensuremath{\mathbf{w}}}
\newcommand{\x}{\ensuremath{\mathbf{x}}}
\newcommand{\y}{\ensuremath{\mathbf{y}}}
\newcommand{\z}{\ensuremath{\mathbf{z}}}
\newcommand{\0}{\ensuremath{\mathbf{0}}}

\newcommand{\bbN}{\ensuremath{\mathbb{N}}}
\newcommand{\bbR}{\ensuremath{\mathbb{R}}}

\newcommand{\calO}{\ensuremath{\mathcal{O}}}

\newcommand{\abs}[1]{\left\lvert#1\right\rvert}
\newcommand{\norm}[1]{\left\lVert#1\right\rVert}

\newcommand{\caja}[4][1]{{%
    \renewcommand{\arraystretch}{#1}%
    \begin{tabular}[#2]{@{}#3@{}}%
      #4%
    \end{tabular}%
    }}

{%
\begin{list}{#1}{
\vspace{-\topsep}
\vspace{-\partopsep}
\setlength{\itemindent}{0cm}
\setlength{\rightmargin}{0cm}
\setlength{\listparindent}{0cm}
\settowidth{\labelwidth}{#1}
\setlength{\leftmargin}{\labelwidth}
\addtolength{\leftmargin}{\labelsep}
\setlength{\itemsep}{0cm}
}%
}%
{%
\end{list}
\vspace{-\topsep}
\vspace{-\partopsep}
}

{\begin{enumerate}%
}%
{\end{enumerate}}

\newcommand{\sgnop}{\operatorname{sgn}}
\newcommand{\sgn}[1]{\ensuremath{\sgnop\left(#1\right)}}

\newcommand{\diagop}{\operatorname{diag}}
\newcommand{\diag}[1]{\ensuremath{\diagop\left(#1\right)}}

\theoremstyle{plain}
\newtheorem{thm}{Theorem}[section]

\newtheorem*{lemma*}{Lemma}

\newtheorem*{prop*}{Proposition}

\theoremstyle{definition}

\newtheorem*{defn*}{Definition}

\newtheorem*{exmp*}{Example}

\newtheorem*{conj*}{Conjecture}

\theoremstyle{remark}

\newtheorem*{rmk*}{Remark}

\graphicspath{{grf/}}

\bibpunct[, ]{(}{)}{;}{a}{,}{,} 

\title{\vspace*{-1cm}Hashing with Binary Autoencoders}
\author{
  Miguel \'A. Carreira-Perpi\~n\'an\hspace{5ex} Ramin Raziperchikolaei \\
  Electrical Engineering and Computer Science, University of California, Merced \\
  {\url{http://eecs.ucmerced.edu}}
}
\date{January 4, 2015}

\AtBeginDvi{}

\begin{document}

\maketitle

\begin{abstract}
  
  An attractive approach for fast search in image databases is binary hashing, where each high-dimensional, real-valued image is mapped onto a low-dimensional, binary vector and the search is done in this binary space. Finding the optimal hash function is difficult because it involves binary constraints, and most approaches approximate the optimization by relaxing the constraints and then binarizing the result. Here, we focus on the binary autoencoder model, which seeks to reconstruct an image from the binary code produced by the hash function. We show that the optimization can be simplified with the method of auxiliary coordinates. This reformulates the optimization as alternating two easier steps: one that learns the encoder and decoder separately, and one that optimizes the code for each image. Image retrieval experiments, using precision/recall and a measure of code utilization, show the resulting hash function outperforms or is competitive with state-of-the-art methods for binary hashing.

\end{abstract}

\section{Introduction}

We consider the problem of binary hashing, where given a high-dimensional vector $\x \in \bbR^D$, we want to map it to an $L$-bit vector $\z = \h(\x) \in \{0,1\}^L$ using a hash function \h, while preserving the neighbors of \x\ in the binary space. Binary hashing has emerged in recent years as an effective technique for fast search on image (and other) databases. While the search in the original space would cost $\calO(ND)$ in both time and space, using floating point operations, the search in the binary space costs $\calO(NL)$ where $L \ll D$ and the constant factor is much smaller. This is because the hardware can compute binary operations very efficiently and the entire dataset ($NL$ bits) can fit in the main memory of a workstation. And while the search in the binary space will produce some false positives and negatives, one can retrieve a larger set of neighbors and then verify these with the ground-truth distance, while still being efficient.

Many different hashing approaches have been proposed in the last few years. They formulate an objective function of the hash function \h\ or of the binary codes that tries to capture some notion of neighborhood preservation. Most of these approaches have two things in common: \h\ typically performs dimensionality reduction ($L < D$) and, as noted, it outputs binary codes ($\h\mathpunct{:} \bbR^D \rightarrow \{0,1\}^L$). The latter implies a step function or binarization applied to a real-valued function of the input \x. Optimizing this is difficult. In practice, most approaches follow a two-step procedure: first they learn a real hash function ignoring the binary constraints and then the output of the resulting hash function is binarized (e.g.\ by thresholding or with an optimal rotation). For example, one can run a continuous dimensionality reduction algorithm (by optimizing its objective function) such as PCA and then apply a step function. This procedure can be seen as a ``filter'' approach \citep{KohaviJohn98a} and is suboptimal: in the example, the thresholded PCA projection is not necessarily the best thresholded linear projection (i.e., the one that minimizes the objective function under all thresholded linear projections). To obtain the latter, we must optimize the objective jointly over linear mappings and thresholds, respecting the binary constraints while learning \h; this is a ``wrapper'' approach \citep{KohaviJohn98a}. In other words, optimizing real codes and then projecting them onto the binary space is not the same as optimizing the codes in the binary space.

In this paper we show that this joint optimization, respecting the binary constraints during training, can actually be carried out reasonably efficiently. The idea is to use the recently proposed \emph{method of auxiliary coordinates (MAC)} \citep{CarreirWang12a,CarreirWang14a}. This is a general strategy to transform an original problem involving a nested function into separate problems without nesting, each of which can be solved more easily. In our case, this allows us to reduce drastically the complexity due to the binary constraints. We focus on \emph{binary autoencoders}, i.e., where the code layer is binary. We believe we are the first to apply MAC to this model and construct an efficient optimization algorithm. Section~\ref{s:BA} describes the binary autoencoder model and objective function. Section~\ref{s:MAC} derives a training algorithm using MAC and explains how, with carefully implemented steps, the optimization in the binary space can be carried out efficiently, and parallelizes well. Our hypothesis is that constraining the optimization to the binary space results in better hash functions and we test this in experiments (section~\ref{s:expts}), using several performance measures: the traditional precision/recall, as well as the reconstruction error and an entropy-based measure of code utilization (which we propose in section~\ref{s:entropy}). These show that linear hash functions resulting from optimizing a binary autoencoder using MAC are consistently competitive with the state-of-the-art, even when the latter uses nonlinear hash functions or more sophisticated objective functions for hashing.

\enlargethispage*{0.5cm}

\section{Related work}

The most basic hashing approaches are data-independent, such as Locality-Sensitive Hashing (LSH) \citep{AndoniIndyk08a}, which is based on random projections and thresholding, and kernelized LSH \citep{KulisGrauman12a}. Generally, they are outperformed by data-dependent methods, which learn a specific hash function for a given dataset in an unsupervised or supervised way. We focus here on unsupervised, data-dependent approaches. These are typically based on defining an objective function (usually based on dimensionality reduction) either of the hash function or the binary codes, and optimizing it. However, this is usually achieved by relaxing the binary codes to a continuous space and thresholding the resulting continuous solution. For example, spectral hashing \citep{Weiss_09a} is essentially a version of Laplacian eigenmaps where the binary constraints are relaxed and approximate eigenfunctions are computed that are then thresholded to provide binary codes. Variations of this include using AnchorGraphs \citep{Liu_11a} to define the eigenfunctions, or obtaining the hash function directly as a binary classifier using the codes from spectral hashing as labels \citep{Zhang_10e}. Other approaches optimize instead a nonlinear embedding objective that depends on a continuous, parametric hash function, which is then thresholded to define a binary hash function \citep{Torral_08b,SalakhHinton09b}; or an objective that depends on a thresholded hash function, but the threshold is relaxed during the optimization \citep{NorouzFleet11a}. Some recent work has tried to respect the binary nature of the codes or thresholds by using alternating optimization directly on an objective function, over one entry or one row of the weight matrix in the hash function \citep{KulisDarrel09a,Neyshab_13a}, or over a subset of the binary codes \citep{Lin_13a,Lin_14b}. Since the objective function involves a large number of terms and all the binary codes or weights are coupled, the optimization is very slow. Also, \citet{Lin_13a,Lin_14b} learn the hash function after the codes have been fixed, which is suboptimal.

The closest model to our binary autoencoder is Iterative Quantization (ITQ) \citep{Gong_13a}, a fast and competitive hashing method. ITQ first obtains continuous low-dimensional codes by applying PCA to the data and then seeks a rotation that makes the codes as close as possible to binary. The latter is based on the optimal discretization algorithm of \citet{YuShi03a}, which finds a rotation of the continuous eigenvectors of a graph Laplacian that makes them as close as possible to a discrete solution, as a postprocessing for spectral clustering. The ITQ objective function is
\begin{equation}
  \label{e:ITQ}
  \min_{\B,\RR}{ \norm{\B - \V\RR}^2 } \qquad \text{s.t.} \qquad \B\in\{-1,+1\}^{NL},\ \RR\ \text{ orthogonal}
\end{equation}
where \V\ of $N \times L$ are the continuous codes obtained by PCA. This is an NP-complete problem, and a local minimum is found using alternating optimization over \B, with solution $\B=\sgn{\V\RR}$ elementwise, and over \RR, which is a Procrustes alignment problem with a closed-form solution based on a SVD. The final hash function is $\h(\x) = \sgn{\W\x}$, which has the form of a thresholded linear projection. Hence, ITQ is a postprocessing of the PCA codes, and it can be seen as a suboptimal approach to optimizing a binary autoencoder, where the binary constraints are relaxed during the optimization (resulting in PCA), and then one ``projects'' the continuous codes back to the binary space. Semantic hashing \citep{SalakhHinton09b} also uses an autoencoder objective with a deep encoder (consisting of stacked RBMs), but again its optimization uses heuristics that are not guaranteed to converge to a local optimum: either training it as a continuous problem with backpropagation and then applying a threshold to the encoder \citep{SalakhHinton09b}, or rounding the encoder output to 0 or 1 during the backpropagation forward pass but ignoring the rounding during the backward pass \citep{KrizhevHinton11a}.

\section{Our hashing models: binary autoencoder and binary factor analysis}
\label{s:BA}

We consider a well-known model for continuous dimensionality reduction, the (continuous) autoencoder, defined in a broad sense as the composition of an encoder $\h(\x)$ which maps a real vector $\x\in\bbR^D$ onto a real code vector $\z\in\bbR^L$ (with $L<D$), and a decoder $\f(\z)$ which maps \z\ back to $\bbR^D$ in an effort to reconstruct \x. Although our ideas apply more generally to other encoders, decoders and objective functions, in this paper we mostly focus on the least-squares error with a linear encoder and decoder. As is well known, the optimal solution is PCA.

For hashing, the encoder maps continuous inputs onto \emph{binary} code vectors with $L$ bits, $\z\in\{0,1\}^L$. Let us write $\h(\x) = \sigma(\W\x)$ (\W\ includes a bias by having an extra dimension $x_0=1$ for each \x) where $\W\in\bbR^{L\times (D+1)}$ and $\sigma(t)$ is a step function applied elementwise, i.e., $\sigma(t) = 1$ if $t\ge 0$ and $\sigma(t) = 0$ otherwise (we can fix the threshold at 0 because the bias acts as a threshold for each bit). Our \emph{desired hash function} will be \h, and it should minimize the following problem, given a dataset of high-dimensional patterns $\X = (\x_1,\dots,\x_N)$:
\begin{equation}
  \label{e:nested}
  E_{\text{BA}}(\h,\f) = \sum^N_{n=1}{ \norm{\x_n - \f(\h(\x_n))}^2 }
\end{equation}
which is the usual least-squares error but where the code layer is binary. Optimizing this nonsmooth function is difficult and NP-complete. Where the gradients do exist wrt \W\ they are zero nearly everywhere. We call this a \emph{binary autoencoder (BA)}.

We will also consider a related model (see later):
\begin{equation}
  \label{e:BFA}
  E_{\text{BFA}}(\Z,\f) = \sum^N_{n=1}{ \norm{\x_n - \f(\z_n)}^2 } \qquad \text{s.t.} \qquad \z_n\in\{0,1\}^L,\ n=1,\dots,N
\end{equation}
where \f\ is linear and we optimize over the decoder \f\ and the binary codes $\Z = (\z_1,\dots,\z_N)$ of each input pattern. Without the binary constraint, i.e., $\Z\in\bbR^{L\times N}$, this model dates back to the 50s and is sometimes called least-squares factor analysis \citep{Whittl52a}, and its solution is PCA. With the binary constraints, the problem is NP-complete because it includes as particular case solving a linear system over $\{0,1\}^L$, which is an integer LP feasibility problem. We call this model (least-squares) \emph{binary factor analysis (BFA)}. We believe this model has not been studied before, at least in hashing. A hash function \h\ can be obtained from BFA by fitting a binary classifier of the inputs to each of the $L$ code bits. It is a filter approach, while the BA is the optimal (wrapper) approach, since it optimizes jointly over \f\ and \h.

Let us compare BFA (without bias term) with ITQ in eq.~\eqref{e:ITQ}. BFA takes the form $\min_{\Z,\A} \smash{\norm{\X - \A\Z}^2}$ s.t.\ $\Z\in\{0,1\}^{NL}$. The main difference is that in BFA the binary variables do not appear in separable form and so the step over them cannot be solved easily.

\section{Optimization of BA and BFA using the method of auxiliary coordinates (MAC)}
\label{s:MAC}

We use the recently proposed \emph{method of auxiliary coordinates (MAC)} \citep{CarreirWang12a,CarreirWang14a}. The idea is to break nested functional relationships judiciously by introducing variables as equality constraints. These are then solved by optimizing a penalized function using alternating optimization over the original parameters and the coordinates, which results in a coordination-minimization (CM) algorithm. Recall eq.~\eqref{e:nested}, this is our nested problem, where the model is $\y = \f(\h(\x))$. We introduce as auxiliary coordinates the outputs of \h, i.e., the codes for each of the $N$ input patterns, and obtain the following equality-constrained problem:
\begin{equation}
  \label{e:MAC-constrained}
  \min_{\h,\f,\Z}{ \sum^N_{n=1}{ \norm{\x_n - \f(\z_n)}^2 } } \qquad \text{s.t.} \qquad \z_n = \h(\x_n)\in\{0,1\}^L,\ n=1,\dots,N.
\end{equation}
Note the codes are binary. We now apply the quadratic-penalty method (it is also possible to apply the augmented Lagrangian method instead; \citealp{NocedalWright06a}) and minimize the following objective function while progressively increasing $\mu$, so the constraints are eventually satisfied:
\begin{equation}
  \label{e:MAC-QP}
  E_Q(\h,\f,\Z;\mu) = \sum^N_{n=1}{ \left( \norm{\x_n - \f(\z_n)}^2 + \mu \norm{\z_n - \h(\x_n)}^2 \right) } \qquad \text{s.t.} \qquad \z_n\in\{0,1\}^L,\ n=1,\dots,N.
\end{equation}
Now we apply alternating optimization over \Z\ and $(\h,\f)$. This results in the following two steps:
\begin{itemize}
\item Over \Z\ for fixed $(\h,\f)$, the problem separates for each of the $N$ codes. The optimal code vector for pattern $\x_n$ tries to be close to the prediction $\h(\x_n)$ while reconstructing $\x_n$ well.
\item Over $(\h,\f)$ for fixed \Z, we obtain $L+1$ independent problems for each of the $L$ single-bit hash functions (which try to predict \Z\ optimally from \X), and for \f\ (which tries to reconstruct \X\ optimally from \Z).
\end{itemize}
We can now see the advantage of the auxiliary coordinates: the individual steps are (reasonably) easy to solve (although some work is still needed, particularly for the \Z\ step), and besides they exhibit significant parallelism. We describe the steps in detail below. The resulting algorithm alternates steps over the encoder ($L$ classifications) and decoder (one regression) and over the codes ($N$ binary proximal operators; \citealp{Rockaf76b,CombetPesquet11a}). During the iterations, we allow the encoder and decoder to be mismatched, since the encoder output does not equal the decoder input, but they are coordinated by \Z\ and as $\mu$ increases the mismatch is reduced. The overall MAC algorithm to optimize a BA is in fig.~\ref{f:BA-alg}.

\begin{figure}[t]
  \begin{center}
    \setlength{\fboxsep}{1ex}
    \framebox{%
      \begin{minipage}[c]{0.70\columnwidth}
        \begin{tabbing}
          n \= n \= n \= n \= n \= \kill
          \underline{\textbf{input}} $\X_{D \times N} = (\x_1,\dots,\x_N)$, $L \in \bbN$ \\
          Initialize $\Z_{L \times N} = (\z_1,\dots,\z_N) \in \{0,1\}^{LN}$ \\
          \underline{\textbf{for}} $\mu = 0 < \mu_1 < \dots < \mu_{\infty}$ \+ \\
          \underline{\textbf{for}} $l = 1,\dots,L$ \` {\small\textsf{\h\ step}} \+ \\
          $h_l \leftarrow$ fit SVM to $(\X,\Z_{\cdot l})$ \- \\
          $\f \leftarrow$ least-squares fit to $(\Z,\X)$ \` {\small\textsf{\f\ step}} \\
          \underline{\textbf{for}} $n = 1,\dots,N$ \` {\small\textsf{\Z\ step}} \+ \\
          $\z_n \leftarrow \arg\min_{\z_n\in\{0,1\}^L}{\norm{\y_n-\f(\z_n)}^2 + \mu \norm{\z_n-\h(\x_n)}^2}$ \- \\
          \underline{\textbf{if}} $\Z = \h(\X)$ \underline{\textbf{then}} stop \- \\
          \underline{\textbf{return}} \h, $\Z = \h(\X)$
        \end{tabbing}
      \end{minipage}
    }
    \caption{Binary autoencoder MAC algorithm.}
    \label{f:BA-alg}
  \end{center}
\end{figure}

Although a MAC algorithm can be shown to produce convergent algorithms as $\mu\rightarrow\infty$ with a differentiable objective function, we cannot apply the theorem of \citet{CarreirWang12a} because of the binary nature of the problem. Instead, we show that our algorithm converges to a local minimum for a \emph{finite} $\mu$, where ``local minimum'' is understood as in $k$-means: a point where \Z\ is globally minimum given $(\h,\f)$ and vice versa. The following theorem is valid for any choice of \h\ and \f, not just linear.
\begin{thm}
  \label{th:finite-mu}
  Assume the steps over \h\ and \f\ are solved exactly. Then the MAC algorithm for the binary autoencoder stops at a finite $\mu$.
\end{thm}
\begin{proof}
  $\mu$ appears only in the \Z\ step, and if \Z\ does not change there, \f\ and \h\ will not change either, since the \h\ and \f\ steps are exact. The \Z\ step over $\z_n$ minimizes $\smash{\norm{\x_n - \f(\z_n)}^2} + \mu \smash{\norm{\z_n - \h(\x_n)}^2}$, and from theorem~\ref{th:h-bound} we have that $\h(\x_n)$ is a global minimizer if $\mu > \smash{\norm{\x_n - \f(\h(\x_n))}^2}$. The statement follows from the fact that $\smash{\norm{\x_n - \f(\z)}^2}$ is bounded over all $n$, \z\ and \f. Let us prove this fact. Clearly this holds for fixed \f\ because $(n,\z_n)$ take values on a finite set, namely $\{1,\dots,N\} \times \{0,\dots,2^L-1\}$. As for \f, even if the set of functions \f\ is infinite, the number of different functions \f\ that are possible is finite, because \f\ results from an exact fit to $(\Z,\X)$, where \X\ is fixed and the set of possible \Z\ is finite (since each $z_{nl}$ is binary).
\end{proof}
The minimizers of $E_Q(\h,\f,\Z;\mu)$ trace a path as a function of $\mu \ge 0$ in the $(\h,\f,\Z)$ space. BA and BFA can be seen as the limiting cases of $E_Q(\h,\f,\Z;\mu)$ when $\mu\rightarrow \infty$ and $\mu\rightarrow 0^+$, respectively (for BFA, \f\ and \Z\ can be optimized independently from \h, but \h\ must optimally fit the resulting \Z). Figure~\ref{f:path} shows graphically the connection between the BA and BFA objective functions, as the two ends of the continuous path in $(\h,\f,\Z)$ space in which the quadratic-penalty function $E_Q(\h,\f,\Z;\mu)$ is defined.

\begin{figure}[t]
  \centering
  \psfrag{lim0}[][]{BFA: $\mu\rightarrow 0^+$}
  \psfrag{limInf}[r][r]{BA: $\mu\rightarrow \infty$}
  \psfrag{hfZ_mu}[][]{$(\h,\f,\Z)(\mu)$}
  \psfrag{h}{\h}
  \psfrag{f}[r][r]{\f}
  \psfrag{Z}[r][r]{\Z}
  \psfrag{BFAEquation}[lb][l]{$E_{\text{BFA}}(\Z,\f) = \displaystyle\sum^N_{n=1}{ \norm{\x_n - \f(\z_n)}^2 }$}
  \psfrag{BAEquation}[r][r]{$E_{\text{BA}}(\h,\f) = \displaystyle\sum^N_{n=1}{ \norm{\x_n - \f(\h(\x_n))}^2 }$}
  $E_Q(\h,\f,\Z;\mu) = \displaystyle\sum^N_{n=1}{ \left( \norm{\x_n - \f(\z_n)}^2 + \mu \norm{\z_n - \h(\x_n)}^2 \right) }$ \\
  \includegraphics[width=0.8\linewidth]{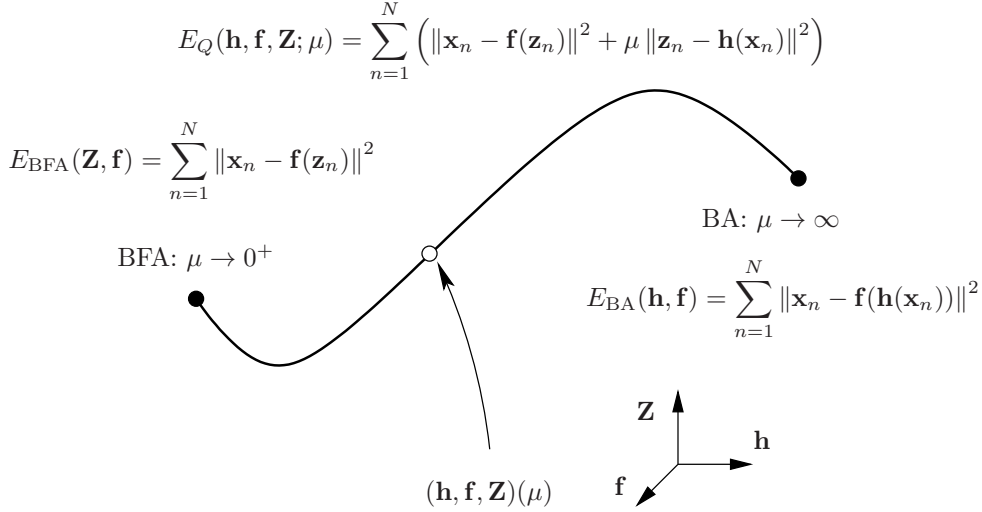}
  \caption{A continuous path induced by the quadratic-penalty objective function as a function of $\mu$ from BFA to BA. The objective functions of BA, BFA and the quadratic-penalty objective are related as shown in fig.~\ref{f:path}. The minimizers of the quadratic-penalty objective $E_Q(\h,\f,\Z;\mu)$ trace a path $(\h,\f,\Z)(\mu)$ continuously as a function of $\mu \ge 0$ in the $(\h,\f,\Z)$ space. The BFA and BA objective functions correspond to the limits $\mu\rightarrow 0^+$ and $\mu\rightarrow \infty$, respectively.}
  \label{f:path}
\end{figure}

In practice, to learn the BFA model we set $\mu$ to a small value and keep it constant while running the BA algorithm. As for BA itself, we increase $\mu$ (times a constant factor, e.g.\ $2$) and iterate the \Z\ and $(\h,\f)$ steps for each $\mu$ value. Usually the algorithm stops in 10 to 15 iterations, when no further changes to the parameters occur.

\subsection{\f\ step}
\label{s:f-step}

With a linear decoder this is a simple linear regression
\begin{equation*}
  \min_{\A,\b}{\sum^N_{n=1}{\norm{\x_n - \A\z_n - \b}^2}}
\end{equation*}
with data $(\Z,\X)$. A bias parameter \b\ is necessary, for example to be able to use $\{-1,+1\}$ instead of $\{0,1\}$ equivalently as bit values. The solution of this regression is (ignoring the bias for simplicity) $\A = \X\Z^T(\Z\Z^T)^{-1}$ and can be computed in $\calO(NDL)$. Note the constant factor in the \calO-notation is small because \Z\ is binary, e.g.\ $\X\Z^T$ involves only sums, not multiplications.

\subsection{\h\ step}
\label{s:h-step}

This has the following form
\begin{equation}
  \label{e:h-step}
  \min_{\h}{ \sum^N_{n=1}{ \norm{\z_n - \h(\x_n)}^2 } } = \min_{\W}{ \sum^N_{n=1}{ \norm{\z_n - \sigma(\W\x_n)}^2 } } = \sum^L_{l=1}{ \min_{\w_l}{ \sum^N_{n=1}{ (z_{nl} - \sigma(\w^T_l\x_n))^2 } } }.
\end{equation}
Since \Z\ and $\sigma(\cdot)$ are binary, $\smash{\norm{\cdot}^2}$ is the Hamming distance and the objective function is the number of misclassified patterns, so it separates for each bit. So it is a classification problem for each bit, using as labels the auxiliary coordinates, where $h_l$ is a linear classifier (a perceptron). However, rather than minimizing this, we will solve an easier, closely related problem: fit a linear SVM $h_l$ to $(\X,\Z_{\cdot l})$ where we use a high penalty for misclassified patterns but optimize the margin plus the slack. Besides being easier (by reusing well-developed codes for SVMs), this surrogate loss has the advantage of making the solution unique (no local optima) and generalizing better to test data (maximum margin). Also, although we used a quadratic penalty, the spirit of penalty methods is to penalize constraint violations ($\z_n - \h(\x_n)$) increasingly. Since in the limit $\mu\rightarrow\infty$ the constraints are satisfied exactly, the classification error using \h\ is zero, hence the linear SVM will find an optimum of the nested problem anyway. We use LIBLINEAR \citep{Fan_08a} with warm start (i.e., the SVM optimization is initialized from the previous iteration's SVM). Note the $L$ SVMs and the decoder function \f\ can be trained in parallel.

\subsection{\Z\ step}
\label{s:Z-step}

From eq.~\eqref{e:MAC-QP}, this is a binary optimization on $NL$ variables, but it separates into $N$ independent optimizations each on only $L$ variables, with the form of a binary proximal operator \citep{Moreau62a,Rockaf76b,CombetPesquet11a} (where we omit the index $n$):
\begin{equation*}
  \label{e:z-step}
  \textstyle\min_{\z}{ e(\z) = \smash{\norm{\x - \f(\z)}^2} + \mu \smash{\norm{\z - \h(\x)}^2} } \ \text{ s.t.\ } \ \z\in\{0,1\}^L.
\end{equation*}
Thus, although the problem over each $\z_n$ is binary and NP-complete, a good or even exact solution may be obtained, because practical values of $L$ are small (typically 8 to 32 bits). Further, because of the intensive computation and large number of independent problems, this step can take much advantage of parallel processing.

We have spent significant effort into making this step efficient while yielding good, if not exact, solutions. Before proceeding, let us show%
\footnote{This result can also be derived by expanding the norms into an $L \times L$ matrix $\A^T\A$ and computing its Cholesky decomposition. However, the product $\smash{\A^T\A}$ squares the singular values of \A\ and loses precision because of roundoff error (\citealp[p.~237ff]{GolubLoan96a}; \citealp[pp.~251]{NocedalWright06a}).}
how to reduce the problem, which as stated uses a matrix of $D \times L$, to an equivalent problem using a matrix of $L \times L$.
\begin{thm}
  \label{th:reduced-Z}
  Let $\x\in\bbR^D$ and $\A\in\bbR^{D\times L}$, with QR factorisation $\A = \Q\RR$, where \Q\ is of $D\times L$ with $\Q^T\Q = \I$ and \RR\ is upper triangular of $L\times L$, and $\y = \Q^T\x \in \bbR^L$. The following two problems have the same minima over \z:
  \begin{equation}
    \label{e:equiv-Z}
    \smash{\textstyle\min_{\z\in\{0,1\}^L}{ \norm{\x-\A\z}^2 } \quad \min_{\z\in\{0,1\}^L}{ \norm{\y-\RR\z}^2 }.}
  \end{equation}
\end{thm}
\begin{proof}
  \fussy
  Let $(\Q\ \Q_{\perp})$ of $D \times D$ be orthogonal, where the columns of $\Q_{\perp}$ are an orthonormal basis of the nullspace of $\Q^T$. Then, since orthogonal matrices preserve Euclidean distances, we have: $\smash{\norm{\x-\A\z}}^2 = \smash{\norm{(\Q\ \Q_{\perp})^T (\x-\A\z)}}^2 = \smash{\norm{\Q^T (\x-\A\z)}}^2 + \smash{\norm{\Q^T_{\perp} (\x-\A\z)}}^2 = \smash{\norm{\y-\RR\z}}^2 + \smash{\norm{\Q^T_{\perp} \x}}^2$, where the term $\smash{\norm{\Q^T_{\perp} \x}}^2$ does not depend on \z.
\end{proof}
This achieves a speedup of $2D/L$ (where the $2$ factor comes from the fact that the new matrix is triangular), e.g.\ this is $40\times$ if using $16$ bits with $D=320$ GIST features in our experiments. Henceforth, we redefine the \z\ step as $\min e(\z) = \smash{\norm{\y - \RR\z}^2} + \mu \smash{\norm{\z - \h(\x)}^2}$ s.t.\ $\z\in\{0,1\}^L$.

\paragraph{Enumeration}

For small $L$, this can be solved \emph{exactly} by enumeration, at a worst-case runtime cost $\calO(L^2 2^L)$, but with small constant factors in practice (see accelerations below). $L=16$ is perfectly practical in a workstation without parallel processing for the datasets in our experiments.

\paragraph{Alternating optimization}

For larger $L$, we use alternating optimization over groups of $g$ bits (where the optimization over a $g$-bit group is done by enumeration and uses the same accelerations). This converges to a local minimum of the \Z\ step, although we find in our experiments that it finds near-global optima \emph{if using a good initialization}. Intuitively, it makes sense to warm-start this, i.e., to initialize \z\ to the code found in the previous iteration's \Z\ step, since this should be close to the new optimum as we converge. However, empirically we find that the codes change a lot in the first few iterations, and that the following initialization works better (in leading to a lower objective value) in early iterations: we solve the relaxed problem on \z\ s.t.\ $\z \in [0,1]^L$ rather than $\{0,1\}^L$. This is a strongly convex bound-constrained quadratic program (QP) in $L$ variables for $\mu > 0$ and its unique minimizer can be found efficiently.

We can further speed up the solution by noting that we have $N$ QPs with some common, special structure. The objective is the sum of a term having the same matrix \RR\ for all QPs, and a term that is separable in \z. We have developed an ADMM algorithm \citep{Carreir14a} that is very simple, parallelizes or vectorizes very well, and reuses matrix factorizations over all $N$ QPs. It is $10$--$100\times$ faster than Matlab's \texttt{quadprog}. We warm-start it from the continuous solution of the QP in the previous \Z\ step.

In order to binarize the continuous minimizer for $\z_n$ we could simply round its $L$ elements, but instead we apply a greedy procedure that is efficient and better (though still suboptimal). We optimally binarize from bit 1 to bit $L$ by evaluating the objective function for bit $l$ in $\{0,1\}$ with all remaining elements fixed (elements $1$ to $l-1$ are already binary and $l+1$ to $L$ are still continuous) and picking the best. Essentially, this is one pass of alternating optimization but having continuous values for some of the bits.

Finally, we pick the best of the binarized relaxed solution or the warm-start value and run alternating optimization. This ensures that the quadratic-penalty function~\eqref{e:MAC-QP} decreases monotonically at each iteration.

\paragraph{Accelerations}

Naively, the enumeration involves evaluating $e(\z)$ for $2^L$ (or $2^g$) vectors, where evaluating $e(\z)$ for one \z\ costs on average roughly $L+1$ multiplications and $\frac{1}{4}L^2$ sums. This enumeration can be sped up or pruned while still finding a global minimum by using upper bounds on $e$, incremental computation of $e$, and necessary and sufficient conditions for the solution. Essentially, we need not evaluate every code vector, or every bit of every code vectors; we know the solution will be ``near'' $\h(\x)$; and we can recognize the solution when we find it.

Call $\z^*$ a global minimizer of $e(\z)$. An initial, good upper bound is $e(\h(\x)) = \smash{\norm{\y - \RR\,\h(\x)}^2}$. In fact, we have the following sufficient condition for $\h(\x)$ to be a global minimizer. (We give it generally for any decoder \f.)
\begin{thm}
  \label{th:h-bound}
  Let $e(\z) = \smash{\norm{\x - \f(\z)}^2} + \mu \smash{\norm{\z - \h(\x)}^2}$. Then: (1) A global minimizer $\z^*$ of $e(\z)$ is at a Hamming distance from $\h(\x)$ of $\frac{1}{\mu} \smash{\norm{\x - \f(\h(\x))}^2}$ or less. (2) If $\mu > \smash{\norm{\x - \f(\h(\x))}^2}$ then $\h(\x)$ is a global minimizer.
\end{thm}
\begin{proof}
  \fussy
  $e(\z^*) = \smash{\norm{\x - \f(\z^*)}^2} + \mu \smash{\norm{\z^* - \h(\x)}^2} \le e(\h(\x)) = \smash{\norm{\x - \f(\h(\x))}^2} \Rightarrow \smash{\norm{\z^* - \h(\x)}^2} \le \frac{1}{\mu} \smash{\norm{\x - \f(\h(\x))}^2}$. (2) follows because the Hamming distance is integer.
\end{proof}
As $\mu$ increases and \h\ improves, this bound becomes more effective and more of the $N$ patterns are pruned. Upon convergence, the \Z\ step costs only $\calO(NL^2)$. If we do have to search for a given $\z_n$, we keep a running bound (current best minimum) $\bar{e} = \smash{\norm{\y_n - \RR\bar{\z}}^2} + \mu \smash{\norm{\bar{\z} - \h(\x_n)}^2}$, and we scan codes in increasing Hamming distance to $\h(\x_n)$ up to a distance of $\bar{e}/\mu$. Thus, we try first the codes that are more likely to be optimal, and keep refining the bound as we find better codes.

Second, since $\smash{\norm{\y - \RR\z}^2}$ separates over dimensions $1,\dots,L$, we evaluate it incrementally (dimension by dimension) and stop as soon as we exceed the running bound.

Finally, there exist global optimality necessary and sufficient conditions for binary quadratic problems that are easy to evaluate \citep{BeckTeboul00a,Jeyakum_07a} (see appendix~\ref{s:binary-cond}). This allows us to recognize the solution as soon as we reach it and stop the search (rather than do a linear search of all values, keeping track of the minimum). These conditions can also determine whether the continuous solution to the relaxed QP is a global minimizer of the binary QP.

\subsection{Schedule for the penalty parameter $\mu$}
\label{s:schedule}

The only user parameters in our method are the initialization for the binary codes \Z\ and the schedule for the penalty parameter $\mu$ (sequence of values $0 < \mu_1 < \dots < \infty$), since we use a penalty or augmented Lagrangian method. In general with these methods, setting the schedule requires some tuning in practice. Fortunately, this is simplified in our case for two reasons. 1) We need not drive $\mu\rightarrow\infty$ because \emph{termination occurs at a finite $\mu$ and can be easily detected}: whenever $\Z = \h(\X)$ at the end of the \Z\ step, no further changes to the parameters can occur. This gives a practical stopping criterion. 2) In order to generalize well to unseen data, we stop iterating not when we (sufficiently) optimize $E_Q(\h,\f,\Z;\mu)$, but \emph{when the precision in a validation set decreases}. This is a form of early stopping that guarantees that we improve (or leave unchanged) the initial \Z, and besides is faster. The initialization for \Z\ and further details about the schedule for $\mu$ appear in section~\ref{s:expts}.

\section{Measuring code utilization using entropy: effective number of bits}
\label{s:entropy}

Here we propose an evaluation measure of binary hash functions that has not been used before as far as we know. Any binary hash function maps a population of high-dimensional real vectors onto a population of $L$-bit vectors (where $L$ is fixed). Intuitively, a good hash function should make best use of the available codes and use each bit equally (since no bit is preferable to any other bit). For example, if we have $L=32$ bits and $N=10^6$ distinct real vectors, a good hash function would ideally assign a different 32-bit code to each vector, in order to avoid collisions. Given an $L$-bit hash function \h\ and a dataset \X\ of $N$ real vectors $\x_1,\dots,\x_N \in \bbR^D$, we then obtain the $N$ $L$-bit binary codes $\z_1,\dots,\z_N \in \{0,1\}^L$. We can measure the \emph{code utilization} of the hash function \h\ for the dataset \X\ by the \emph{entropy of the code distribution}, $S(P(\h(\X)))$, defined as follows. Let $c_i$ be the number of vectors that map to binary code $i$ for $i = 0,\dots,2^L-1$. Then the (sample) code distribution $P$ is a discrete probability distribution defined over the $L$-bit integers $0,\dots,2^L-1$ and has probability $p_i = \frac{c_i}{N}$ for code $i$, since $p_0 + \dots + p_{2^L-1} = N$, i.e., the code probabilities are the normalized counts computed over the $N$ data points. This works whether $N$ is smaller or larger than $2^L$ (if $N < 2^L$ there will be unused codes). If the dataset is a sample of a distribution of high-dimensional vectors, then $P$ is an estimate of the code usage induced by the hash function \h\ for the distribution $P$ based on a sample of size $N$. The entropy of $P$ is
\begin{equation}
  \label{e:entropy}
  S(P) = - \sum^{2^L-1}_{i=0}{ p_i \log_2{p_i} },
\end{equation}
which is measured in bits. The entropy $S$ is a real number and satisfies $S \in [0,\min(L,\log_2{N})]$. It is $0$ when all $N$ codes are coincident (hence only one of the $2^L$ available codes are used). It is $L$ when $N \ge 2^L$ (the number of available codes is no more than the number of data points) and $p_i = 2^{-L}$ for all $i$ (uniform distribution). It is $\log_2{N}$ when $N < 2^L$ and all $N$ codes are distinct. Hence, the entropy is large the more available codes are used, and the more uniform their use is. Since the entropy is measured in bits and cannot be more than $L$, $S$ can be said to measure the \emph{effective number of bits} $L_{\text{eff}}$ of the code distribution induced by the hash function \h\ on dataset \X.

A good hash function will have a good code utilization, making use of the available codes to avoid collisions and preserve neighbors. However, it is crucial to realize that an optimal code utilization does by itself not necessarily result in an optimal hash function, and it is not a fully reliable proxy for precision/recall. Indeed, code utilization is not directly related to the distances between data vectors, and it is easy to construct hash functions that produce an optimal code utilization but are not necessarily very good in preserving neighbors. For example, we can pick the first hash function as a hyperplane that splits the space into two half spaces each with half the data points (this is a high-dimensional median). Then, we do this within each half to set the second hash function, etc. This generates a binary decision tree with oblique cuts, which is impractical because it has $2^L$ hyperplanes, one per internal node. If the real vectors follow a Gaussian distribution (or any other axis-symmetric distribution) in $\bbR^D$, then using as hash functions any $L$ thresholded principal components will give maximum entropy (as long as $L \le D$). This is because each of the $L$ hash functions is a hyperplane that splits the space into two half spaces containing half of the vectors, and the hyperplanes are orthogonal. This gives the thresholded PCA (tPCA) method mentioned earlier, which is generally not competitive with other methods, as seen in our experiments. More generally, we may expect tPCA and random projections through the mean to achieve high code utilization, because for most high-dimensional datasets (of arbitrary distribution), most low-dimensional projections are approximately Gaussian \citep{DiaconFreedm84a}.

With this caveat in mind, code utilization measured in effective number of bits $L_{\text{eff}}$ is still a useful evaluation measure for hash functions. It also has an important advantage: it does not depend on any user parameters. In particular, it does not depend on the ground truth size ($K$ nearest neighbors in data space) or retrieved set size (given by the $k$ nearest neighbors or the vectors with codes within Hamming distance $r$). This allows us to compare all binary hashing methods with a single number $L_{\text{eff}}$ (for a given number of bits $L$). We report values of $L_{\text{eff}}$ in the experiments section and show that it indeed correlates well with precision in terms of the ranking of methods, particularly for methods that use the same model and objective function (such as the binary autoencoder reconstruction error with ITQ and MAC).

\section{Experiments}
\label{s:expts}

We used three datasets in our experiments, commonly used as benchmarks for image retrieval. (1) CIFAR \citep{Krizhev09a} contains $60\,000$ $32\times 32$ color images in $10$ classes. We ignore the labels in this paper and use $N = 50\,000$ images as training set and $10\,000$ images as test set. We extract $D = 320$ GIST features \citep{OlivaTorral01a} from each image. (2) NUS-WIDE \citep{Chua_09a} contains $N=269\,648$ high-resolution color images and use $N = 161\,789$ for training and $107\,859$ for test. We extract $D = 128$ wavelet features \citep{OlivaTorral01a} from each image. In some experiments we also use the NUS-WIDE-LITE subset of this dataset, containing $N = 27\,807$ images for training and $27\,807$ images for test. (3) SIFT-1M \citep{Jegou_11a} contains $N = 1\,000\,000$ training high-resolution color images and $10\,000$ test images, each represented by $D=128$ SIFT features.

We report precision and recall (\%) in the test set using as true neighbors the $K$ nearest images in Euclidean distance in the original space, and as retrieved neighbors in the binary space we either use the $k$ nearest images in Hamming distance, or the images within a Hamming distance $r$ (if no images satisfy the latter, we report zero precision). We also compare algorithms using our entropy-based measure of code utilization described in section~\ref{s:entropy}.

Our experiments evaluate the effectiveness of our algorithm to minimize the BA objective and whether this translates into better hash functions (i.e., better image retrieval); its runtime and parallel speedup; and its precision and recall and code utilization compared to representative state-of-the-art algorithms.

\subsection{How much does respecting the binary constraints help?}

\begin{figure}[t]
  \centering
  \psfrag{rerror}[][t]{reconstruction error}
  \psfrag{bits}[][]{$L$}
  \psfrag{precision}[b][b]{precision}
  \begin{tabular}{@{}c@{\hspace{0\linewidth}}c@{\hspace{0\linewidth}}c@{}}
    \raisebox{0pt}[0pt][0pt]{\includegraphics[width=0.337\linewidth]{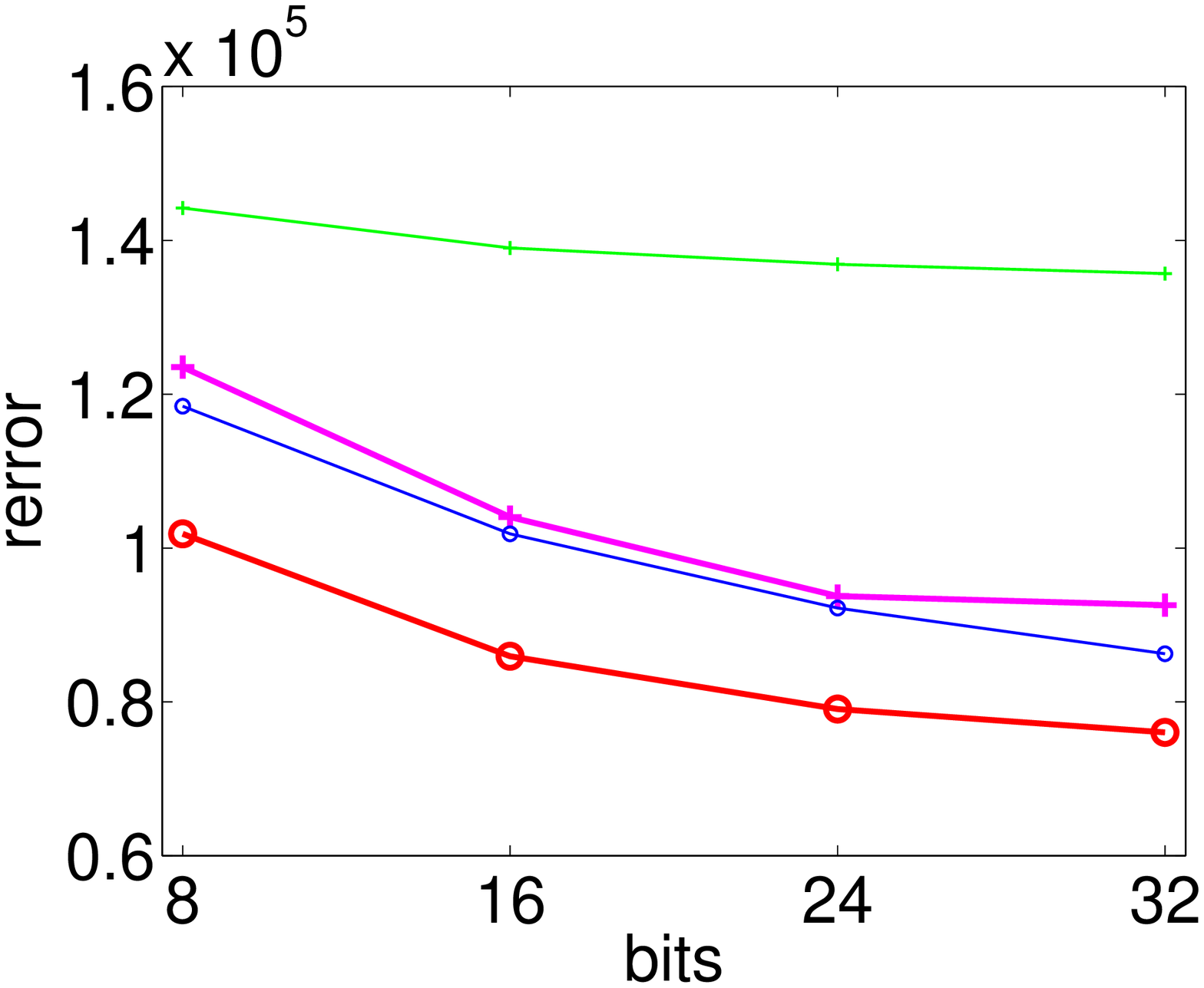}} &
    \includegraphics[width=0.33\linewidth]{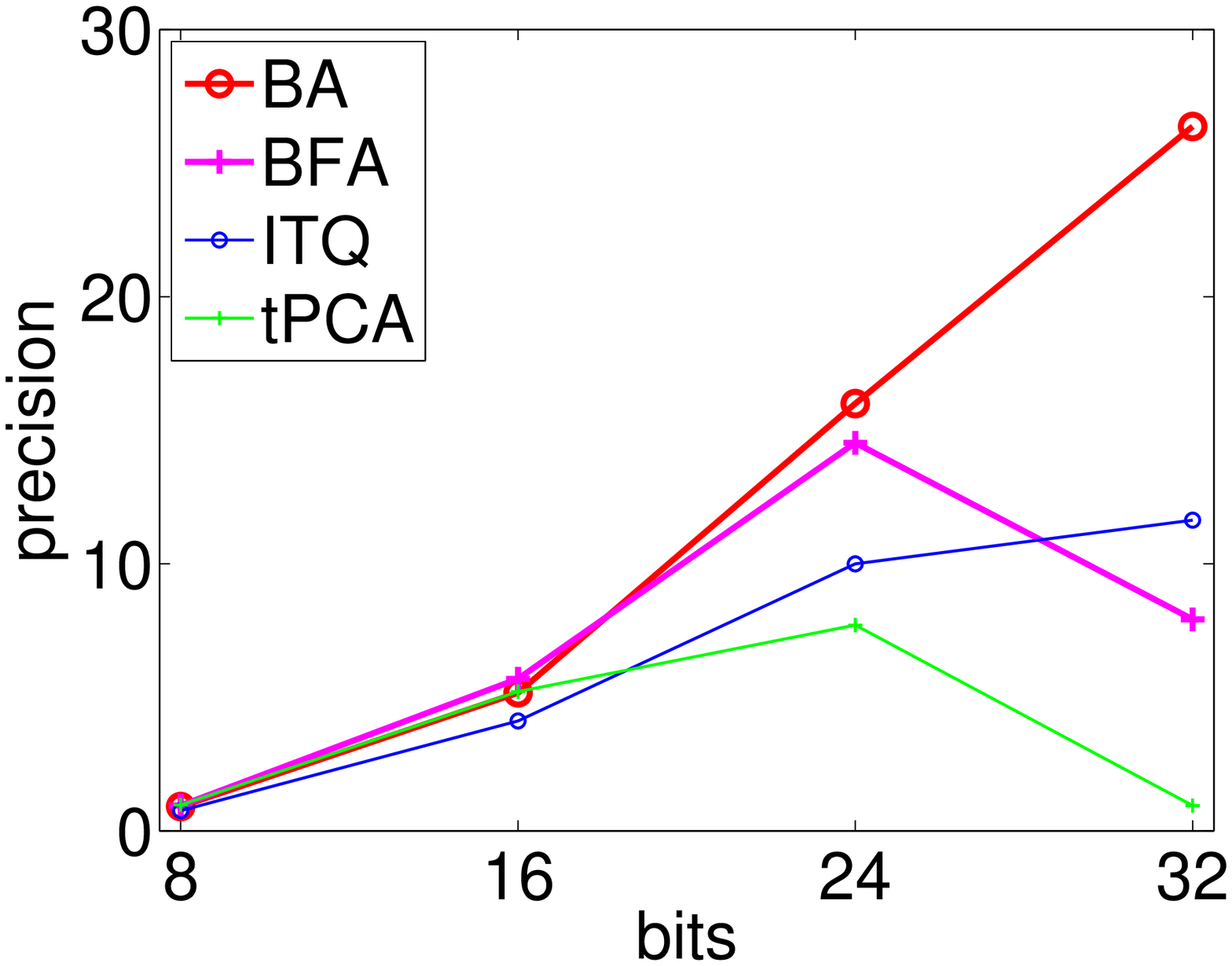} &
    \includegraphics[width=0.33\linewidth]{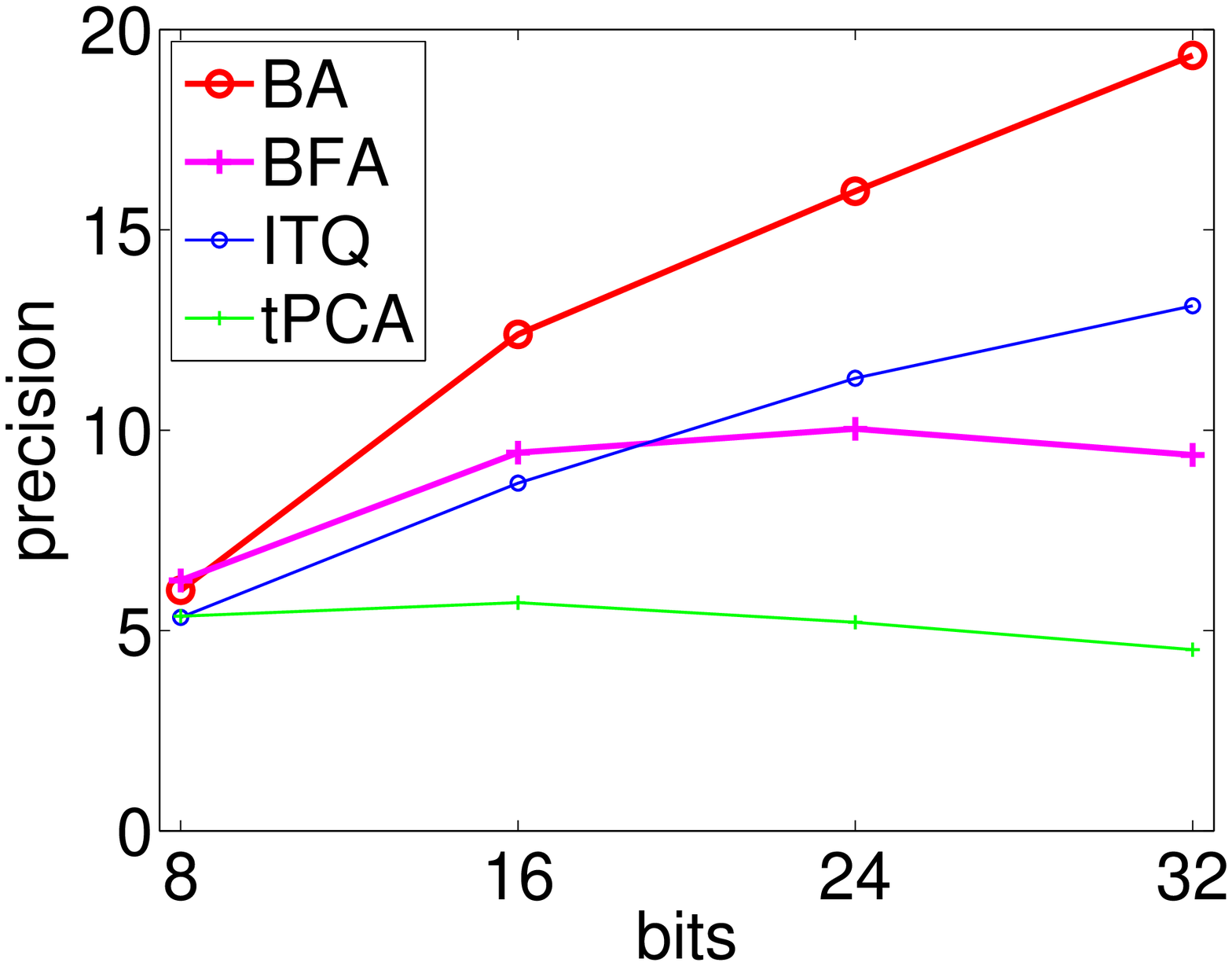}
  \end{tabular}
  \caption{Wrapper vs filter optimization: BA objective function (\emph{left}) and precision using neighbors within Hamming distance $r=2$ (\emph{middle}) and using $k=50$ nearest neighbors (\emph{right}).}
  \label{f:wrapper}
\end{figure}

We focus purely on the BA objective function (reconstruction error) and study the gain obtained by the MAC optimization, which respects the binary constraints, over the suboptimal, ``filter'' approach of relaxing the constraints (i.e., PCA) and then binarizing the result by thresholding at 0 (tPCA) or by optimal rotation (ITQ). To compute the reconstruction error for tPCA and ITQ we find the optimal mapping \f\ given their binary codes. We use the NUS-WIDE-LITE subset of the NUS-WIDE dataset. We initialize BA from AGH and BFA from tPCA and use alternating optimization in the \Z\ steps. We search for $K = 50$ true neighbors and report results over a range of $L = 8$ to $32$ bits in fig.~\ref{f:wrapper}. We can see that BA dominates all other methods in reconstruction error, as expected, and also in precision, as one might expect. tPCA is consistently the worst method by a significant margin, while ITQ and BFA are intermediate. Hence, the more we respect the binary constraints during the optimization, the better the hash function. Further experiments below consistently show that the BA precision significantly increases over the (AGH) initialization and is leading or competitive over other methods.

\subsection{\Z-step: alternating optimization and initialization}

We study the MAC optimization if doing an inexact \Z\ step by using alternating optimization over groups of $g$ bits. Specifically, we study the effect on the number of iterations and runtime of the group size $g$ and of the initialization (warm-start vs relaxed QP). Fig.~\ref{f:altopt} shows the results in the CIFAR dataset using $L=16$ bits (so using $g=16$ gives an exact optimization), without using a validation-based stopping criterion (so we do optimize the training objective).

Surprisingly, the warm-start initialization leads to worse BA objective function values than the binarized relaxed one. Fig.~\ref{f:altopt}(left) shows the dashed lines (warm-start for different $g$) are all above the solid lines (relaxed for different $g$). The reason is that, early during the optimization, the codes \Z\ undergo drastic changes from one iteration to the next, so the warm-start initialization is farther from a good optimum than the relaxed one. Late in the optimization, when the codes change slowly, the warm-start does perform well. The relaxed initialization resulting optima are almost the same as using the exact binary optimization.

Also surprisingly, different group sizes $g$ eventually converge to almost the same result as using the exact binary optimization if using the relaxed initialization. (If using warm-start, the larger $g$ the better the result, as one would expect.) Likewise, in fig.~\ref{f:wrapper}, if using alternating optimization in the \Z\ step rather than enumeration, the curves for BA and BFA barely vary. But, of course, the runtime per iteration grows exponentially on $g$ (middle panel).

Hence, it appears that using faster, inexact \Z\ steps does not impair the model learnt, and we settle on $g=1$ with relaxed initialization as default for all our remaining experiments (unless we use $L < 16$ bits, in which case we simply use enumeration).

\begin{figure}[t]
  \centering
  \psfrag{error}[][t]{reconstruction error}
  \psfrag{iterations}[][b]{number of iterations}
  \psfrag{exact}[][]{exact}
  \psfrag{warm start}[l][l]{\caja[0.9]{c}{c}{warm \\ start}}
  \psfrag{relaxed}[c][c]{relaxed}
  \psfrag{Block 1}[bl][bl]{{~~$g=1$}}
  \psfrag{Block 2}[bl][bl]{{~~$g=2$}}
  \psfrag{Block 4}[bl][bl]{{~~$g=4$}}
  \psfrag{Block 8}[bl][bl]{{~~$g=8$}}
  \psfrag{Block 16}[bl][bl]{{~~$g=16$}}
  \psfrag{time}[][b]{runtime (minutes)}
  \begin{tabular}{@{}c@{\hspace{0\linewidth}}c@{}}
    \includegraphics[width=0.53\linewidth]{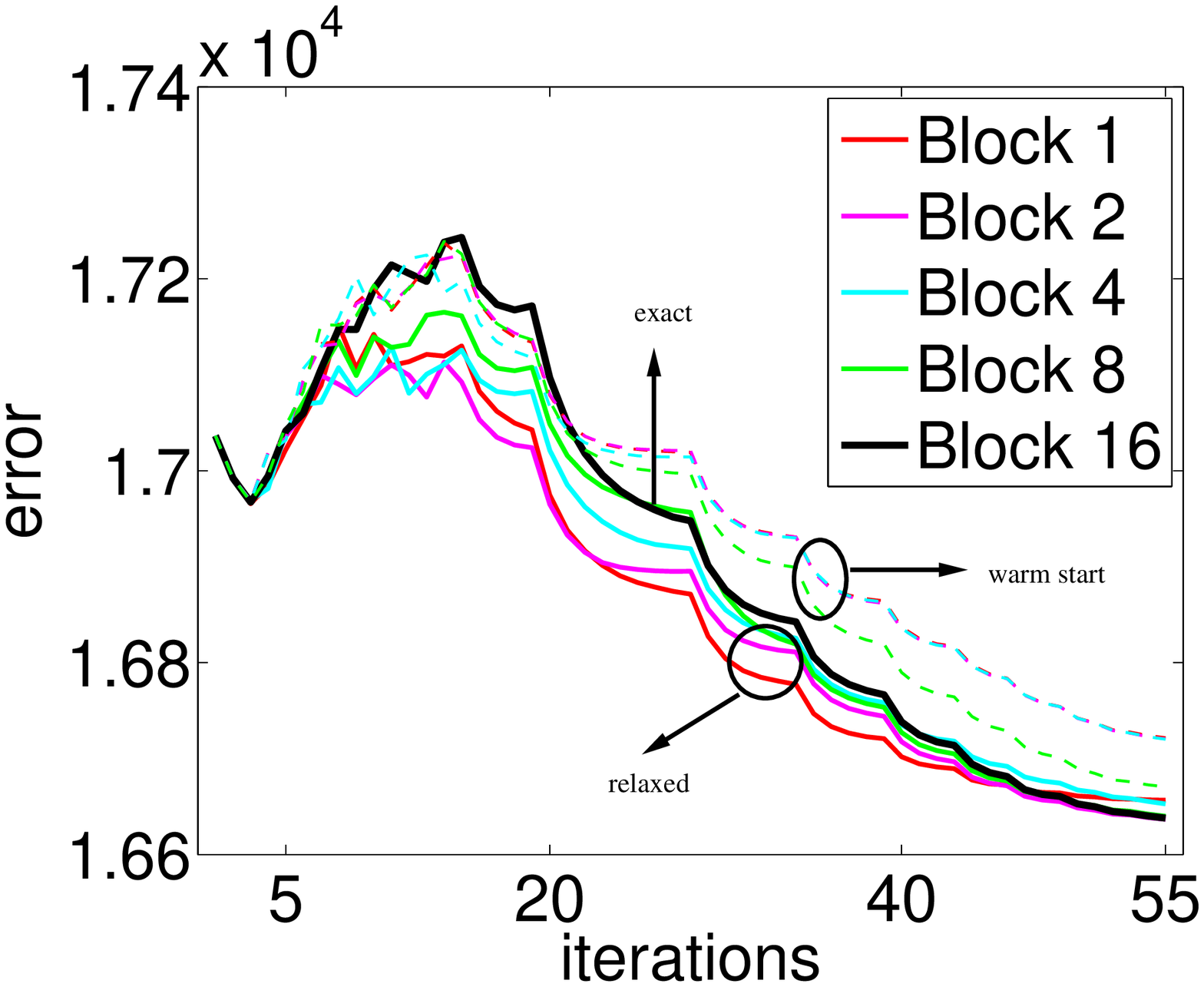} &
    \includegraphics[width=0.47\linewidth]{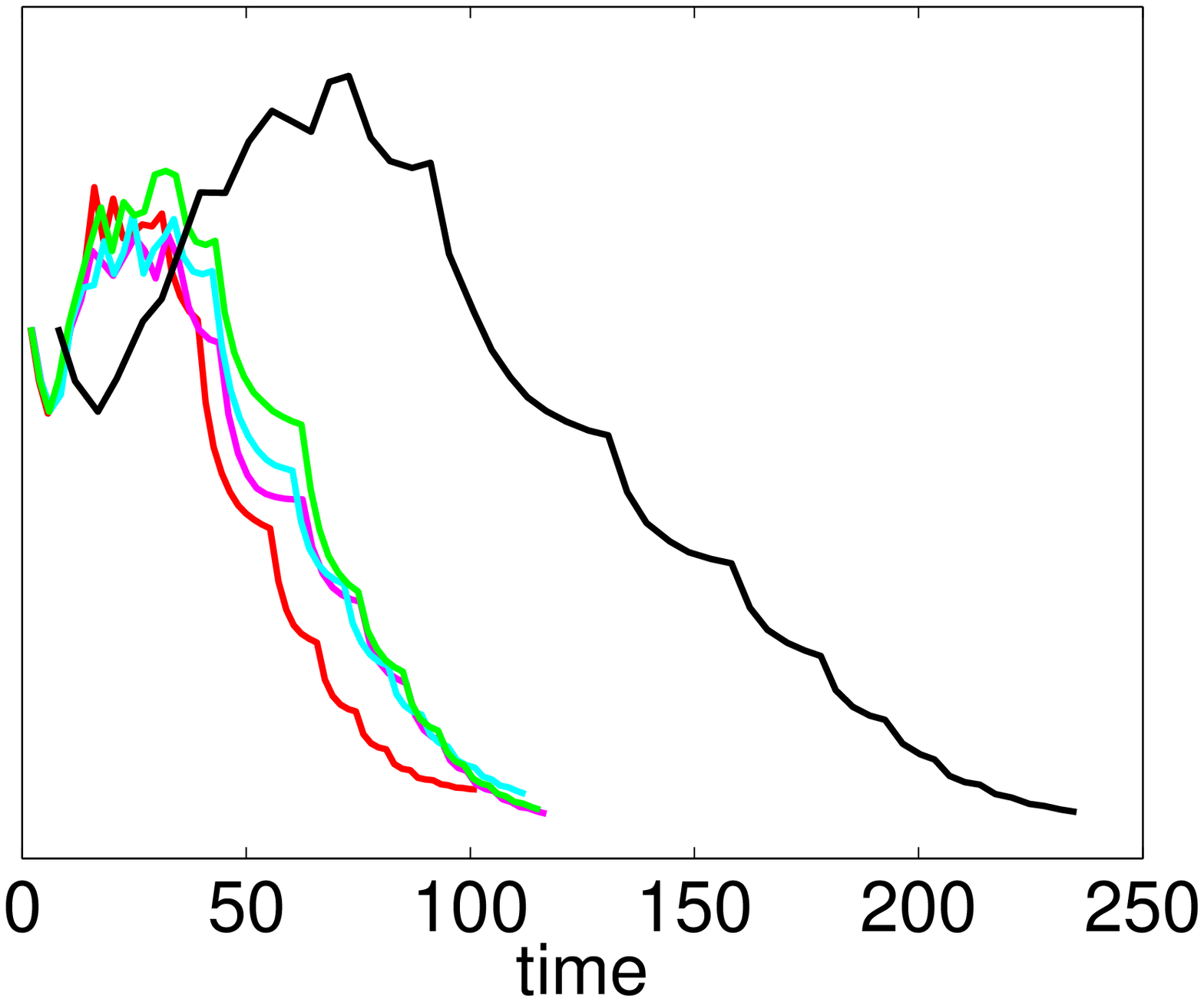}
  \end{tabular}
  \caption{Iterations (\emph{left}) and runtime (\emph{right}) of the MAC optimization of the BA objective function (at each iteration, we run one \Z\ and $(\f,\h)$ step). In the \Z\ step, we use alternating optimization in $g$-bit groups ($g=16$ means exact optimization), and a warm-start vs relaxed initialization of \Z.}
  \label{f:altopt}
\end{figure}

\subsection{Parallel processing}

Fig.~\ref{f:parallel} shows the BA training time speedup achieved with parallel processing, in CIFAR with $L=16$ bits. We use the Matlab Parallel Processing Toolbox with up to 12 processors and simply replace ``for'' with ``parfor'' loops so each iteration (over points in the \Z\ step, over bits in the \h\ step) is run in a different processor. We observe a nearly perfect scaling for this particular problem. Note that, the larger the number of bits, the larger the parallelization ability in the \h\ step. As a rough indication of runtimes for BA, training the $50\,000$ CIFAR images and $161\,789$ NUS-WIDE images using $L=32$ bits with alternating optimization in the \Z\ step takes 20' and 50', respectively (in a 4-core laptop).

\begin{figure}[t]
  \centering
  \psfrag{speedup}[][]{speedup}
  \psfrag{time}[][b]{number of processors}
  \includegraphics[width=0.40\linewidth]{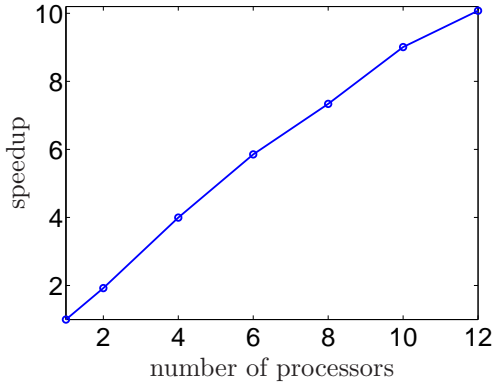}
  \caption{Parallel processing: speedup of the training algorithm as a function of the number of processors (CIFAR dataset with $L=16$ bits).}
  \label{f:parallel}
\end{figure}

\subsection{Schedule of $\mu$ and initial \Z}

Since the objective is nonconvex, our result does depend on the initial codes \Z\ (in the first iteration of the MAC algorithm), but we are guaranteed to improve or leave unchanged the precision (in the validation set) of the codes produced by any algorithm. We have observed that initializing from AGH \citep{Liu_11a} tends to produce best results overall, so we use this in all the experiments. Using ITQ \citep{Gong_13a} produces also good results (occasionally better but generally somewhat worse than AGH), and is a simpler and faster option if so desired. We initialize BFA from tPCA, since this seems to work best.

In order to be able to use a fixed schedule, we make the data zero-mean and rescale it so the largest feature range is 1. This does not alter the Euclidean distances and normalizes the scale. We start with $\mu_1 = 0.01$ and double it after each iteration (one \Z\ and $(\f,\h)$ step). As noted in section~\ref{s:MAC}, the algorithm will skip $\mu$ values that do not improve the precision in the validation set, and will stop at a finite $\mu$ value (past which no further changes occur).

It is of course possible to tweak all these settings ($\mu$ schedule and initial \Z) and obtain better results, but these defaults seem robust.

\subsection{Comparison with other algorithms in image retrieval using precision/recall}

We compare BA and BFA with the following algorithms: thresholded PCA (tPCA), Iterative Quantization (ITQ) \citep{Gong_13a}, Spectral Hashing (SH) \citep{Weiss_09a}, Kernelized Locality-Sensitive Hashing (KLSH) \citep{KulisGrauman12a}, AnchorGraph Hashing (AGH) \citep{Liu_11a}, and Spherical Hashing (SPH) \citep{Heo_12a}. Note several of these learn nonlinear hash functions and use more sophisticated error functions (that better approximate the nearest neighbor ideal), while our BA uses a linear hash function and simply minimizes the reconstruction error. All experiments use the output of AGH and tPCA to initialize BA and BFA, respectively.

It is known that the retrieval performance of a given algorithm depends strongly on the size of the neighbor set used, so we report experiments with small and large number of points in the ground truth set. For NUS-WIDE, we considered as ground truth $K=100$ and $K=1\,500$ neighbors of the query point, and as set of retrieved neighbors, we retrieve either $k$ nearest neighbors ($k=100$, $1500$) or neighbors within Hamming distance $r$ ($r=1$, $2$). Fig.~\ref{f:NUS-WIDE} shows the results. For ANNSIFT-1M, we considered ground truth $K=10\,000$ neighbors and set of retrieved neighbors for $k=10\,000$ or $r=1$ to $3$. Fig.~\ref{f:ANNSIFT-1M} shows the results. All curves are the average over the test set. We also show precision and recall results for the CIFAR dataset (fig.~\ref{f:CIFAR}) as well as top retrieved images for sample query images in the test set (fig.~\ref{f:CIFAR_sample}), and for the NUS-WIDE-LITE dataset (fig.~\ref{f:NUS-WIDE-LITE}), for different ground truth sizes.

Although the relative performance of the different methods varies depending on the reported set size, some trends are clear. Generally (though not always) BA beats all other methods, sometime by a significant margin. ITQ and SPH become close (sometimes comparable) to BA in CIFAR and NUS-WIDE dataset, respectively. BFA is also quite competitive, but consistently worse than BA. The only situation when the precision of BA and BFA appears to decrease is when $L$ is large and $r$ is small. The reason is that many test images have no neighbors at a Hamming distance of $r$ or less and we report zero precision for them. This suggests the hash function finds a way to avoid collisions as more bits are available. In practice, one would simply increase $r$ to retrieve sufficient neighbors.

Fig.~\ref{f:ANNSIFT-1M} shows the results of BA using two different initializations, from the codes of AGH and ITQ, respectively. Although the initialization does affect the local optimum that the MAC algorithm finds, this optimum is always better than the initialization and it tends to give competitive results with other algorithms.

\subsection{Comparison with other algorithms in image retrieval using code utilization}

Fig.~\ref{f:entropy} shows the effective number of bits $L_{\text{eff}}$ for all methods on two datasets, NUS-WIDE and ANNSIFT-1M, and should be compared with figure~\ref{f:NUS-WIDE} for NUS-WIDE and figure~\ref{f:ANNSIFT-1M} (top row, initialized from AGH) for ANNSIFT-1M. For each method, we show $L_{\text{eff}}$ for the training set (solid line) and test set (dashed line). We also show a diagonal-horizontal line to indicate the upper bound $L_{\text{eff}} \le \min(L,\log_2{N})$ (again, one such line for each of the training and test sets). All methods lie below this line, and the closer they are to it, the better their code utilization.

As mentioned in section~\ref{s:entropy}, tPCA consistently gives the largest $L_{\text{eff}}$, and $L_{\text{eff}}$ is by itself not a consistently reliable indicator of good precision/recall. However, there is a reasonably good correlation between $L_{\text{eff}}$ and precision (with a few exceptions, notably tPCA), in that the methods display a comparable order in both measures. This is especially true if comparing $L_{\text{eff}}$ with the precision plots corresponding to retrieving neighbors within a Hamming distance $r$. This is particularly clear with the ANNSIFT-1M dataset.

Methods for which the precision drops for high $L$ (if using a small Hamming distance $r$ to retrieve neighbors) also show $L_{\text{eff}}$ stagnating for those $L$ values. This is because the small number of points (hence possibly used codes) limits the entropy achievable. These methods are making good use of the available codes, having few collisions, so they need a larger Hamming distance to find neighbors. In particular, this explains the drop in precision of BA for small $r$ in fig.~\ref{f:ANNSIFT-1M}, which we also noted earlier. If the training or test set was larger, we would expect the precision and $L_{\text{eff}}$ to continue to increase for those same values of $L$.

The correlation in performance between $L_{\text{eff}}$ and precision is also seen if comparing methods that are using the same model and (approximately) optimizing the same objective, such as the binary autoencoder reconstruction error for ITQ and MAC. The consistent improvement of MAC over ITQ in precision is seen in $L_{\text{eff}}$ too. This further suggests that the better optimization effected by MAC is improving the hash function (in precision, and in code utilization).

We can see that if the number of points $N$ in the dataset (training or test) is small compared to the number of available codes $2^L$, then $L_{\text{eff}}$ stagnates as $L$ increases (since $L_{\text{eff}} \le \log_2{N}$). This is most obvious with the test set in ANNSIFT-1M, since it is small. Together with the precision/recall curves, $L_{\text{eff}}$ can be used to determine the number of bits $L$ to use with a given database. 

The leftmost plot of fig.~\ref{f:entropy} shows (for selected methods on NUS-WIDE) the actual code distribution as a histogram. That is, we plot the number of high-dimensional vectors $c_i$ that map to code $i$, for each binary code $i \in \{0,\dots,2^L-1\}$ that is used (i.e., that has at least one vector mapping to it). This is the code distribution $P$ of section~\ref{s:entropy} but unnormalized and without plotting the zero-count codes. The entropy of this distribution gives the $L_{\text{eff}}$ value in the middle plot. High entropy corresponds to a large number of used codes, and uniformity in their counts.

\section{Discussion}
\label{s:disc}

One contribution of this paper is to reveal a connection between ITQ \citep{Gong_13a} (a popular, effective hashing algorithm) and binary autoencoders. ITQ can be seen as a fast, approximate optimization of the BA objective function, using a ``filter'' approach (relax the problem to obtain continuous codes, iteratively quantize the codes, then fit the hash function). Our BA algorithm is a corrected version of ITQ.

Admittedly, there are objective functions that are more suited for information retrieval than the autoencoder, by explicitly encouraging distances in the original and Hamming space to match in order to preserve nearest neighbors \citep{Weiss_09a,Liu_11a,KulisGrauman12a,Lin_13a,Lin_14b}. However, autoencoders do result in good hash functions, as evidenced by the good performance of ITQ and our method (or of semantic hashing \citep{SalakhHinton09b}, using neural nets). The reason is that, with continuous codes, autoencoders can capture the data manifold in a smooth way and indirectly preserve distances, encouraging (dis)similar images to have (dis)similar codes---even if this is worsened to some extent because of the quantization introduced with discrete codes. Autoencoders are also faster and easier to optimize and scale up better to large datasets.

Note that, although similar in some respects, the binary autoencoder is not a graphical model, in particular it is not a stacked restricted Boltzmann machine (RBM). A binary autoencoder composes two \emph{deterministic} mappings: one (the encoder) with binary outputs, and another (the decoder) with binary inputs. Hence, the objective function for binary autoencoders is discontinuous. The objective function for RBMs is differentiable, but it involves a normalization factor that is computationally intractable. This results in a very different optimization: a nonsmooth optimization for binary autoencoders (with combinatorial optimization in the step over the codes if using MAC), and a smooth optimization using sampling and approximate gradients for RBMs. See \citet[p.~3372]{Vincen_10a}.

\section{Conclusion and future work}

Up to now, many hashing approaches have essentially ignored the binary nature of the problem and have approximated it through relaxation and truncation, possibly disregarding the hash function when learning the binary codes. The inspiration for this work was to capitalize on the decoupling introduced by the method of auxiliary coordinates to be able to break the combinatorial complexity of optimizing with binary constraints, and to introduce parallelism into the problem. Armed with this algorithm, we have shown that respecting the binary nature of the problem during the optimization is possible in an efficient way and that it leads to better hash functions, competitive with the state-of-the-art. This was particularly encouraging given that the autoencoder objective is not the best for retrieval, and that we focused on linear hash functions.

The algorithm has an intuitive form (alternating classification, regression and binarization steps) that can reuse existing, well-developed code. The extension to nonlinear hash and reconstruction mappings is straightforward and it will be interesting to see how much these can improve over the linear case. This paper is a step towards constructing better hash functions using the MAC framework. We believe it may apply more widely to other objective functions.

\section*{Acknowledgments}

Work supported by NSF award IIS--1423515. We thank Ming-Hsuan Yang and Yi-Hsuan Tsai (UC Merced) for helpful discussions about binary hashing.

\clearpage
\appendix

\section{Global optimality conditions for binary quadratic functions}
\label{s:binary-cond}

Consider the following quadratic optimization with binary variables:
\begin{equation}
  \label{e:quad-bin}
  \min_{\x}{ \frac{1}{2} \x^T \Q \x + \b^T \x } \quad \text{s.t.} \quad \x\in\{-1,1\}^n
\end{equation}
where \Q\ is a symmetric matrix of of $n \times n$. In general, this is an NP-hard problem \citep{GareyJohnson79a}. \citet{BeckTeboul00a} gave global optimality conditions for this problem that are simply expressed in terms of the problem's data (\Q, \b) involving only primal variables and no dual variables:
\begin{description}
\item[Sufficient] If $\lambda_{\text{min}} \e \ge \X \Q \X \e + \X \b$ then \x\ is a global optimizer for~\eqref{e:quad-bin}.
\item[Necessary] If \x\ is a global optimizer for~\eqref{e:quad-bin} then $\X \Q \X \e + \X \b \le \diag{\Q} \e$.
\end{description}
Here, $\lambda_{\text{min}}$ is the smallest eigenvalue of \Q, \e\ is a vector of ones, $\X = \diag{\x}$ is a diagonal matrix of $n \times n$ with diagonal entries as in vector \x, and $\diag{\Q}$ is a diagonal matrix of $n \times n$ with diagonal entries $q_{ii}$. Intuitively, if \Q\ is ``smaller'' than \b\ then we can disregard the quadratic term and trivially solve the separable, linear term.

\citet{Jeyakum_07a} further gave additional global optimality conditions:
\begin{description}
\item[Sufficient] If $\diag{\tilde{\q}} \e \ge \X \Q \X \e + \X \b$ then \x\ is a global optimizer for~\eqref{e:quad-bin}, where $\tilde{q}_i = q_{ii} - \sum^n_{j\neq i}{\abs{q_{ij}}},\ i=1,\dots,n$.
\item[Necessary] If \x\ is a global optimizer for~\eqref{e:quad-bin} then $\b^T \x \le \0$.
\end{description}
It is possible to give tighter conditions \citep{Xia09a} but which are more complicated to compute.

Furthermore, \citet{BeckTeboul00a} also gave conditions for the minimizer of the relaxed problem to be the global minimizer of the binary problem. Let the continuous relaxation be:
\begin{equation}
  \label{e:quad-relaxed}
  \min_{\x}{ \frac{1}{2} \x^T \Q \x + \b^T \x } \quad \text{s.t.} \quad \x\in[-1,1]^n
\end{equation}
and consider the case where \Q\ is positive semidefinite, so this problem is convex. Then, a necessary and sufficient condition for \x\ to be a solution of both~\eqref{e:quad-bin} and~\eqref{e:quad-relaxed} is:
\begin{equation*}
  \x\in\{-1,1\}^n \text{ is a global optimizer for~\eqref{e:quad-bin} and~\eqref{e:quad-relaxed}} \Longleftrightarrow \X \Q \X \e + \X \b \le \0.
\end{equation*}
Furthermore, we have a relation when the solution of the binary problem is ``close enough'' to that of the relaxed one. If \x\ is an optimizer of the relaxed problem~\eqref{e:quad-relaxed} and $\y = \sgn{\x}$ satisfies $\Y \Q (\y-\x) \le \lambda_{\text{min}} \e$ then \y\ is a global optimizer of the binary problem~\eqref{e:quad-bin} (where $\sgn{t} = 1$ if $t \ge 0$ and $-1$ if $t<0$).

In our case, we are particularly interested in the sufficient conditions, which we can combine as
\begin{equation*}
  \min{(\diag{\tilde{\q}},\lambda_{\text{min}})} \e \ge \X \Q \X \e + \X \b \Rightarrow \x\ \text{ is a global minimizer of~\eqref{e:quad-bin}.}
\end{equation*}
Computationally, these conditions have a comparable cost to evaluating the objective function $e(\z) = \smash{\norm{\x - \f(\z)}^2} + \mu \smash{\norm{\z - \h(\x)}^2}$ for an $L$-bit vector (in the main paper). Besides, some of the arithmetic operations are common to evaluating $e(\z)$ and the necessary and sufficient conditions for global optimality of the binary problem, so we can use the latter as a fast test to determine whether we have found the global minimizer and stop the search (when enumerating all $2^L$ vectors). Likewise, we can use the necessary and sufficient conditions to determine whether the solution to the relaxed problem is the global minimizer for the discrete problem, since in our $e(\z)$ function the quadratic function is convex. Also, note that computing $\diag{\tilde{\q}}$ and $\lambda_{\text{min}}$ is done just once for all $N$ data points in the training set, since all the problems (i.e., finding the binary codes for each data point) share the same matrix \A.


\begin{thebibliography}{38}
\providecommand{\natexlab}[1]{#1}
\providecommand{\url}[1]{\texttt{#1}}
\expandafter\ifx\csname urlstyle\endcsname\relax
  \providecommand{\doi}[1]{doi: #1}\else
  \providecommand{\doi}{doi: \begingroup \urlstyle{rm}\Url}\fi

\bibitem[Andoni and Indyk(2008)]{AndoniIndyk08a}
A.~Andoni and P.~Indyk.
\newblock Near-optimal hashing algorithms for approximate nearest neighbor in
  high dimensions.
\newblock \emph{Comm. ACM}, 51\penalty0 (1):\penalty0 117--122, Jan. 2008.

\bibitem[Beck and Teboulle(2000)]{BeckTeboul00a}
A.~Beck and M.~Teboulle.
\newblock Global optimality conditions for quadratic optimization problems with
  binary constraints.
\newblock \emph{SIAM Journal on Optimization}, 11\penalty0 (1):\penalty0
  179--188, 2000.

\bibitem[Carreira-Perpi{\~n}{\'a}n(2014)]{Carreir14a}
M.~{\'A}. Carreira-Perpi{\~n}{\'a}n.
\newblock An {ADMM} algorithm for solving a proximal bound-constrained
  quadratic program.
\newblock arXiv:1412.8493 [math.OC], Dec.~29 2014.

\bibitem[Carreira-Perpi{\~n}{\'a}n and Wang(2012)]{CarreirWang12a}
M.~{\'A}. Carreira-Perpi{\~n}{\'a}n and W.~Wang.
\newblock Distributed optimization of deeply nested systems.
\newblock arXiv:1212.5921 [cs.LG], Dec.~24 2012.

\bibitem[Carreira-Perpi{\~n}{\'a}n and Wang(2014)]{CarreirWang14a}
M.~{\'A}. Carreira-Perpi{\~n}{\'a}n and W.~Wang.
\newblock Distributed optimization of deeply nested systems.
\newblock In S.~Kaski and J.~Corander, editors, \emph{Proc. of the 17th Int.
  Conf. Artificial Intelligence and Statistics (AISTATS 2014)}, pages 10--19,
  Reykjavik, Iceland, Apr.~22--25 2014.

\bibitem[Chua et~al.(2009)Chua, Tang, Hong, Li, Luo, and Zheng]{Chua_09a}
T.-S. Chua, J.~Tang, R.~Hong, H.~Li, Z.~Luo, and Y.-T. Zheng.
\newblock {NUS-WIDE}: A real-world web image database from {National}
  {University} of {Singapore}.
\newblock In \emph{Proc. ACM Conf. Image and Video Retrieval (CIVR'09)},
  Santorini, Greece, July~8--10 2009.

\bibitem[Combettes and Pesquet(2011)]{CombetPesquet11a}
P.~L. Combettes and J.-C. Pesquet.
\newblock Proximal splitting methods in signal processing.
\newblock In H.~H. Bauschke, R.~S. Burachik, P.~L. Combettes, V.~Elser, D.~R.
  Luke, and H.~Wolkowicz, editors, \emph{Fixed-Point Algorithms for Inverse
  Problems in Science and Engineering}, Springer Series in Optimization and Its
  Applications, pages 185--212. Springer-Verlag, 2011.

\bibitem[Diaconis and Freedman(1984)]{DiaconFreedm84a}
P.~Diaconis and D.~Freedman.
\newblock Asymptotics of graphical projection pursuit.
\newblock \emph{Annals of Statistics}, 12\penalty0 (3):\penalty0 793--815,
  Sept. 1984.

\bibitem[Fan et~al.(2008)Fan, Chang, Hsieh, Wang, and Lin]{Fan_08a}
R.-E. Fan, K.-W. Chang, C.-J. Hsieh, X.-R. Wang, and C.-J. Lin.
\newblock {LIBLINEAR}: A library for large linear classification.
\newblock \emph{J. Machine Learning Research}, 9:\penalty0 1871--1874, Aug.
  2008.

\bibitem[Garey and Johnson(1979)]{GareyJohnson79a}
M.~R. Garey and D.~S. Johnson.
\newblock \emph{Computers and Intractability: A Guide to the Theory of
  {NP-Completeness}}.
\newblock W.H. Freeman, 1979.

\bibitem[Getoor and Scheffer(2011)]{icml11}
L.~Getoor and T.~Scheffer, editors.
\newblock \emph{Proc. of the 28th Int. Conf. Machine Learning (ICML 2011)},
  Bellevue, WA, June~28 -- July~2 2011.

\bibitem[Golub and van Loan(1996)]{GolubLoan96a}
G.~H. Golub and C.~F. van Loan.
\newblock \emph{Matrix Computations}.
\newblock Johns Hopkins University Press, Baltimore, third edition, 1996.

\bibitem[Gong et~al.(2013)Gong, Lazebnik, Gordo, and Perronnin]{Gong_13a}
Y.~Gong, S.~Lazebnik, A.~Gordo, and F.~Perronnin.
\newblock Iterative quantization: A {Procrustean} approach to learning binary
  codes for large-scale image retrieval.
\newblock \emph{IEEE Trans. Pattern Analysis and Machine Intelligence},
  35\penalty0 (12):\penalty0 2916--2929, Dec. 2013.

\bibitem[Heo et~al.(2012)Heo, Lee, He, Chang, and Yoon]{Heo_12a}
J.-P. Heo, Y.~Lee, J.~He, S.-F. Chang, and S.-E. Yoon.
\newblock Spherical hashing.
\newblock In \emph{Proc. of the 2012 IEEE Computer Society Conf. Computer
  Vision and Pattern Recognition (CVPR'12)}, pages 2957--2964, Providence, RI,
  June~16--21 2012.

\bibitem[J{\'e}gou et~al.(2011)J{\'e}gou, Douze, and Schmid]{Jegou_11a}
H.~J{\'e}gou, M.~Douze, and C.~Schmid.
\newblock Product quantization for nearest neighbor search.
\newblock \emph{IEEE Trans. Pattern Analysis and Machine Intelligence},
  33\penalty0 (1):\penalty0 117--128, Jan. 2011.

\bibitem[Jeyakumar et~al.(2007)Jeyakumar, Rubinov, and Wu]{Jeyakum_07a}
V.~Jeyakumar, A.~M. Rubinov, and Z.~Y. Wu.
\newblock Non-convex quadratic minimization problems with quadratic
  constraints: Global optimality conditions.
\newblock \emph{Math. Prog.}, 110\penalty0 (3):\penalty0 521--541, Sept. 2007.

\bibitem[Kohavi and John(1998)]{KohaviJohn98a}
R.~Kohavi and G.~H. John.
\newblock The wrapper approach.
\newblock In H.~Liu and H.~Motoda, editors, \emph{Feature Extraction,
  Construction and Selection. {A} Data Mining Perspective}. Springer-Verlag,
  1998.

\bibitem[Krizhevsky(2009)]{Krizhev09a}
A.~Krizhevsky.
\newblock Learning multiple layers of features from tiny images.
\newblock Master's thesis, Dept. of Computer Science, University of Toronto,
  Apr.~8 2009.

\bibitem[Krizhevsky and Hinton(2011)]{KrizhevHinton11a}
A.~Krizhevsky and G.~E. Hinton.
\newblock Using very deep autoencoders for content-based image retrieval.
\newblock In \emph{Proc. of the 19th European Symposium on Artificial Neural
  Networks (ESANN 2011)}, Bruges, Belgium, Apr.~27--29 2011.

\bibitem[Kulis and Darrell(2009)]{KulisDarrel09a}
B.~Kulis and T.~Darrell.
\newblock Learning to hash with binary reconstructive embeddings.
\newblock In Y.~Bengio, D.~Schuurmans, J.~Lafferty, C.~K.~I. Williams, and
  A.~Culotta, editors, \emph{Advances in Neural Information Processing Systems
  (NIPS)}, volume~22, pages 1042--1050. MIT Press, Cambridge, MA, 2009.

\bibitem[Kulis and Grauman(2012)]{KulisGrauman12a}
B.~Kulis and K.~Grauman.
\newblock Kernelized locality-sensitive hashing.
\newblock \emph{IEEE Trans. Pattern Analysis and Machine Intelligence},
  34\penalty0 (6):\penalty0 1092--1104, June 2012.

\bibitem[Lin et~al.(2013)Lin, Shen, Suter, and van~den Hengel]{Lin_13a}
G.~Lin, C.~Shen, D.~Suter, and A.~van~den Hengel.
\newblock A general two-step approach to learning-based hashing.
\newblock In \emph{Proc. 14th Int. Conf. Computer Vision (ICCV'13)}, pages
  2552--2559, Sydney, Australia, Dec.~1--8 2013.

\bibitem[Lin et~al.(2014)Lin, Shen, Shi, van~den Hengel, and Suter]{Lin_14b}
G.~Lin, C.~Shen, Q.~Shi, A.~van~den Hengel, and D.~Suter.
\newblock Fast supervised hashing with decision trees for high-dimensional
  data.
\newblock In \emph{Proc. of the 2014 IEEE Computer Society Conf. Computer
  Vision and Pattern Recognition (CVPR'14)}, pages 1971--1978, Columbus, OH,
  June~23--28 2014.

\bibitem[Liu et~al.(2011)Liu, Wang, Kumar, and Chang]{Liu_11a}
W.~Liu, J.~Wang, S.~Kumar, and S.-F. Chang.
\newblock Hashing with graphs.
\newblock In  \citet{icml11}, pages 1--8.

\bibitem[Moreau(1962)]{Moreau62a}
J.-J. Moreau.
\newblock Fonctions convexes duales et points proximaux dans un espace
  hilbertien.
\newblock \emph{C. R. Acad. Sci. Paris S{\'e}r. A Math.}, 255:\penalty0
  2897--2899, 1962.

\bibitem[Neyshabur et~al.(2013)Neyshabur, Srebro, Salakhutdinov, Makarychev,
  and Yadollahpour]{Neyshab_13a}
B.~Neyshabur, N.~Srebro, R.~Salakhutdinov, Y.~Makarychev, and P.~Yadollahpour.
\newblock The power of asymmetry in binary hashing.
\newblock In C.~J.~C. Burges, L.~Bottou, M.~Welling, Z.~Ghahramani, and K.~Q.
  Weinberger, editors, \emph{Advances in Neural Information Processing Systems
  (NIPS)}, volume~26, pages 2823--2831. MIT Press, Cambridge, MA, 2013.

\bibitem[Nocedal and Wright(2006)]{NocedalWright06a}
J.~Nocedal and S.~J. Wright.
\newblock \emph{Numerical Optimization}.
\newblock Springer Series in Operations Research and Financial Engineering.
  Springer-Verlag, New York, second edition, 2006.

\bibitem[Norouzi and Fleet(2011)]{NorouzFleet11a}
M.~Norouzi and D.~Fleet.
\newblock Minimal loss hashing for compact binary codes.
\newblock In  \citet{icml11}.

\bibitem[Oliva and Torralba(2001)]{OlivaTorral01a}
A.~Oliva and A.~Torralba.
\newblock Modeling the shape of the scene: A holistic representation of the
  spatial envelope.
\newblock \emph{Int. J. Computer Vision}, 42\penalty0 (3):\penalty0 145--175,
  May 2001.

\bibitem[Rockafellar(1976)]{Rockaf76b}
R.~T. Rockafellar.
\newblock Monotone operators and the proximal point algorithm.
\newblock \emph{SIAM J. Control and Optim.}, 14\penalty0 (5):\penalty0
  877--898, 1976.

\bibitem[Salakhutdinov and Hinton(2009)]{SalakhHinton09b}
R.~Salakhutdinov and G.~Hinton.
\newblock Semantic hashing.
\newblock \emph{Int. J. Approximate Reasoning}, 50\penalty0 (7):\penalty0
  969--978, July 2009.

\bibitem[Torralba et~al.(2008)Torralba, Fergus, and Weiss]{Torral_08b}
A.~Torralba, R.~Fergus, and Y.~Weiss.
\newblock Small codes and large image databases for recognition.
\newblock In \emph{Proc. of the 2008 IEEE Computer Society Conf. Computer
  Vision and Pattern Recognition (CVPR'08)}, Anchorage, AK, June~23--28 2008.

\bibitem[Vincent et~al.(2010)Vincent, Larochelle, Lajoie, Bengio, and
  Manzagol]{Vincen_10a}
P.~Vincent, H.~Larochelle, I.~Lajoie, Y.~Bengio, and P.~A. Manzagol.
\newblock Stacked denoising autoencoders: Learning useful representations in a
  deep network with a local denoising criterion.
\newblock \emph{J. Machine Learning Research}, 11:\penalty0 3371--3408, 2010.

\bibitem[Weiss et~al.(2009)Weiss, Torralba, and Fergus]{Weiss_09a}
Y.~Weiss, A.~Torralba, and R.~Fergus.
\newblock Spectral hashing.
\newblock In D.~Koller, Y.~Bengio, D.~Schuurmans, L.~Bottou, and A.~Culotta,
  editors, \emph{Advances in Neural Information Processing Systems (NIPS)},
  volume~21, pages 1753--1760. MIT Press, Cambridge, MA, 2009.

\bibitem[Whittle(1952)]{Whittl52a}
P.~Whittle.
\newblock On principal components and least square methods of factor analysis.
\newblock \emph{Skand. Aktur. Tidskr.}, 36:\penalty0 223--239, 1952.

\bibitem[Xia(2009)]{Xia09a}
Y.~Xia.
\newblock New optimality conditions for quadratic optimization problems with
  binary constraints.
\newblock \emph{Opt. Lett.}, 3\penalty0 (2):\penalty0 253--263, Mar. 2009.

\bibitem[Yu and Shi(2003)]{YuShi03a}
S.~X. Yu and J.~Shi.
\newblock Multiclass spectral clustering.
\newblock In \emph{Proc. 9th Int. Conf. Computer Vision (ICCV'03)}, pages
  313--319, Nice, France, Oct.~14--17 2003.

\bibitem[Zhang et~al.(2010)Zhang, Wang, Cai, and Lu]{Zhang_10e}
D.~Zhang, J.~Wang, D.~Cai, and J.~Lu.
\newblock Self-taught hashing for fast similarity search.
\newblock In \emph{Proc. of the 33rd ACM Conf. Research and Development in
  Information Retrieval (SIGIR 2010)}, pages 18--25, Geneva, Switzerland,
  July~19--23 2010.

\end{thebibliography}

\begin{figure}[t!]
  \centering
  \begin{tabular}{@{}c@{\hspace{0\linewidth}}c@{}c@{}}
    $k=K$ neighbors are retrieved & Hamming distance $\le 1$ & Hamming distance $\le 2$ \\
    \psfrag{precision}[][t]{precision $K=100$}
    \psfrag{bits}[][]{}
    \includegraphics[width=.33\linewidth,height=0.265\linewidth]{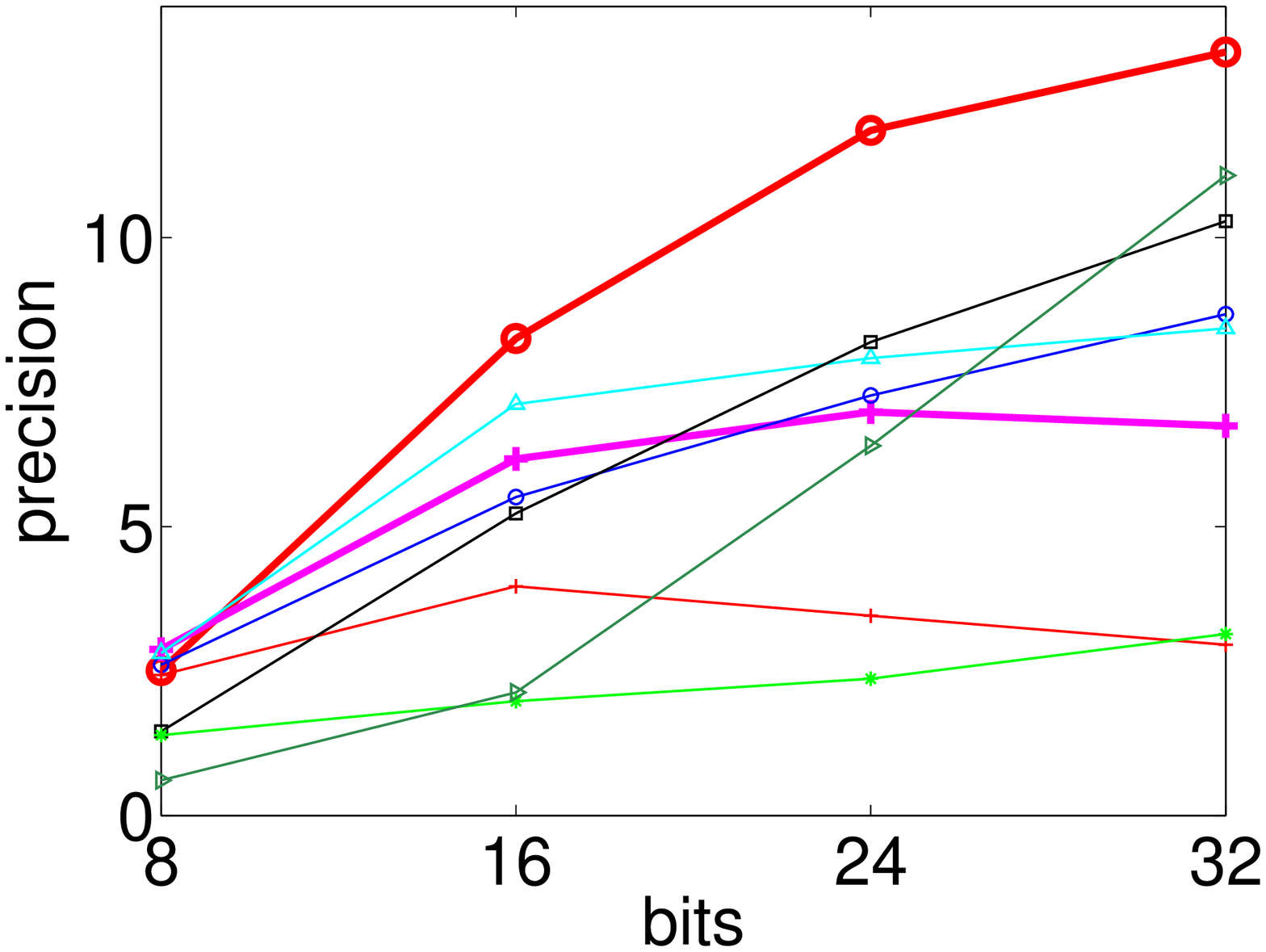} &
    \psfrag{bits}[][]{}
    \includegraphics[width=.33\linewidth]{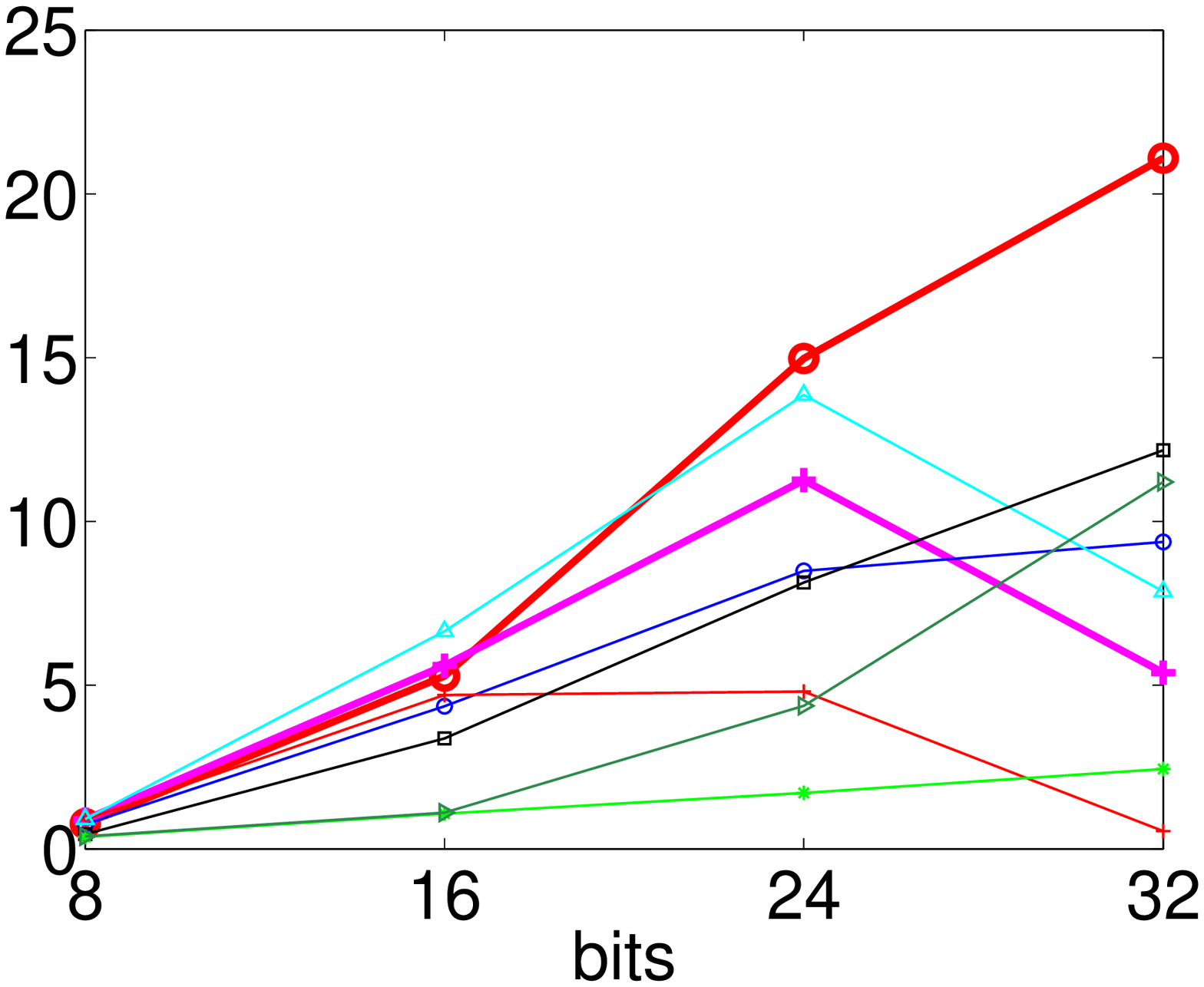} &
    \psfrag{bits}[][]{}
    \includegraphics[width=.33\linewidth]{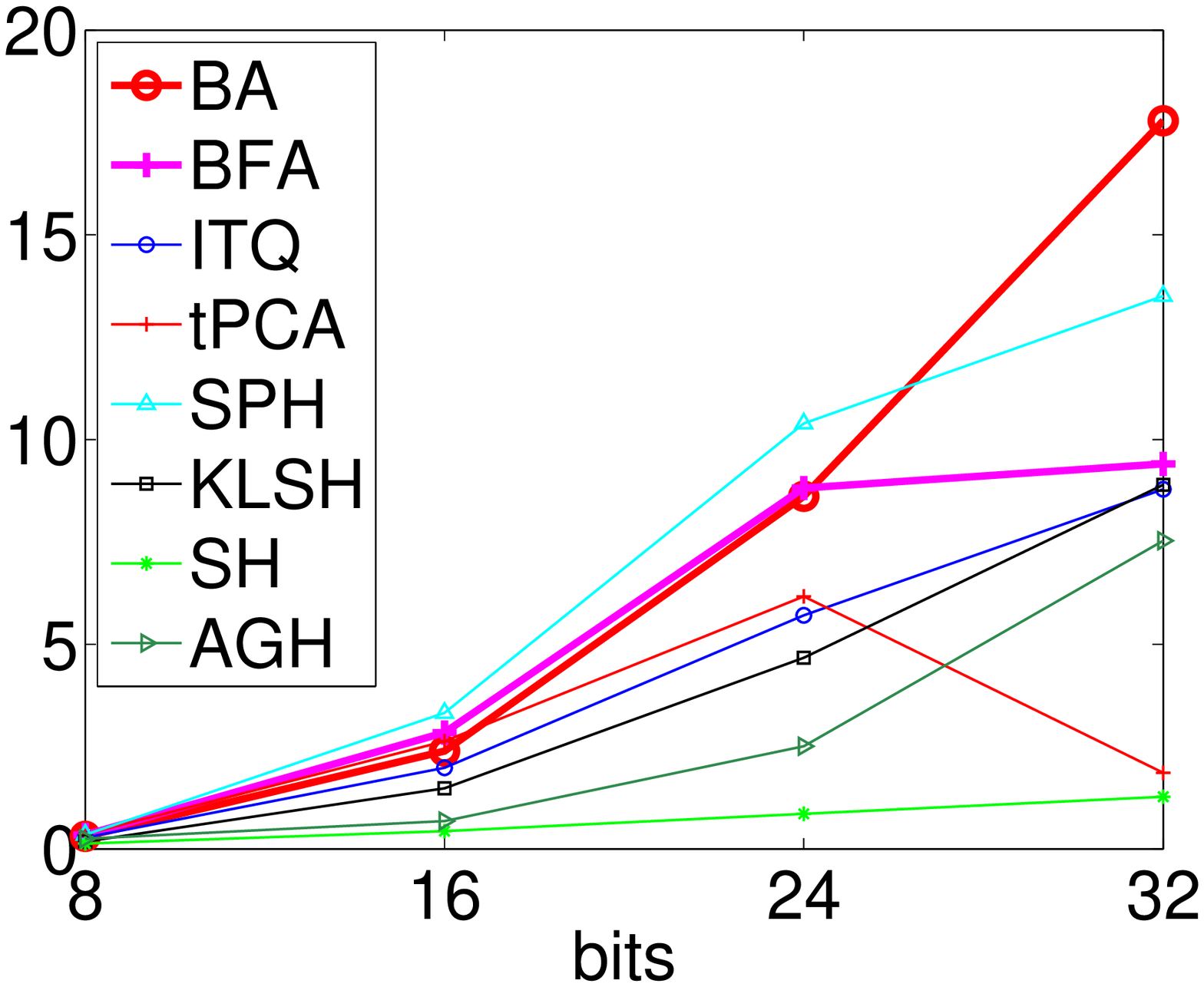} \\
    \psfrag{precision}[][t]{precision $K=500$}
    \psfrag{bits}[][]{}
    \includegraphics[width=.33\linewidth,height=0.265\linewidth]{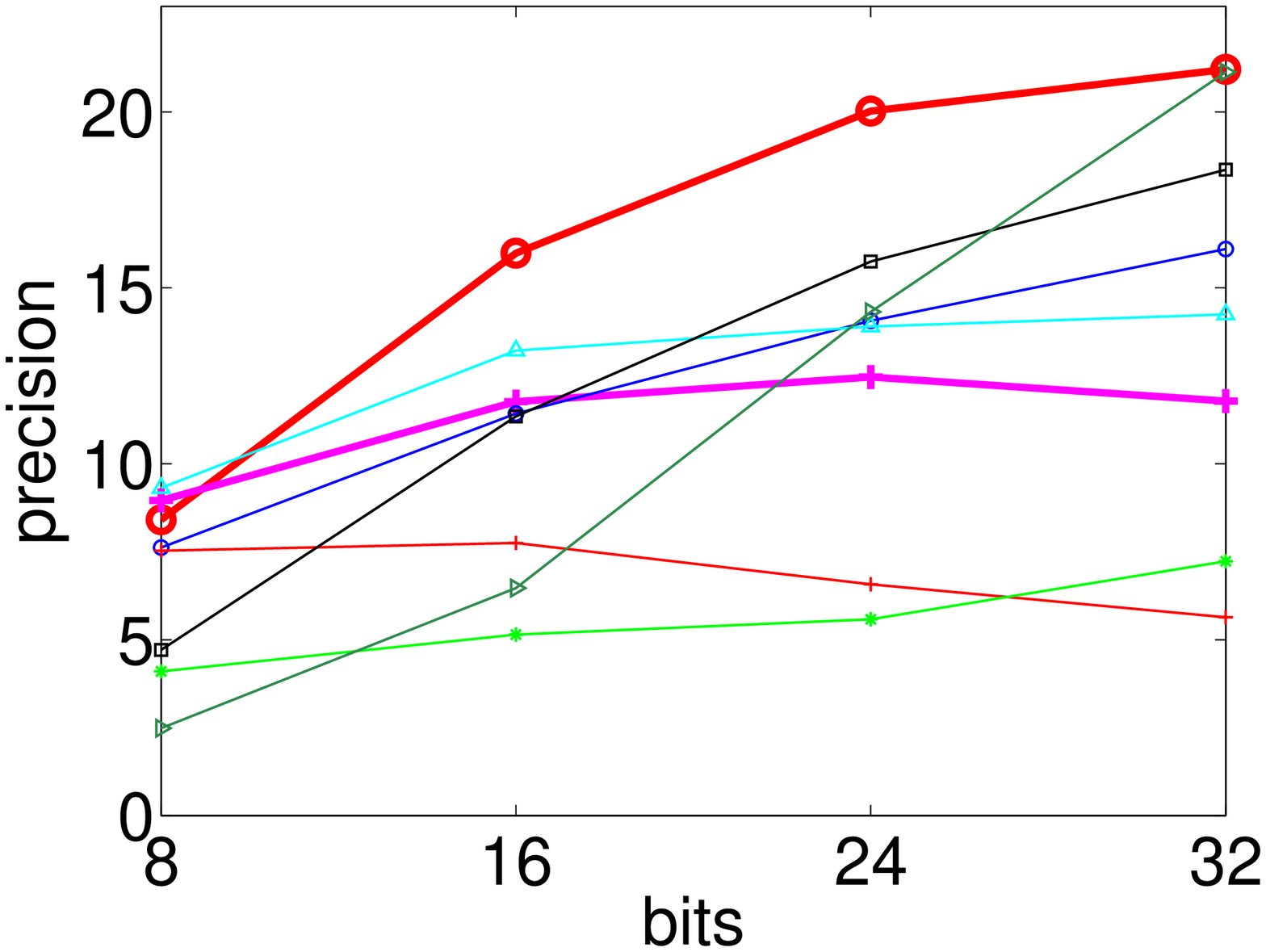} &
    \psfrag{bits}[][]{}
    \includegraphics[width=.33\linewidth]{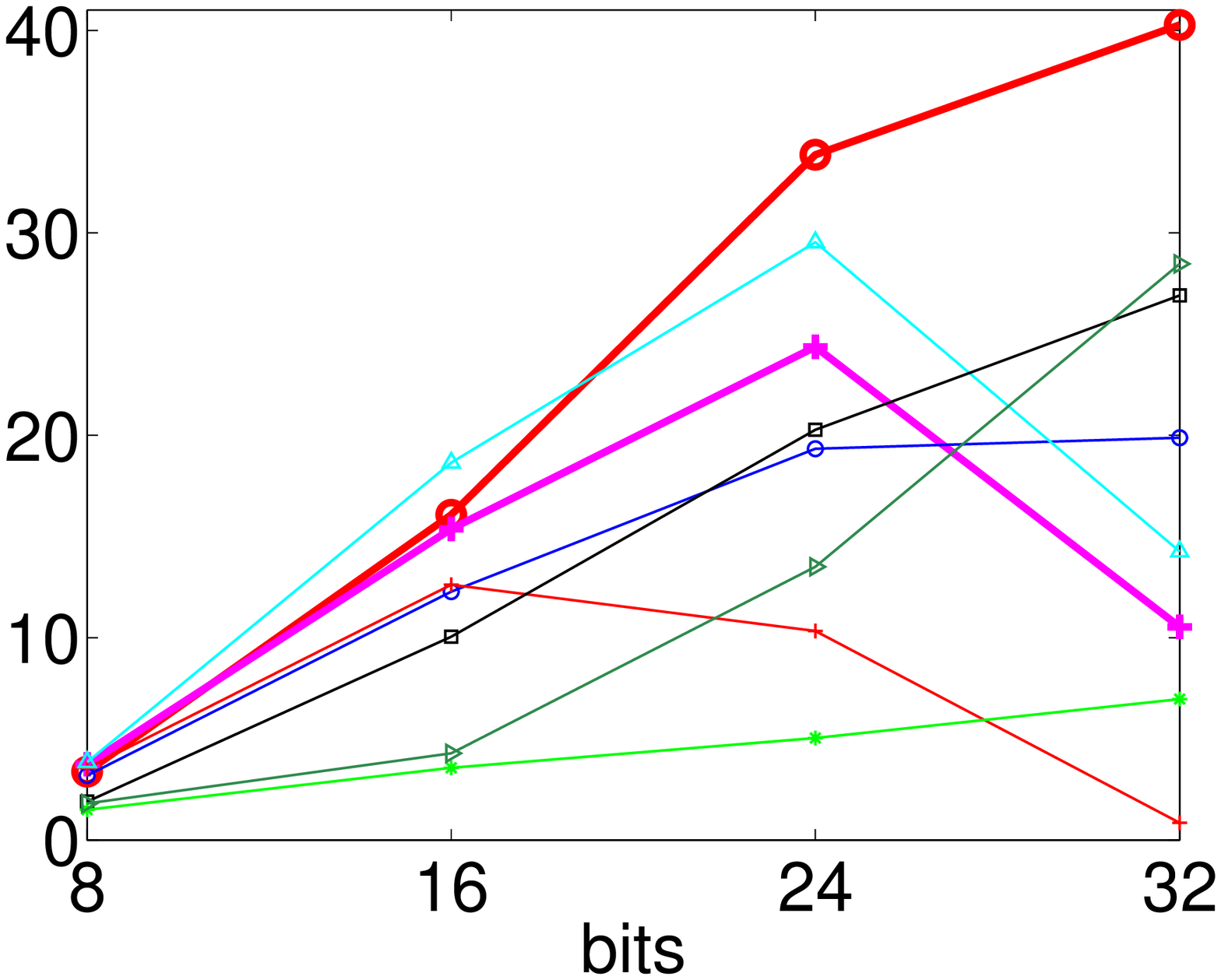} &
    \psfrag{bits}[][]{}
    \includegraphics[width=.33\linewidth]{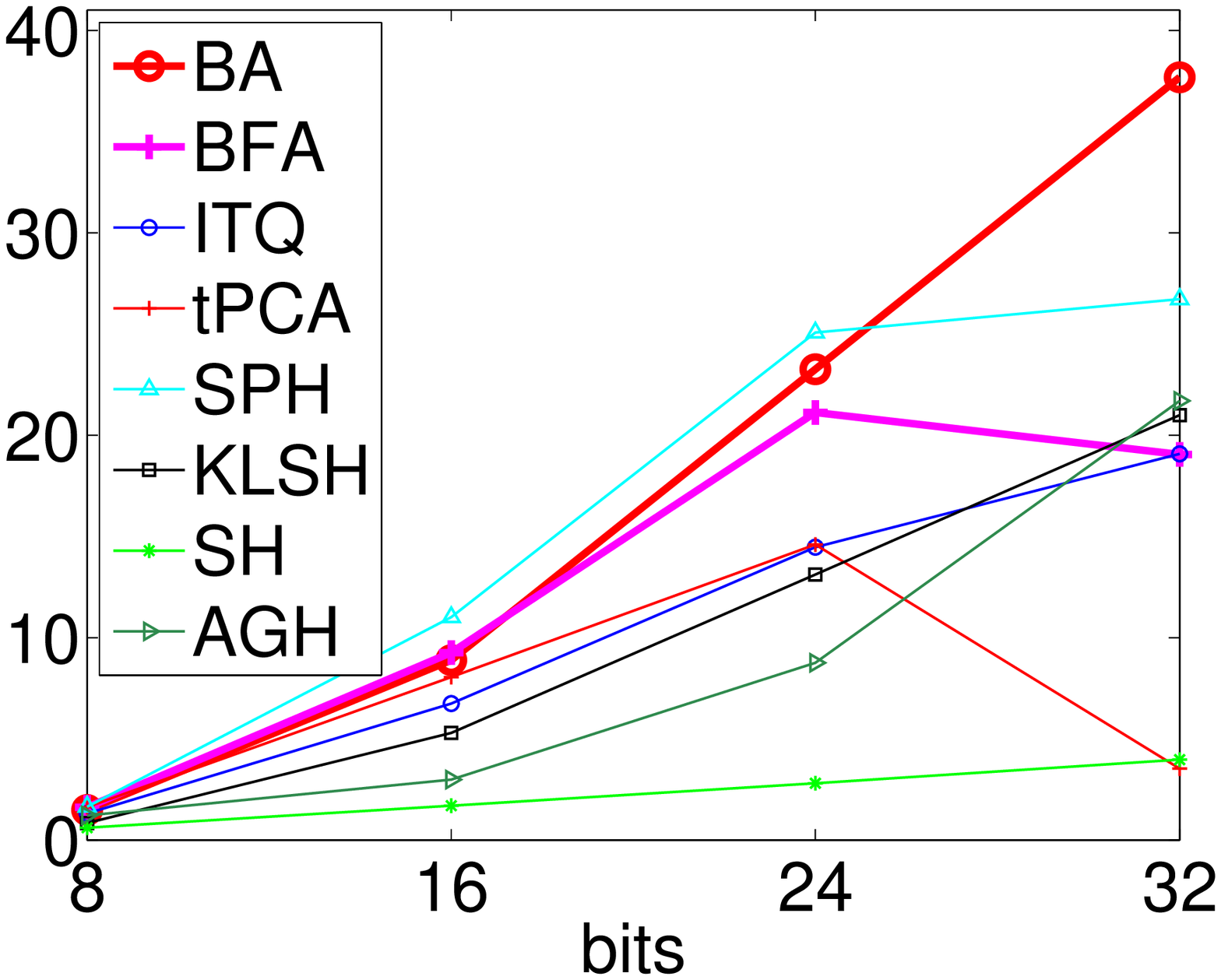} \\
    \psfrag{precision}[][t]{precision $K=1\,500$}
    \psfrag{bits}[][b]{$L$}
    \includegraphics[width=.33\linewidth,height=0.265\linewidth]{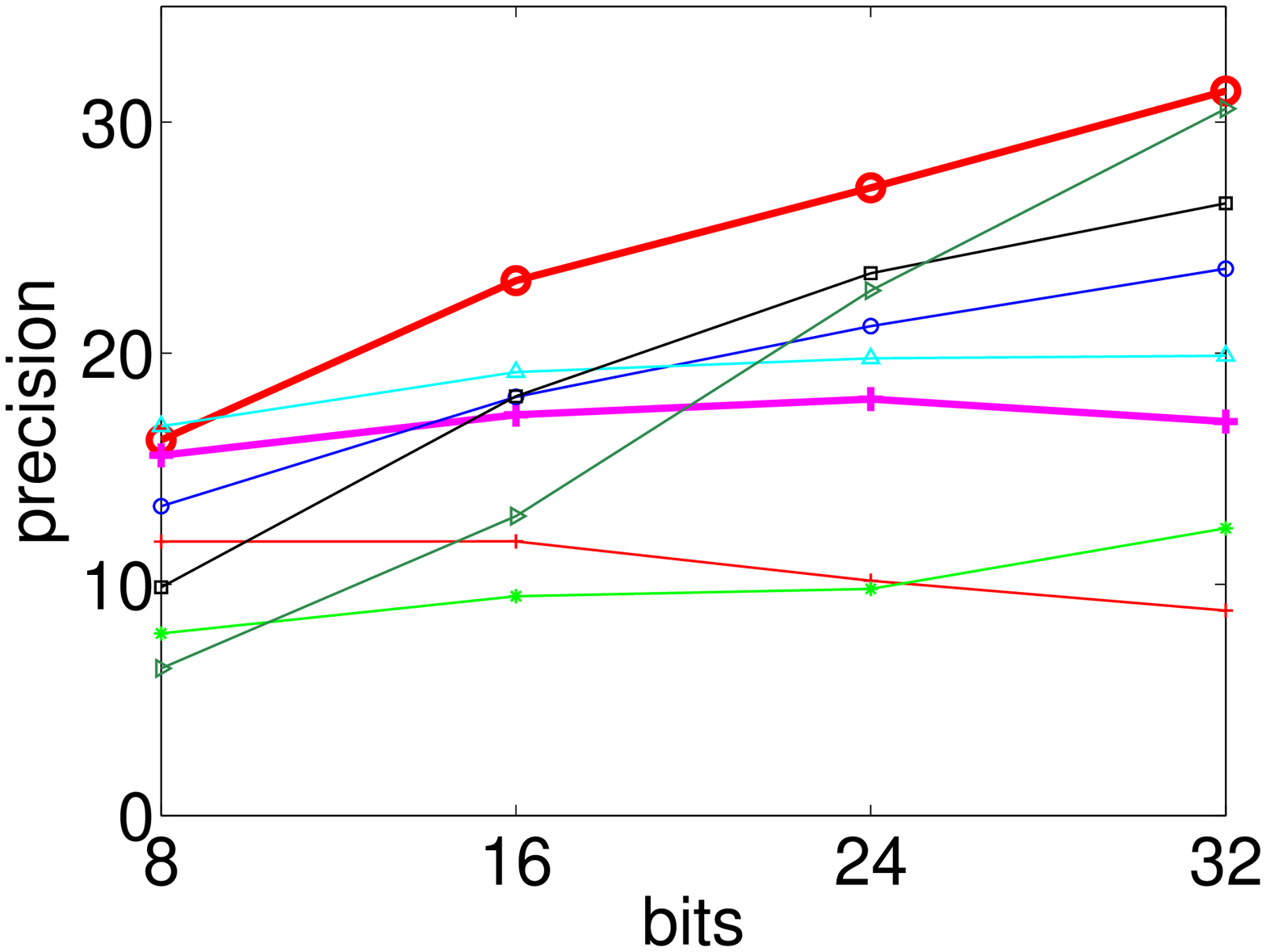} &
    \psfrag{bits}[][b]{$L$}
    \includegraphics[width=.33\linewidth]{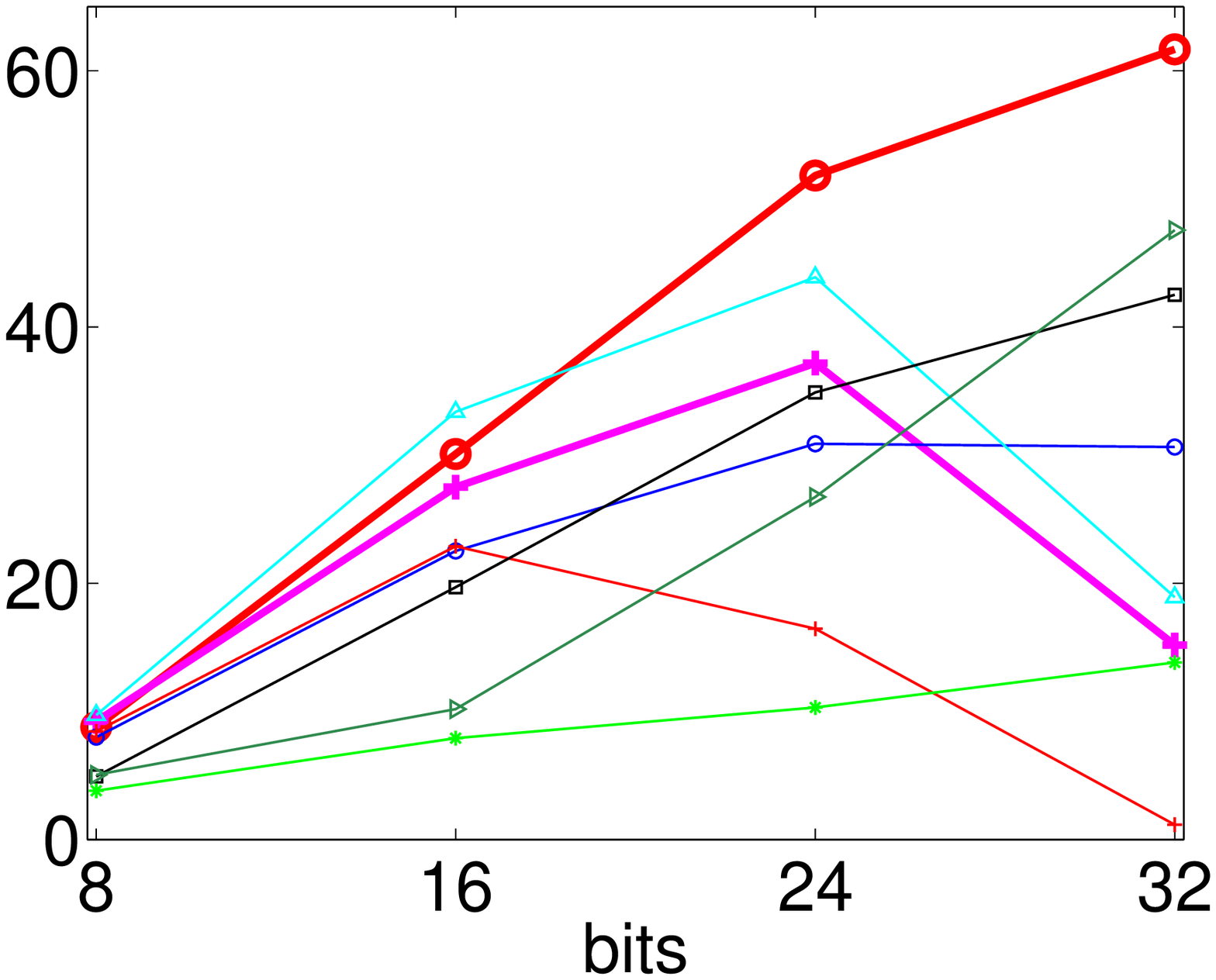} &
    \psfrag{bits}[][b]{$L$}
    \includegraphics[width=.33\linewidth]{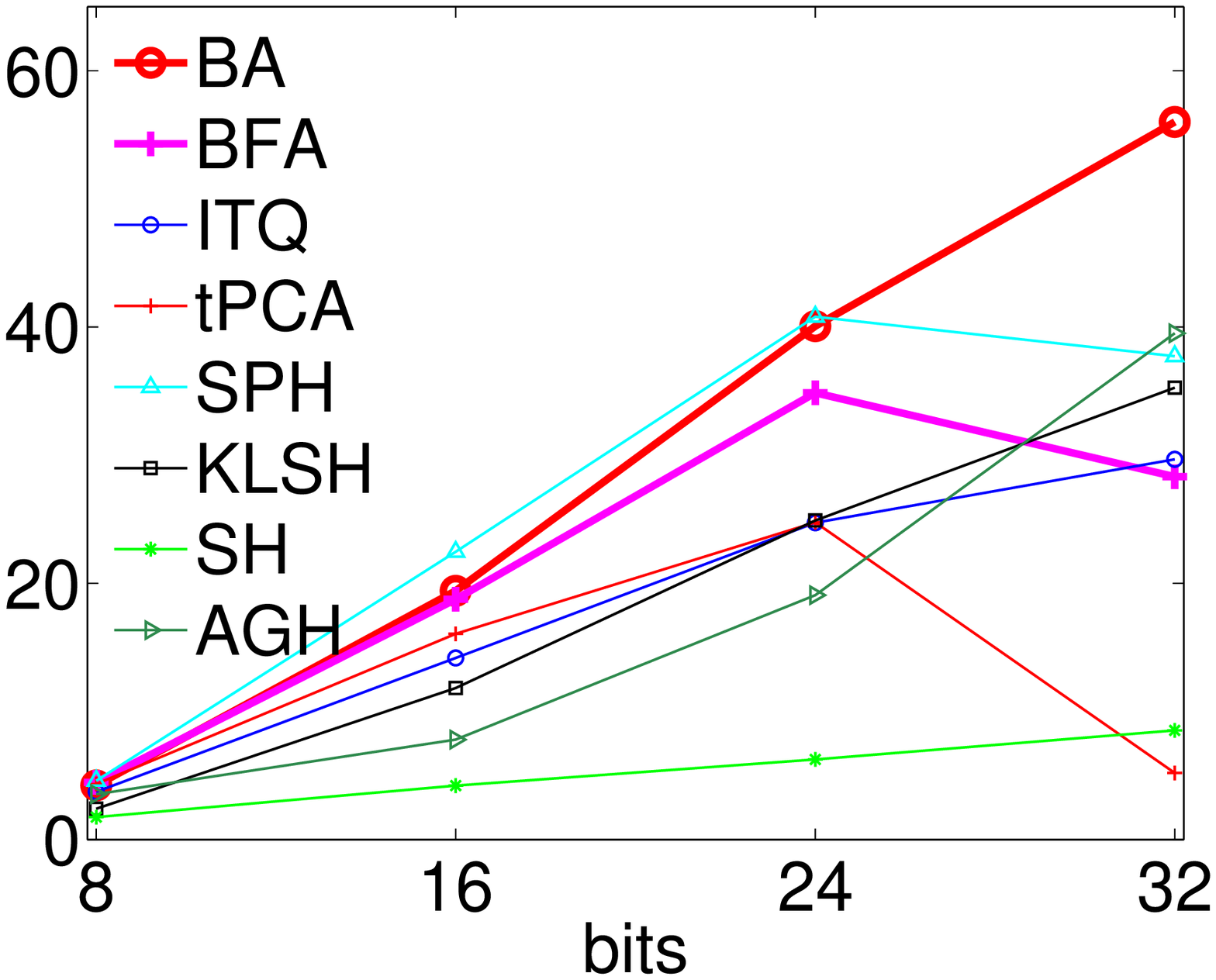}
  \end{tabular} \\[2ex]
  \psfrag{recall}[][b]{recall}
  \psfrag{precision}[][t]{precision}
  \begin{tabular}{@{}c@{\hspace{0\linewidth}}c@{\hspace{0\linewidth}}c@{\hspace{0\linewidth}}c@{\hspace{0\linewidth}}c@{}}
    & $L=8$ & $L=16$ & $L=24$ & $L=32$ \\
    \rotatebox{90}{\hspace{8ex}precision} &
    \includegraphics[width=0.240\linewidth]{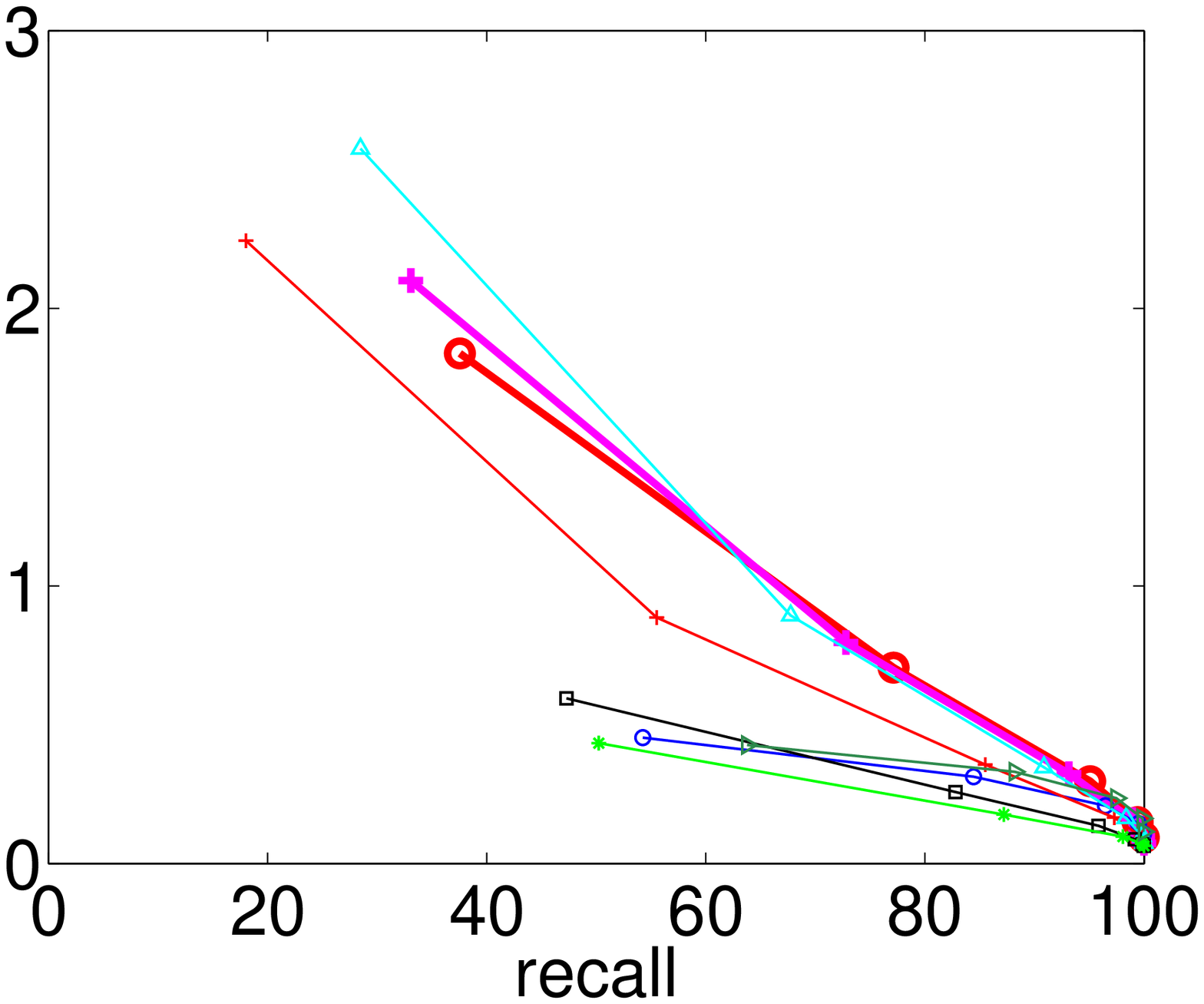} &
    \includegraphics[width=0.248\linewidth]{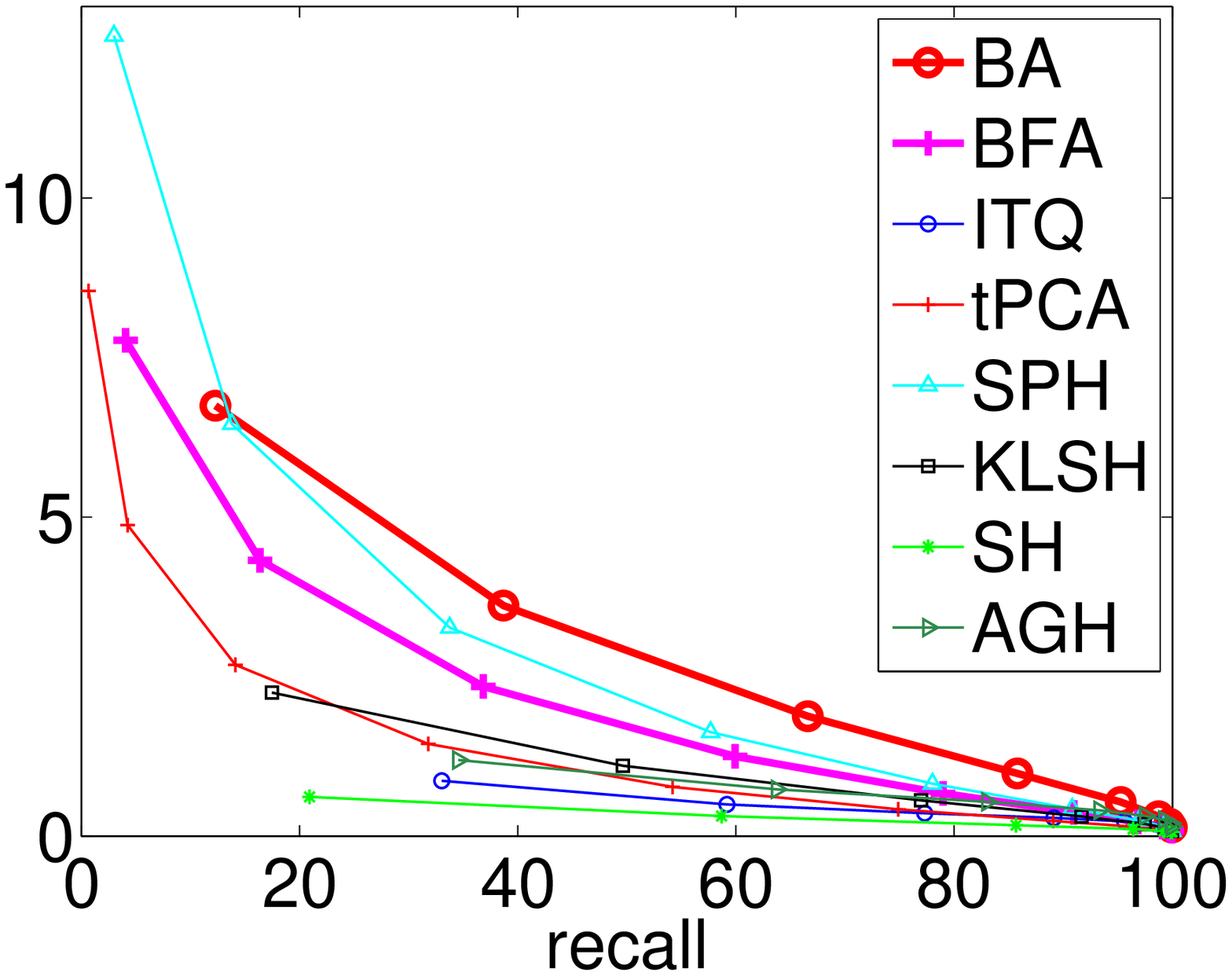} &
    \includegraphics[width=0.248\linewidth]{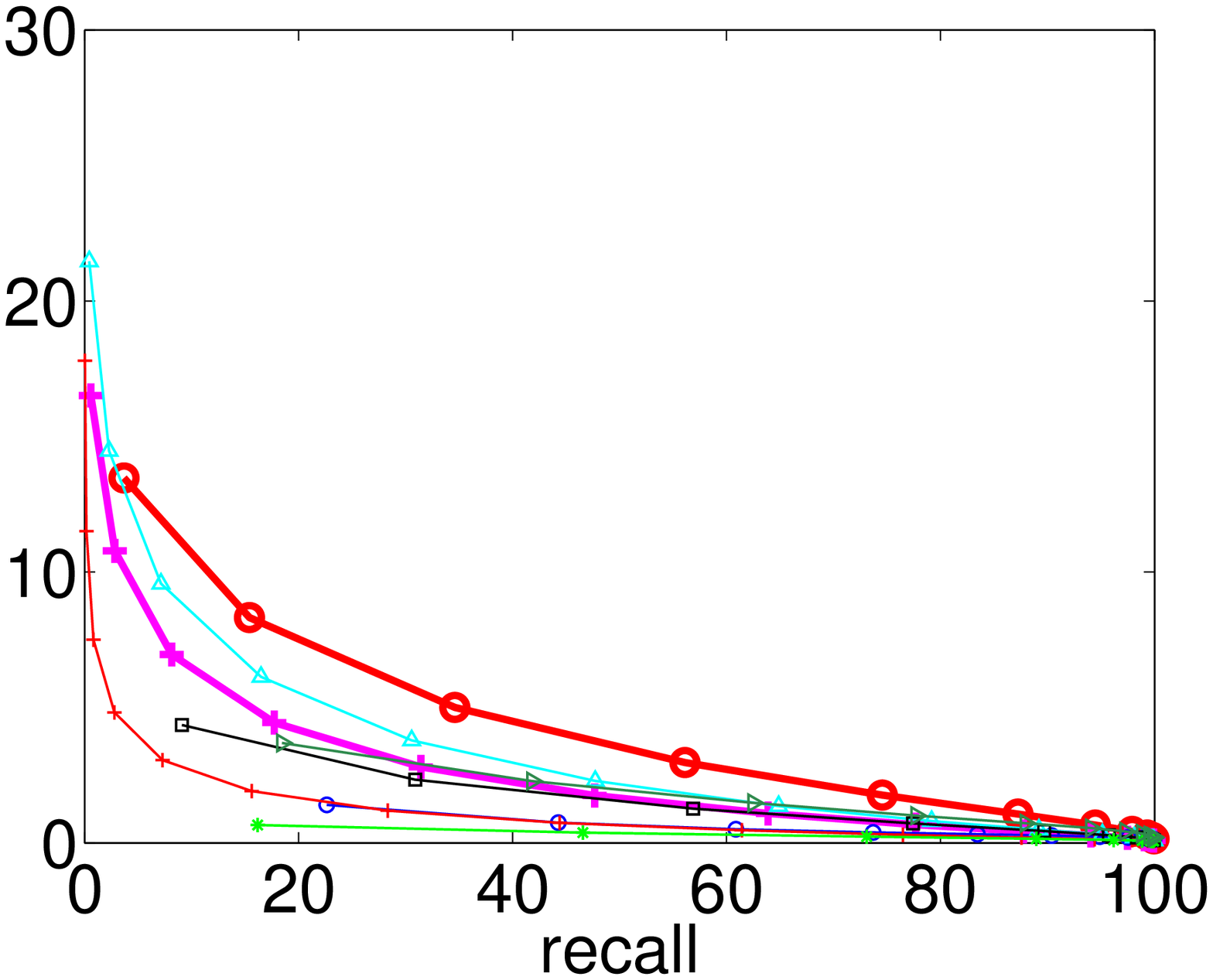} &
    \includegraphics[width=0.248\linewidth]{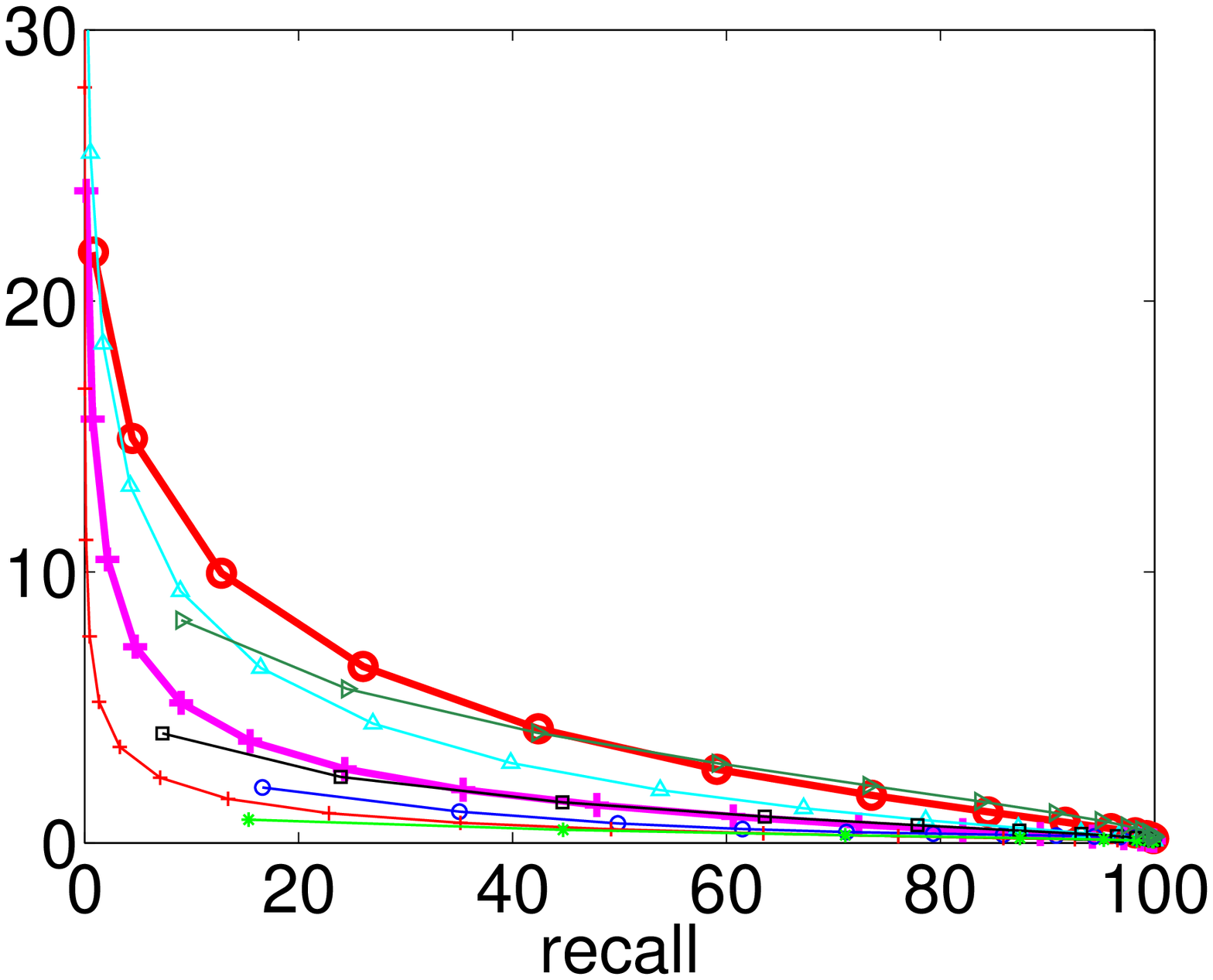}
  \end{tabular}
  \caption{Precision and precision/recall in NUS-WIDE dataset. \emph{Top block}: precision using as ground truth $K=100$ (\emph{top panels}), $500$ (\emph{middle panels}) and $1\,500$ (\emph{bottom panels}) nearest images to the query image in the training set. Retrieved neighbors: $k$ nearest images to the query image, where $k$ is equal to the size of the ground-truth set (\emph{left}), or neighbors at Hamming distance $\le r =1$ (\emph{middle}) or $\le r =2$ (\emph{right}), searching the training set binary codes. \emph{Bottom block}: precision/recall using $L=8$ to $32$ bits. Ground truth: $K=100$ nearest images to the query image in the training set. Retrieved neighbors: training images at Hamming distance $\le r$ of the query in binary space. Within each plot, the markers left to right along each curve show $r = 0,1,2\dots$. Test points not returning any neighbors are ignored in the precision/recall plots.}
  \label{f:NUS-WIDE}
\end{figure}

\begin{figure}[t]
  \centering
  \psfrag{precision}[][t]{precision}
  \begin{tabular}{@{}c@{\hspace{0.0\linewidth}}c@{\hspace{0.0\linewidth}}c@{\hspace{0.0\linewidth}}c@{}}
    $k=10\,000$ neighbors & Hamming distance $\le 1$ & Hamming distance $\le 2$ & Hamming distance $\le 3$ \\
    \psfrag{bits}[][b]{$L$}
    \includegraphics[width=0.265\linewidth]{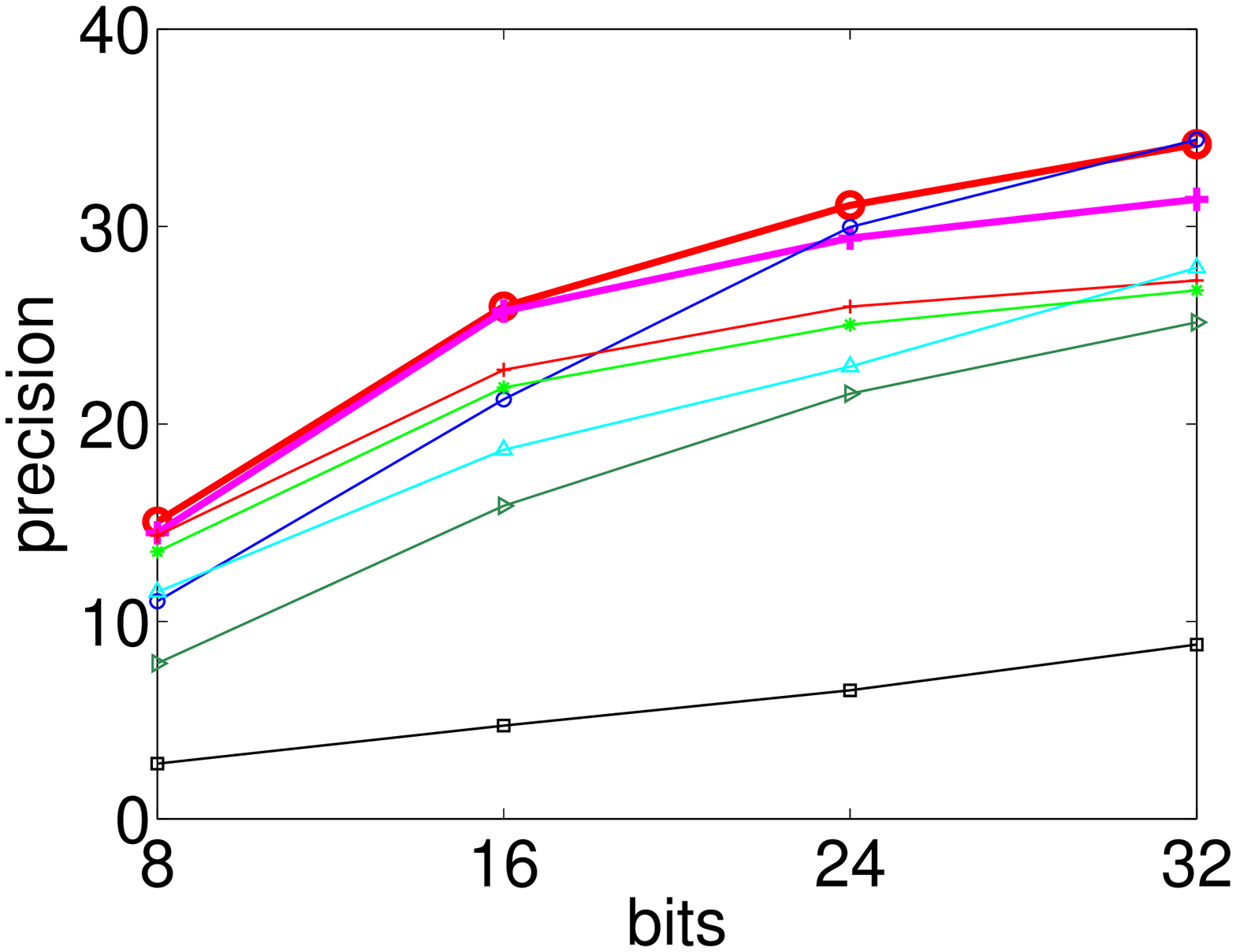} &
    \psfrag{bits}[][b]{$L$}
    \includegraphics[width=0.245\linewidth]{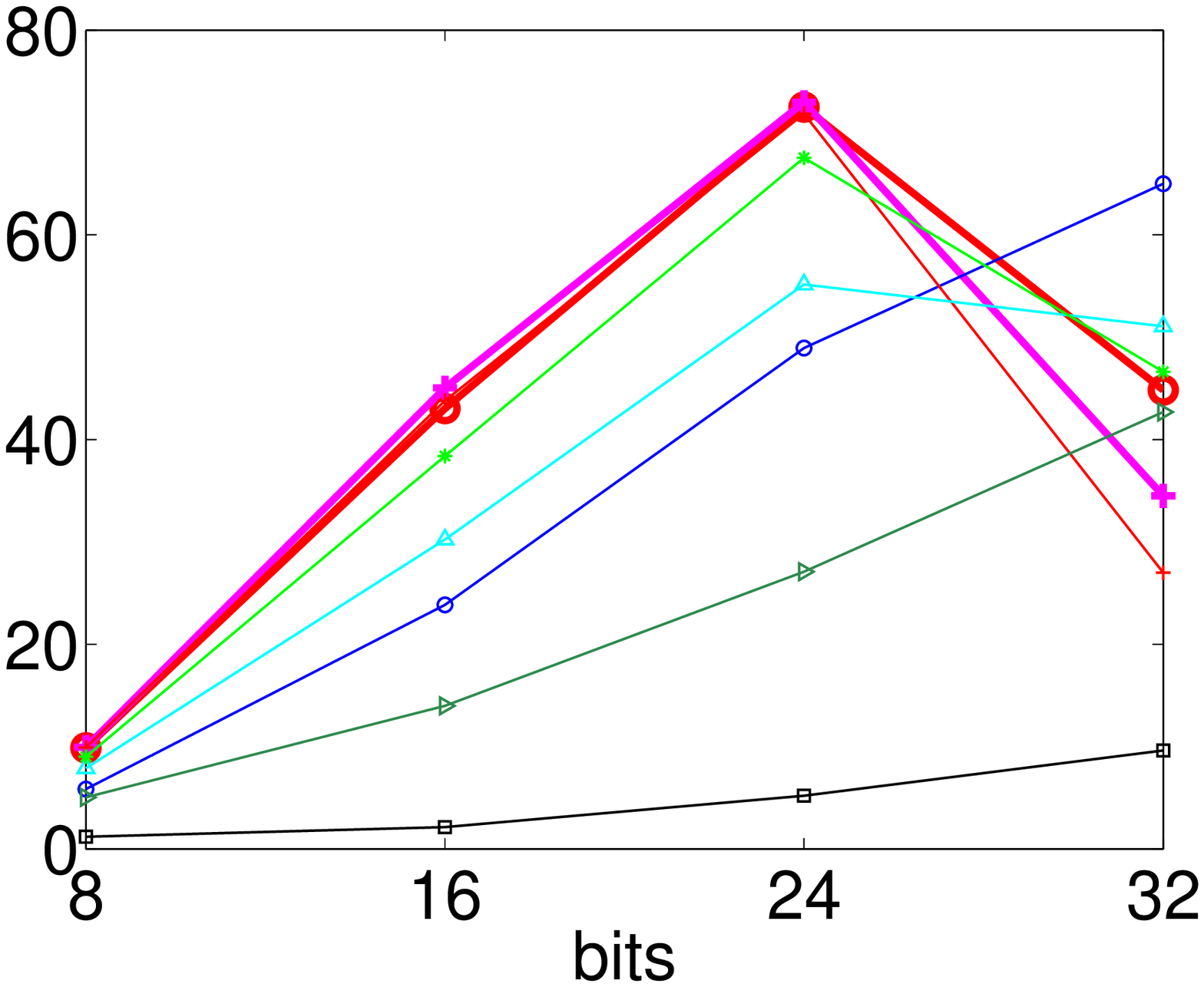} &
    \psfrag{bits}[][b]{$L$}
    \includegraphics[width=0.245\linewidth]{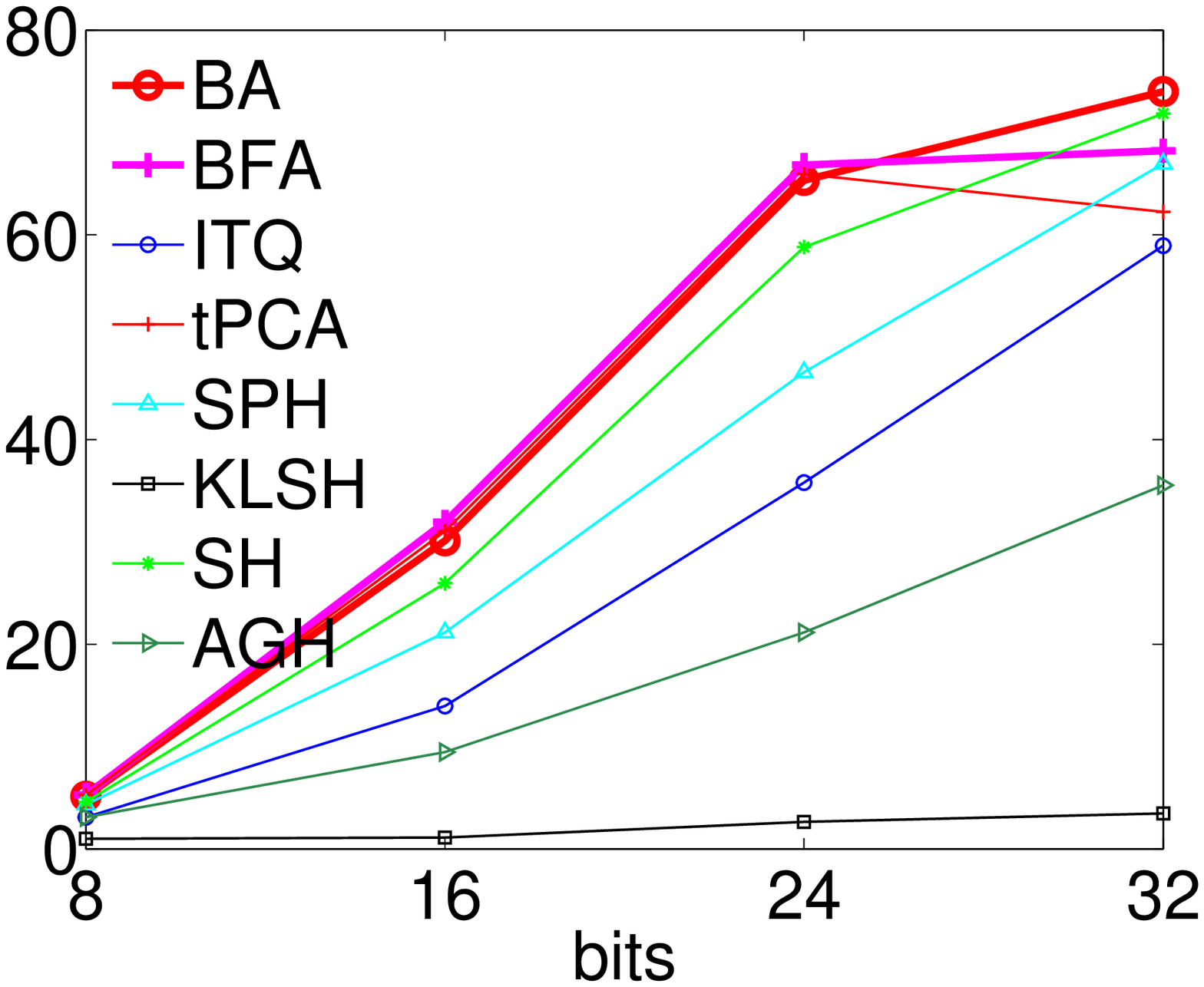} &
    \psfrag{bits}[][b]{$L$}
    \includegraphics[width=0.245\linewidth]{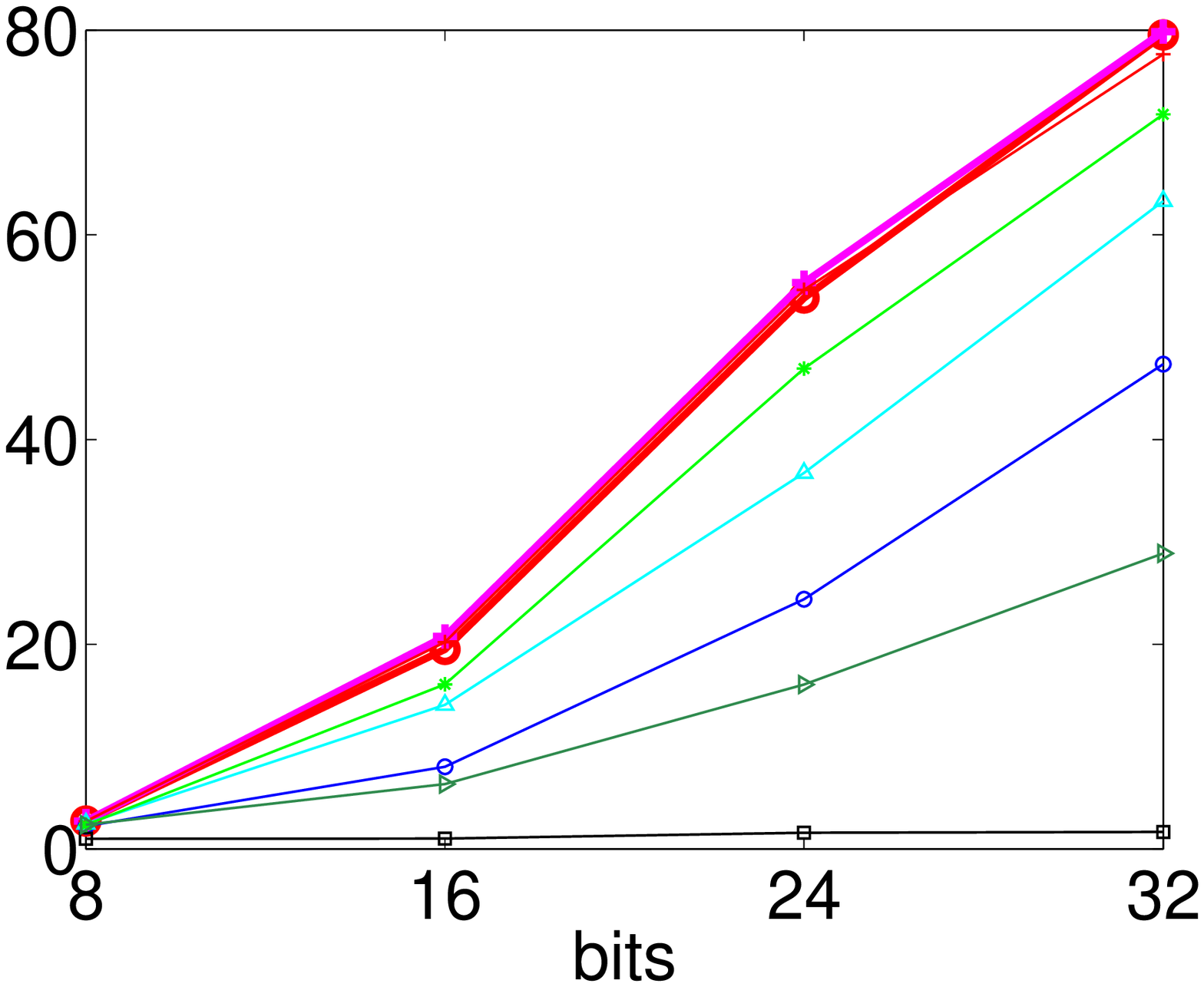} \\
    \psfrag{bits}{}
    \includegraphics[width=0.265\linewidth]{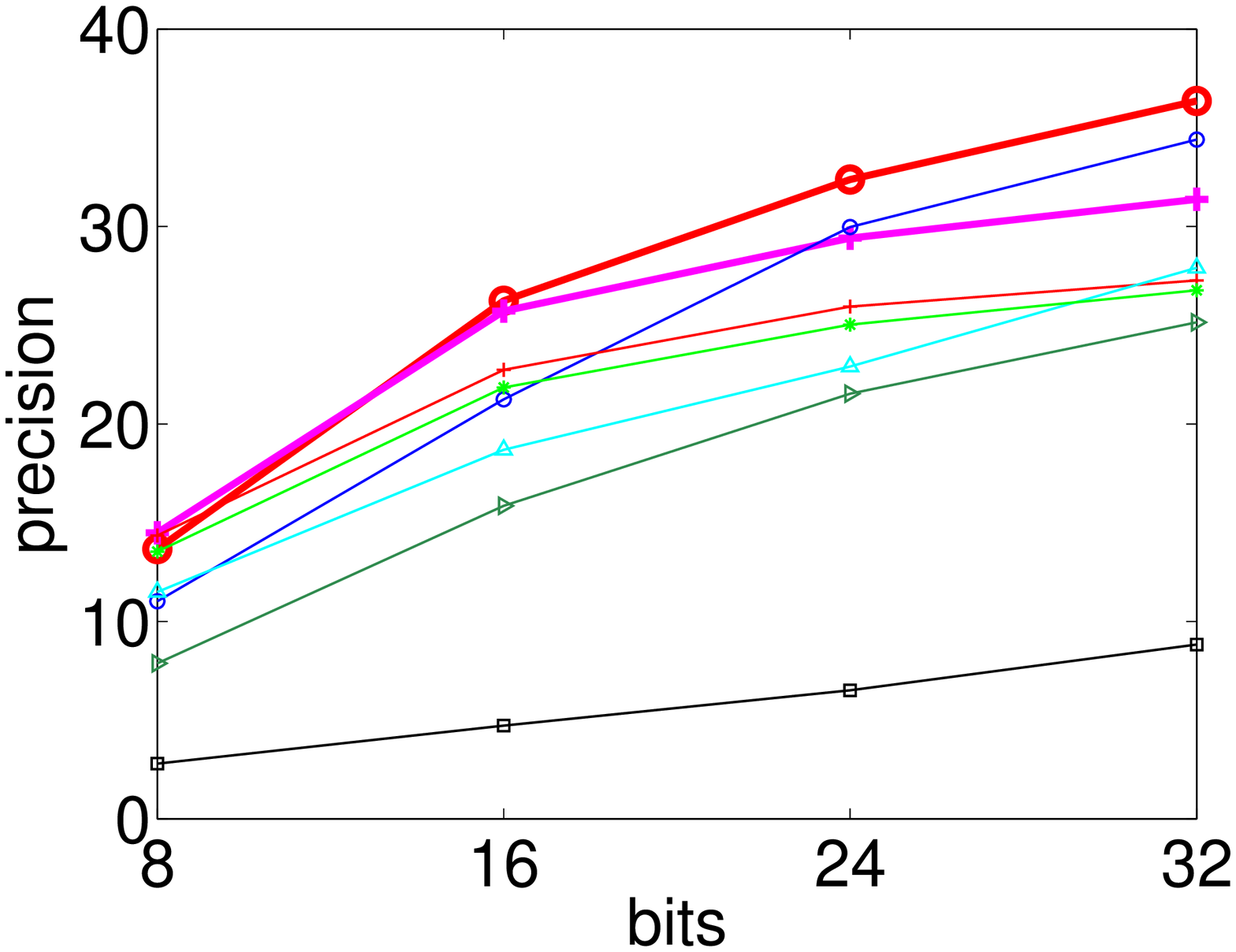} &
    \psfrag{bits}{}
    \includegraphics[width=0.245\linewidth]{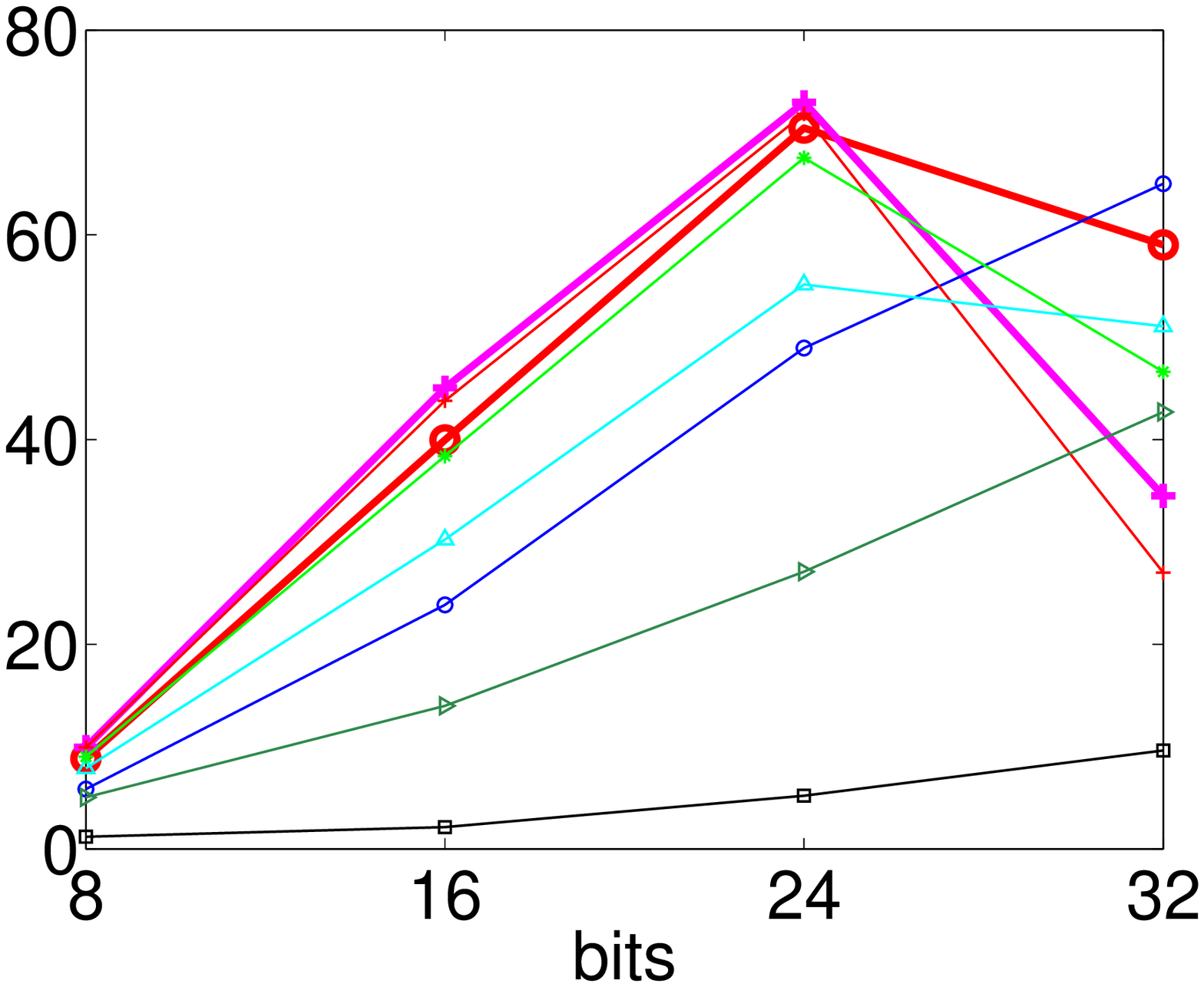} &
    \psfrag{bits}{}
    \includegraphics[width=0.245\linewidth]{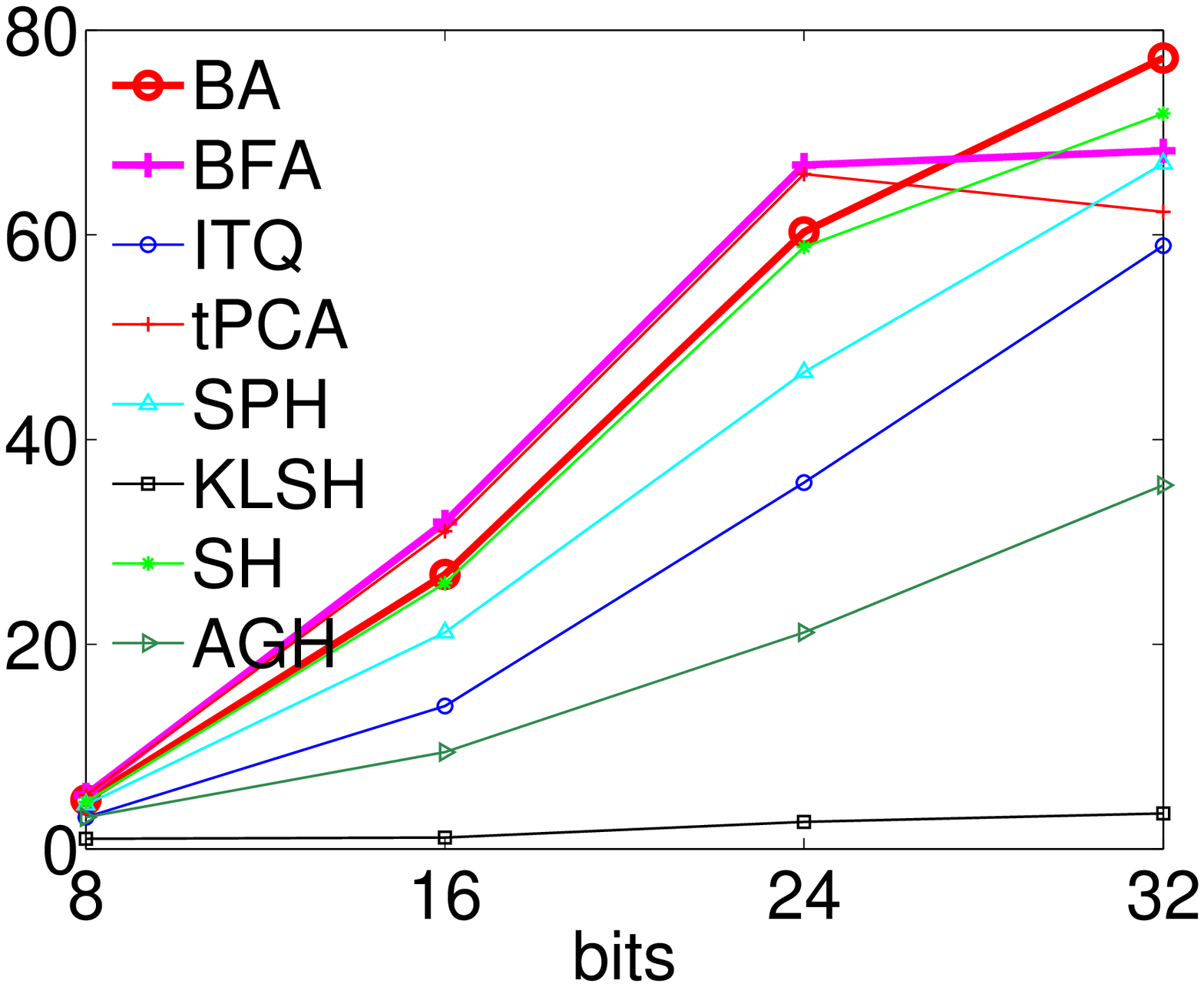} &
    \psfrag{bits}{}
    \includegraphics[width=0.245\linewidth]{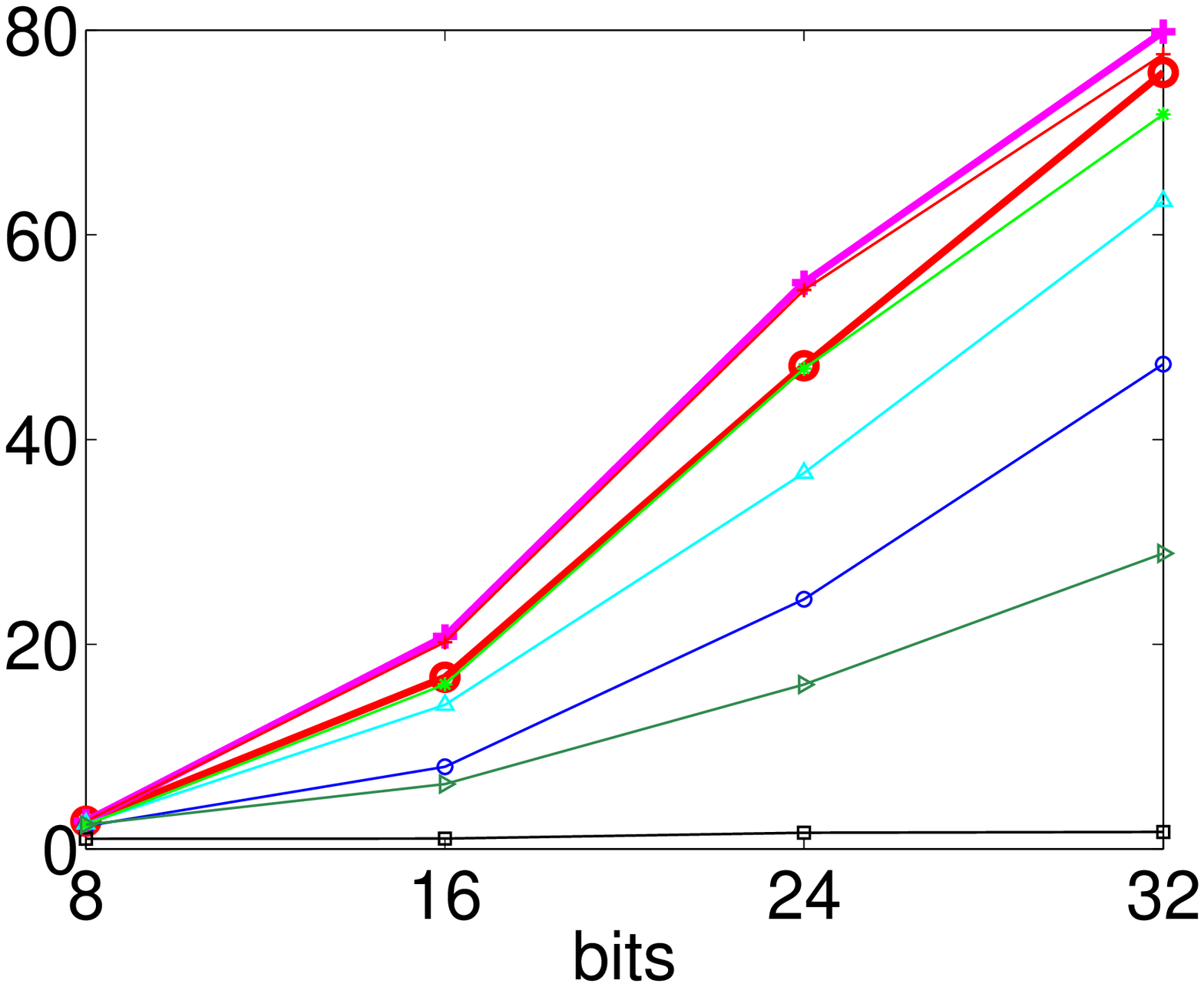}
  \end{tabular}
  \caption{Precision in ANNSIFT-1M using $L=8$ to $32$ bits. Like fig.~\ref{f:NUS-WIDE} (top block) but with ground truth $K=10\,000$ and retrieved neighbors $k=10\,000$ or Hamming distance $\le r =1$ to $3$. BFA is initialized with tPCA. BA is initialized with AGH (\emph{top panel}) and ITQ (\emph{bottom panel}).}
  \label{f:ANNSIFT-1M}
\end{figure}

\begin{figure}[t]
  \centering
  \psfrag{Recall}[][b]{recall}
  \psfrag{bits}[][b]{$L$}
  \begin{tabular}{@{}c@{\hspace{0\linewidth}}c@{\hspace{0\linewidth}}c@{}}
    $k=100$ neighbors retrieved & Hamming distance $\le 3$ & Hamming distance $\le 4$ \\
    \psfrag{precision}[][t]{precision $K=1\,000$}
    \includegraphics[width=0.35\linewidth]{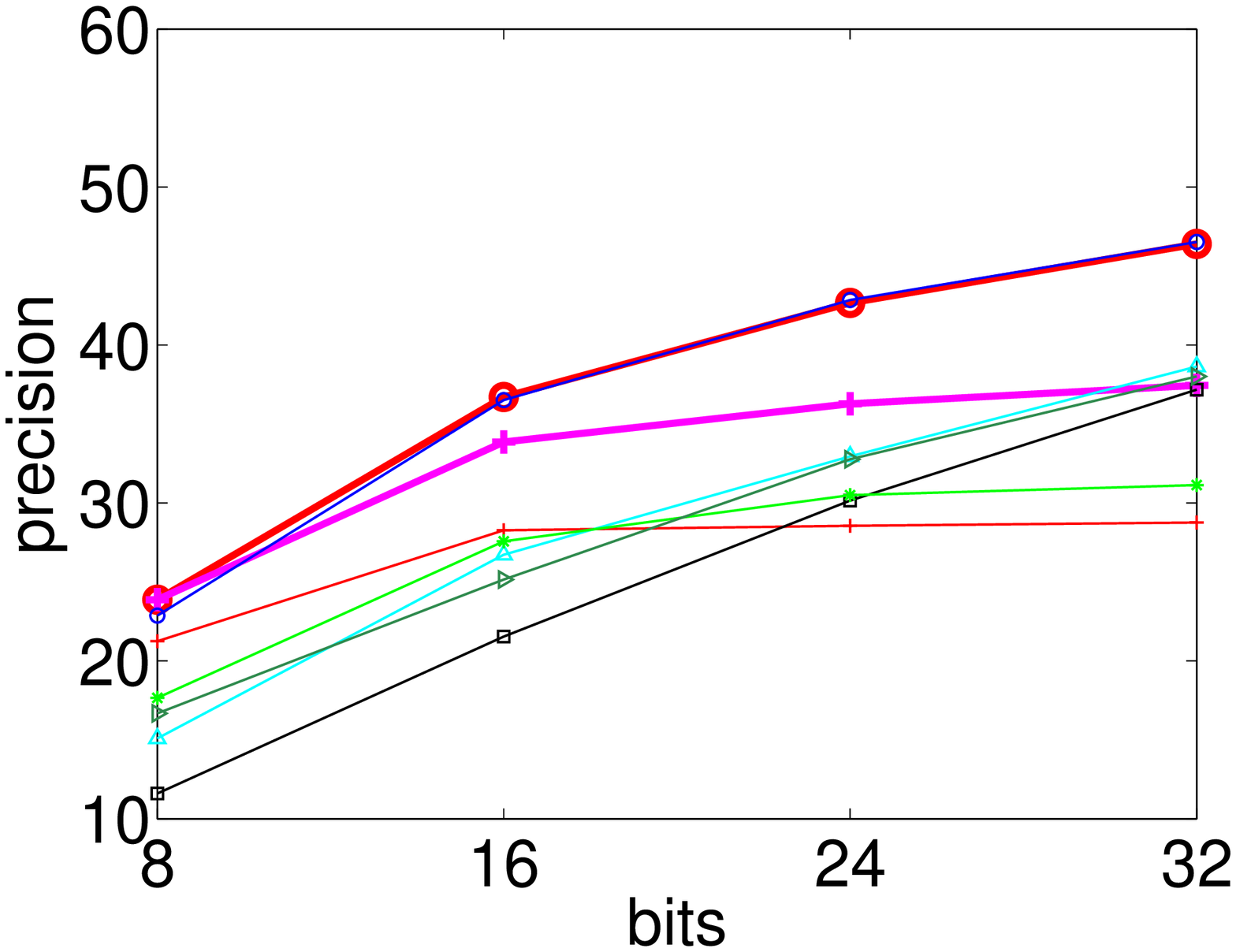} &
    \includegraphics[width=0.325\linewidth]{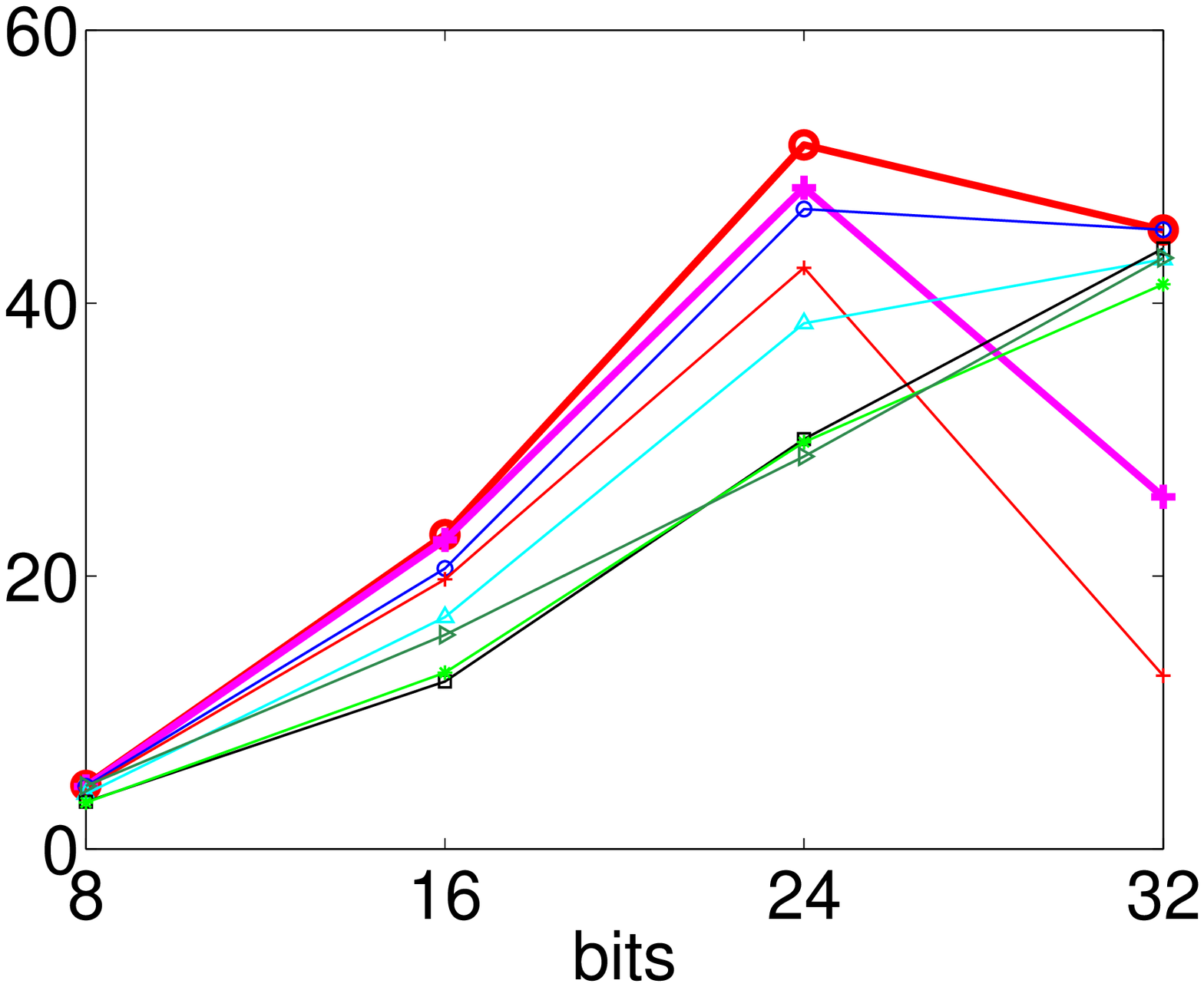} &
    \includegraphics[width=0.325\linewidth]{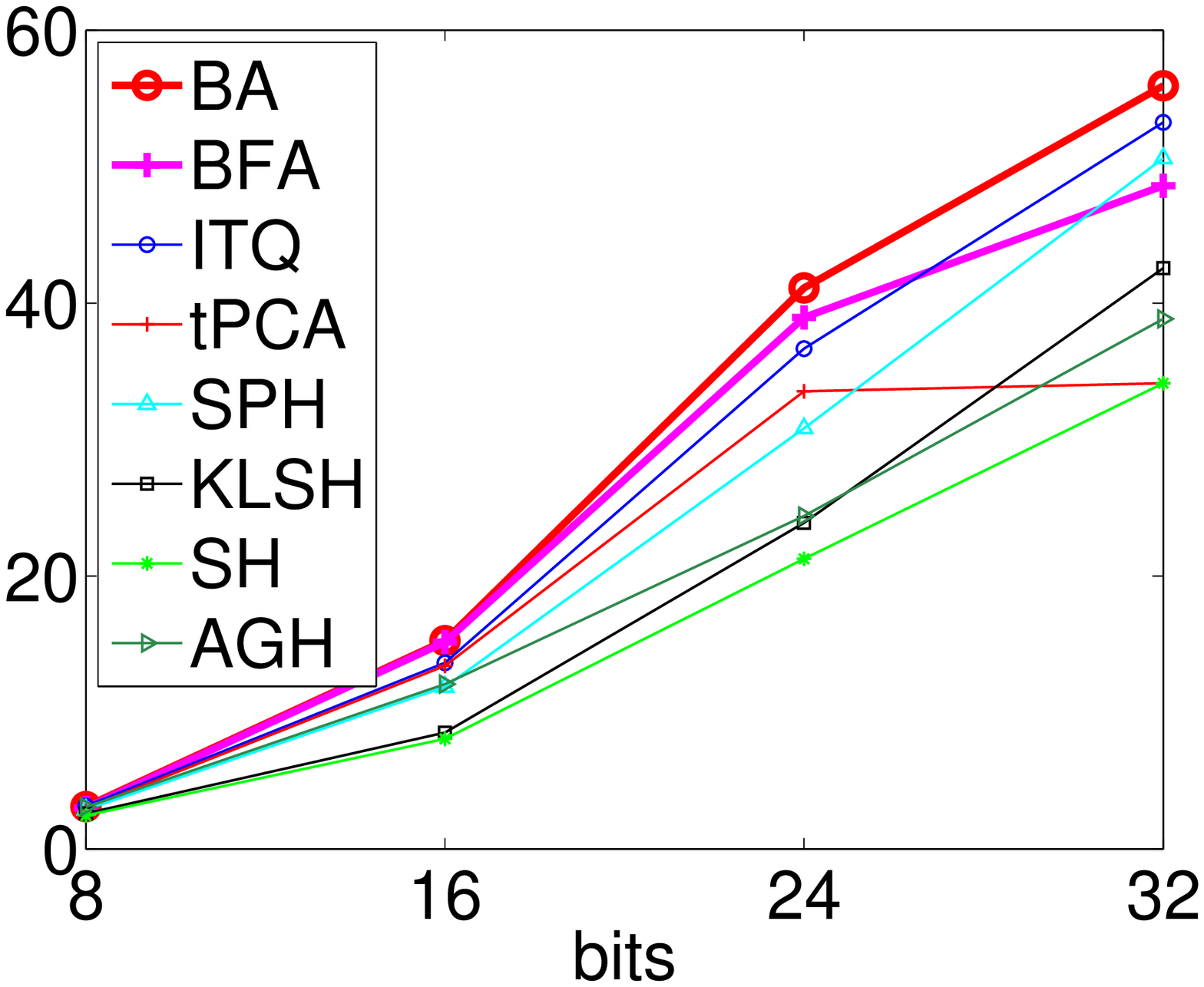}
  \end{tabular} \\[1ex]
  \begin{tabular}{@{}c@{}c@{}c@{}c@{}}
    $L = 8$ bits & $L = 16$ bits & $L = 24$ bits & $L = 32$ bits \\[-0.5ex]
    \psfrag{precision}[][t]{precision $K=1\,000$}
    \includegraphics[width=0.25\linewidth,height=0.202\linewidth]{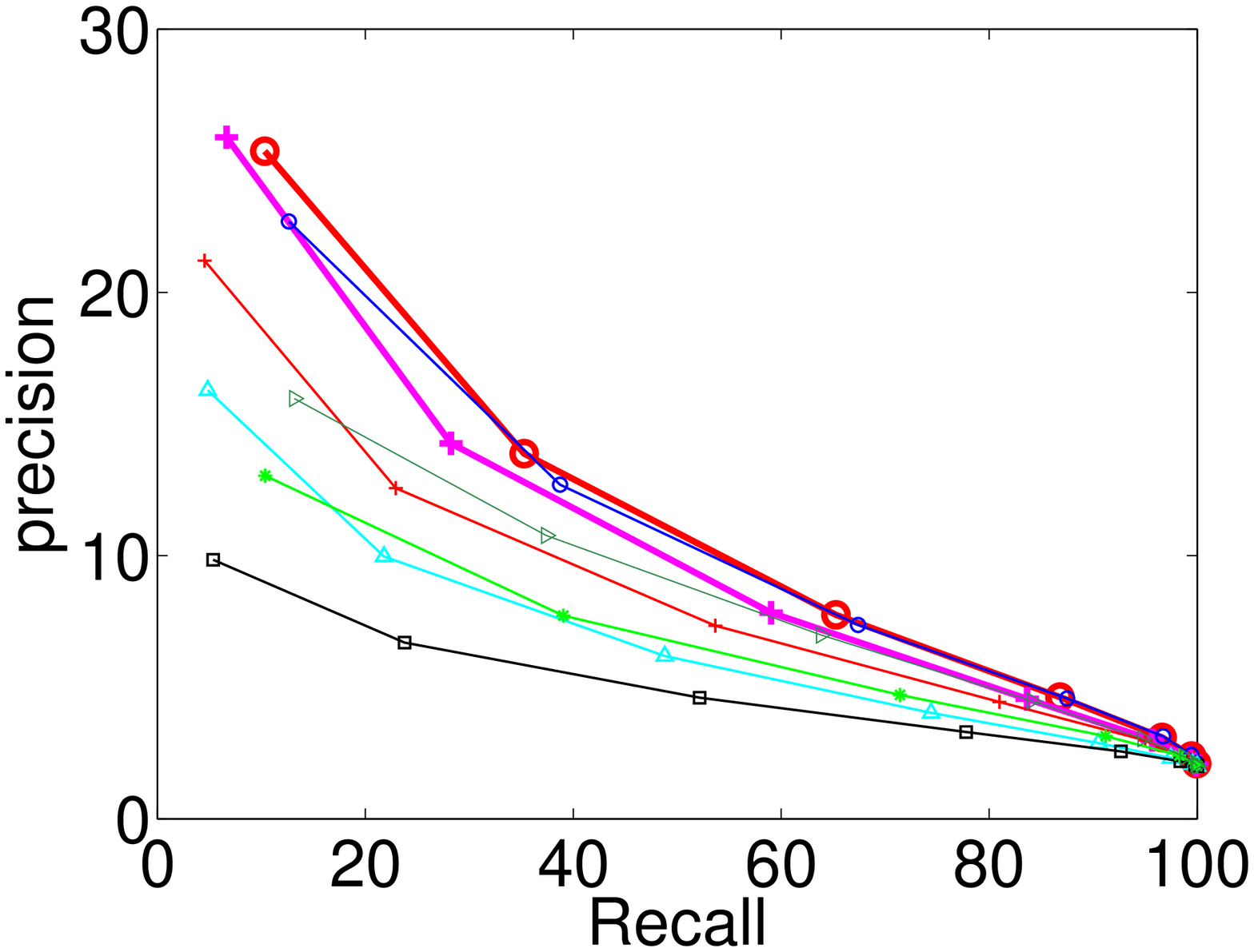} &
    \includegraphics[width=0.25\linewidth,height=0.202\linewidth]{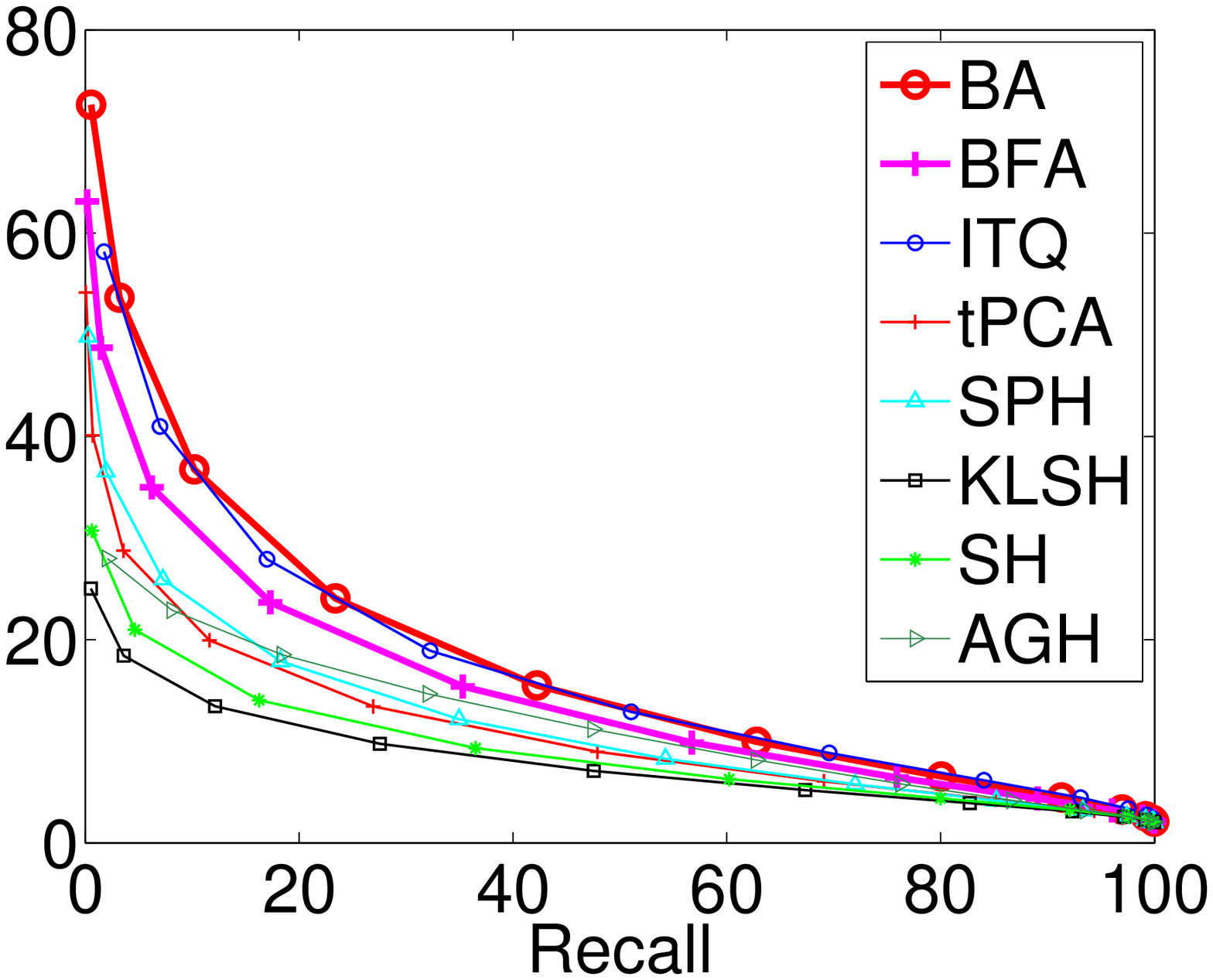} &
    \includegraphics[width=0.25\linewidth]{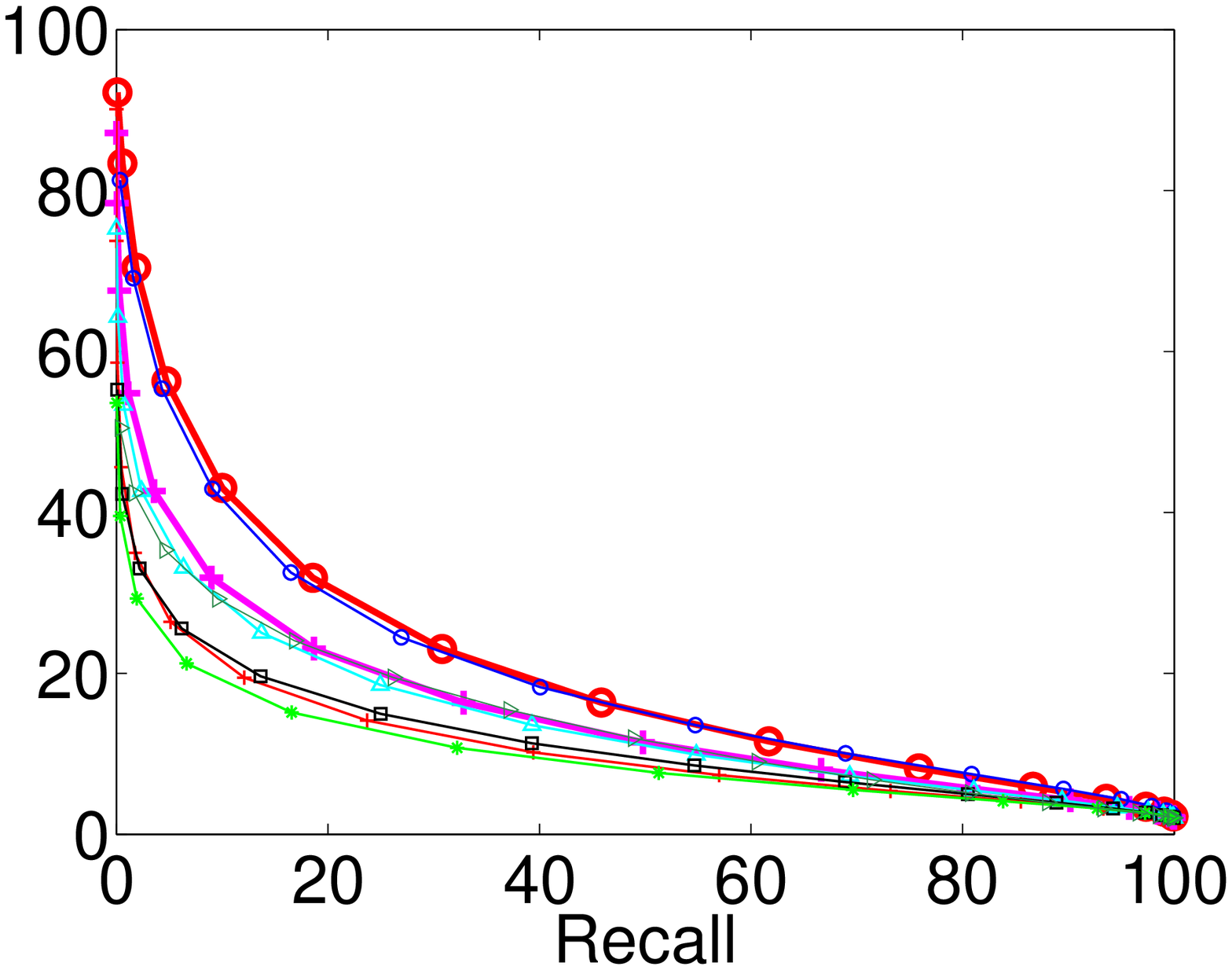} &
    \includegraphics[width=0.25\linewidth]{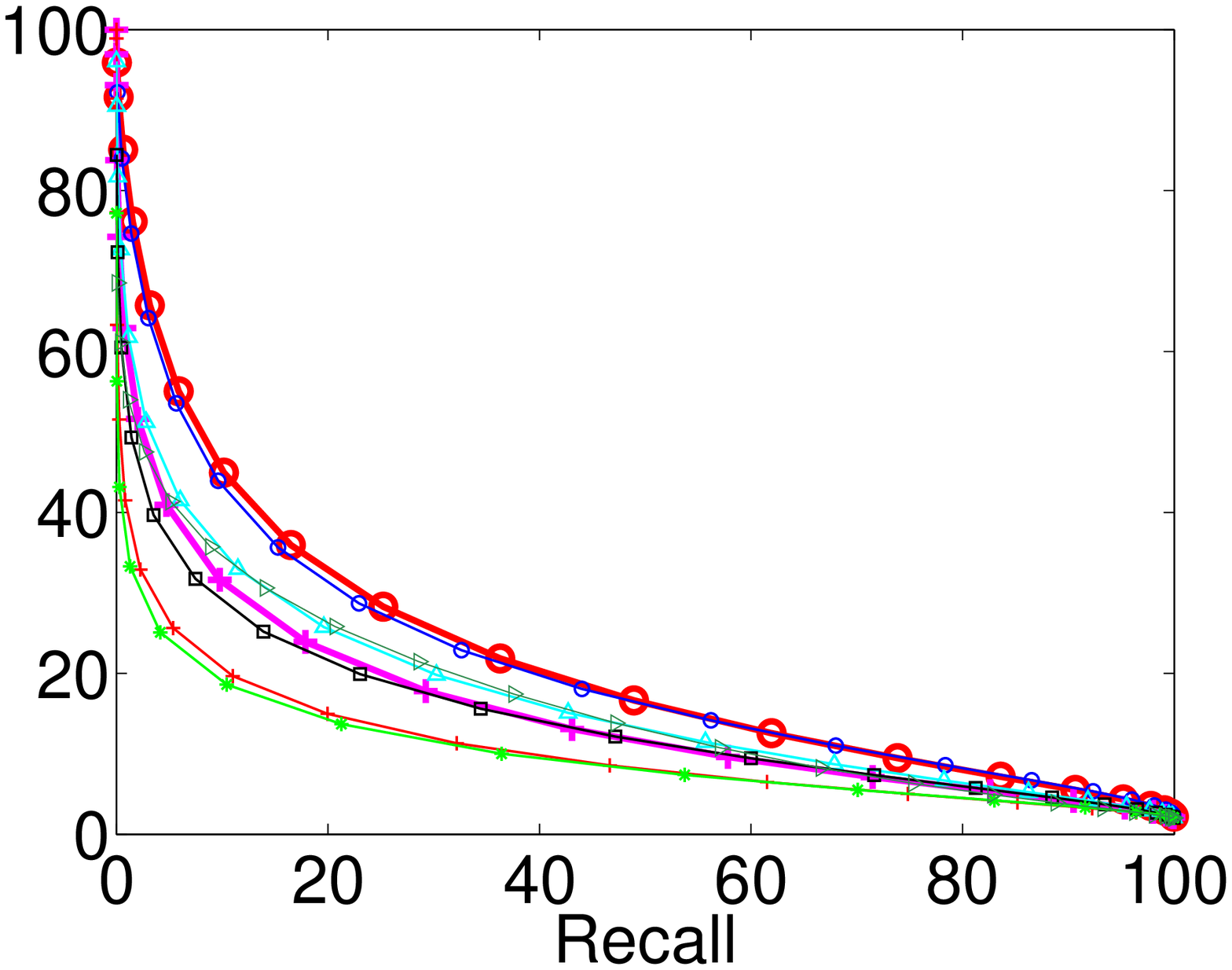}
  \end{tabular} \\[5ex]
  \begin{tabular}{@{}c@{\hspace{0\linewidth}}c@{\hspace{0\linewidth}}c@{}}
    $k=50$ neighbors retrieved & Hamming distance $\le 3$ & Hamming distance $\le 4$ \\
    \psfrag{precision}[][t]{precision $K=50$}
    \includegraphics[width=0.33\linewidth,height=0.276\linewidth]{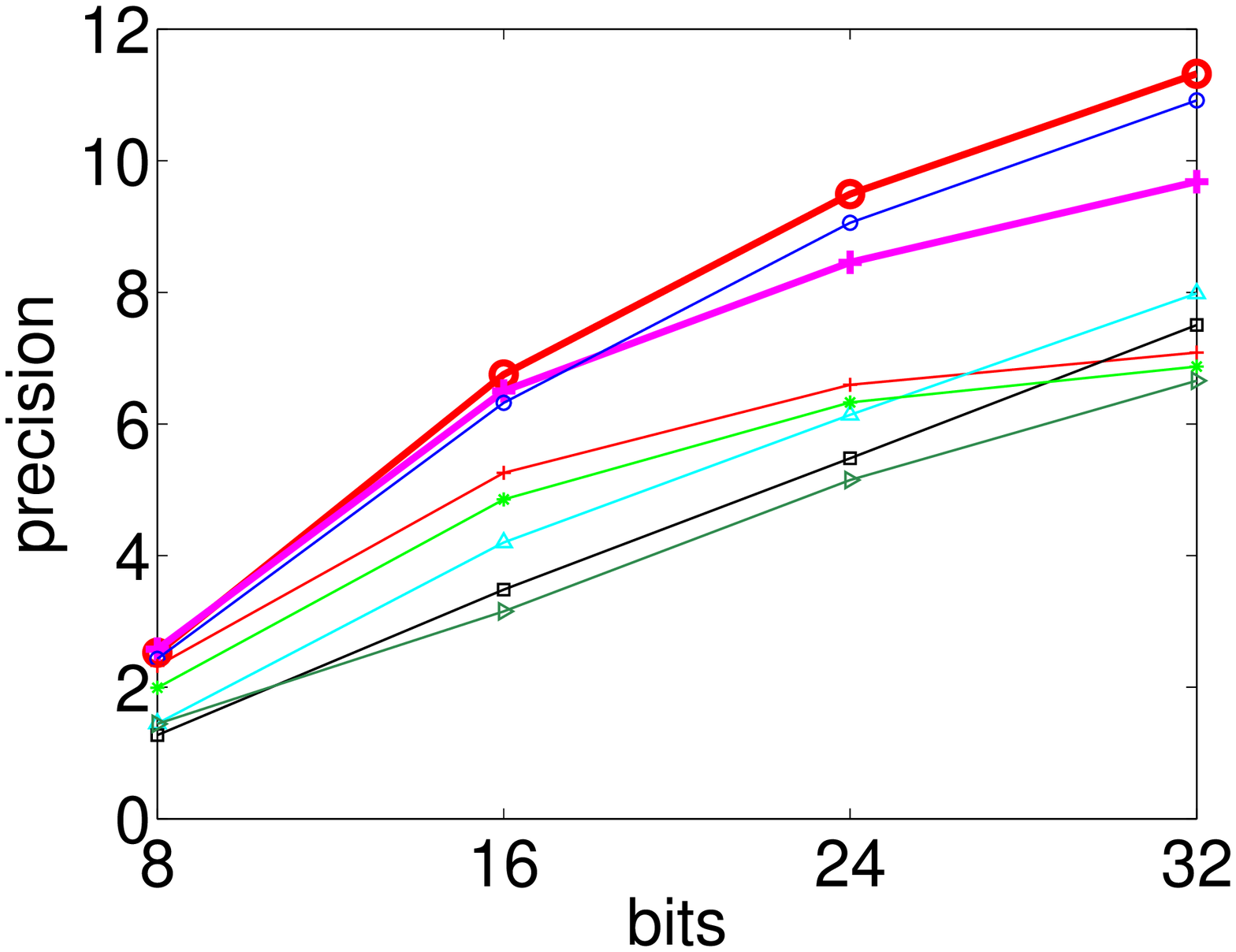} &
    \includegraphics[width=0.33\linewidth,height=0.271\linewidth]{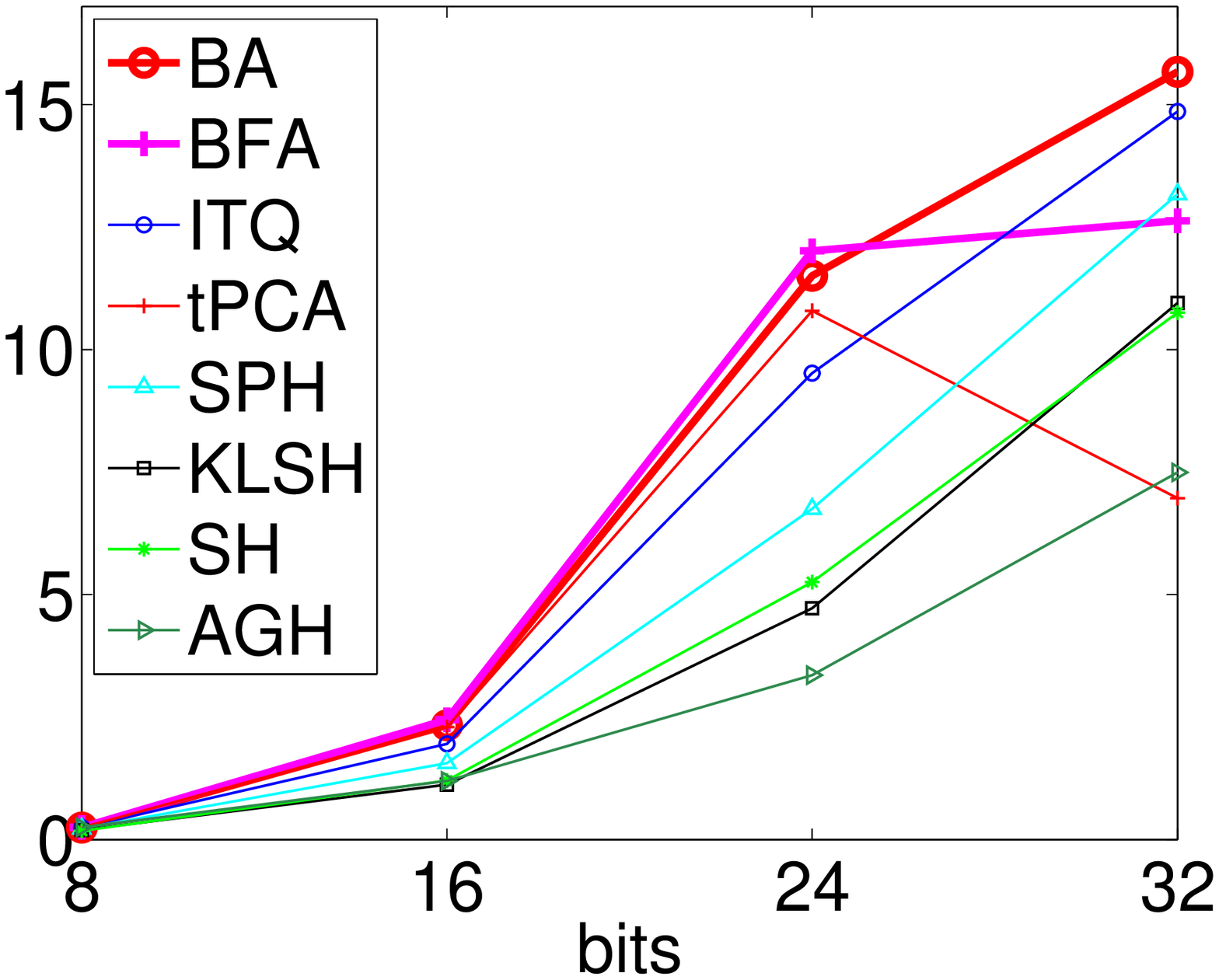} &
    \includegraphics[width=0.33\linewidth]{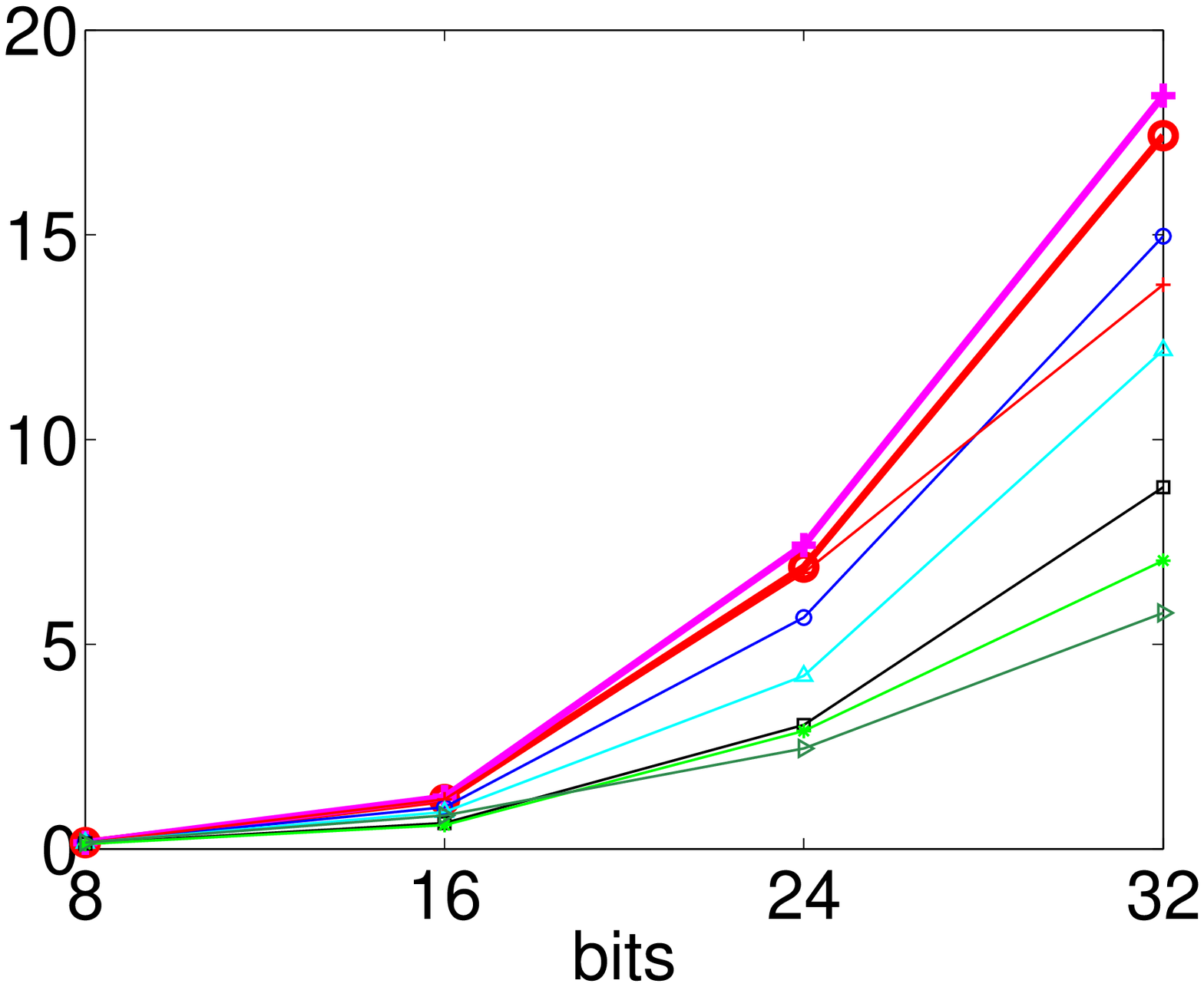}
  \end{tabular} \\
  \begin{tabular}{@{}c@{}c@{\hspace{0\linewidth}}c@{\hspace{0\linewidth}}c@{}}
    $L = 8$ bits & $L = 16$ bits & $L = 24$ bits & $L = 32$ bits \\
    \psfrag{precision}[][t]{precision $K=50$}
    \includegraphics[width=0.25\linewidth,height=0.203\linewidth]{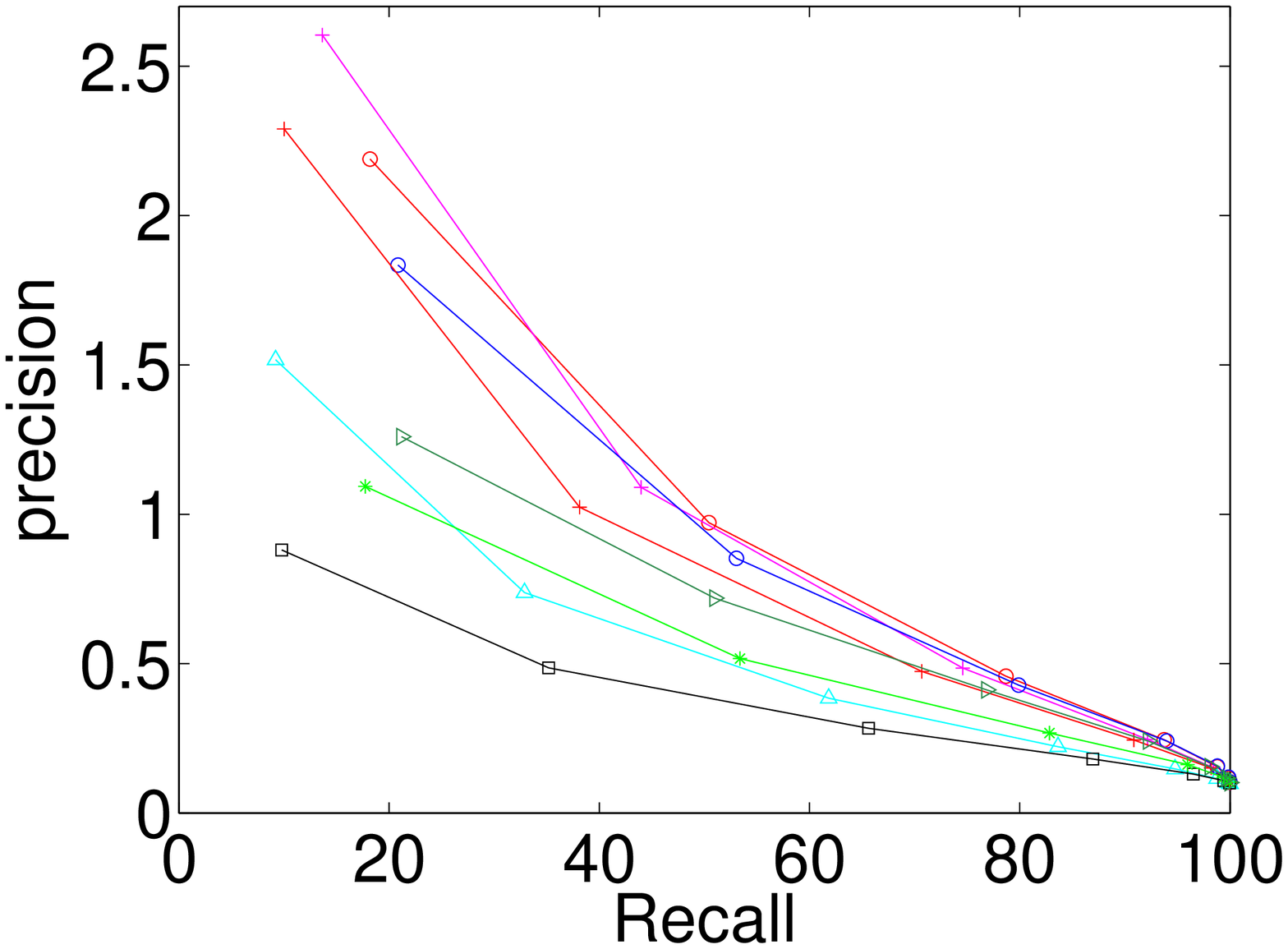} &
    \includegraphics[width=0.25\linewidth,height=0.206\linewidth]{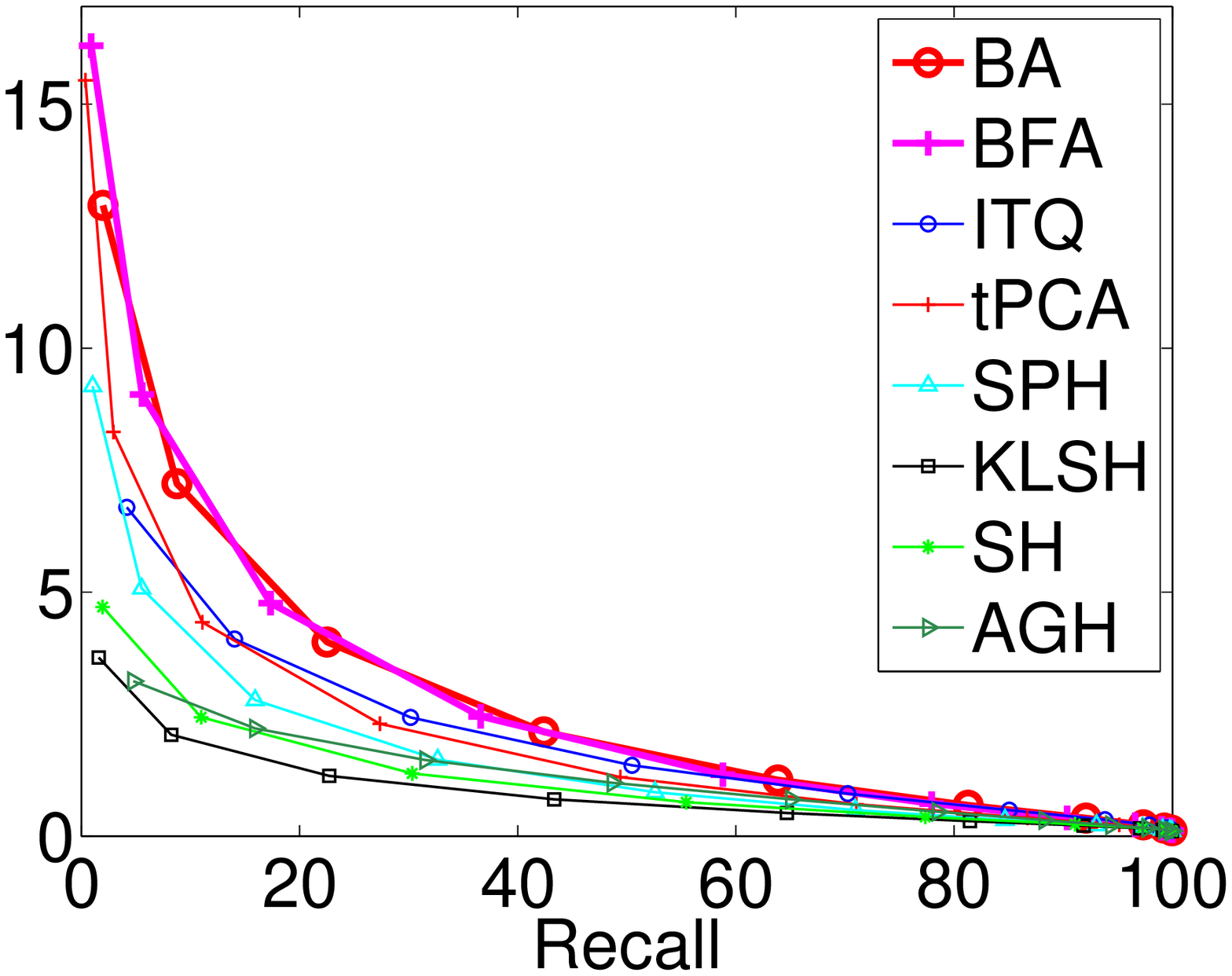} &
    \includegraphics[width=0.25\linewidth]{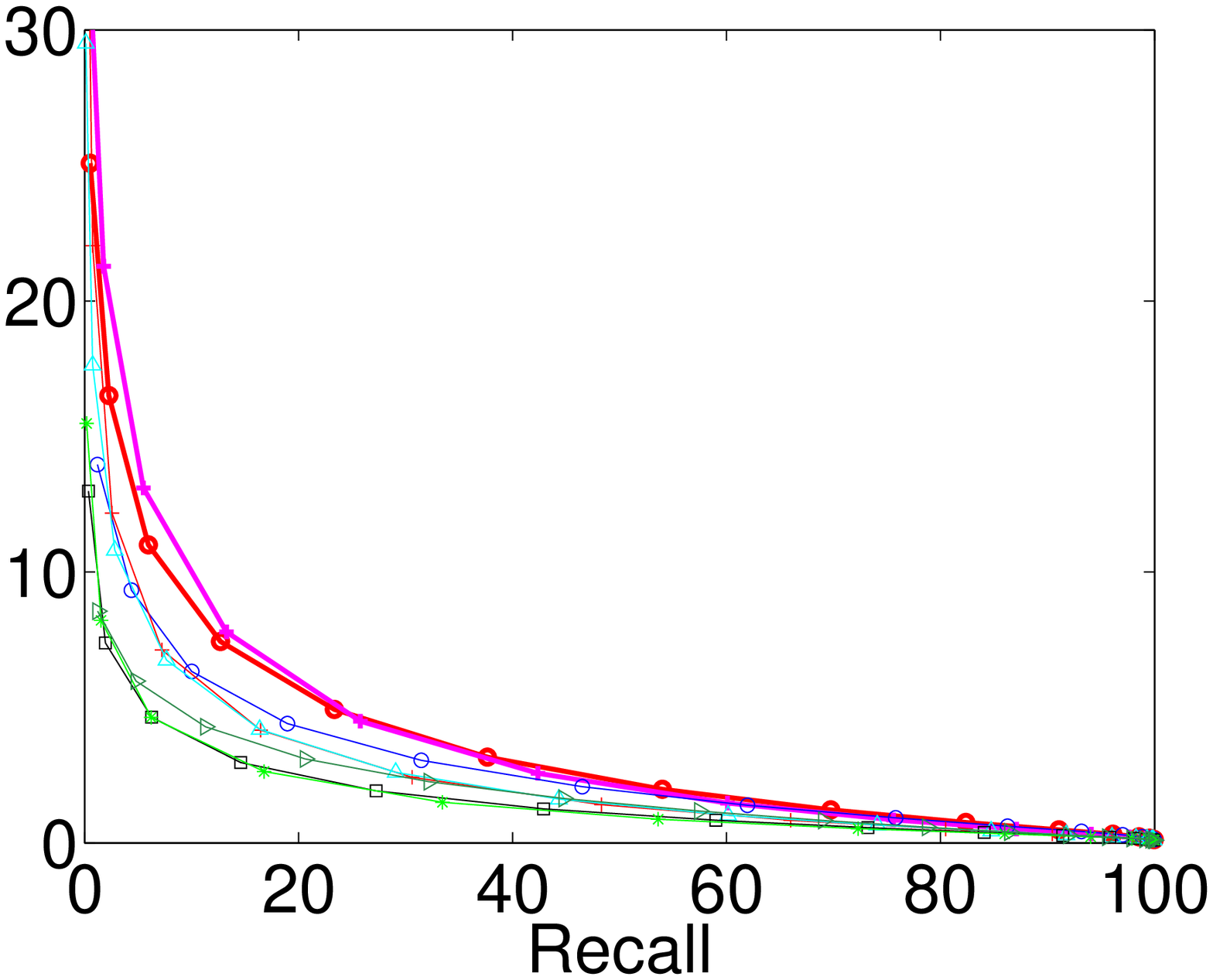} &
    \includegraphics[width=0.25\linewidth]{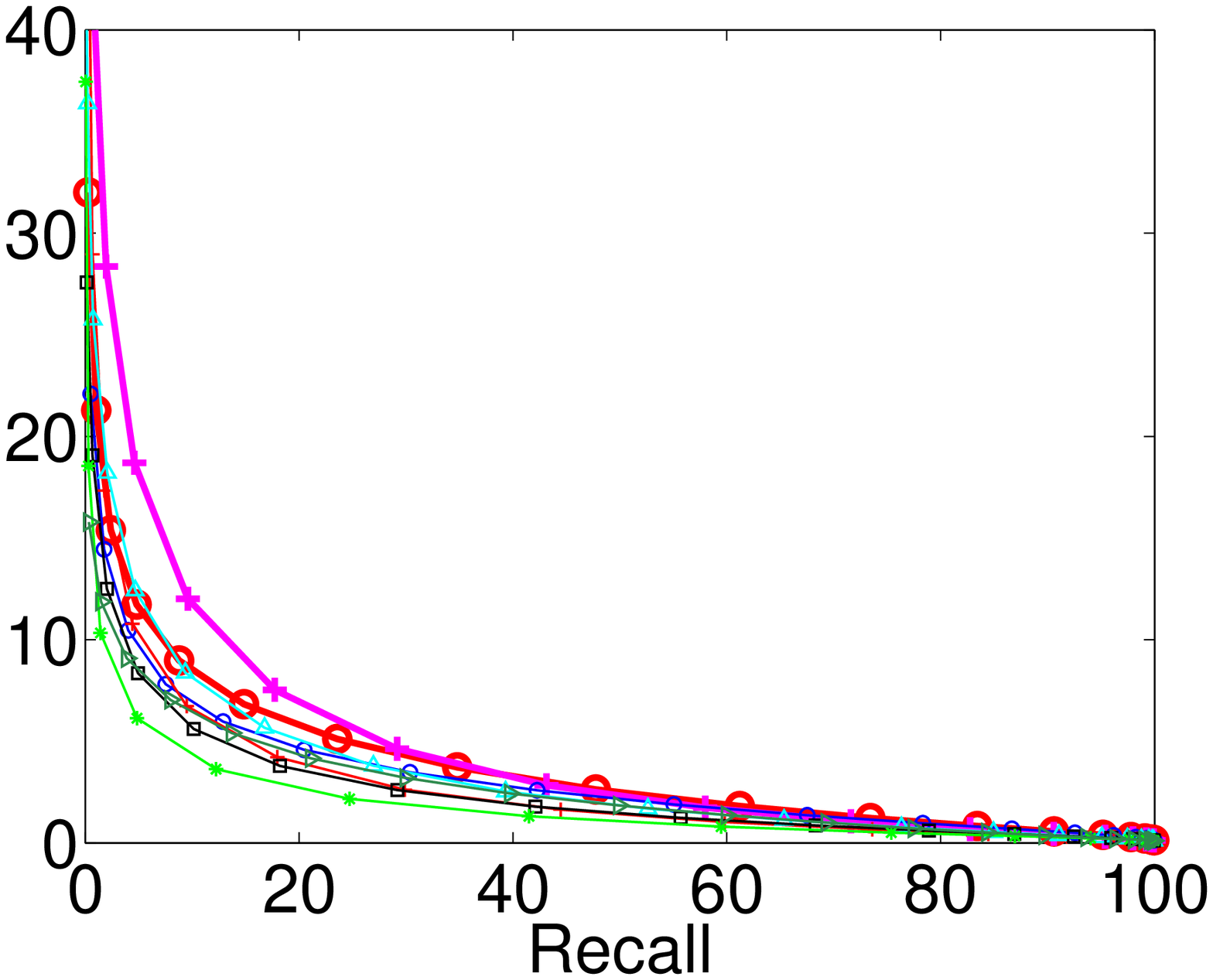}
  \end{tabular}
  \caption{Precision and precision/recall in CIFAR dataset (plotted as in fig.~\ref{f:NUS-WIDE}). Ground truth: $K=1\,000$ (\emph{top block}) and $K=50$ (\emph{bottom block}) nearest images to the query image in the training set. \emph{Top panels within each block}: precision depending on the retrieved set ($k$ nearest neighbors in Hamming distance, or images at Hamming distance $\le 3$ or $4$). \emph{Bottom panels within each block}: precision/recall curves for a retrieved set of images at Hamming distance $\le r$, using $L = 8$ to $32$ bits.}
  \label{f:CIFAR}
\end{figure}

\begin{figure}[t]
  \centering
  \begin{tabular}{@{}c@{\hspace{.01\linewidth}}c@{}c@{}c@{}c@{}c@{}c@{}c@{}c@{}c@{}c@{}}
    \dotfill Query\dotfill & \multicolumn{9}{c}{\dotfill Retrieved neighbors\dotfill}\\
    \includegraphics[width=.1\linewidth]{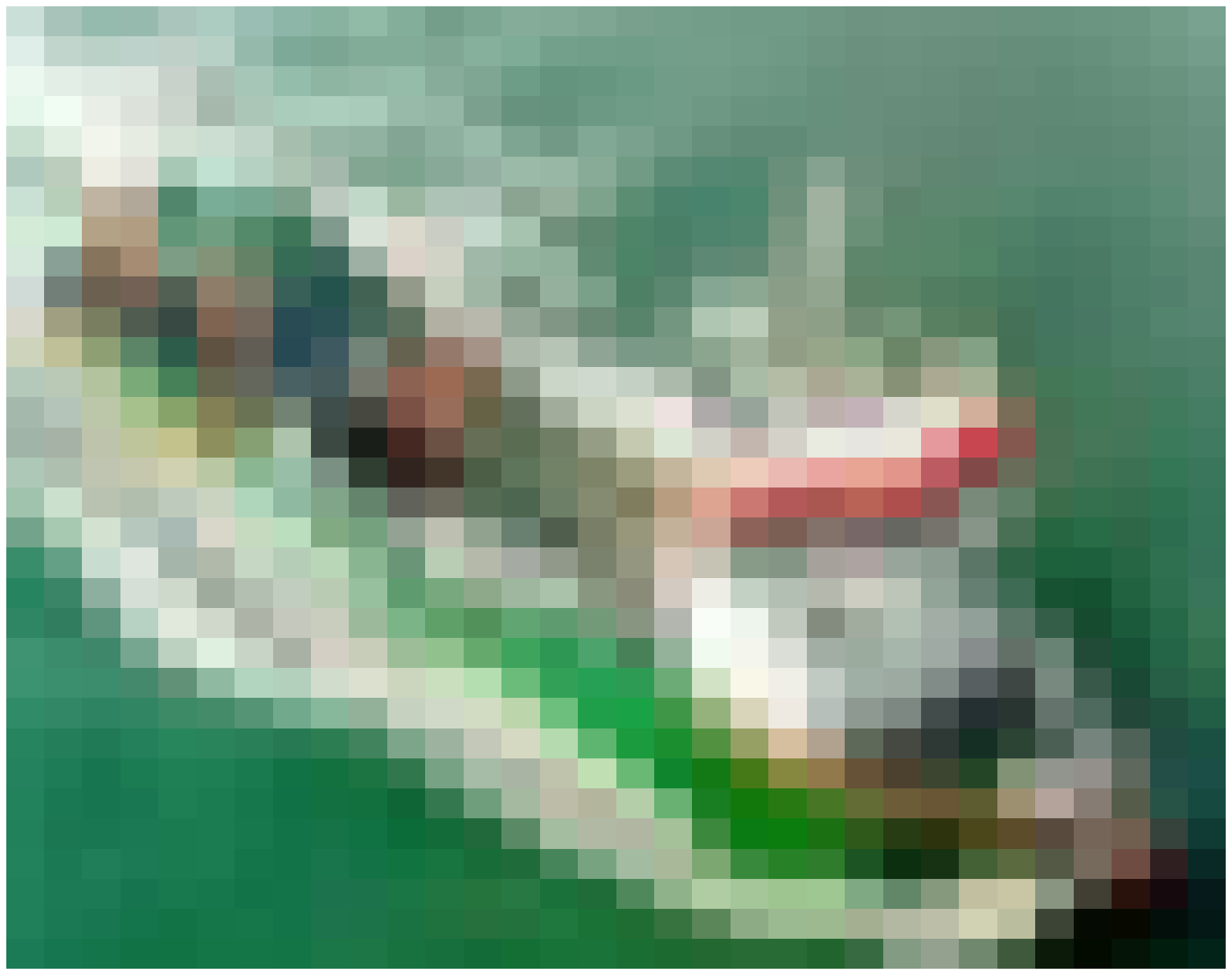}& \includegraphics[width=.1\linewidth]{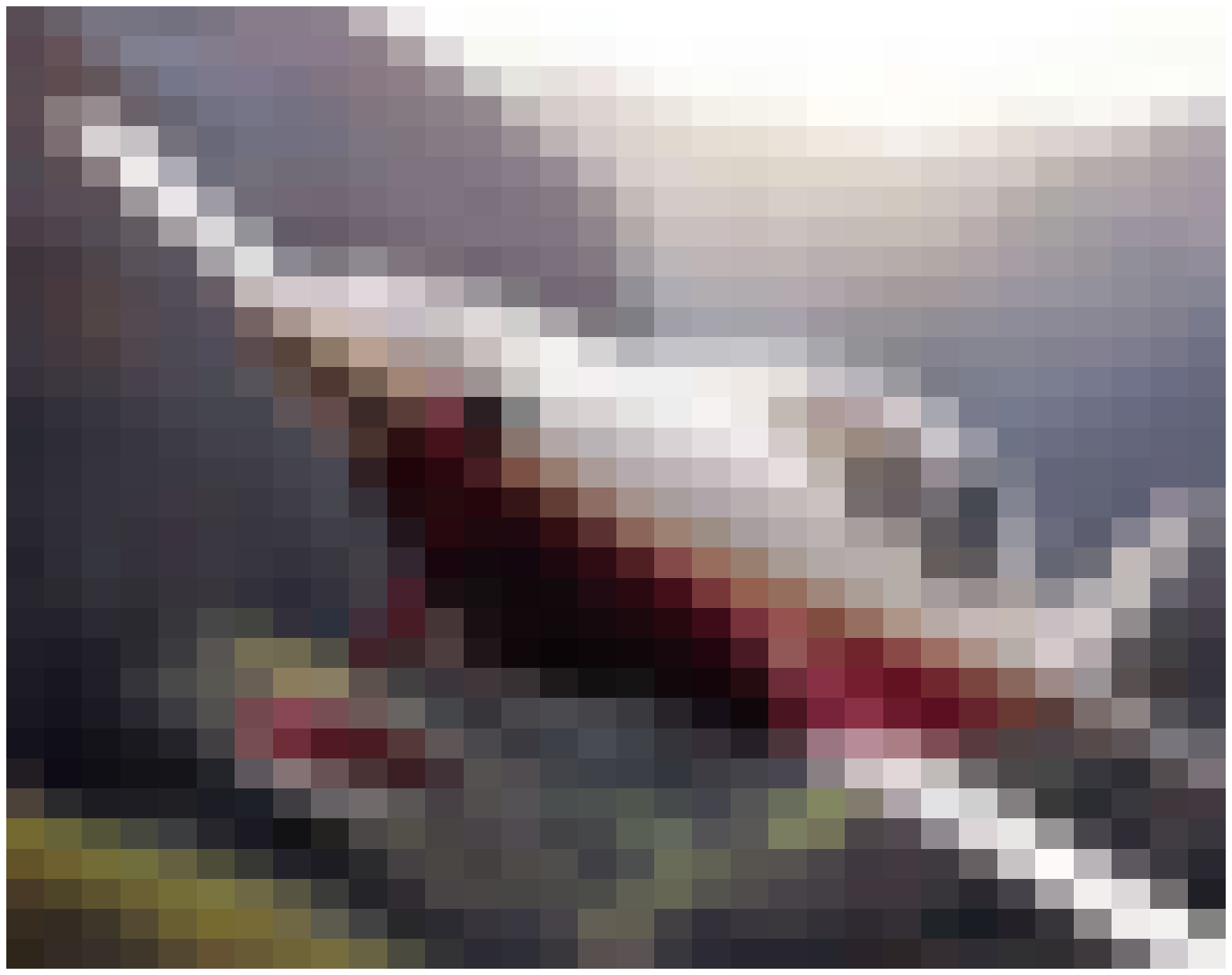}& \includegraphics[width=.1\linewidth]{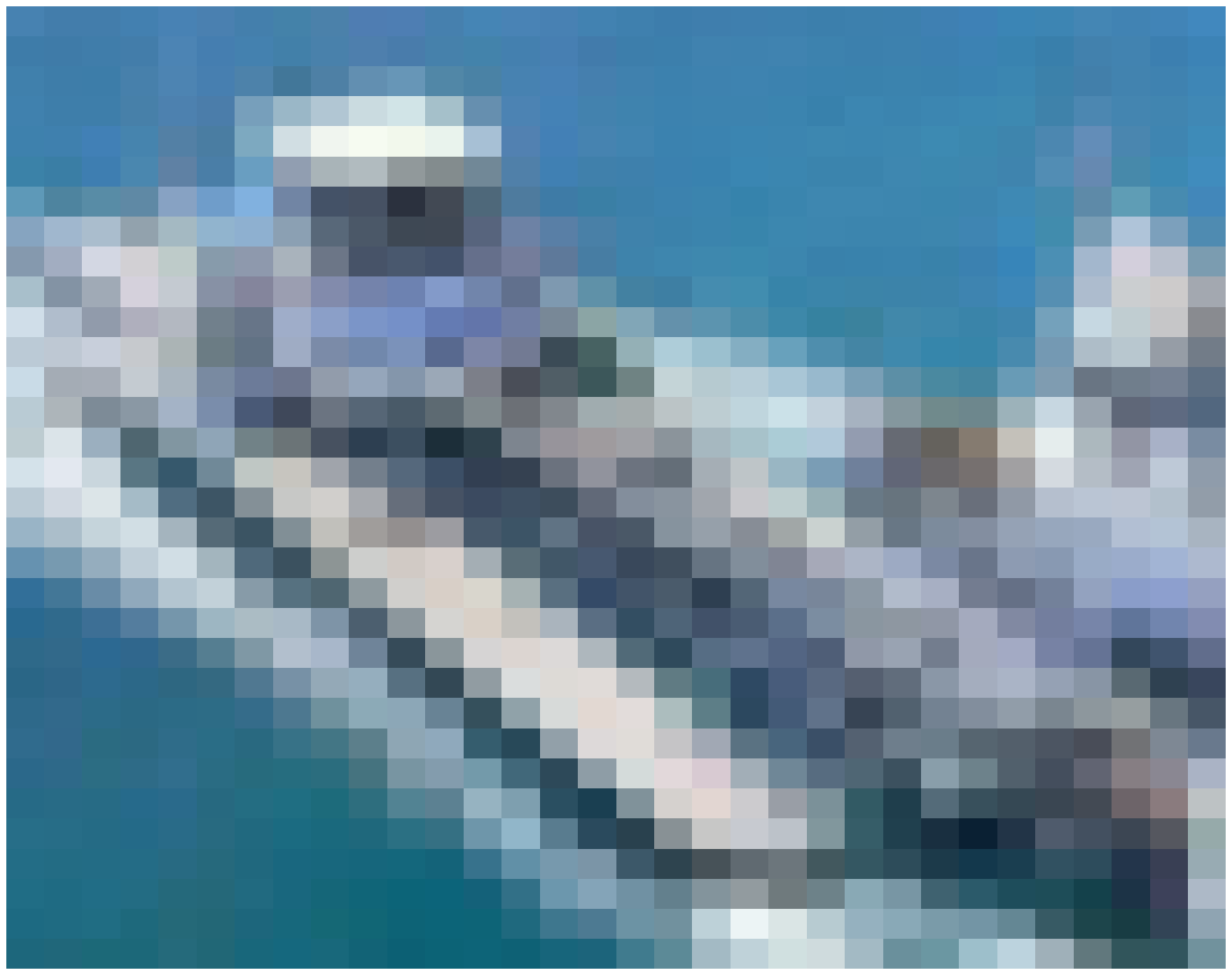}& \includegraphics[width=.1\linewidth]{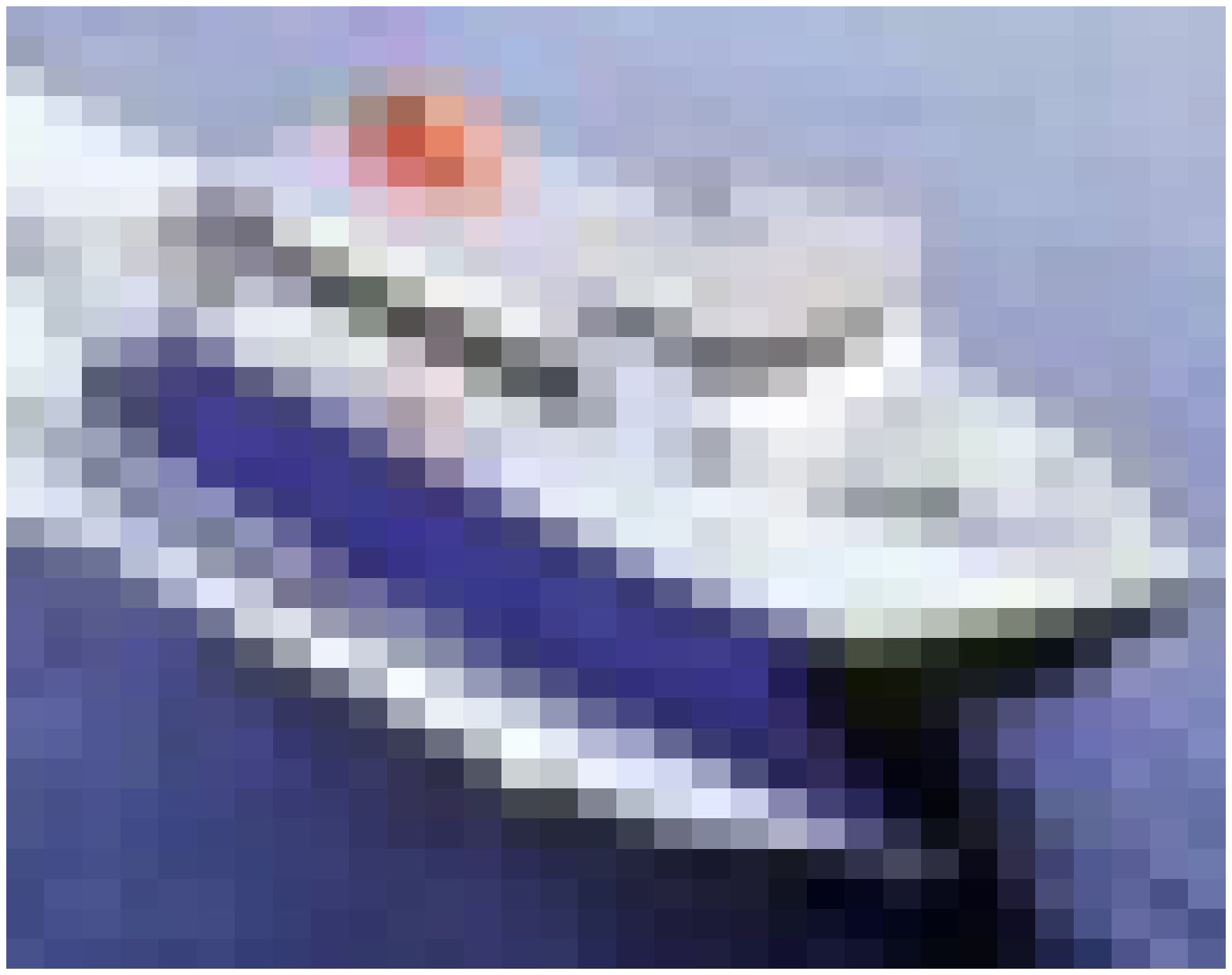}& \includegraphics[width=.1\linewidth]{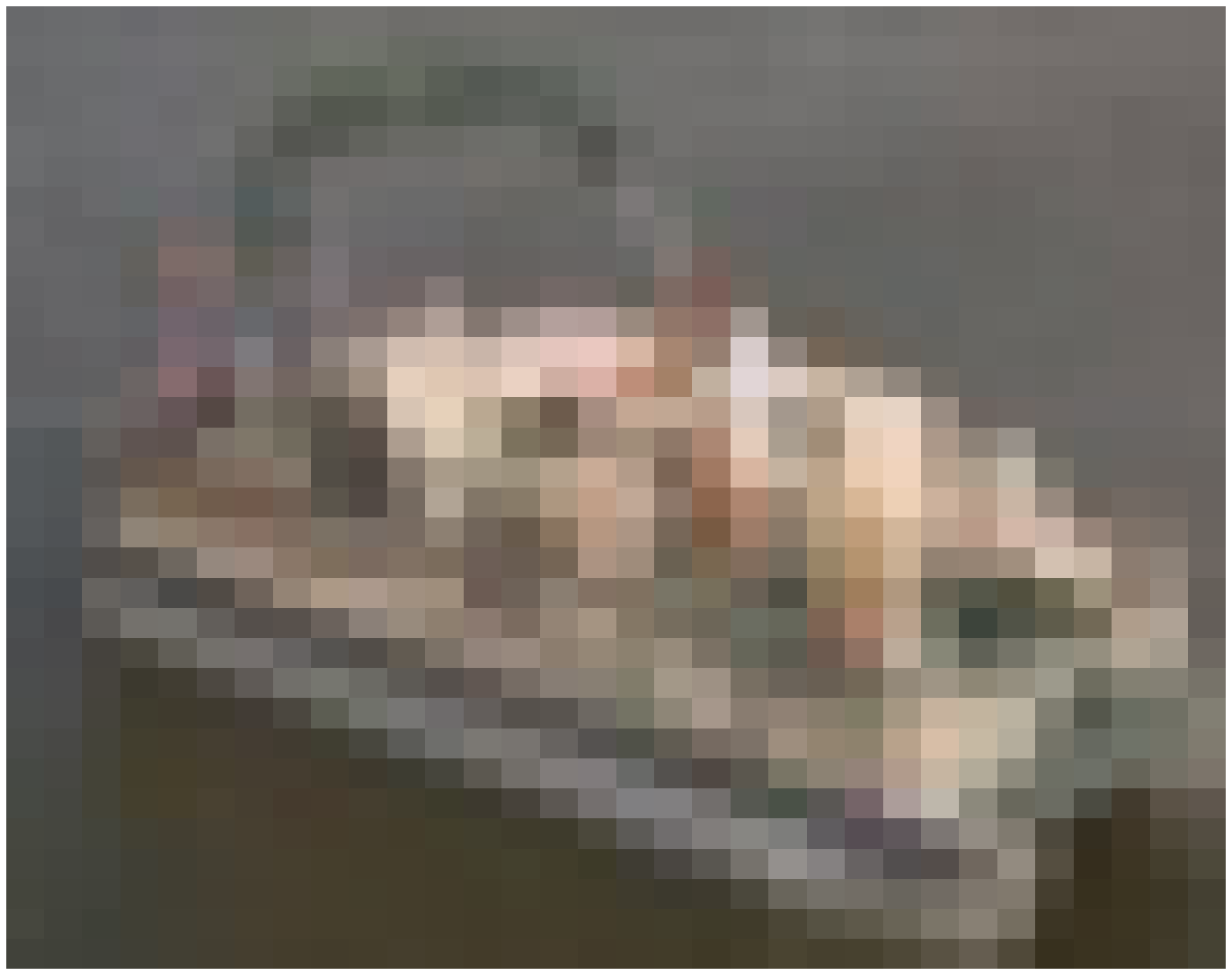}& \includegraphics[width=.1\linewidth]{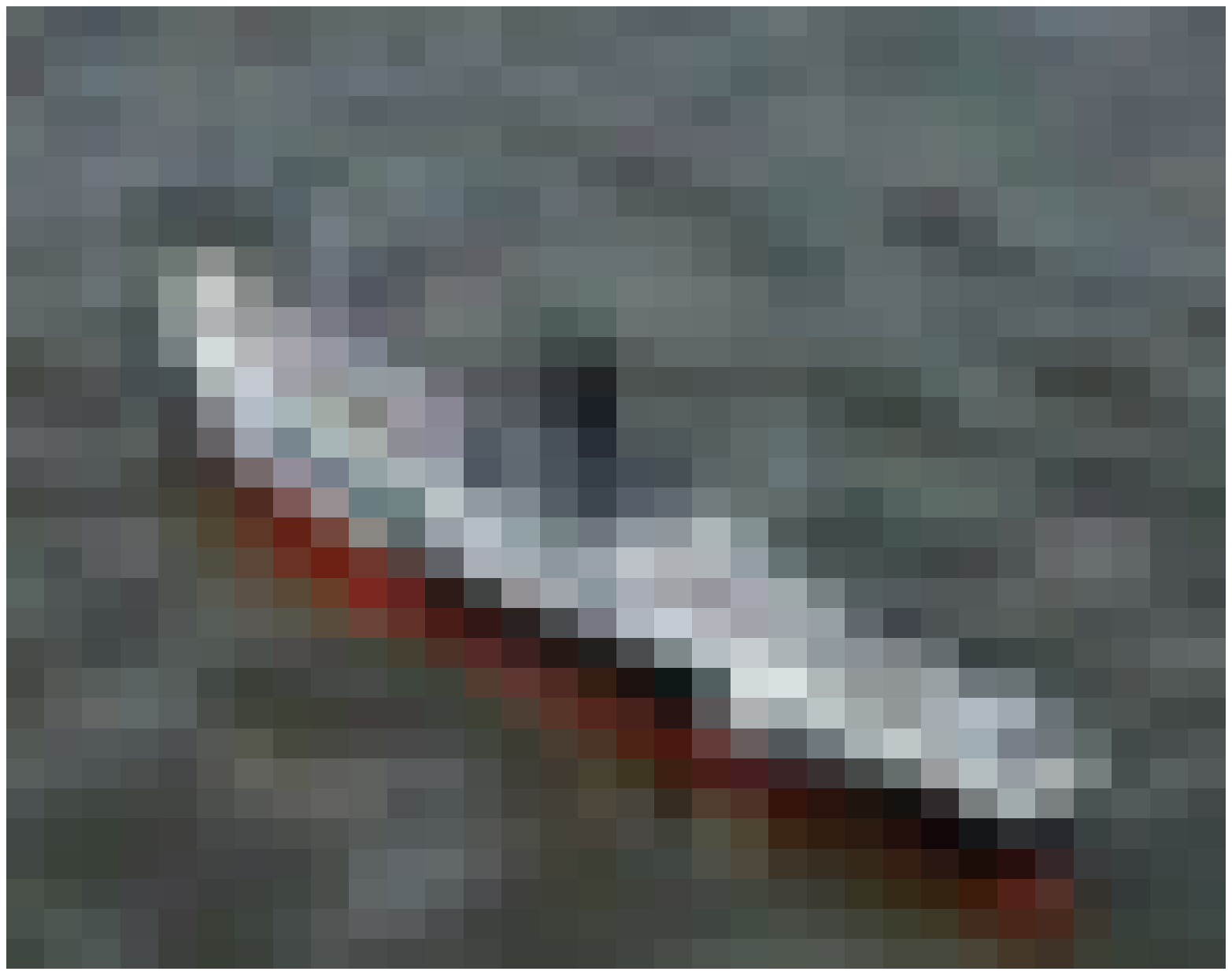}& \includegraphics[width=.1\linewidth]{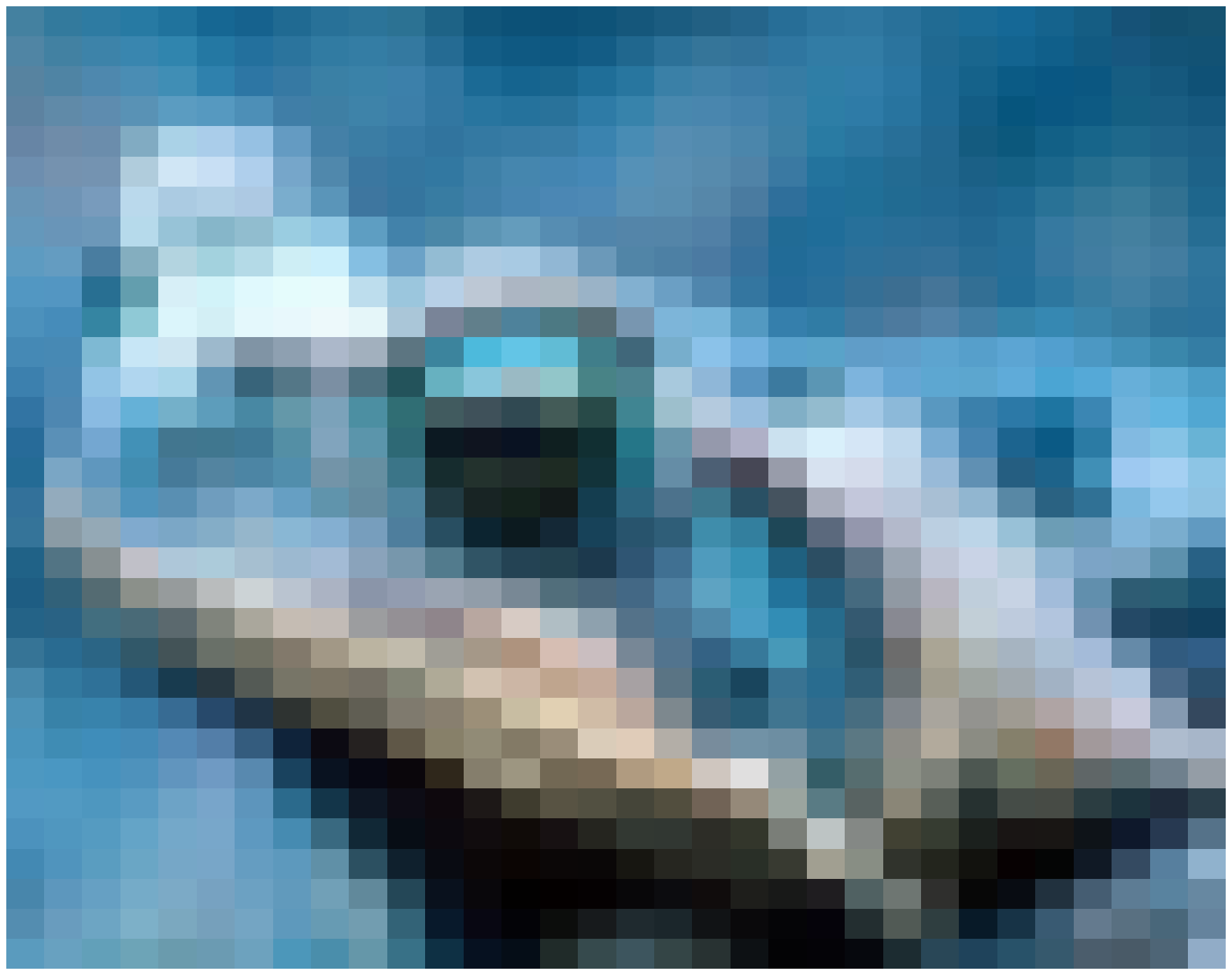}& \includegraphics[width=.1\linewidth]{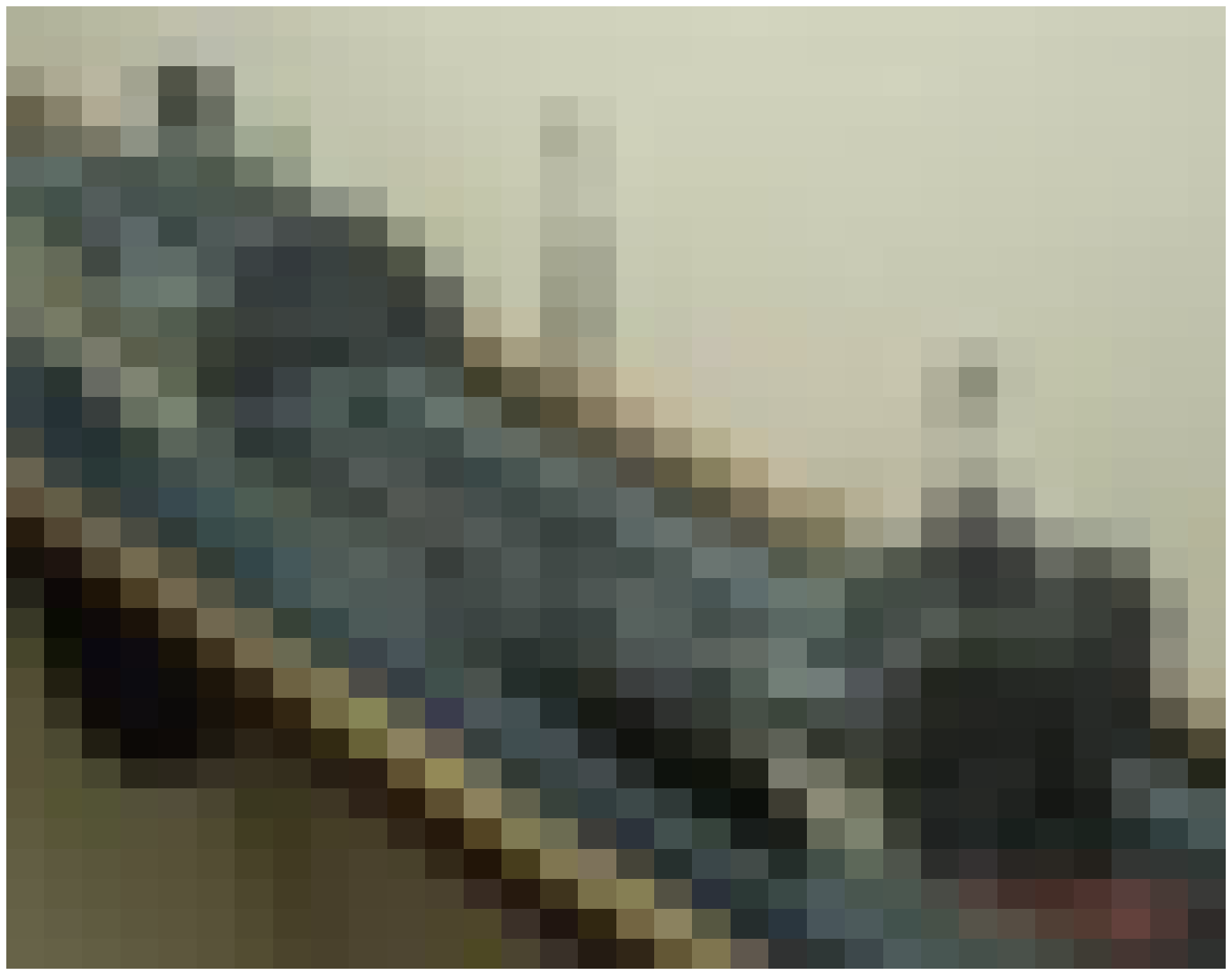} & \includegraphics[width=.1\linewidth]{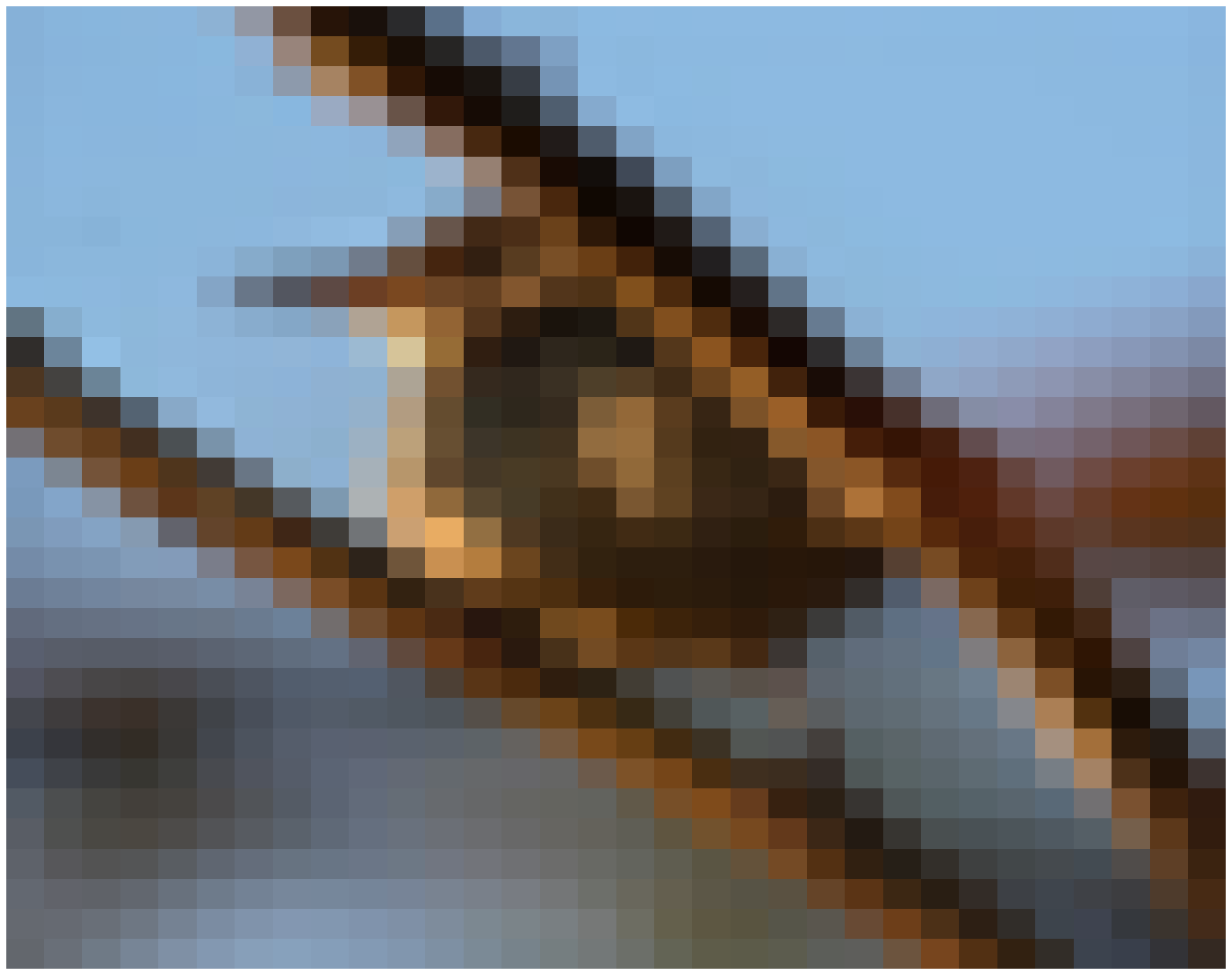}  & \includegraphics[width=.1\linewidth]{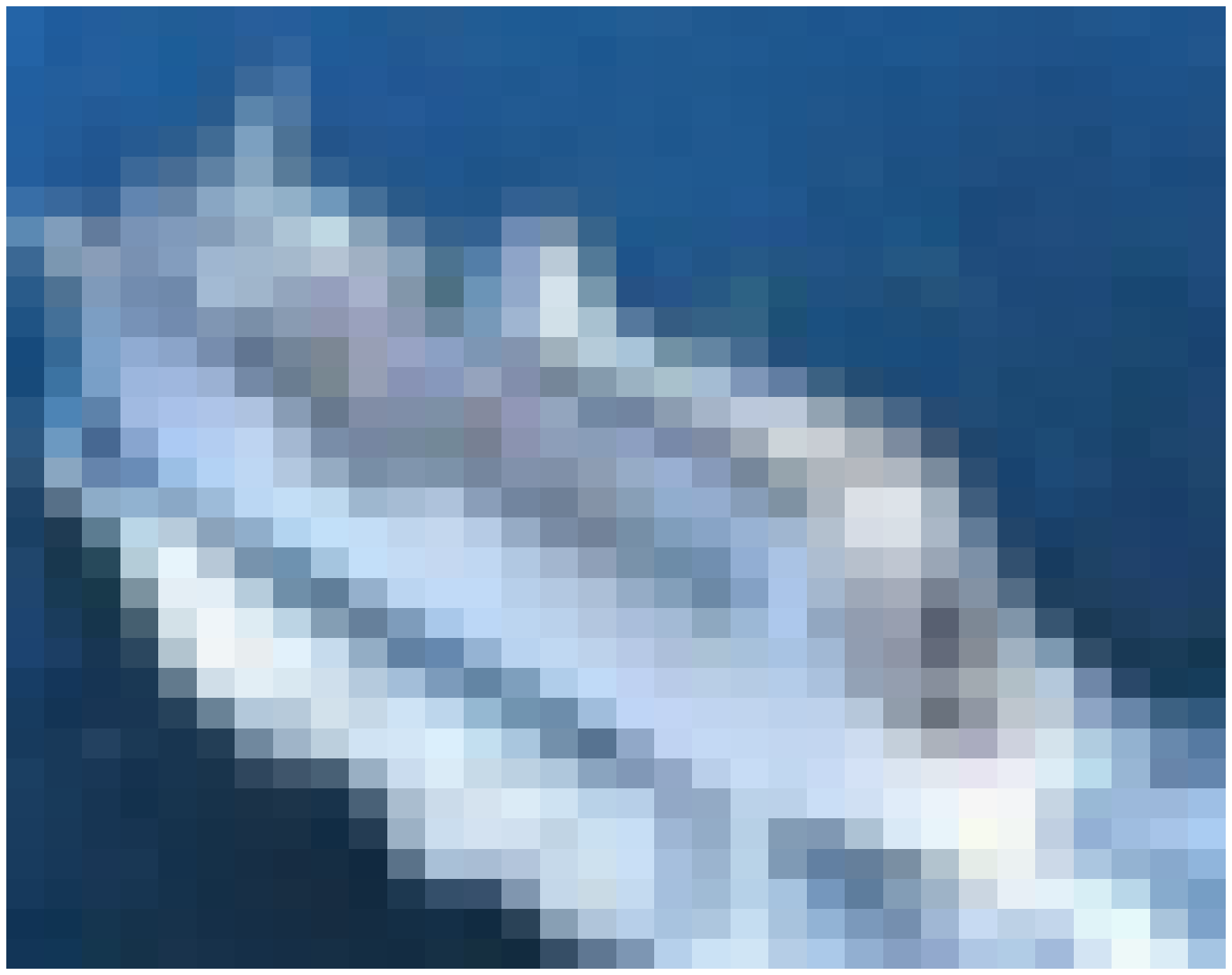}\\
    \includegraphics[width=.1\linewidth]{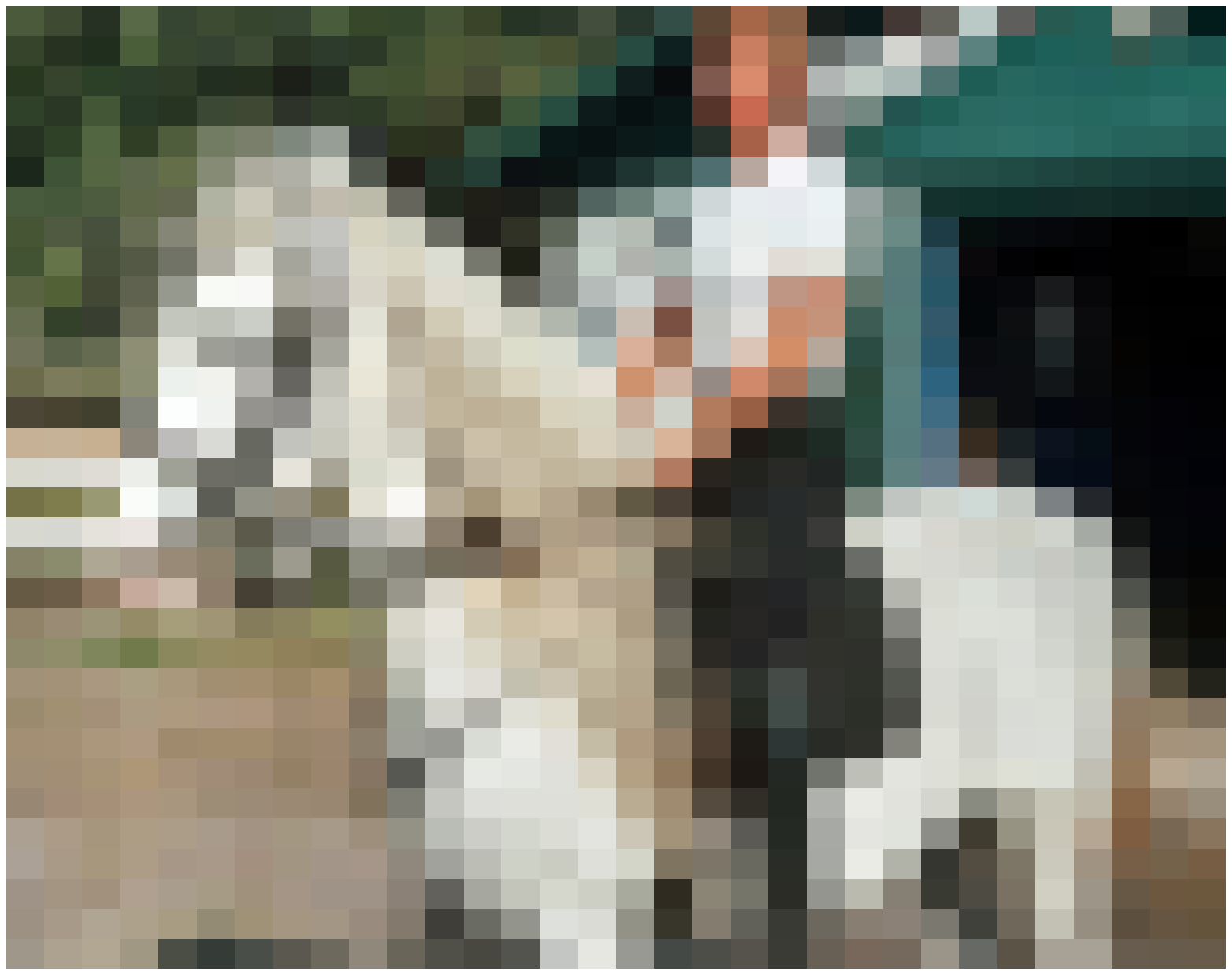}& \includegraphics[width=.1\linewidth]{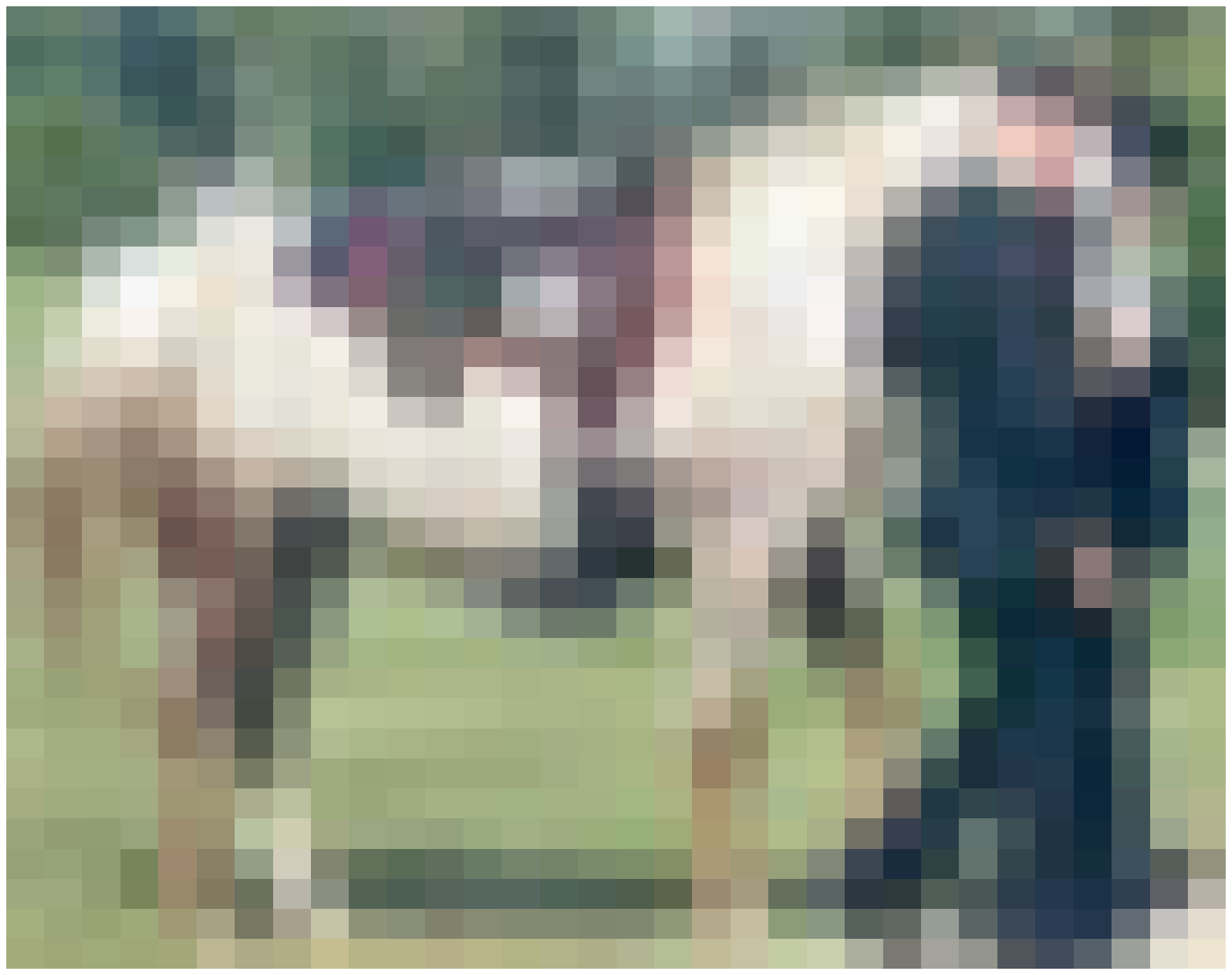}& \includegraphics[width=.1\linewidth]{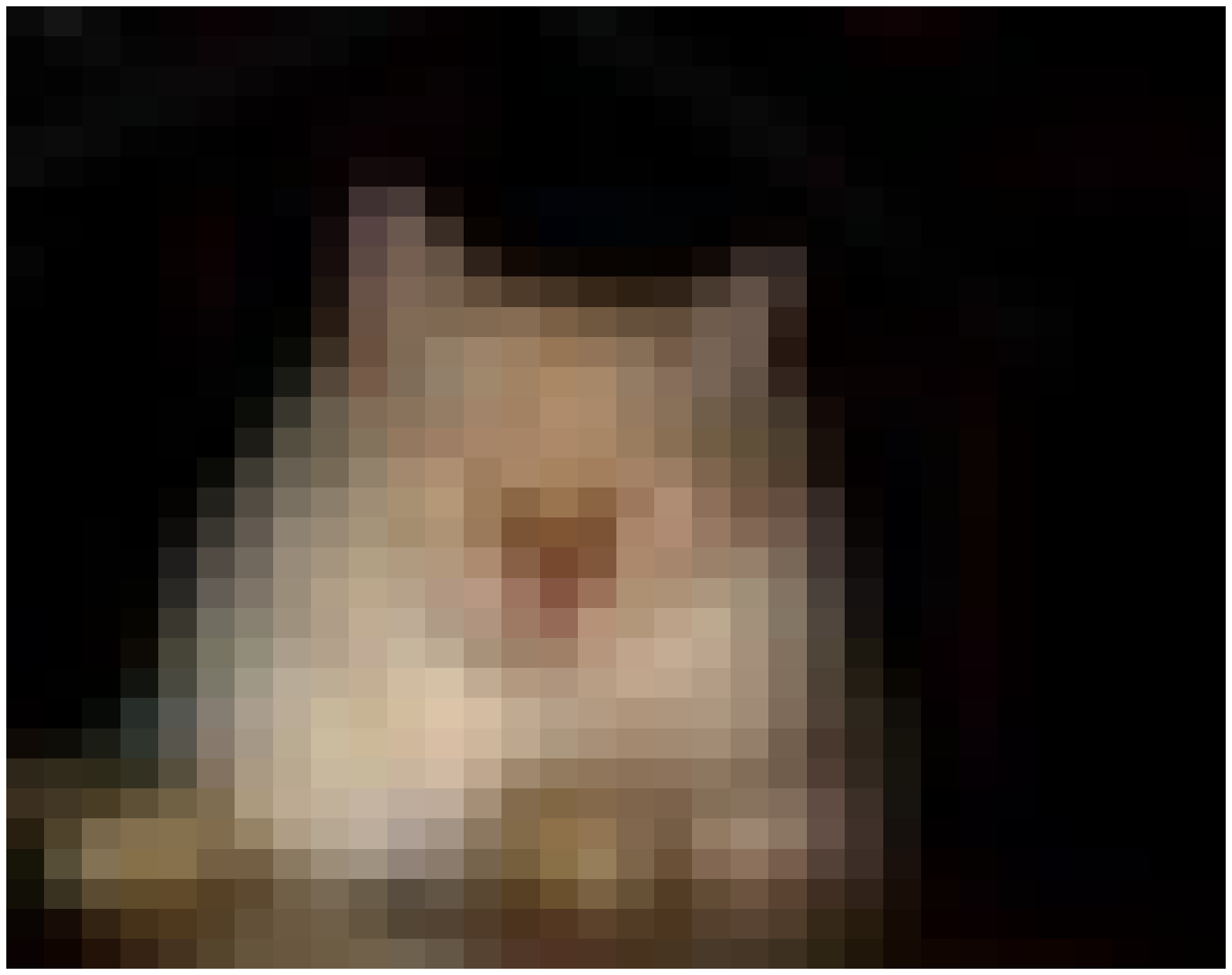}& \includegraphics[width=.1\linewidth]{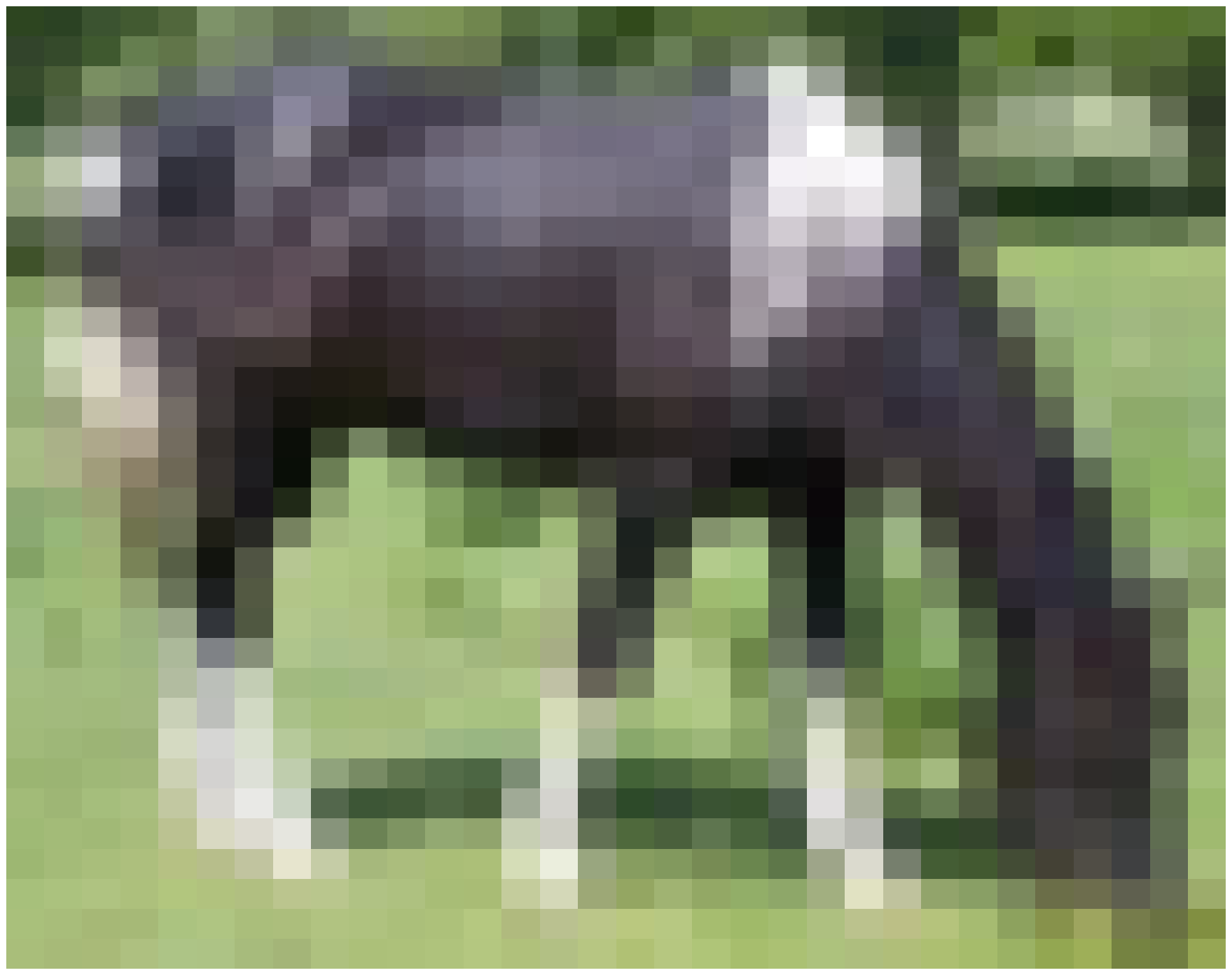}& \includegraphics[width=.1\linewidth]{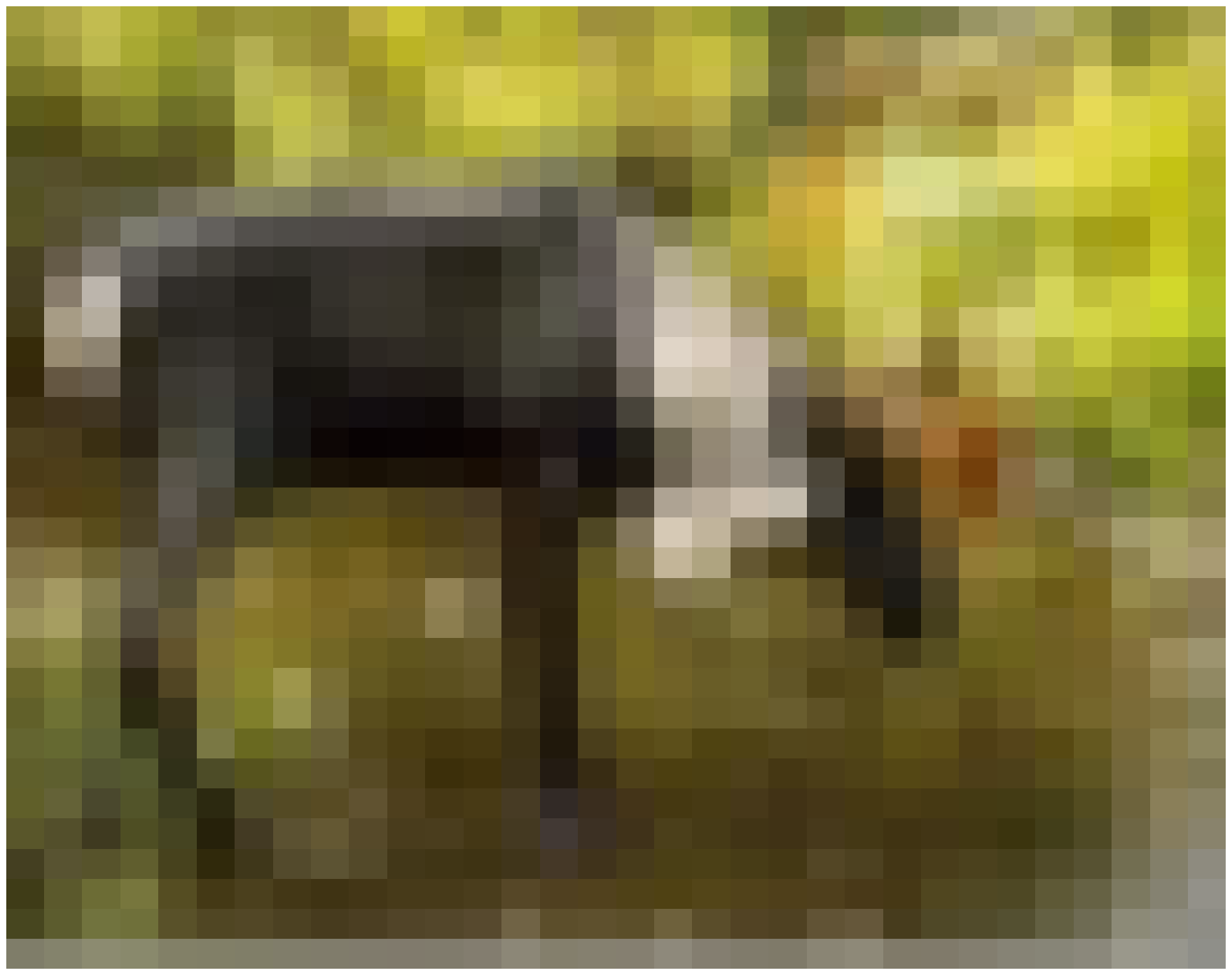}& \includegraphics[width=.1\linewidth]{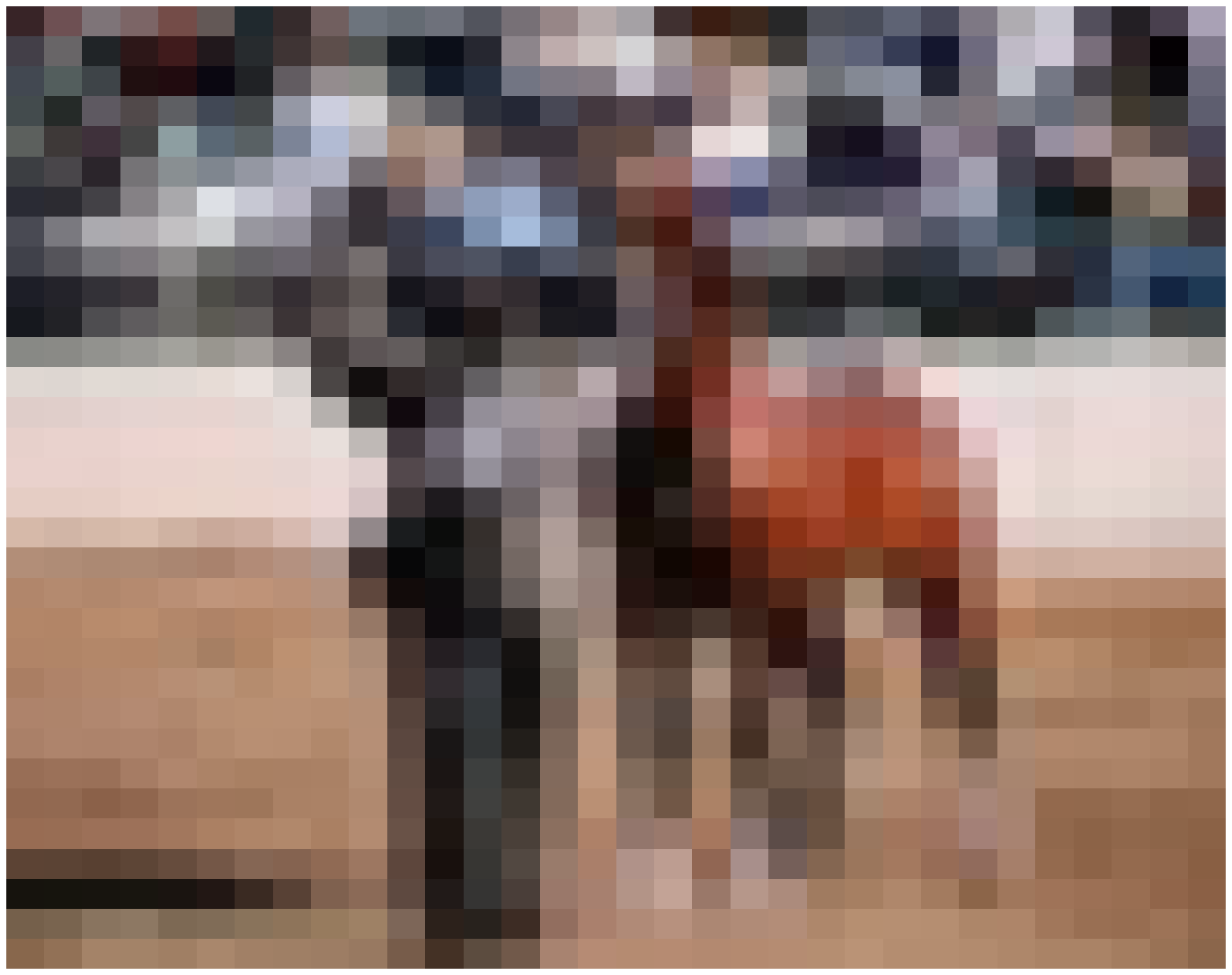}& \includegraphics[width=.1\linewidth]{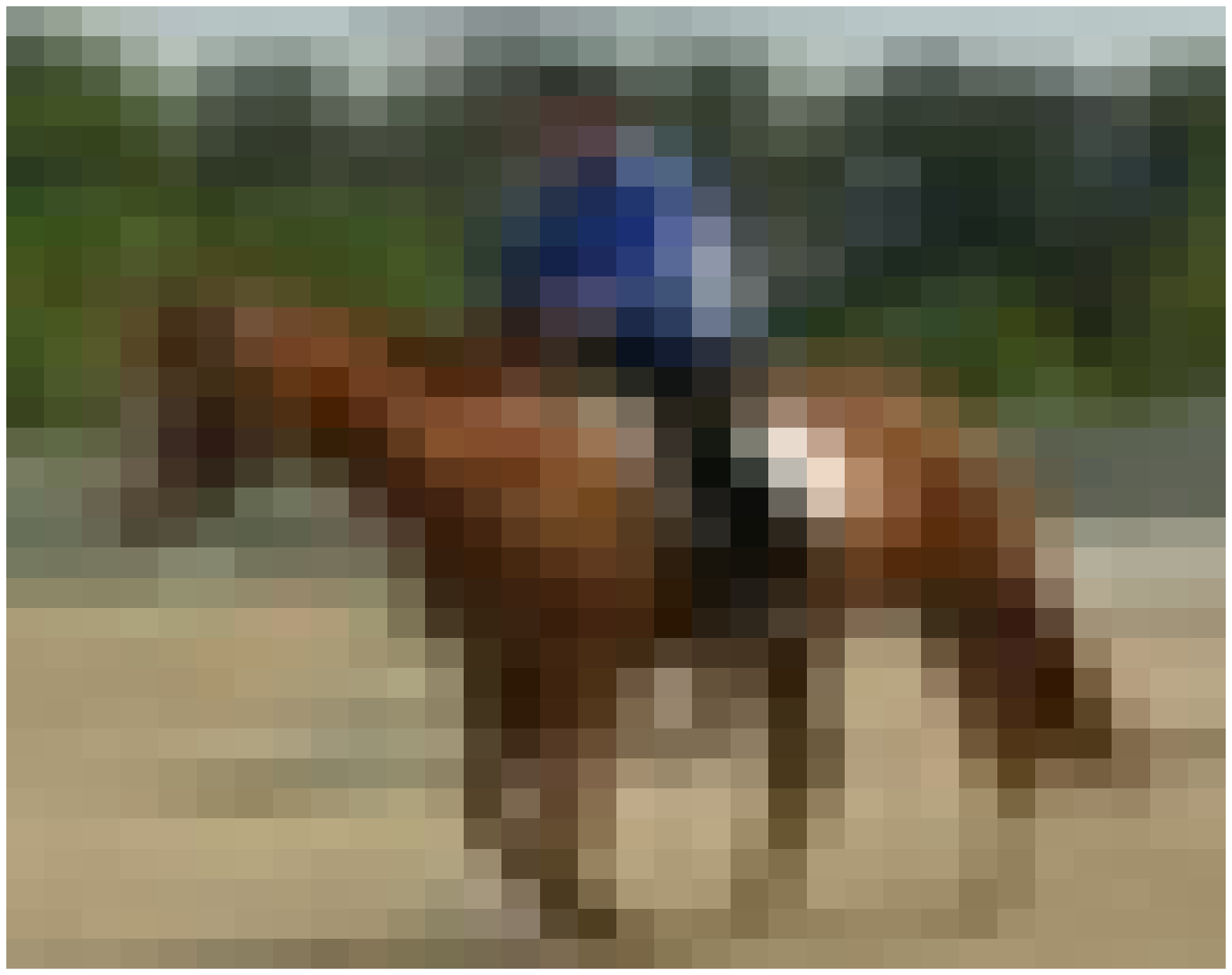}& \includegraphics[width=.1\linewidth]{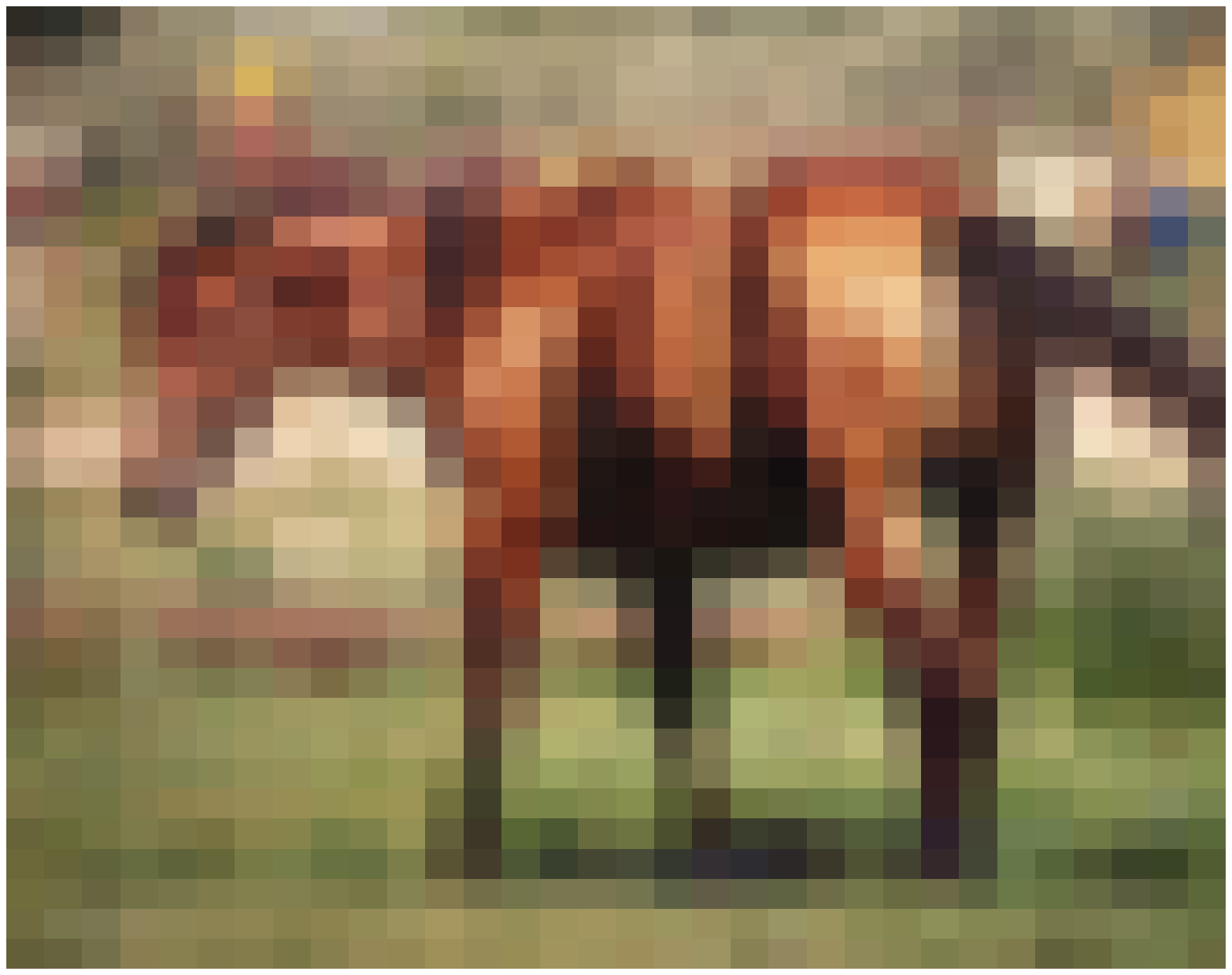}& \includegraphics[width=.1\linewidth]{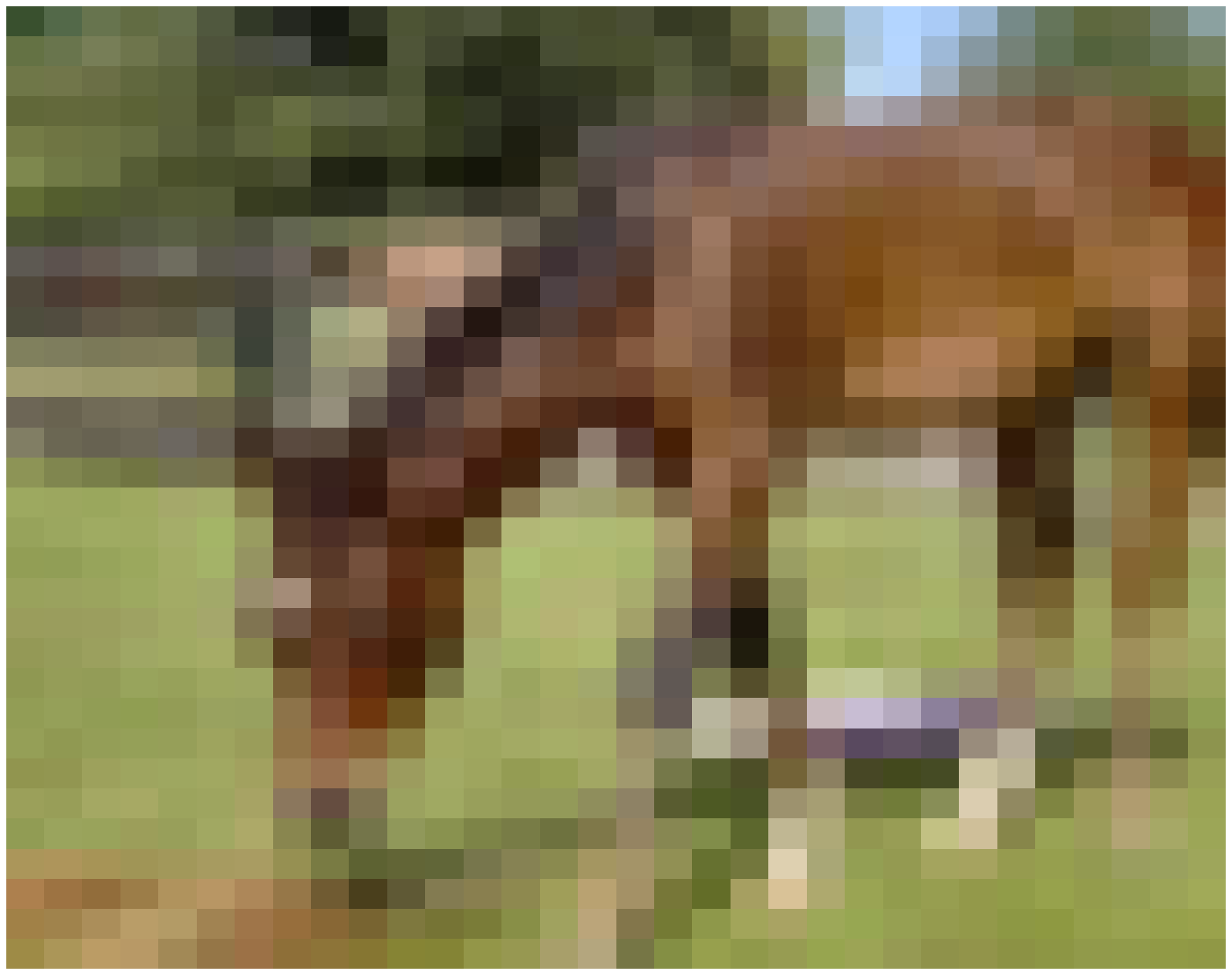} & \includegraphics[width=.1\linewidth]{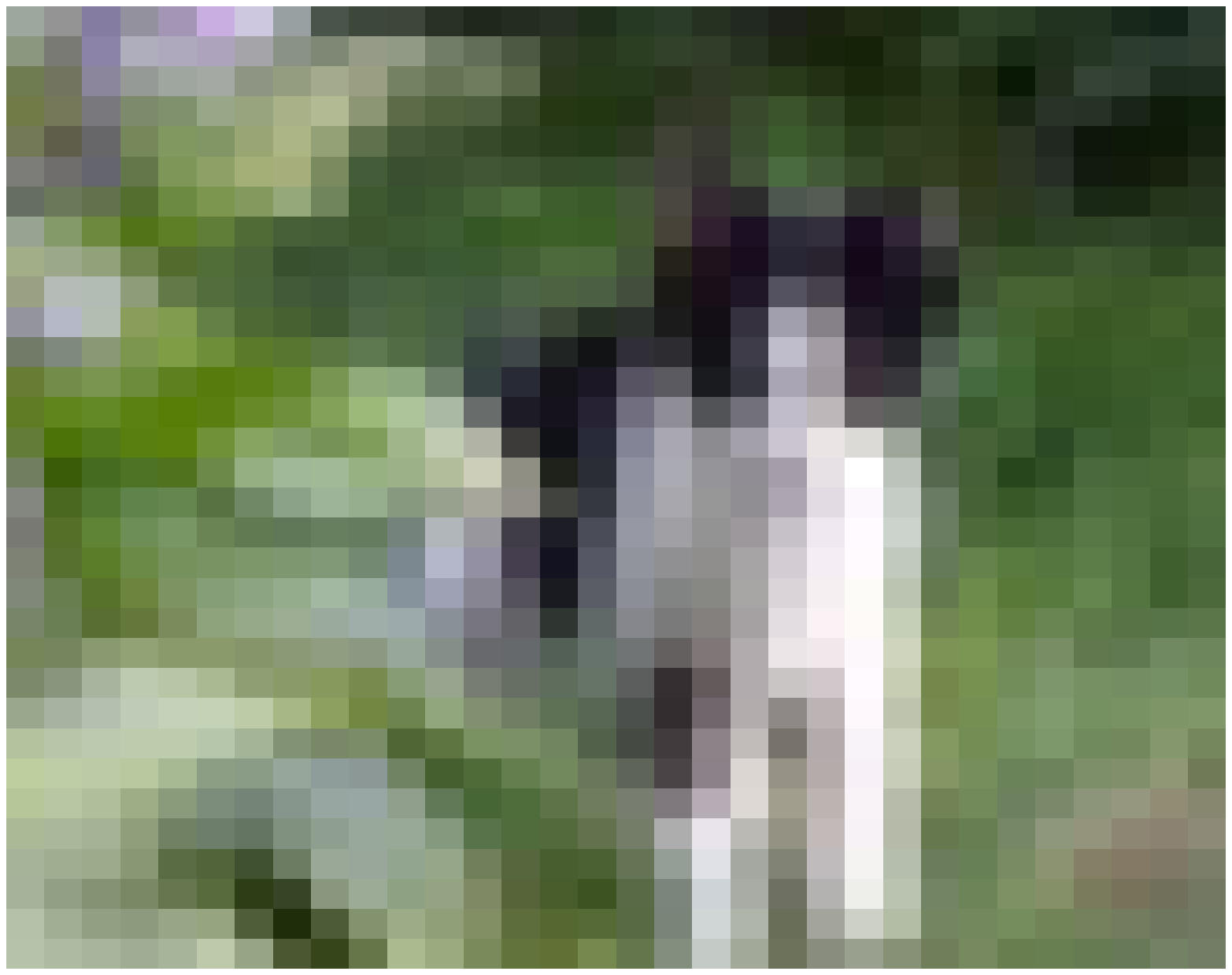}\\
    \includegraphics[width=.1\linewidth]{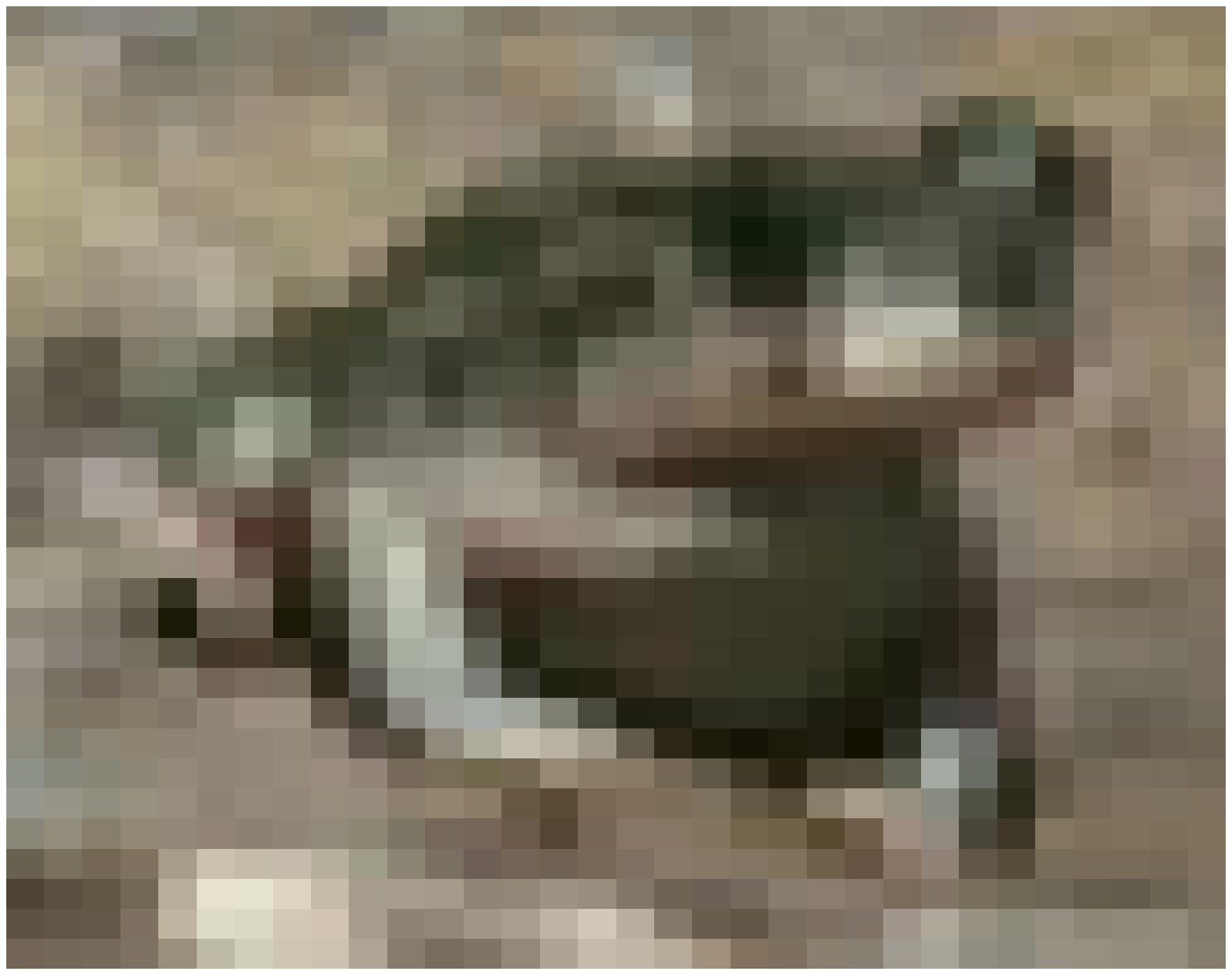}& \includegraphics[width=.1\linewidth]{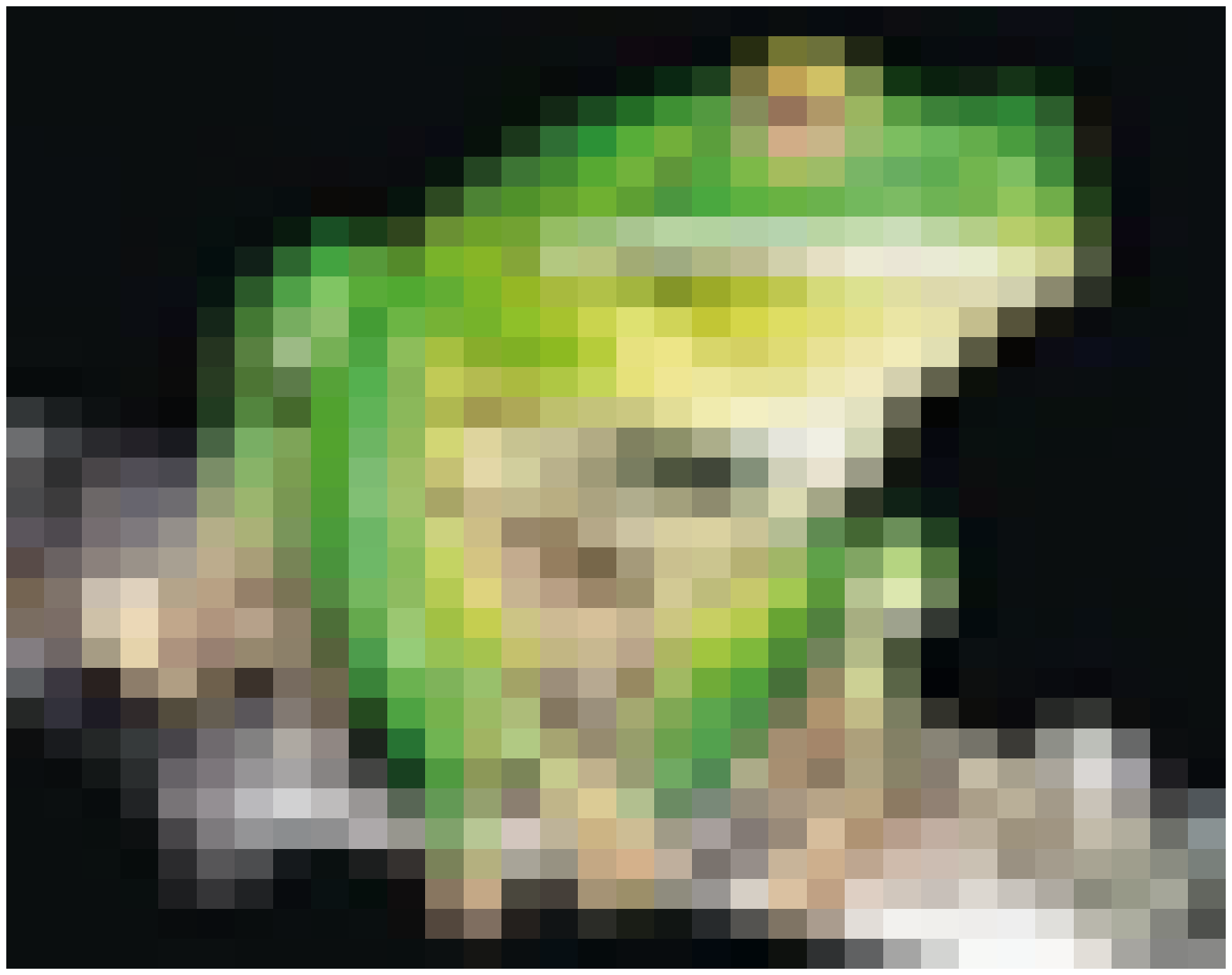}&\includegraphics[width=.1\linewidth]{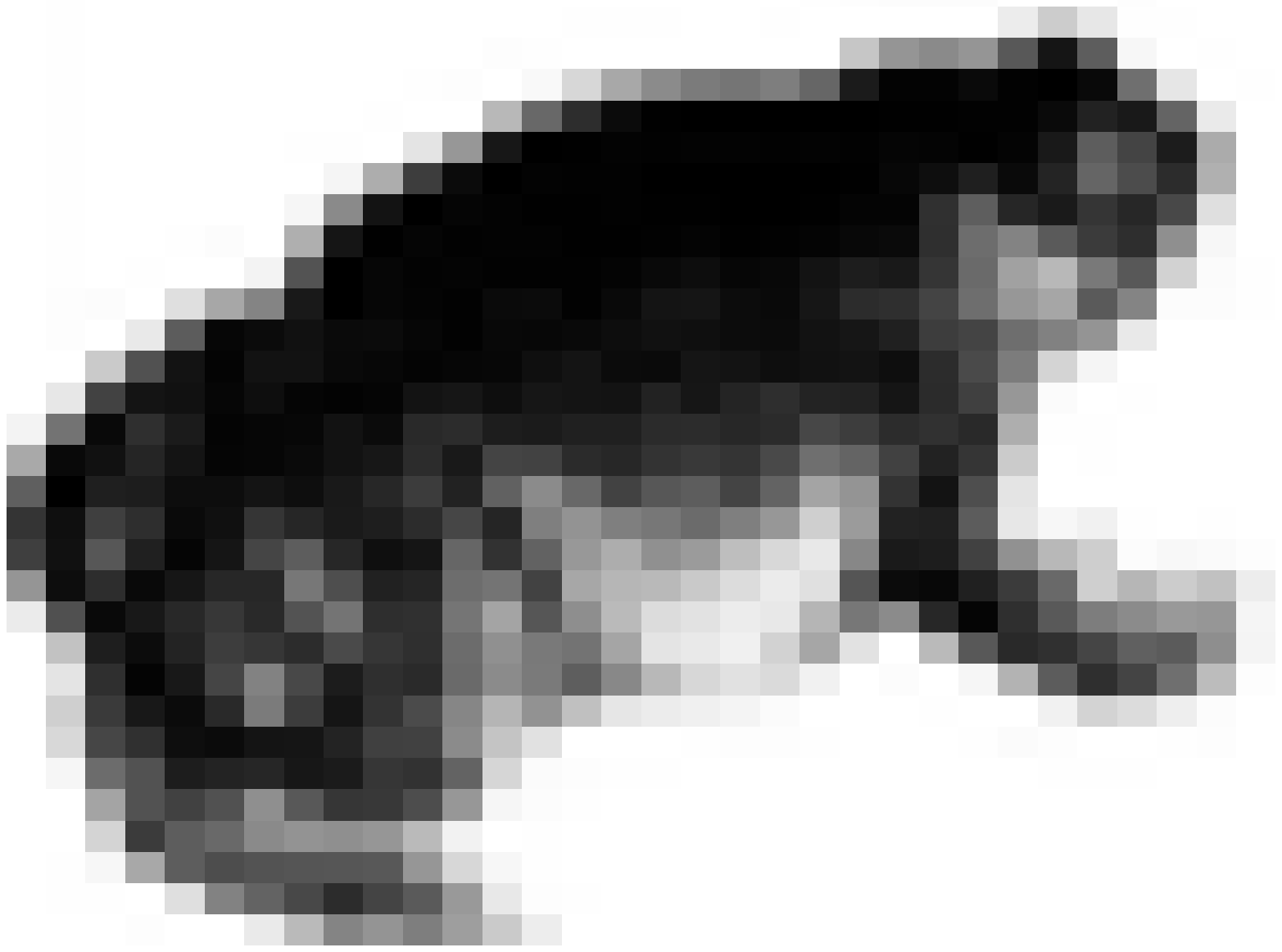}&\includegraphics[width=.1\linewidth]{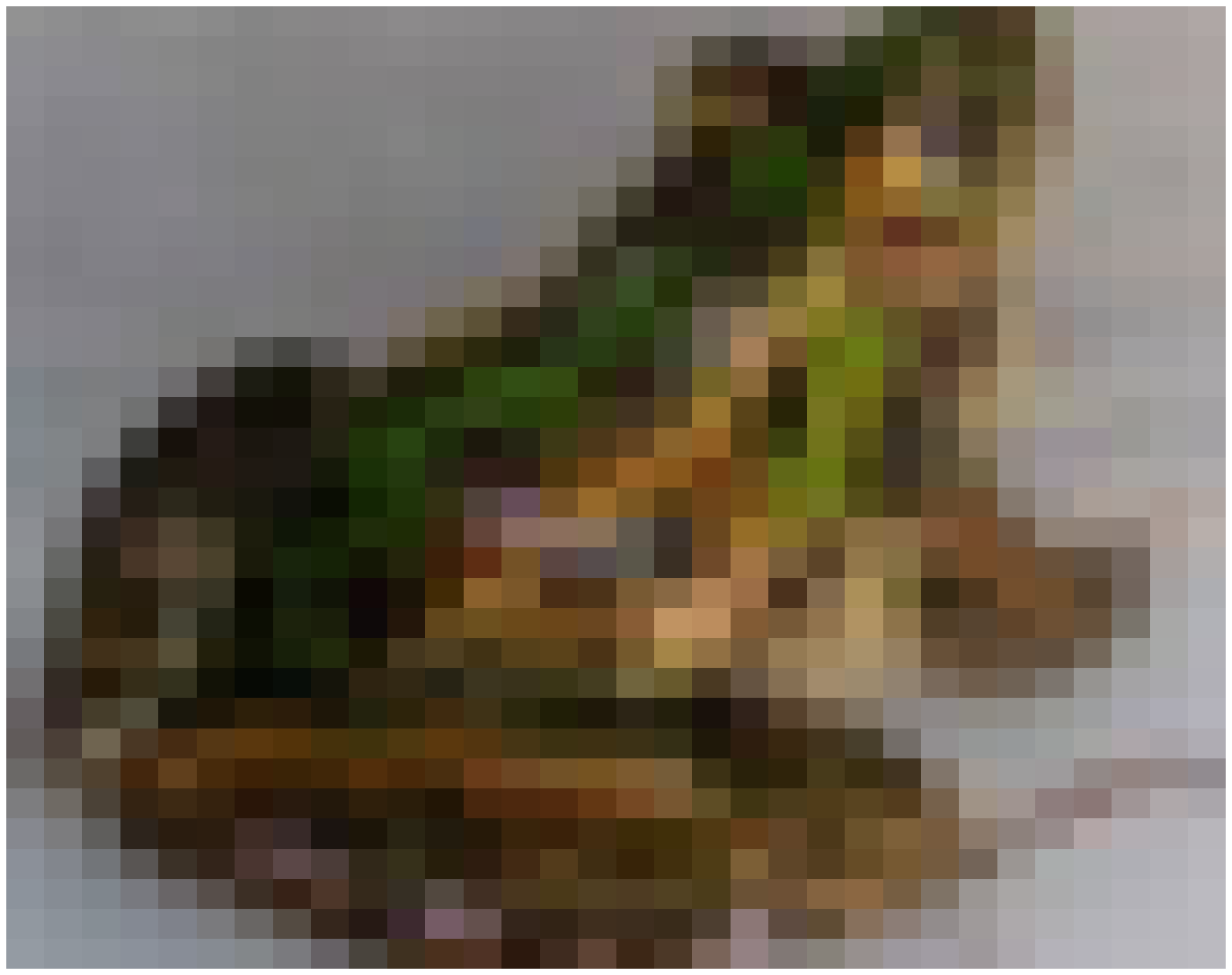}&\includegraphics[width=.1\linewidth]{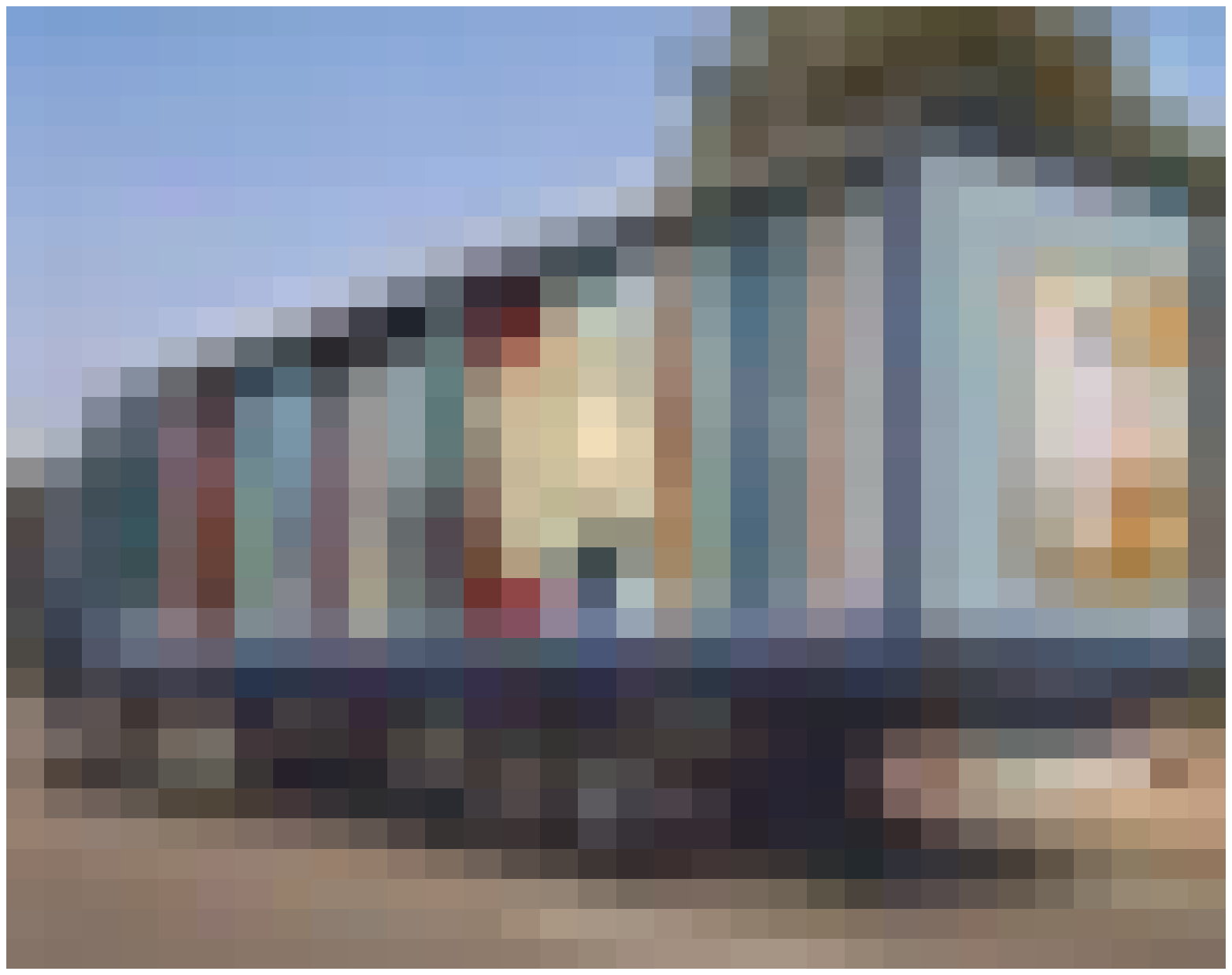}&\includegraphics[width=.1\linewidth]{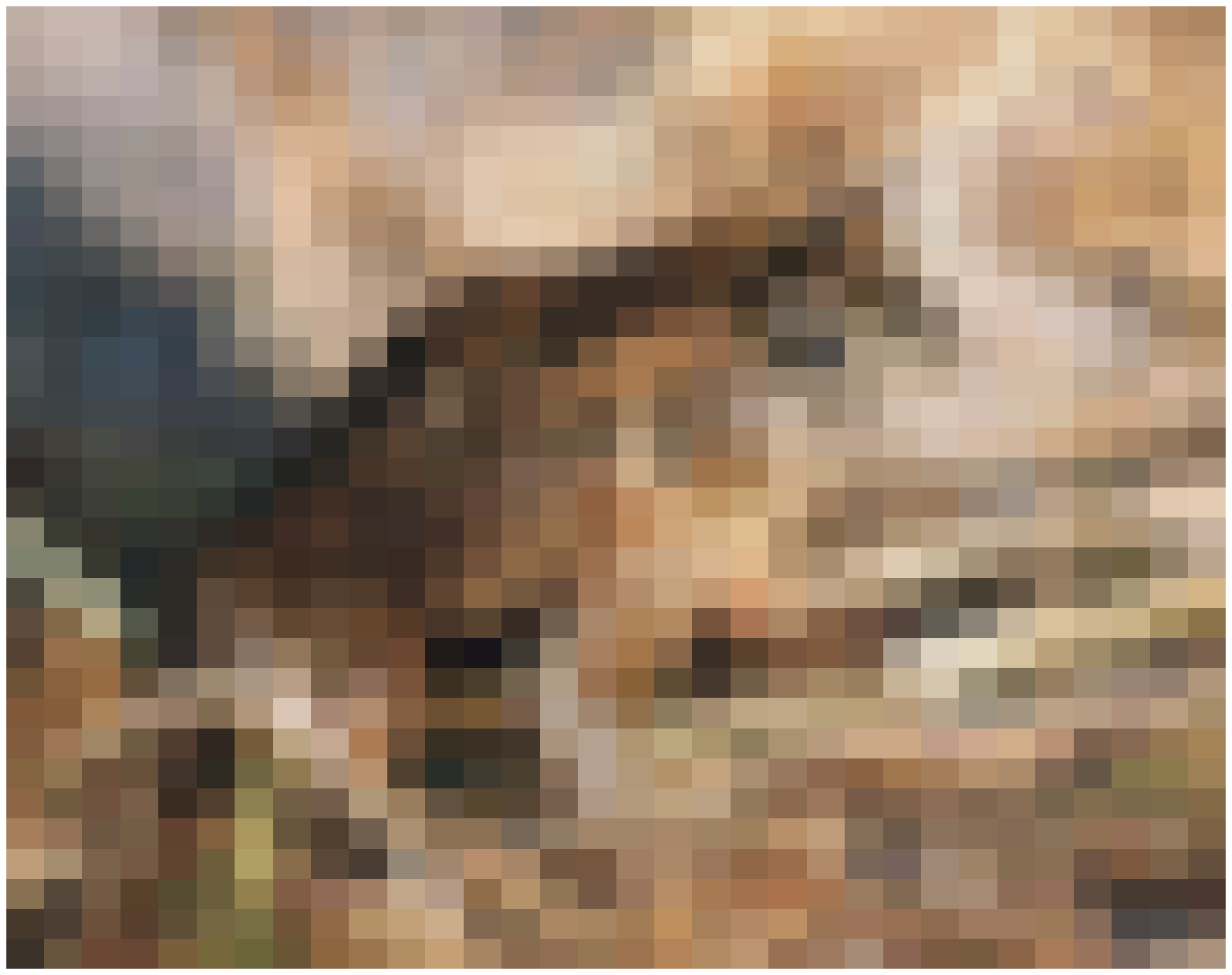}&\includegraphics[width=.1\linewidth]{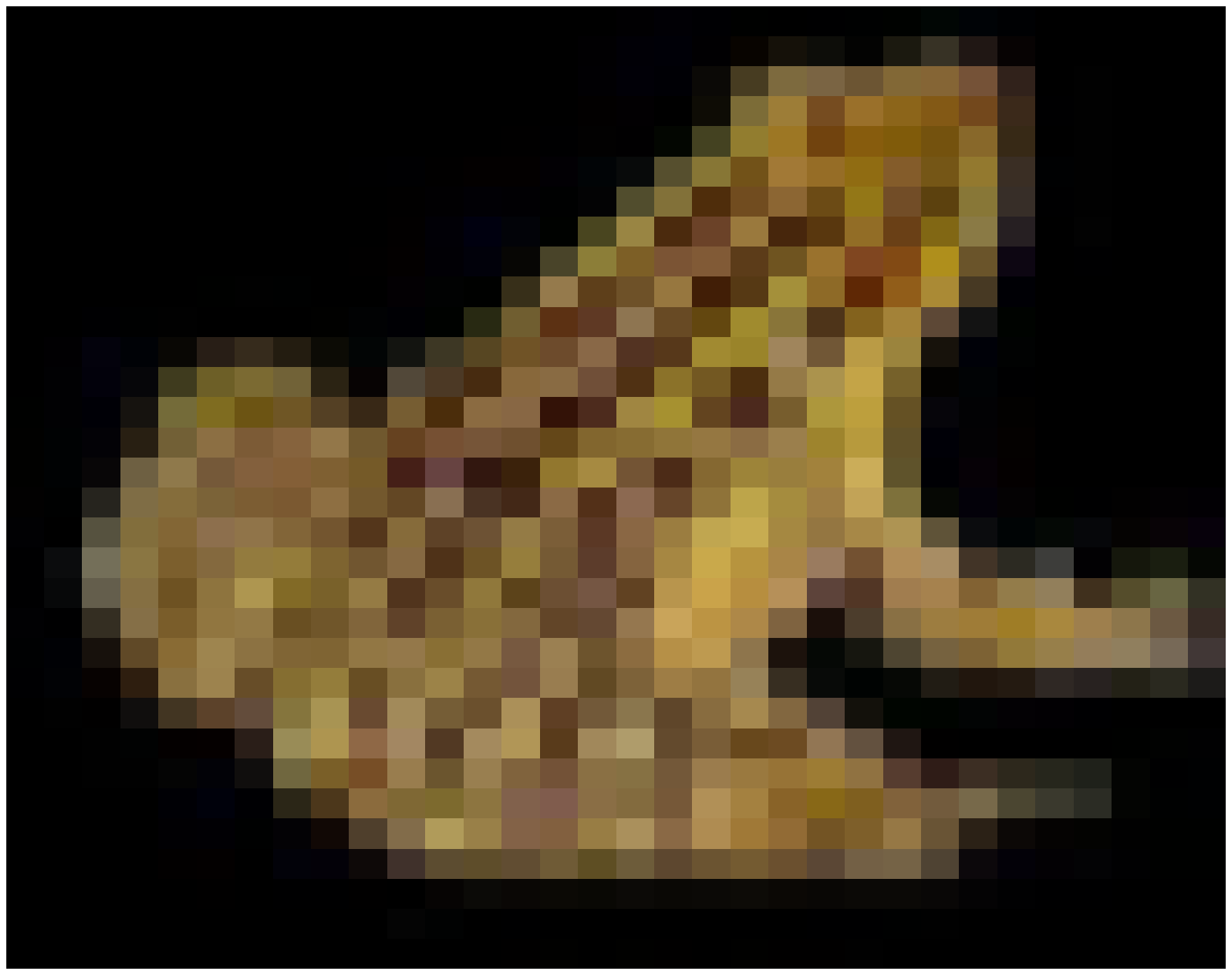}&\includegraphics[width=.1\linewidth]{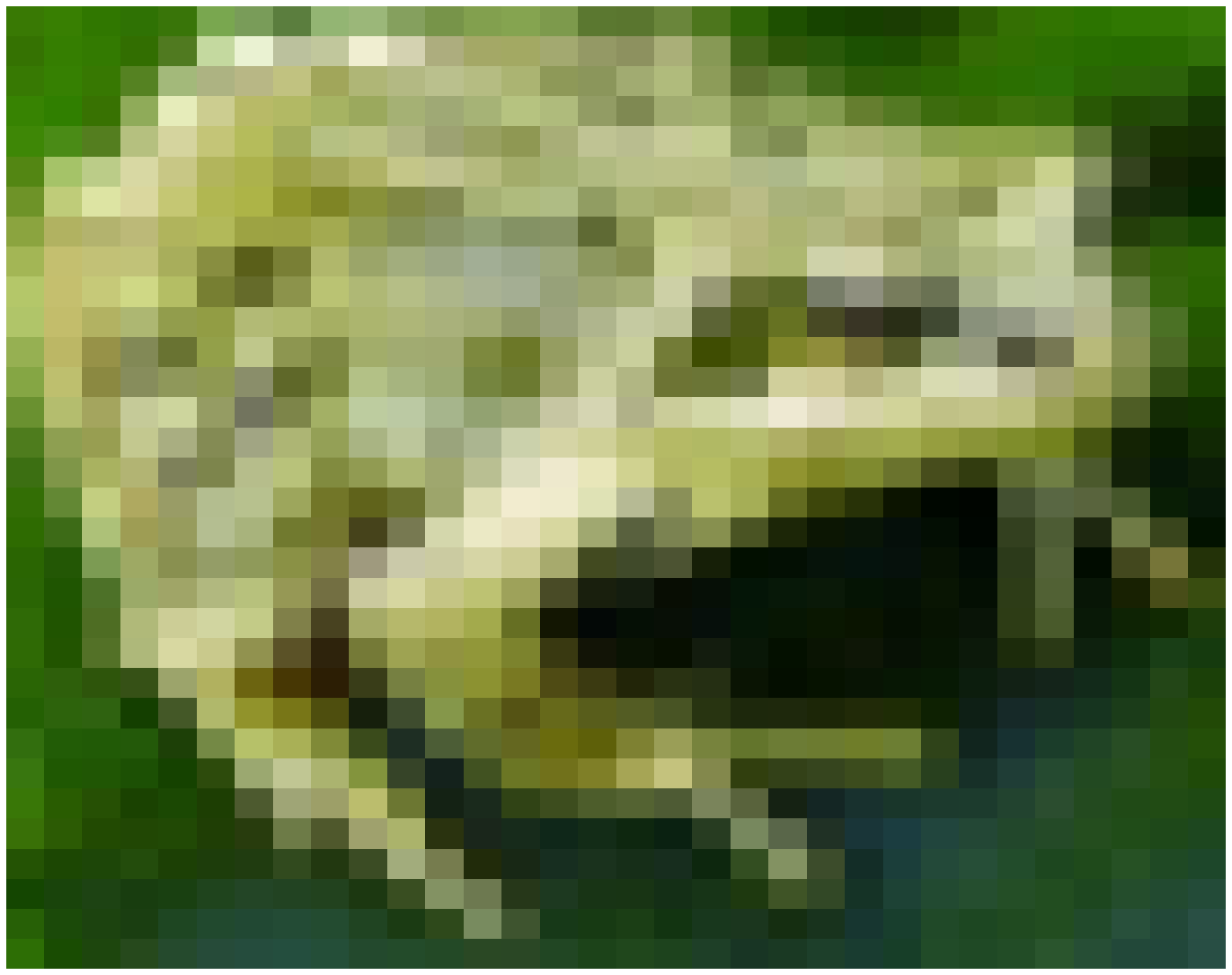}&\includegraphics[width=.1\linewidth]{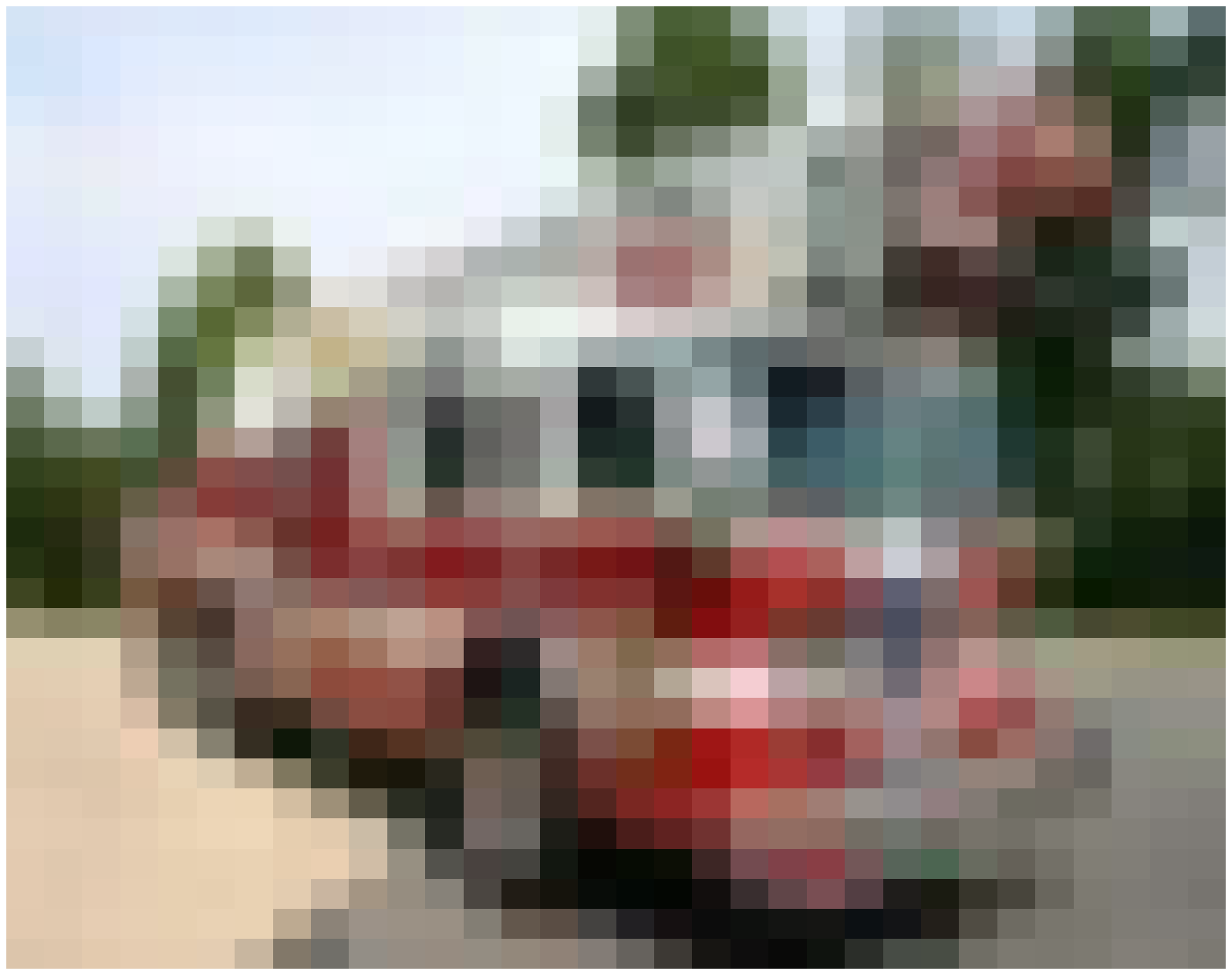}&\includegraphics[width=.1\linewidth]{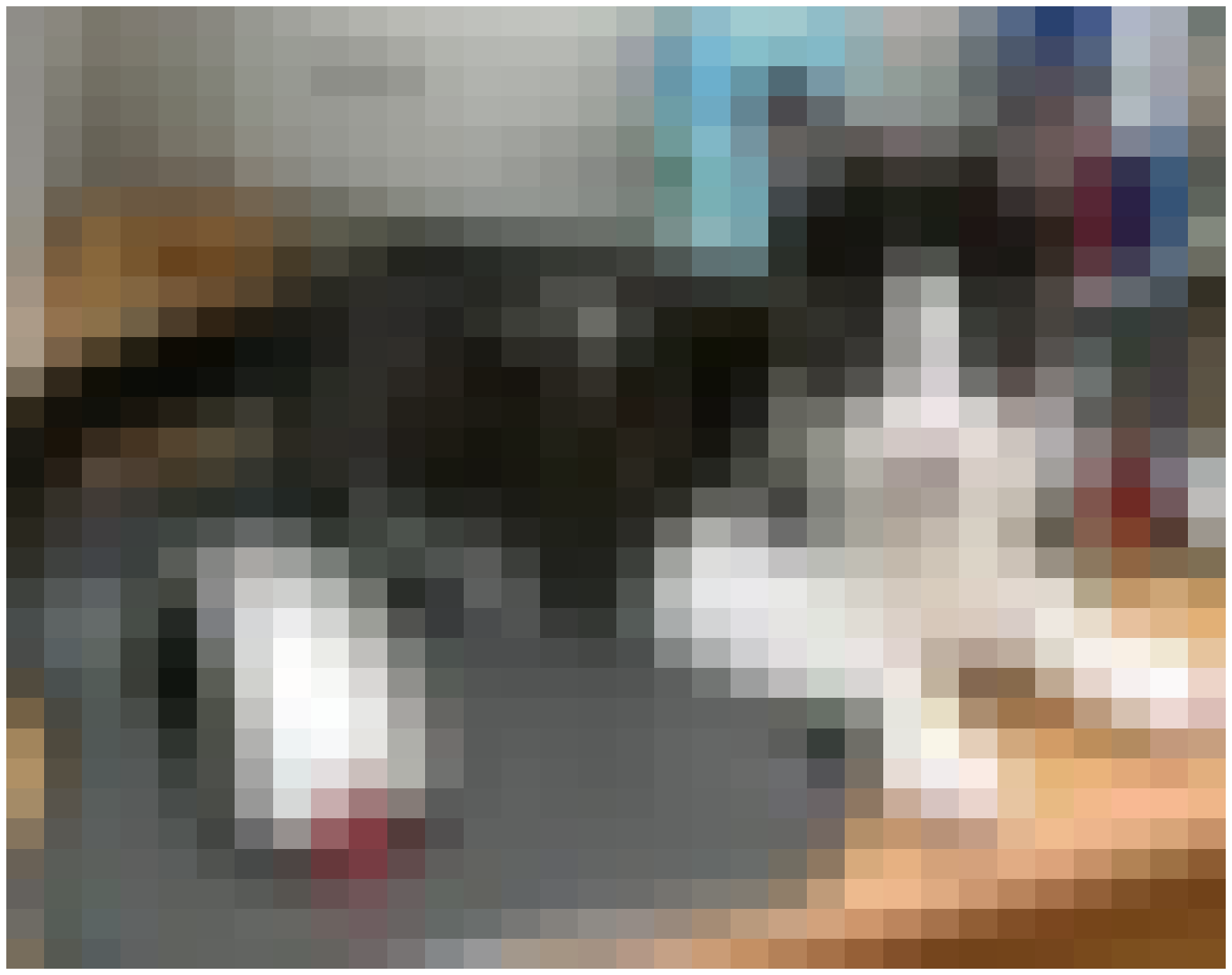}\\
    \includegraphics[width=.1\linewidth]{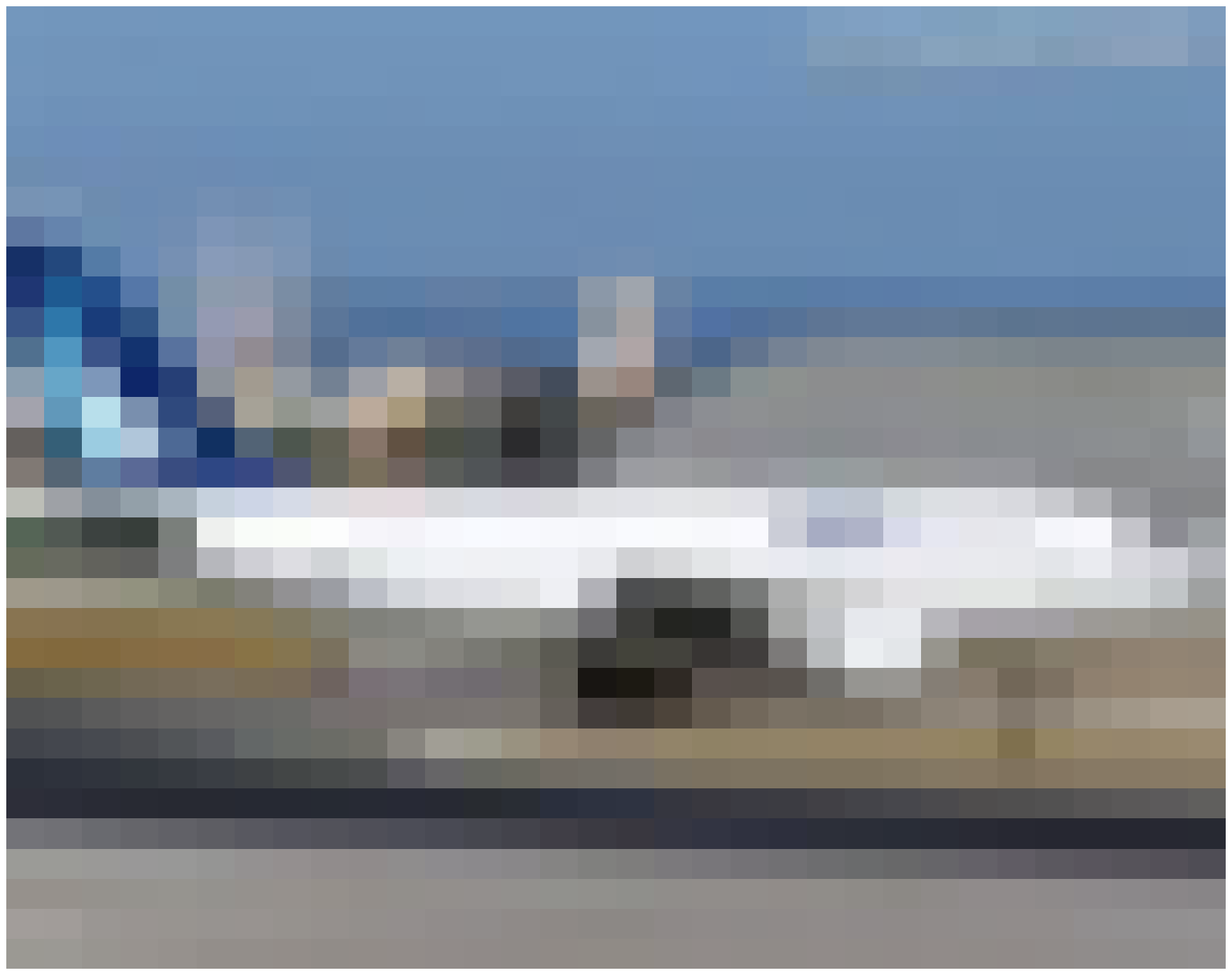}&\includegraphics[width=.1\linewidth]{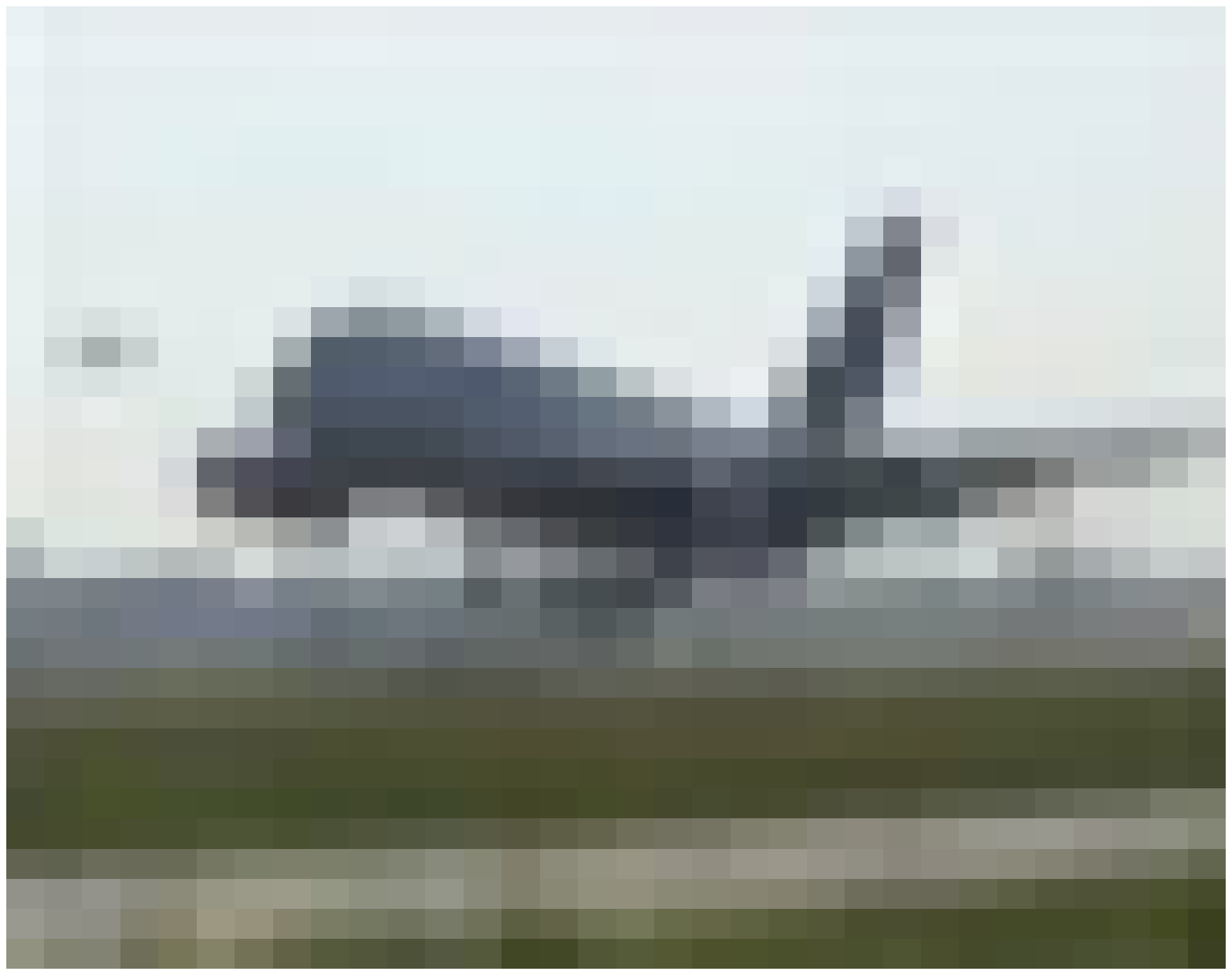}&\includegraphics[width=.1\linewidth]{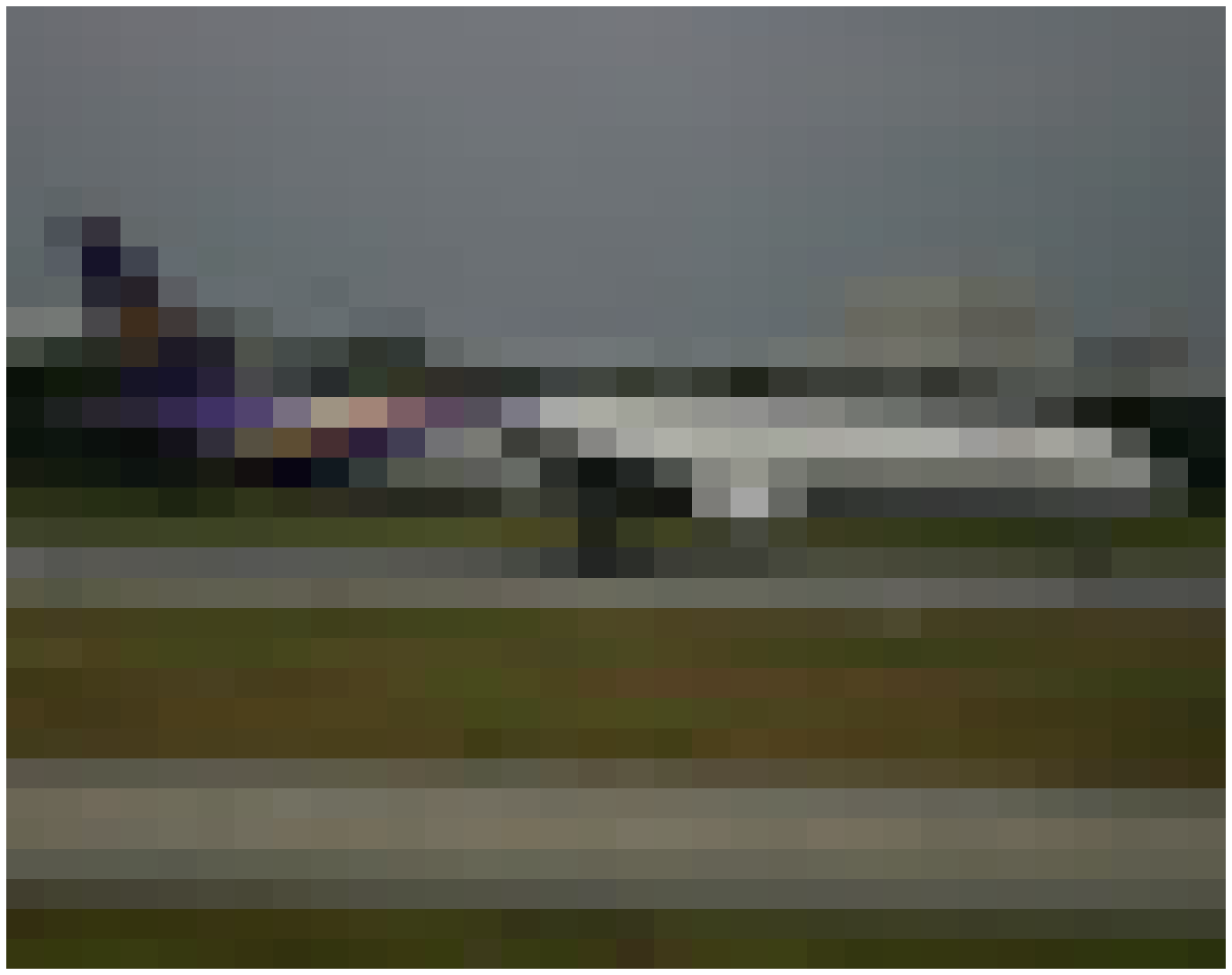}&\includegraphics[width=.1\linewidth]{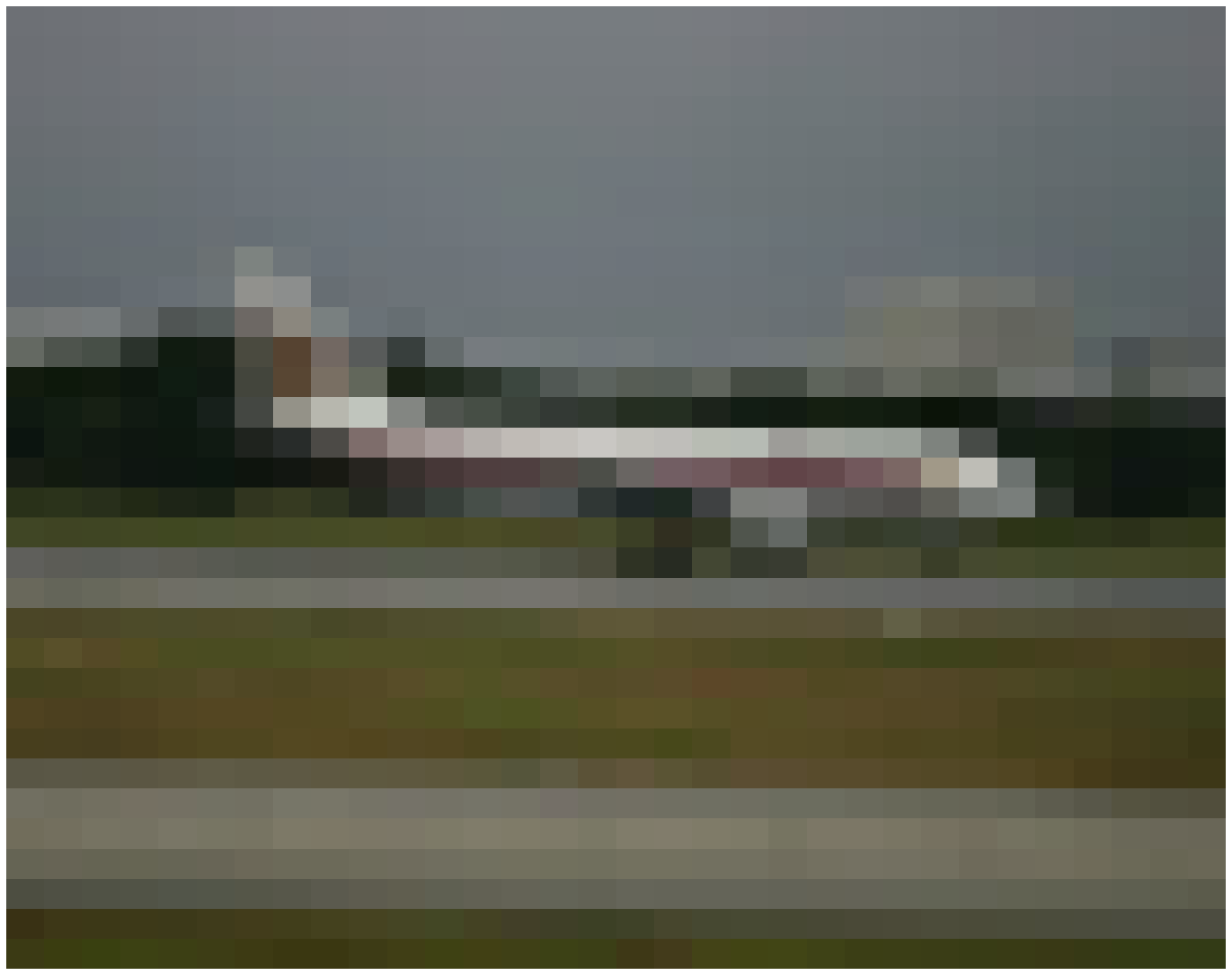}&\includegraphics[width=.1\linewidth]{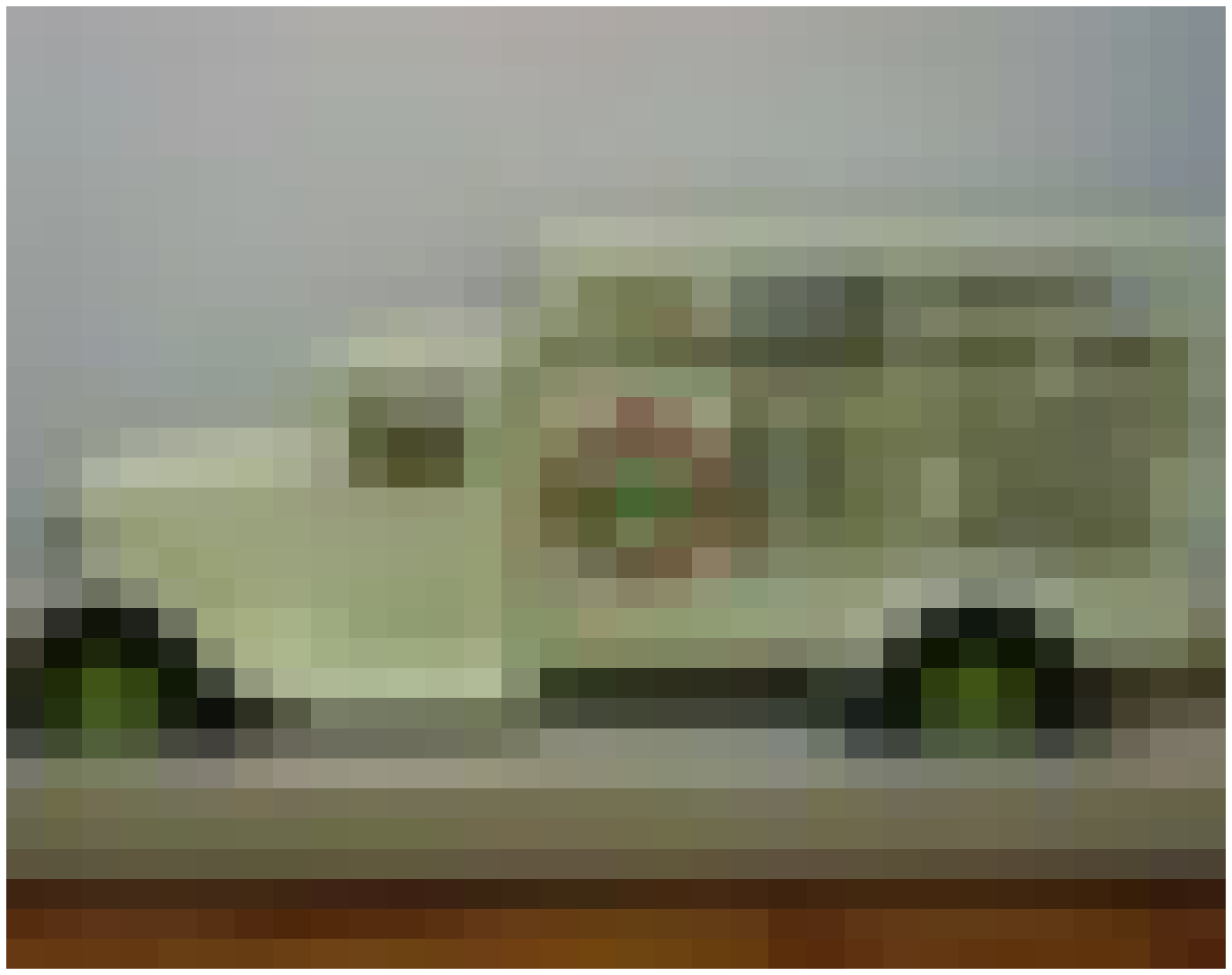}&\includegraphics[width=.1\linewidth]{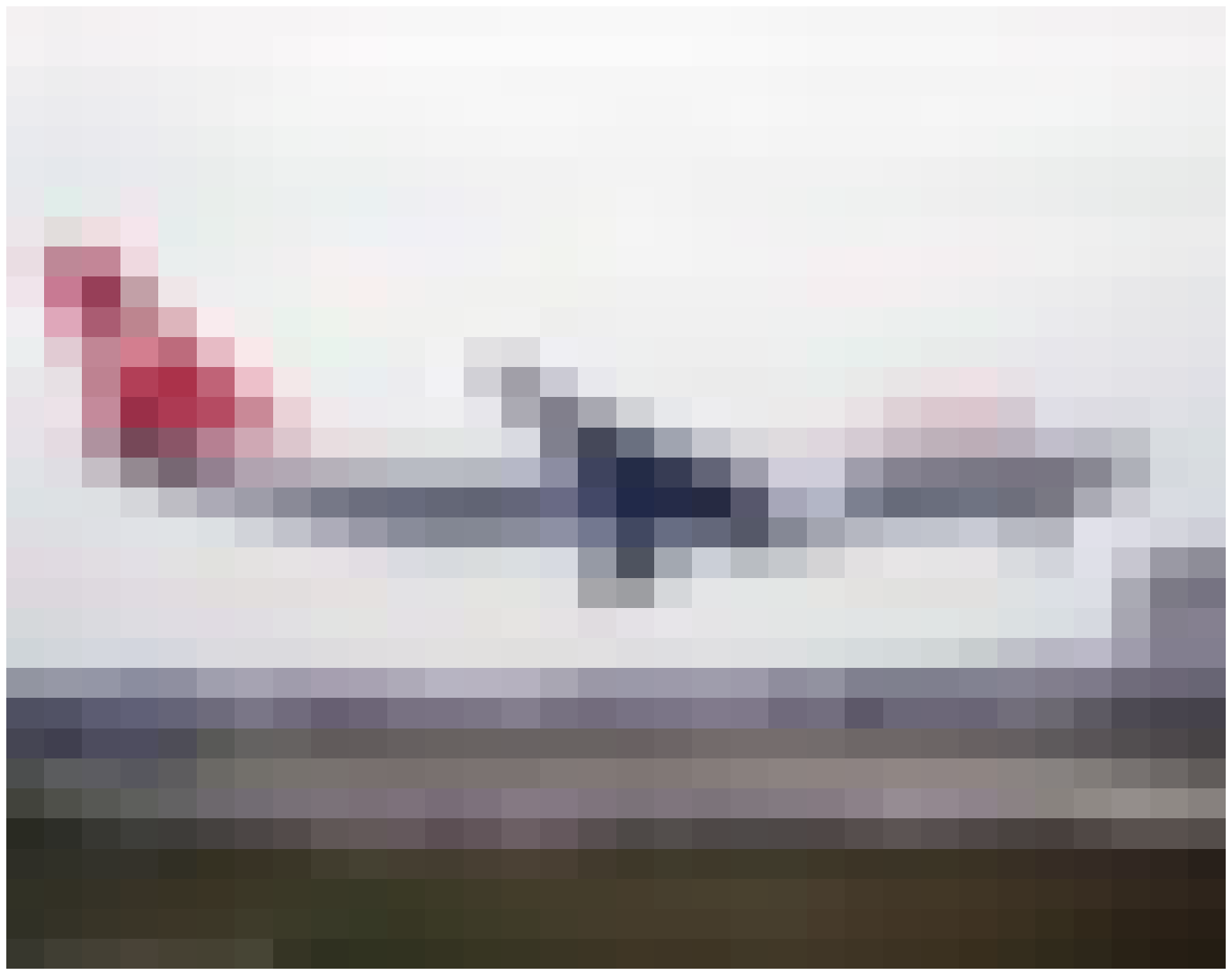}&\includegraphics[width=.1\linewidth]{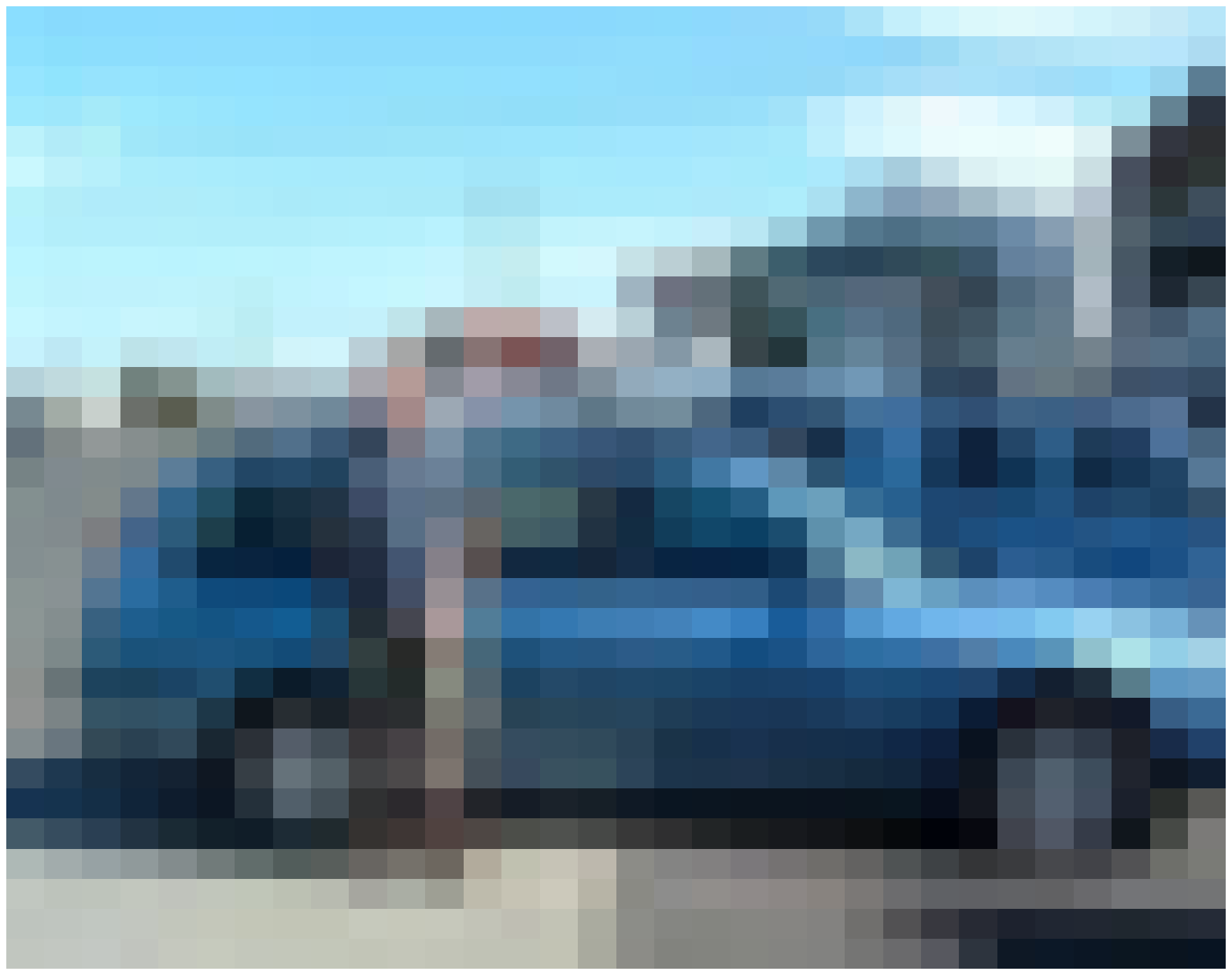}&\includegraphics[width=.1\linewidth]{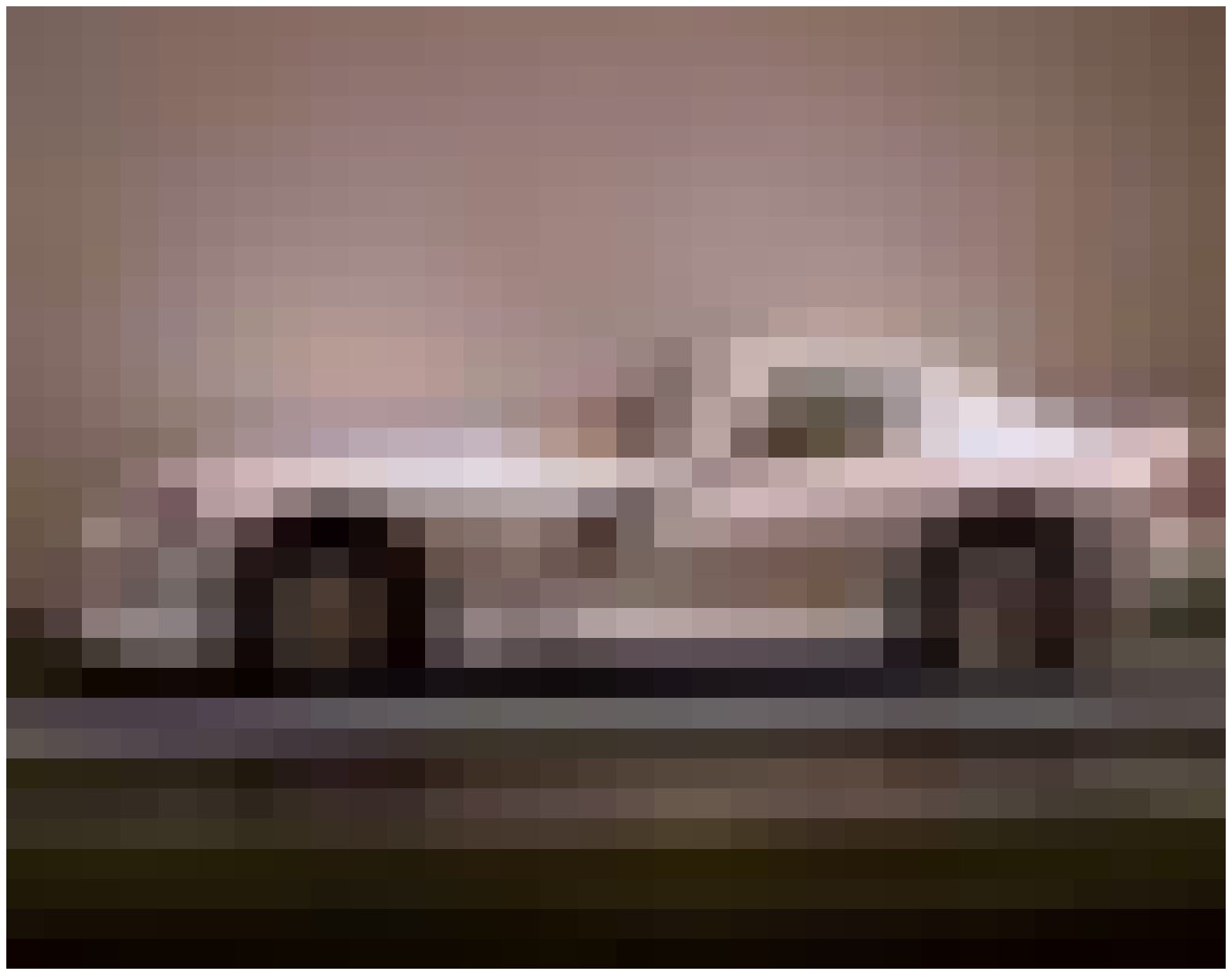}&\includegraphics[width=.1\linewidth]{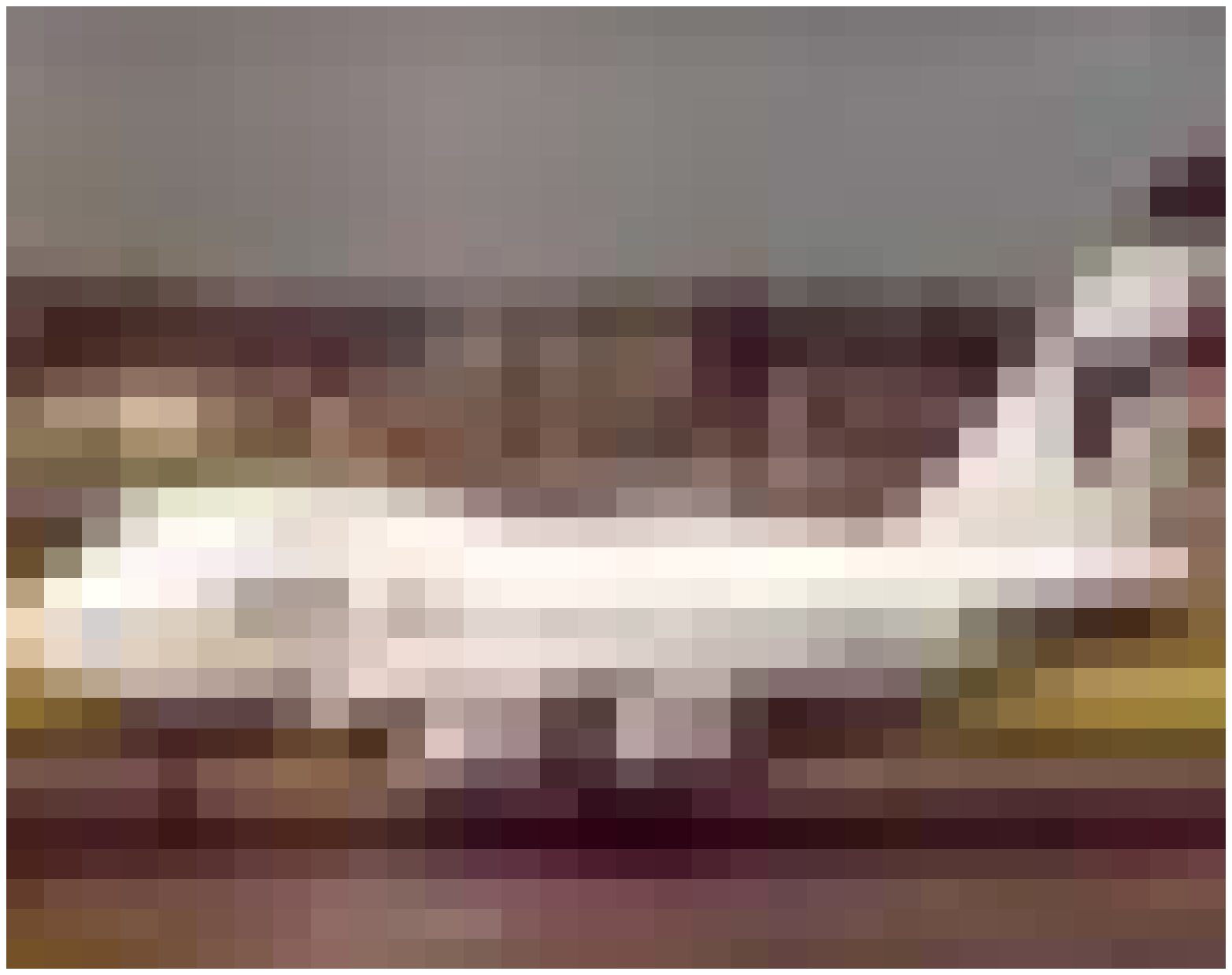}&\includegraphics[width=.1\linewidth]{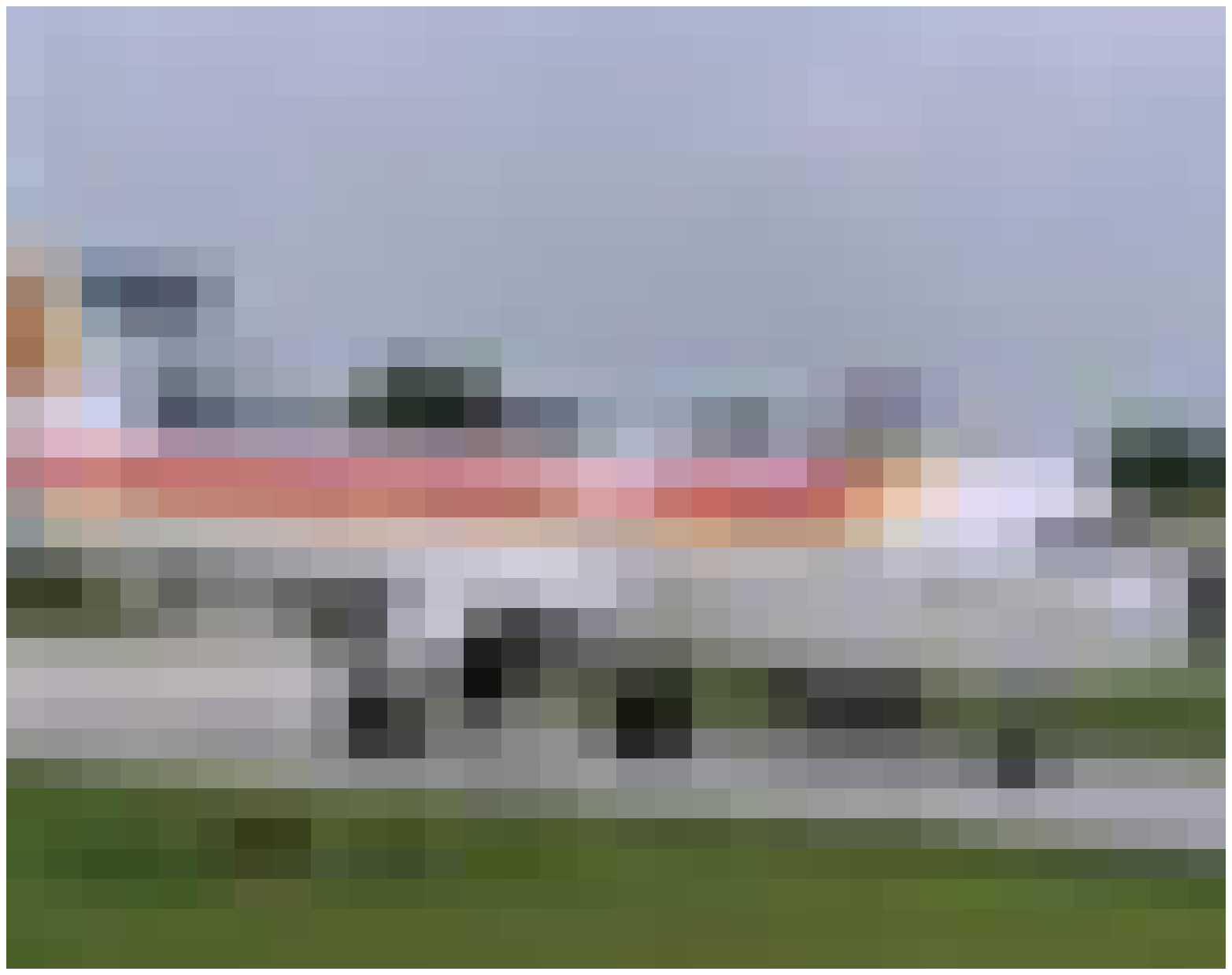}\\
    \includegraphics[width=.1\linewidth]{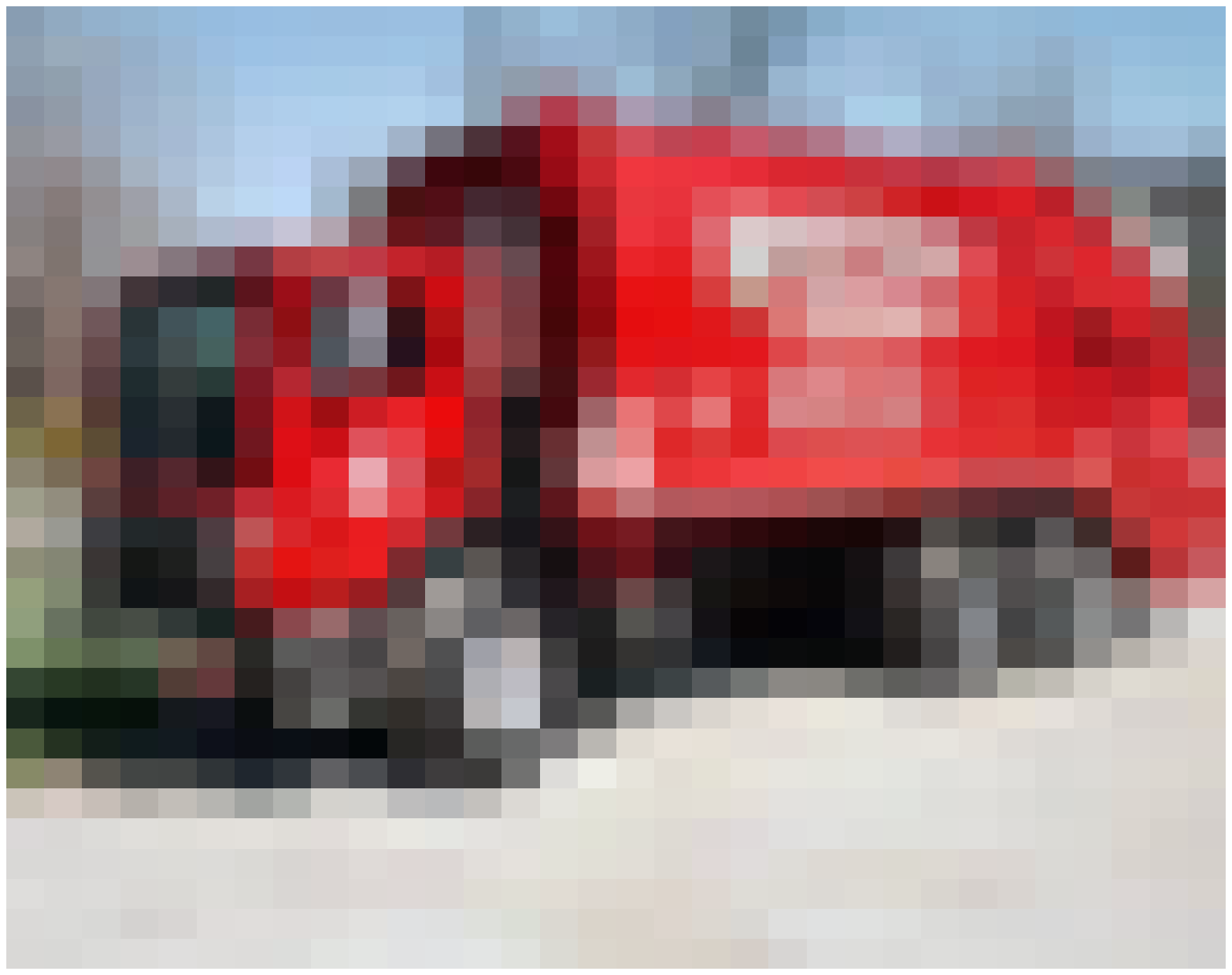}&\includegraphics[width=.1\linewidth]{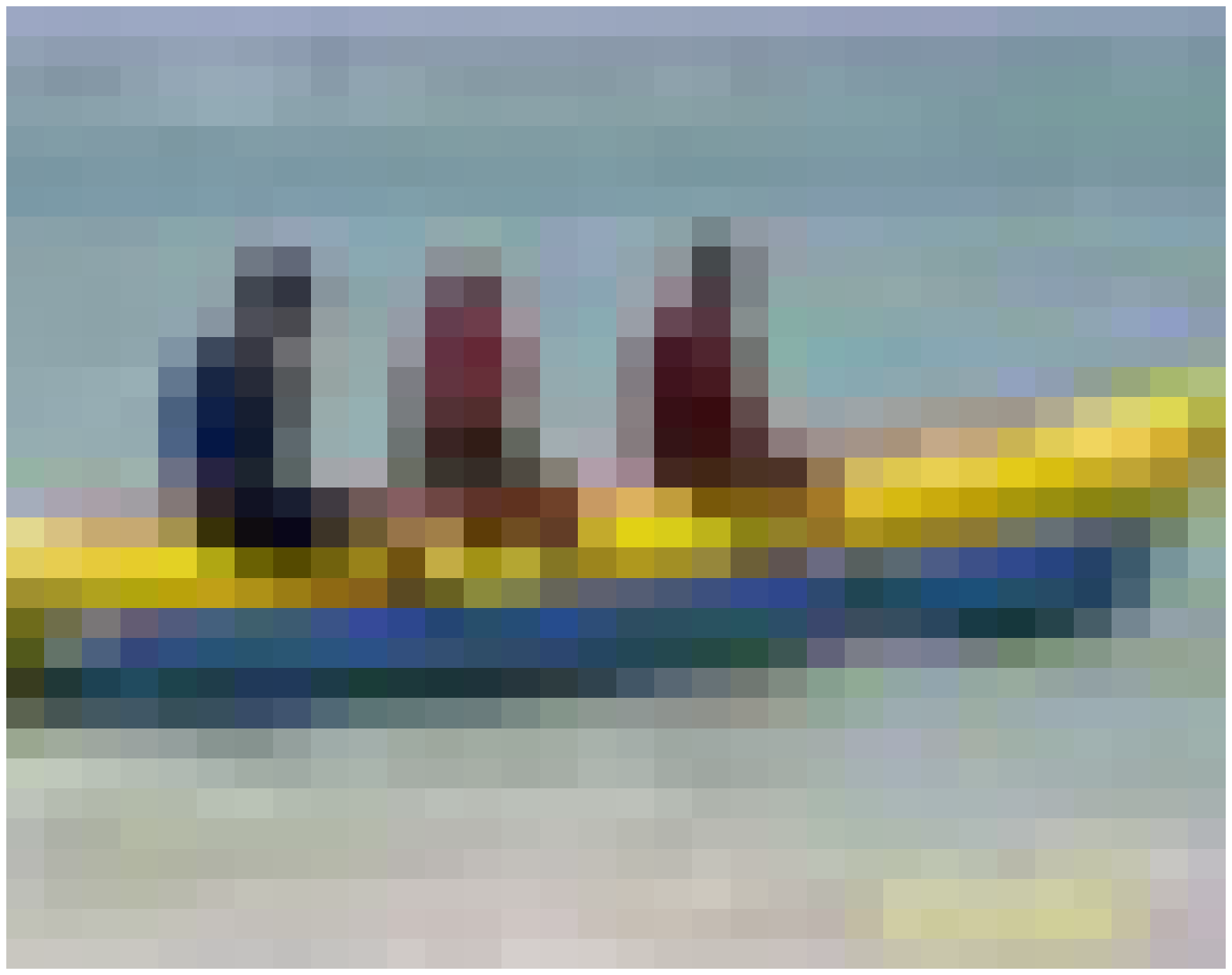}&\includegraphics[width=.1\linewidth]{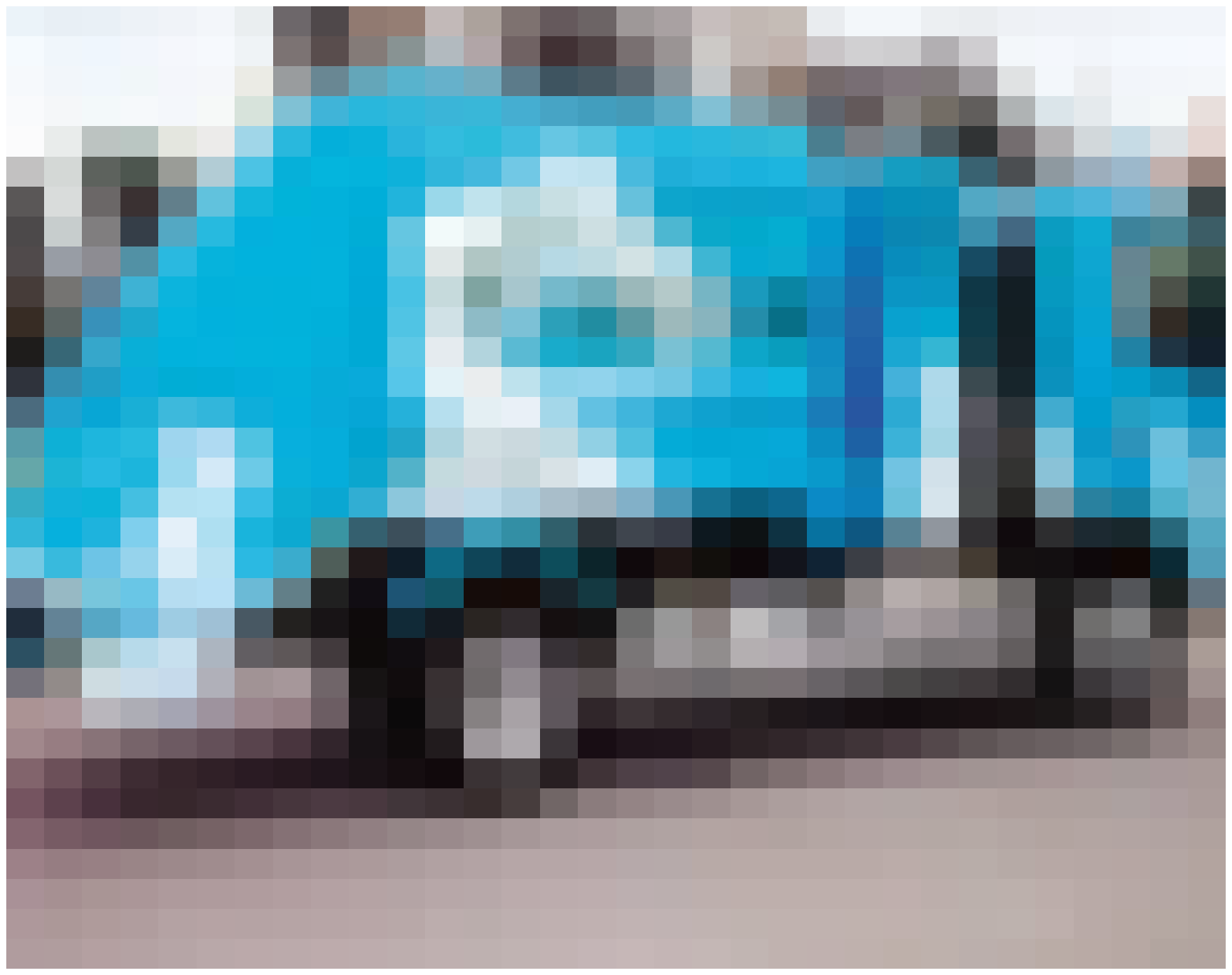}&\includegraphics[width=.1\linewidth]{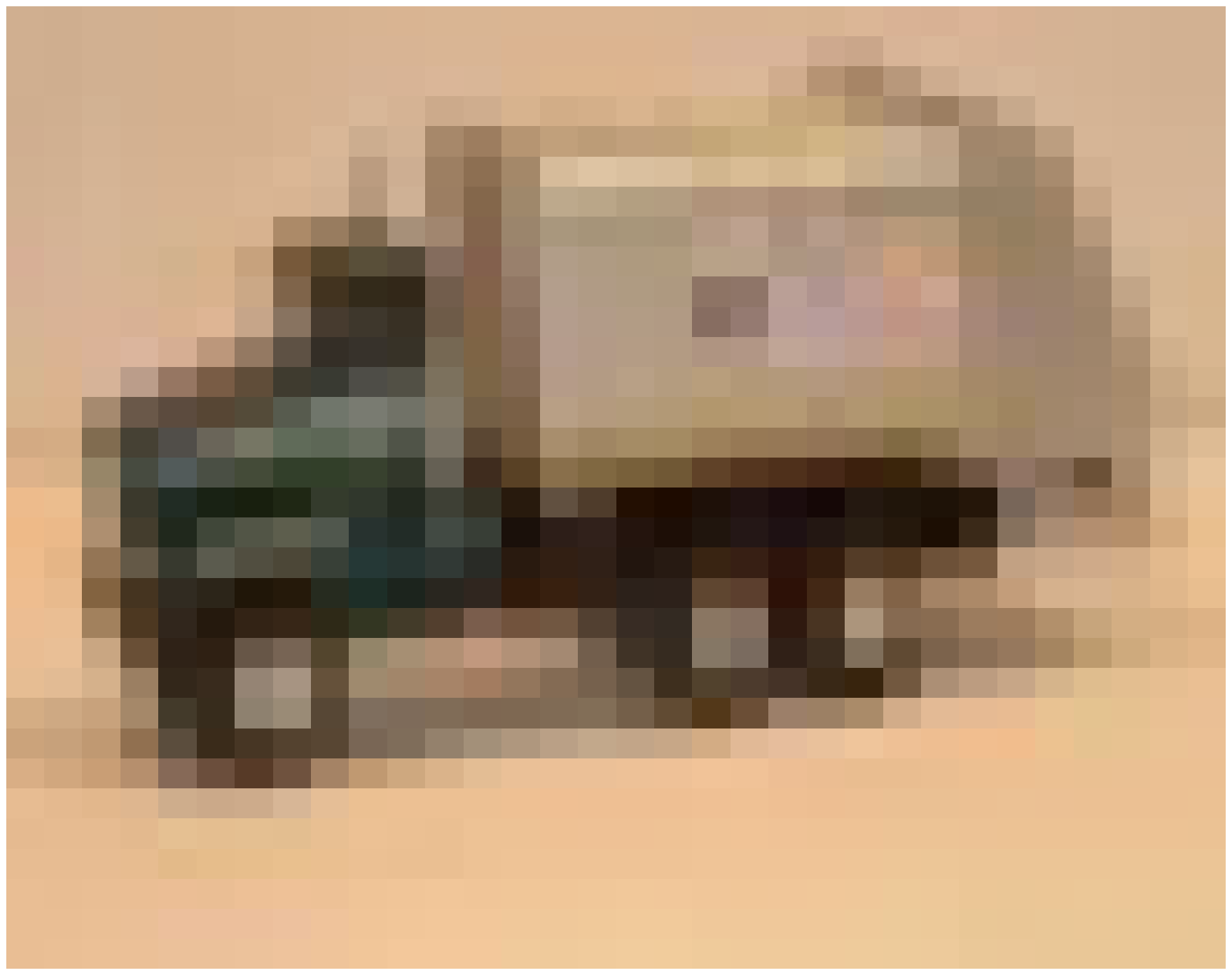}&\includegraphics[width=.1\linewidth]{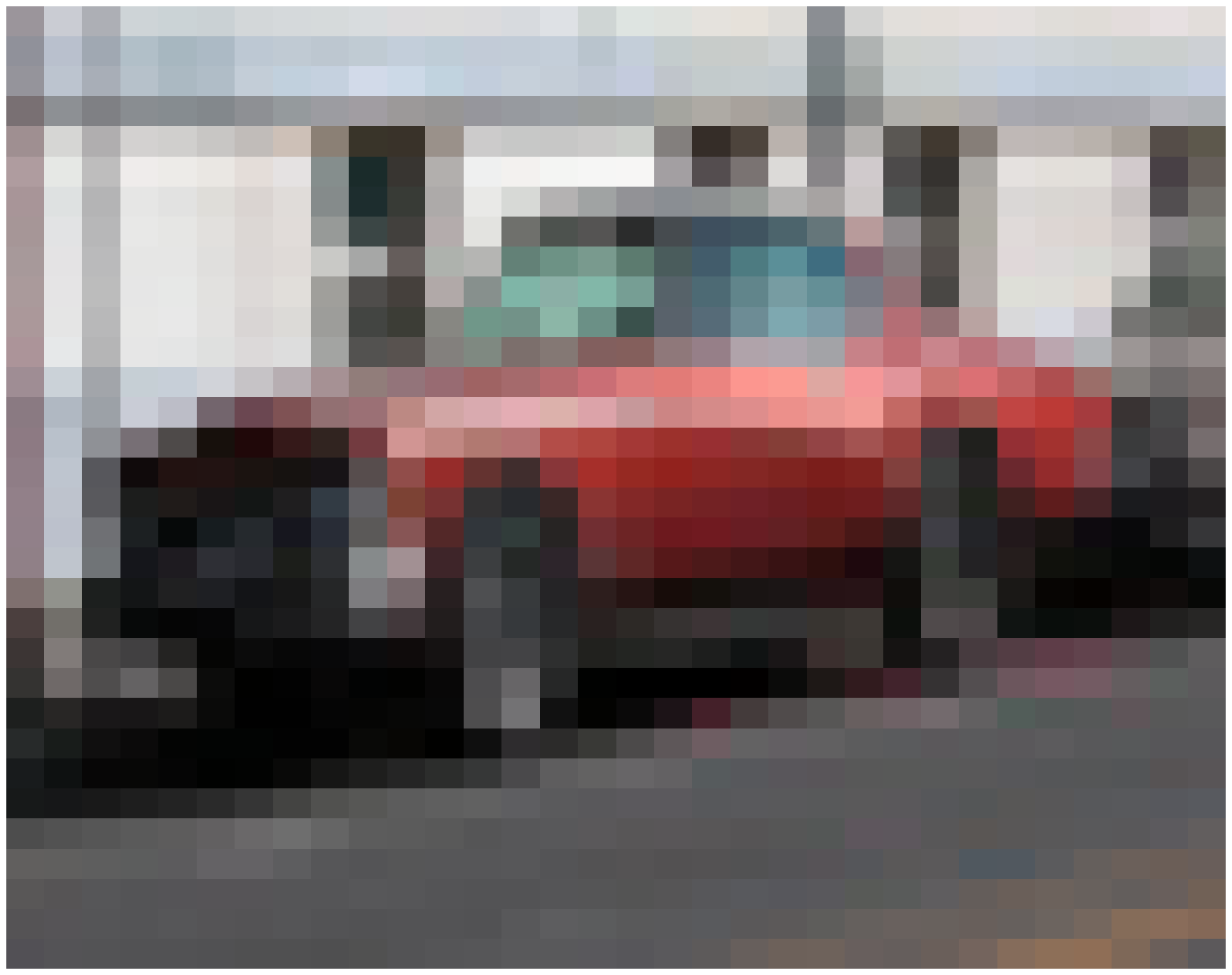}&\includegraphics[width=.1\linewidth]{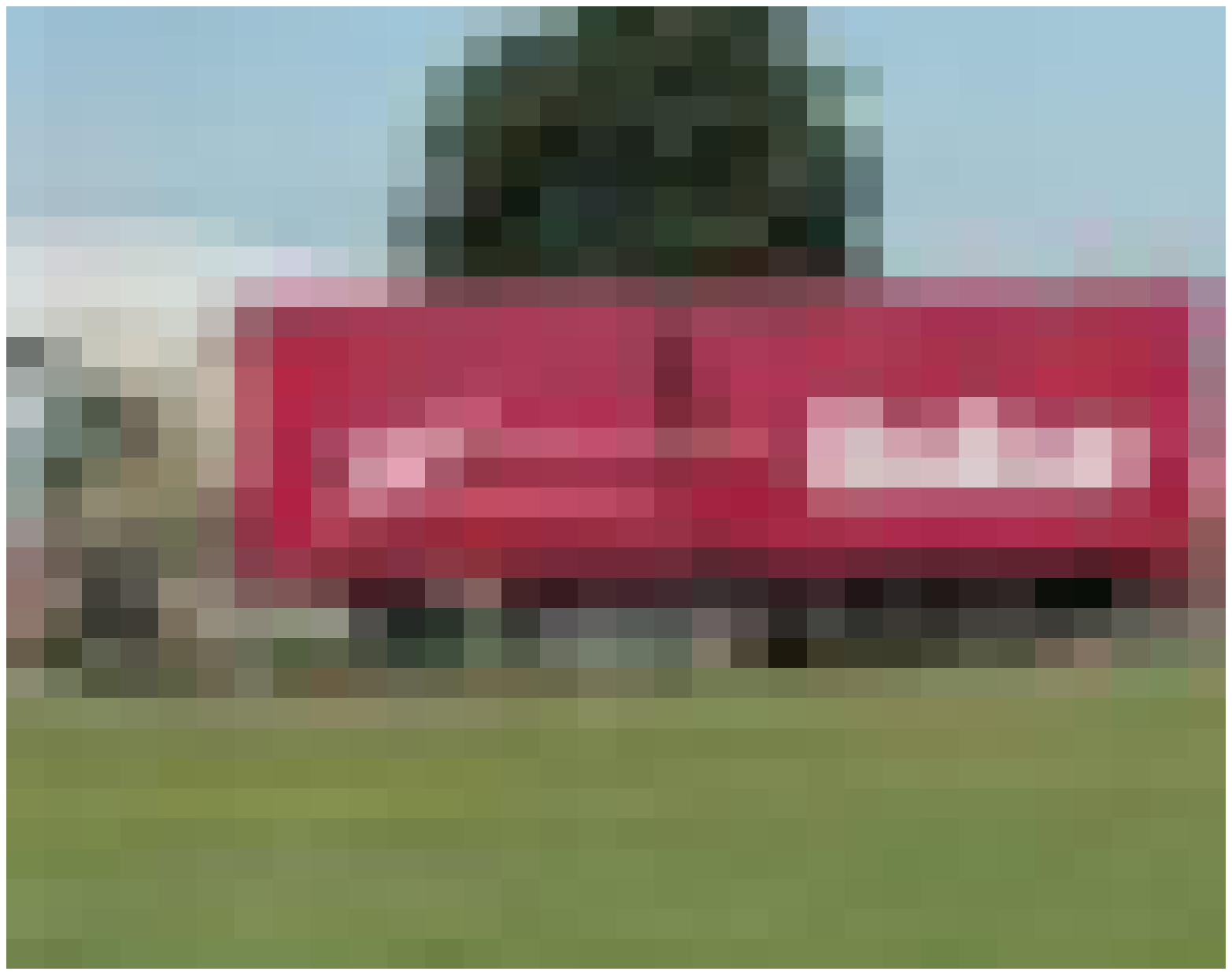}&\includegraphics[width=.1\linewidth]{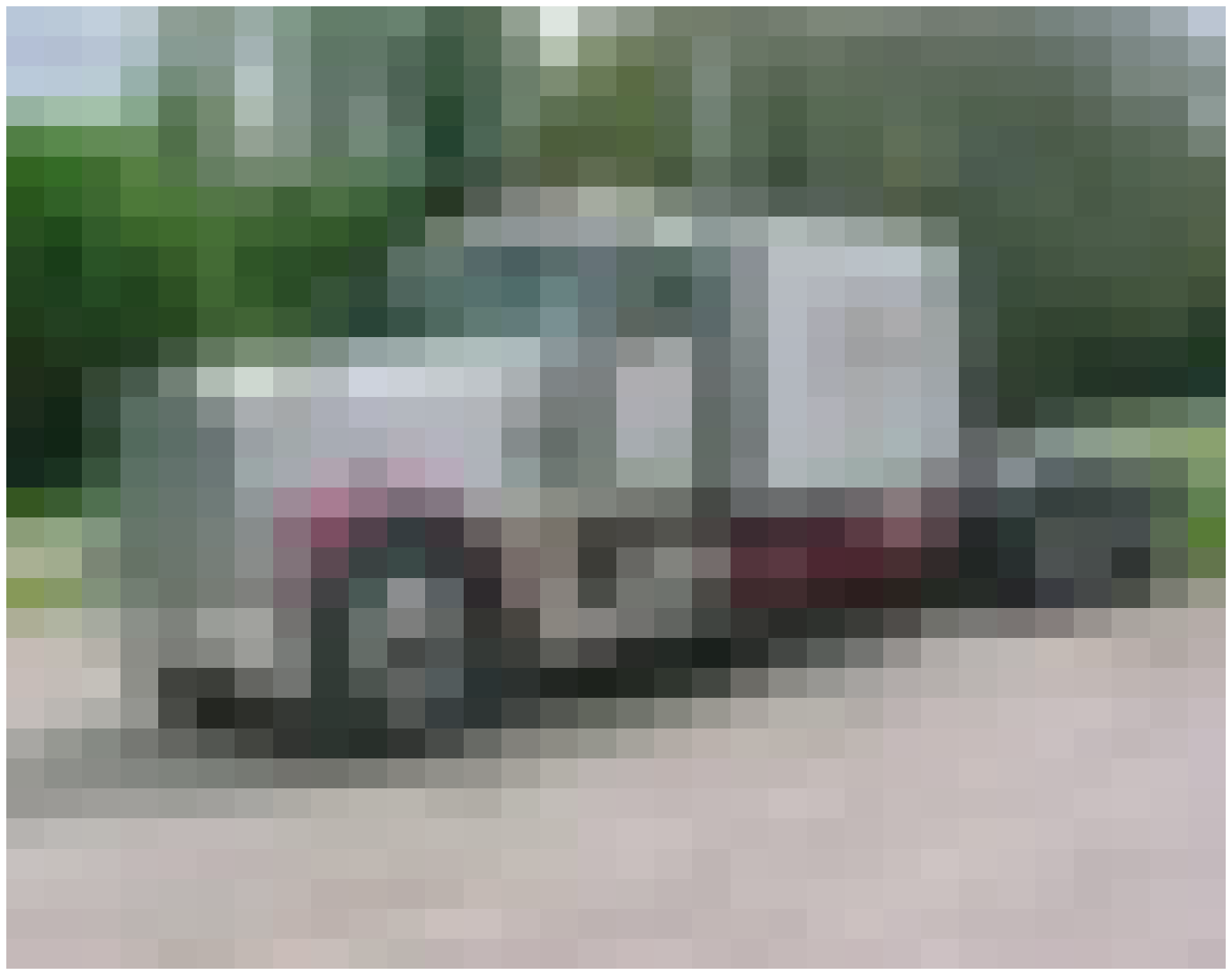}&\includegraphics[width=.1\linewidth]{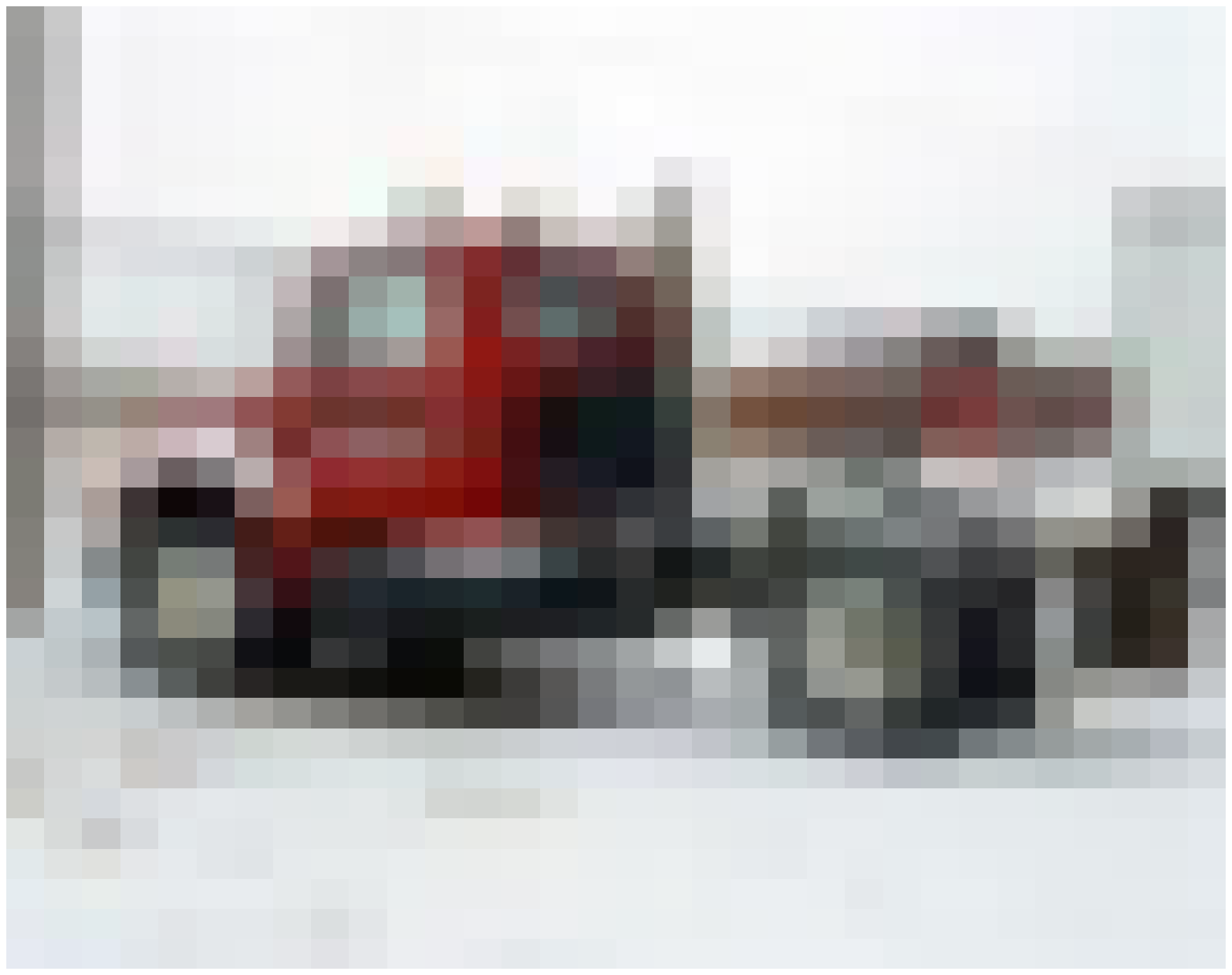}&\includegraphics[width=.1\linewidth]{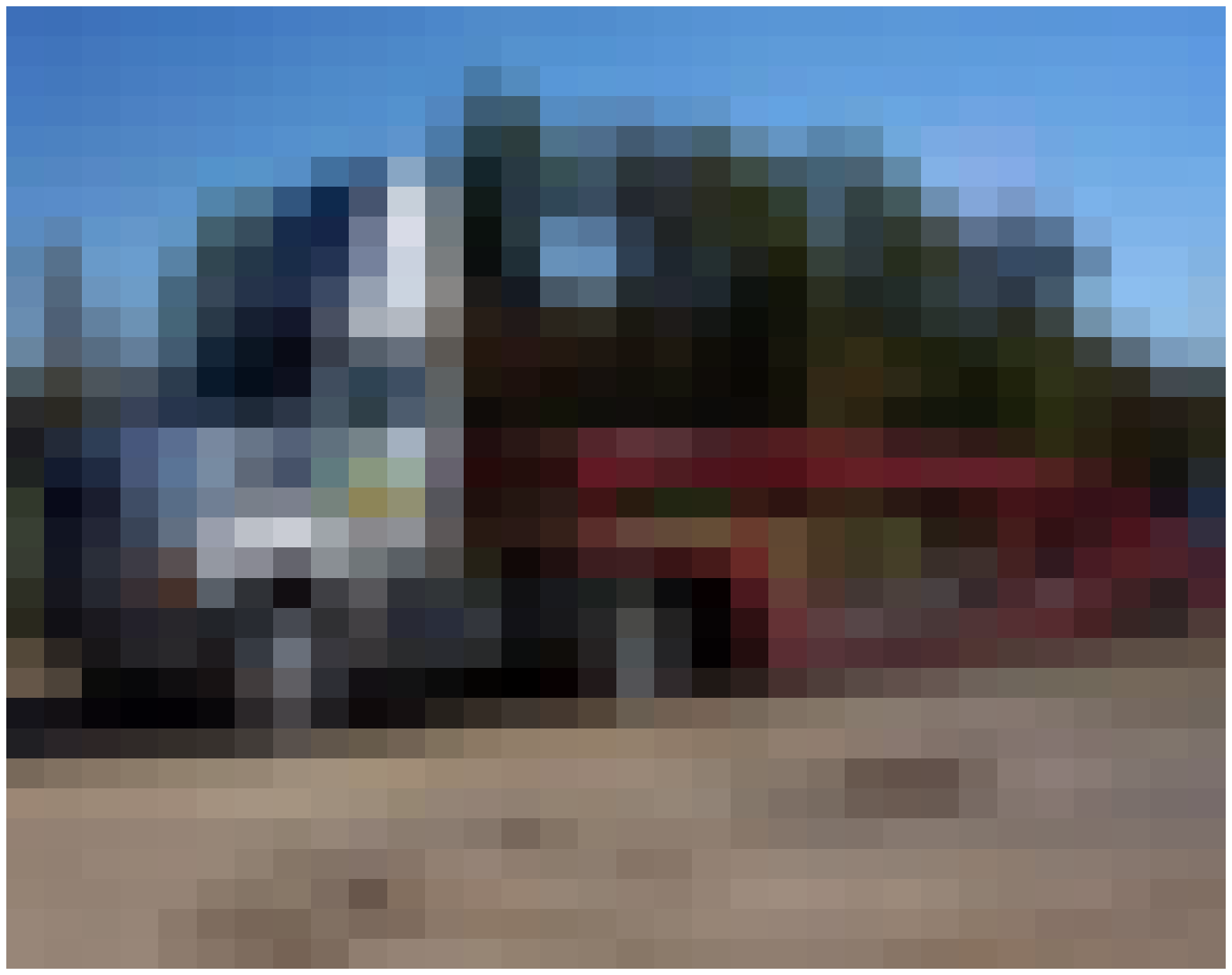}&\includegraphics[width=.1\linewidth]{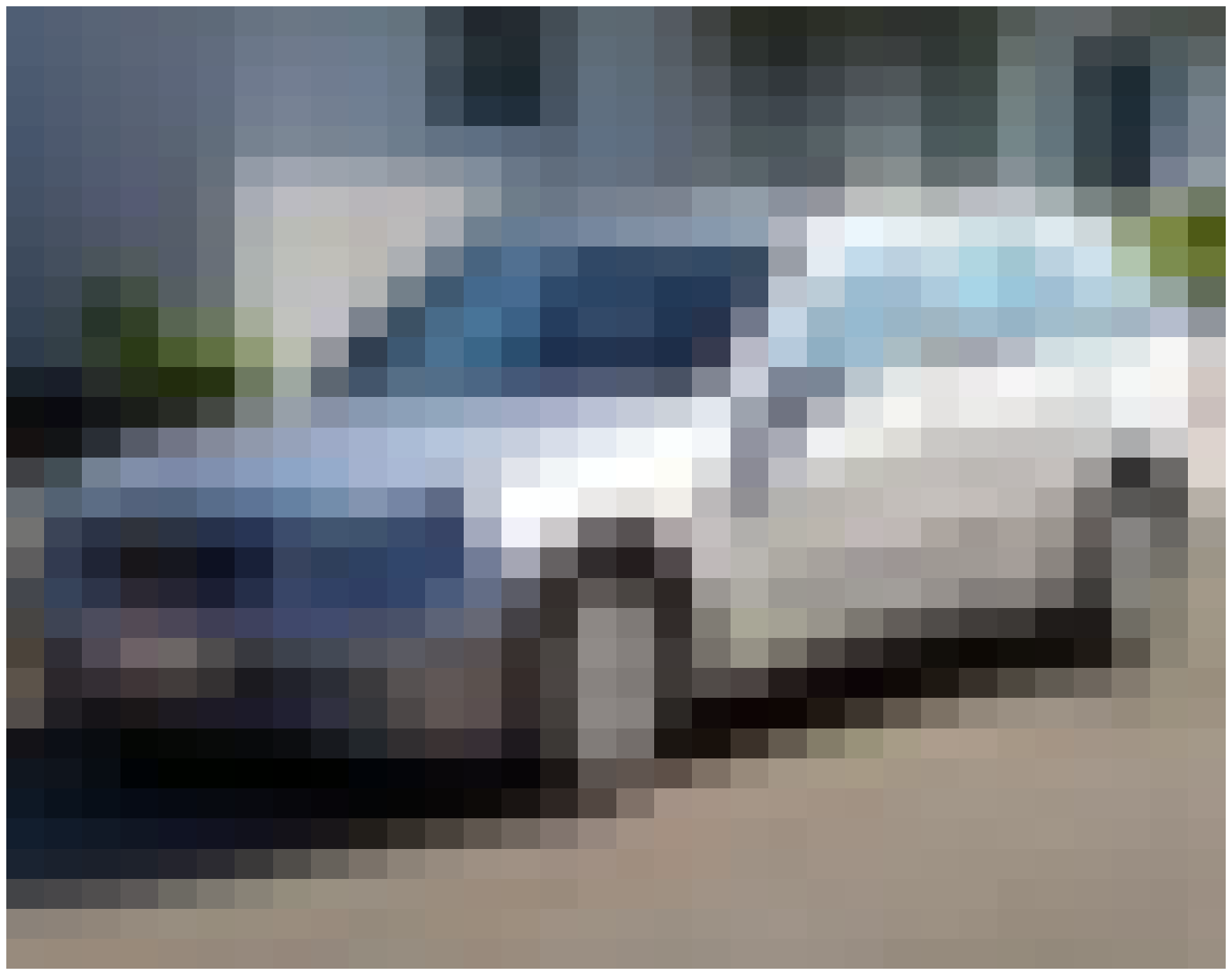}\\
    \includegraphics[width=.1\linewidth]{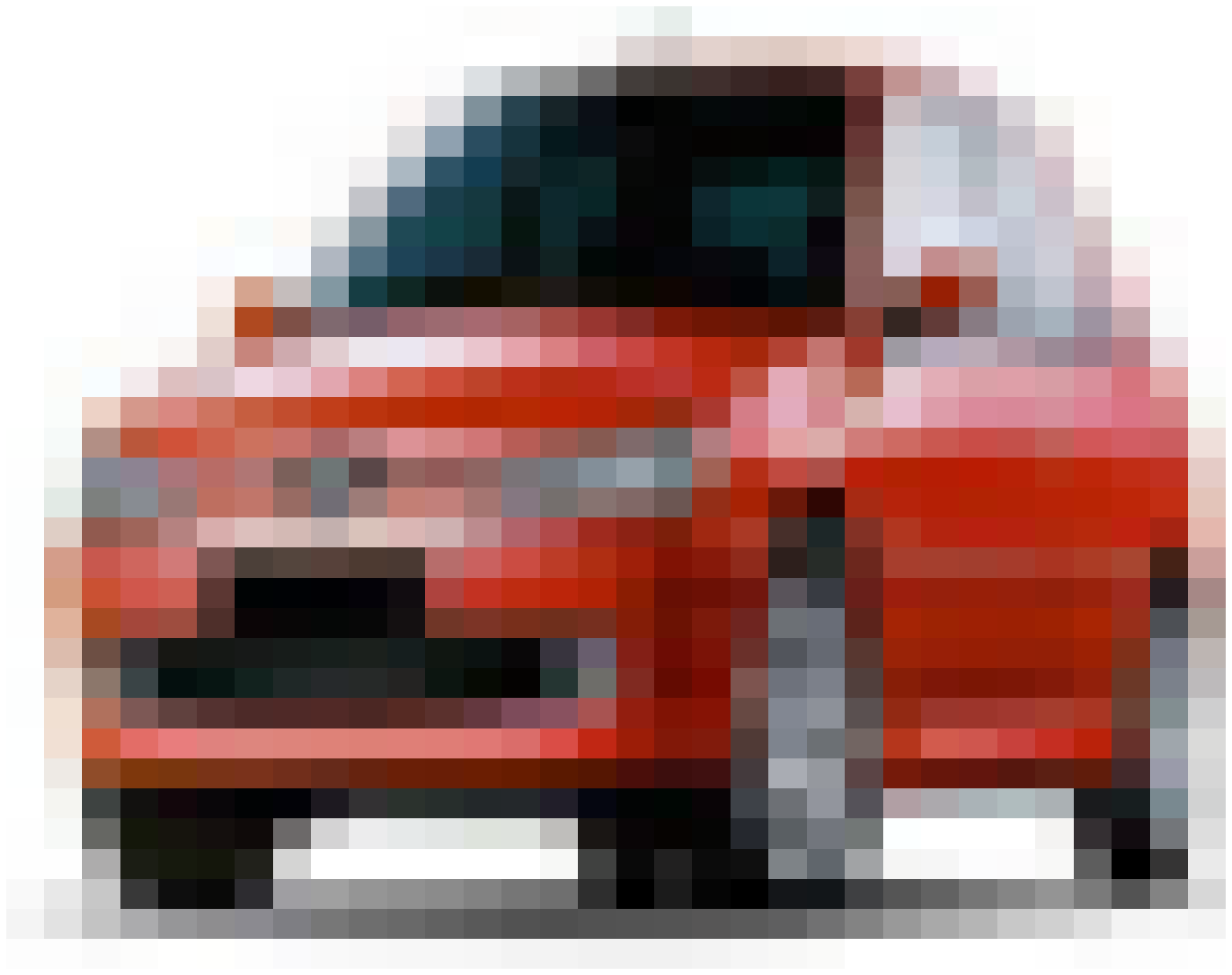}&\includegraphics[width=.1\linewidth]{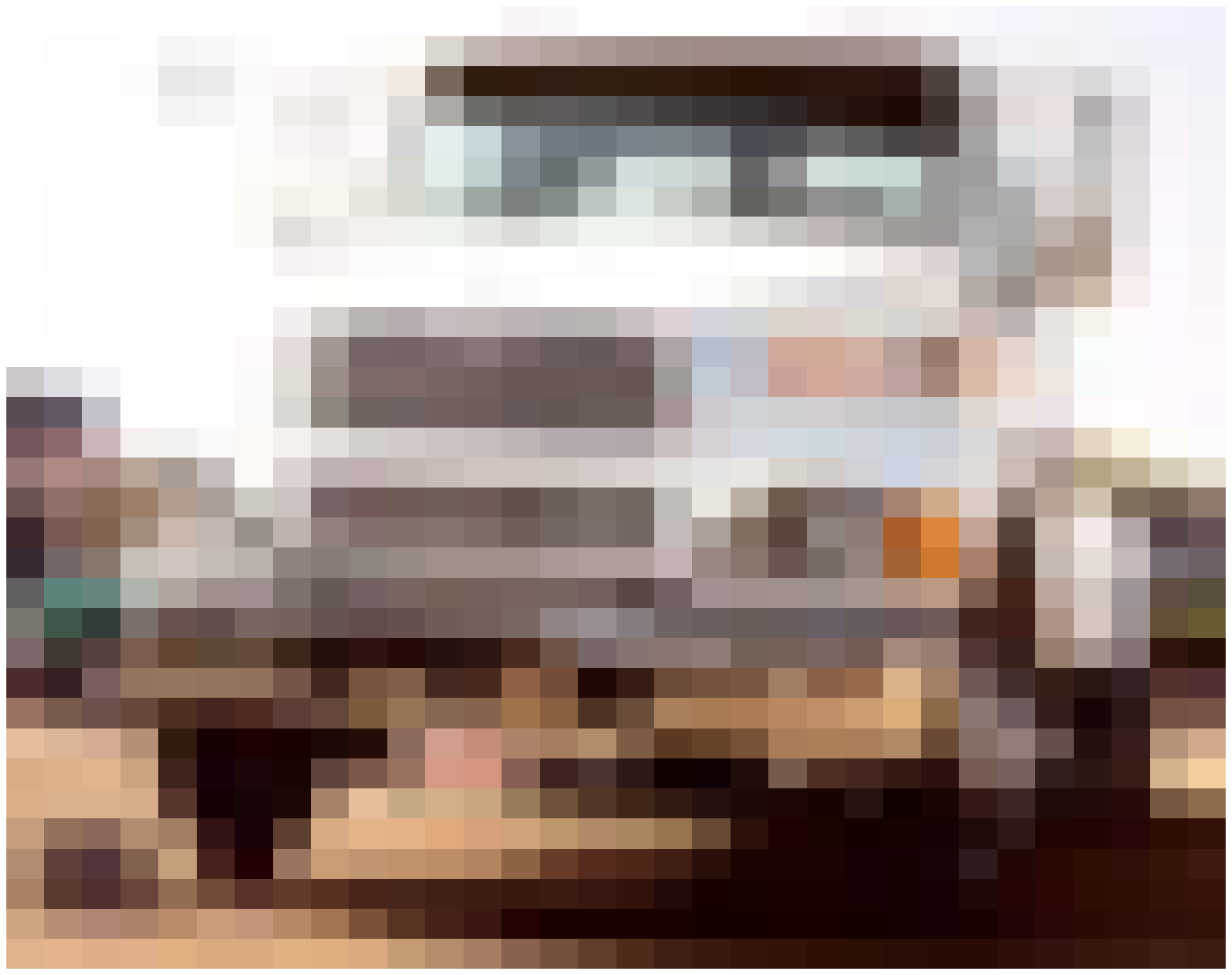}&\includegraphics[width=.1\linewidth]{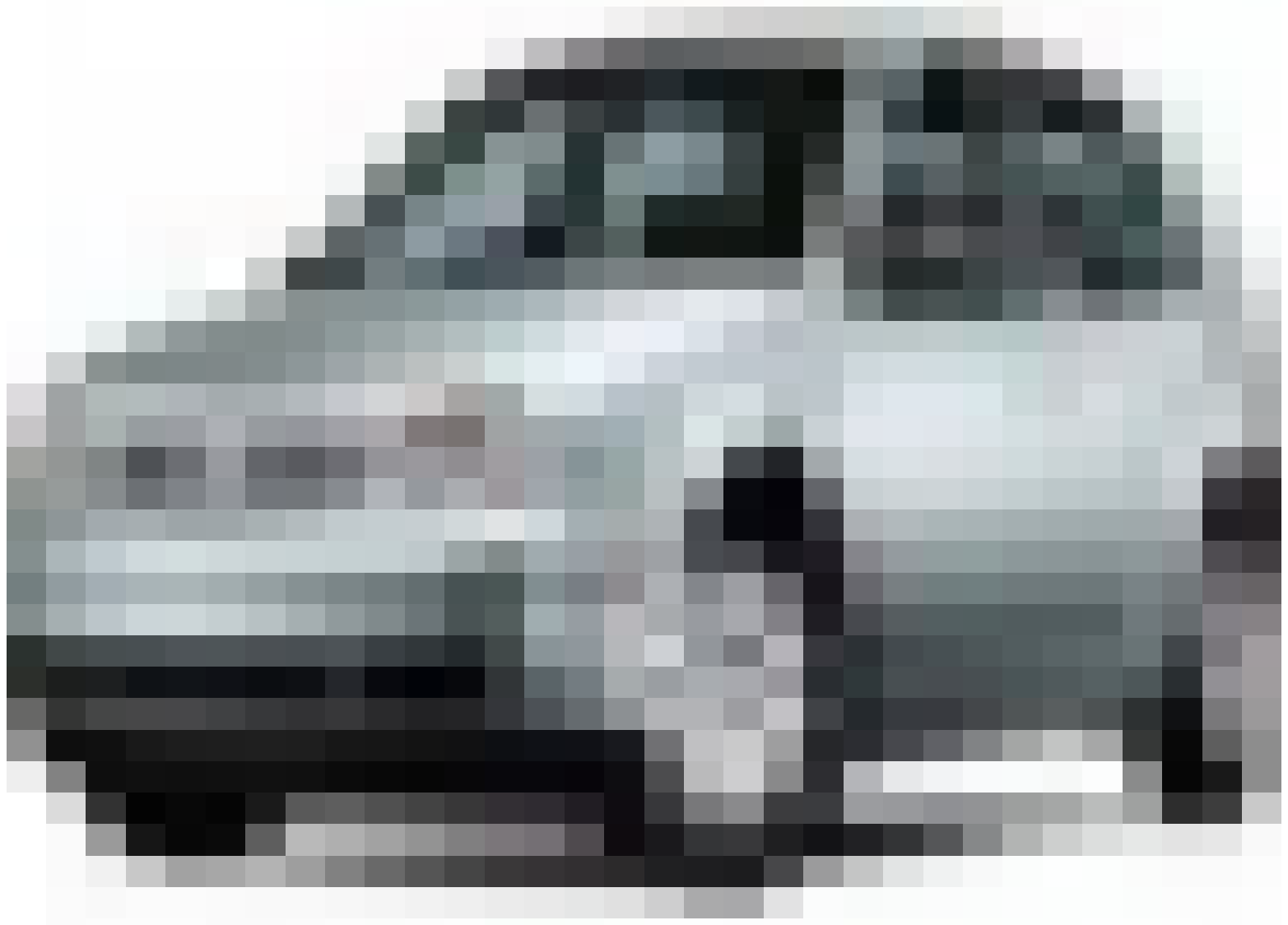}&\includegraphics[width=.1\linewidth]{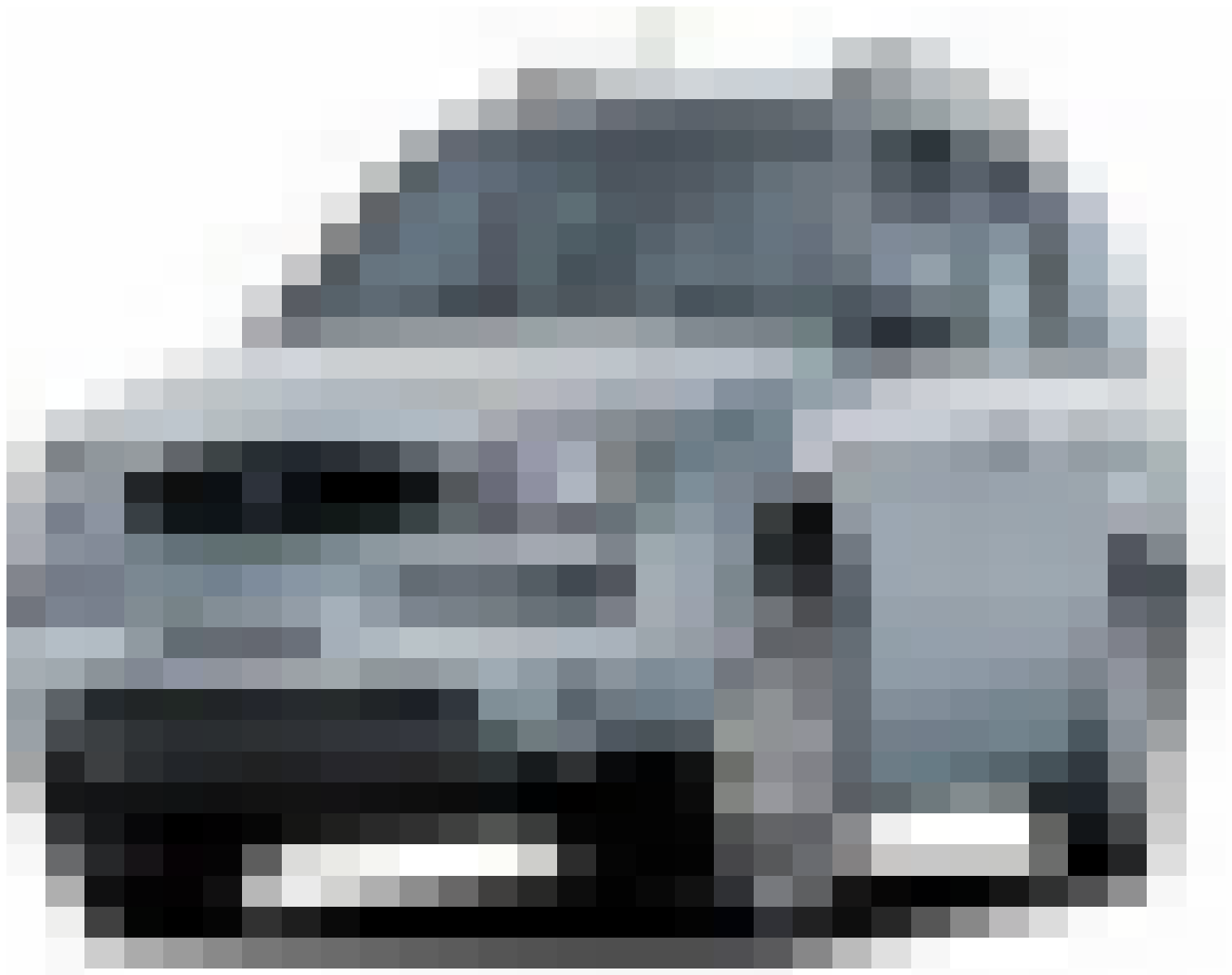}&\includegraphics[width=.1\linewidth]{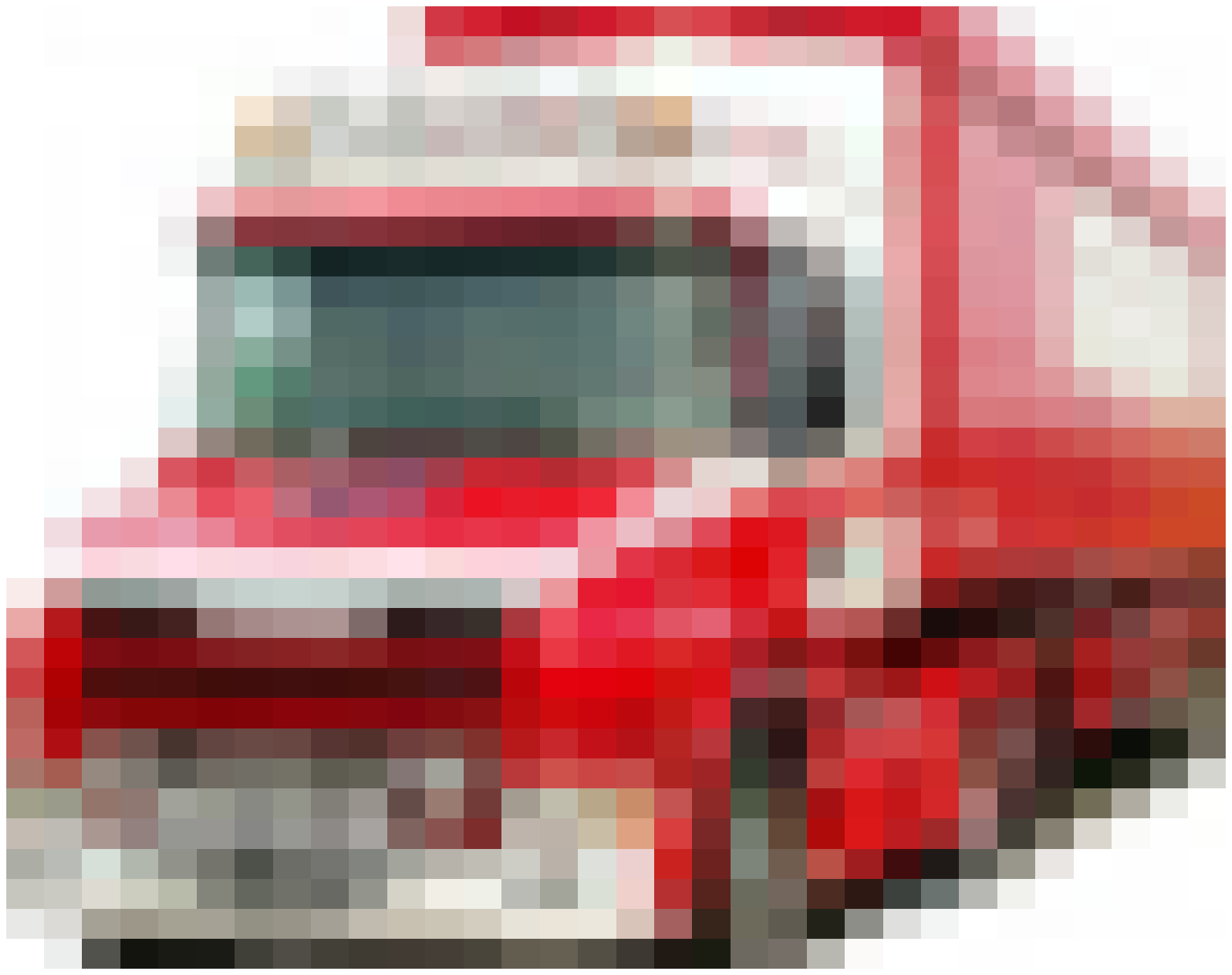}&\includegraphics[width=.1\linewidth]{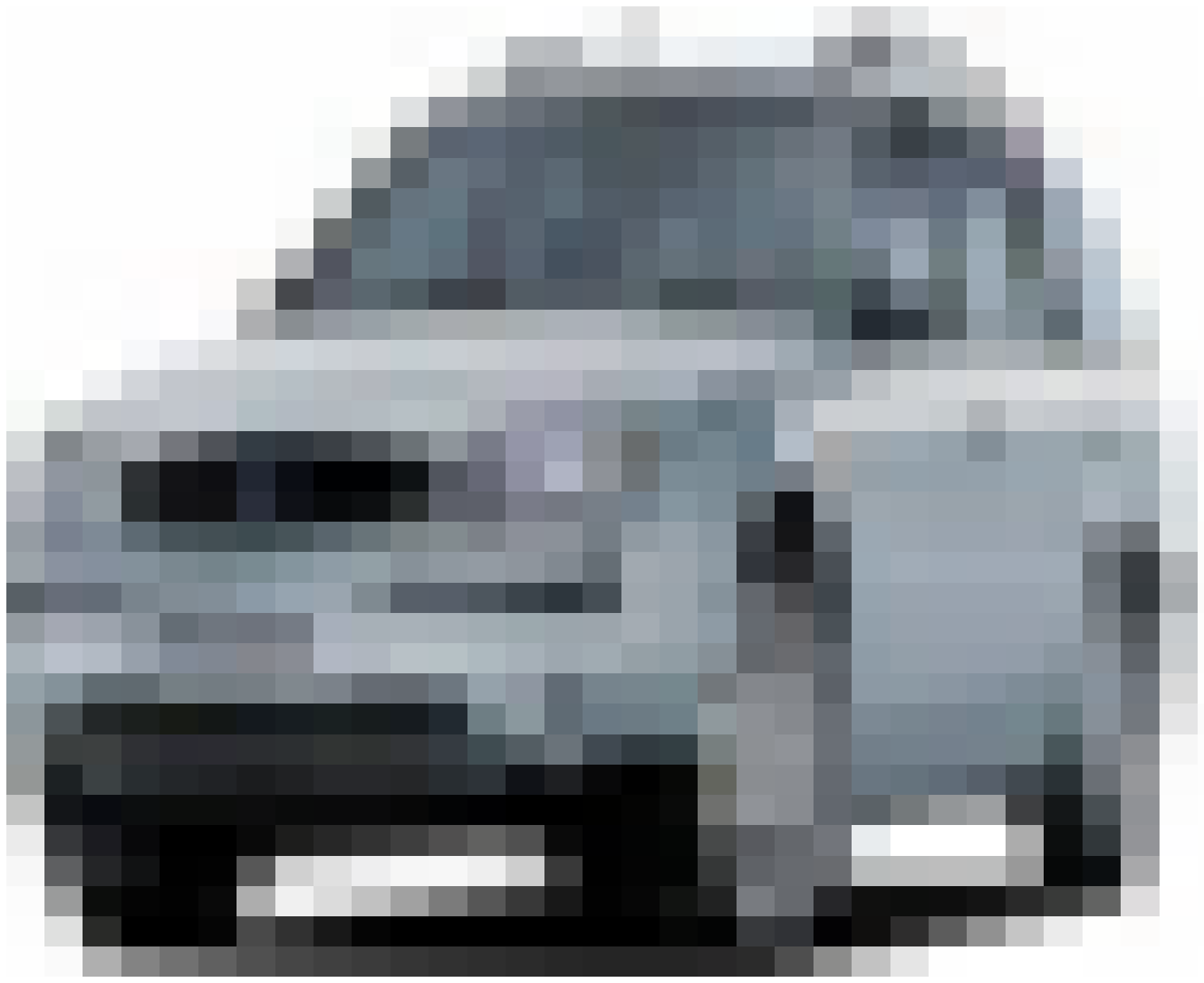}&\includegraphics[width=.1\linewidth]{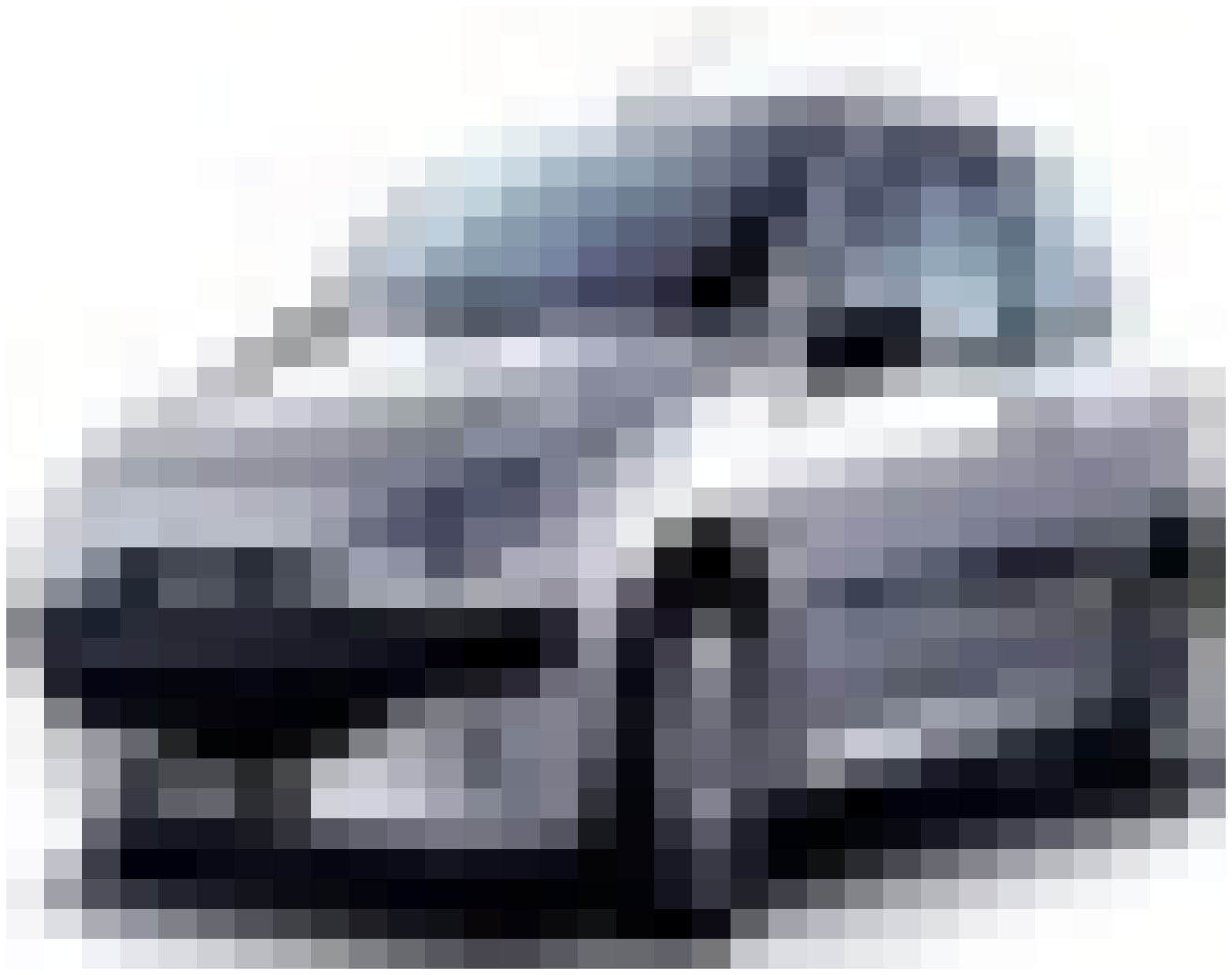}&\includegraphics[width=.1\linewidth]{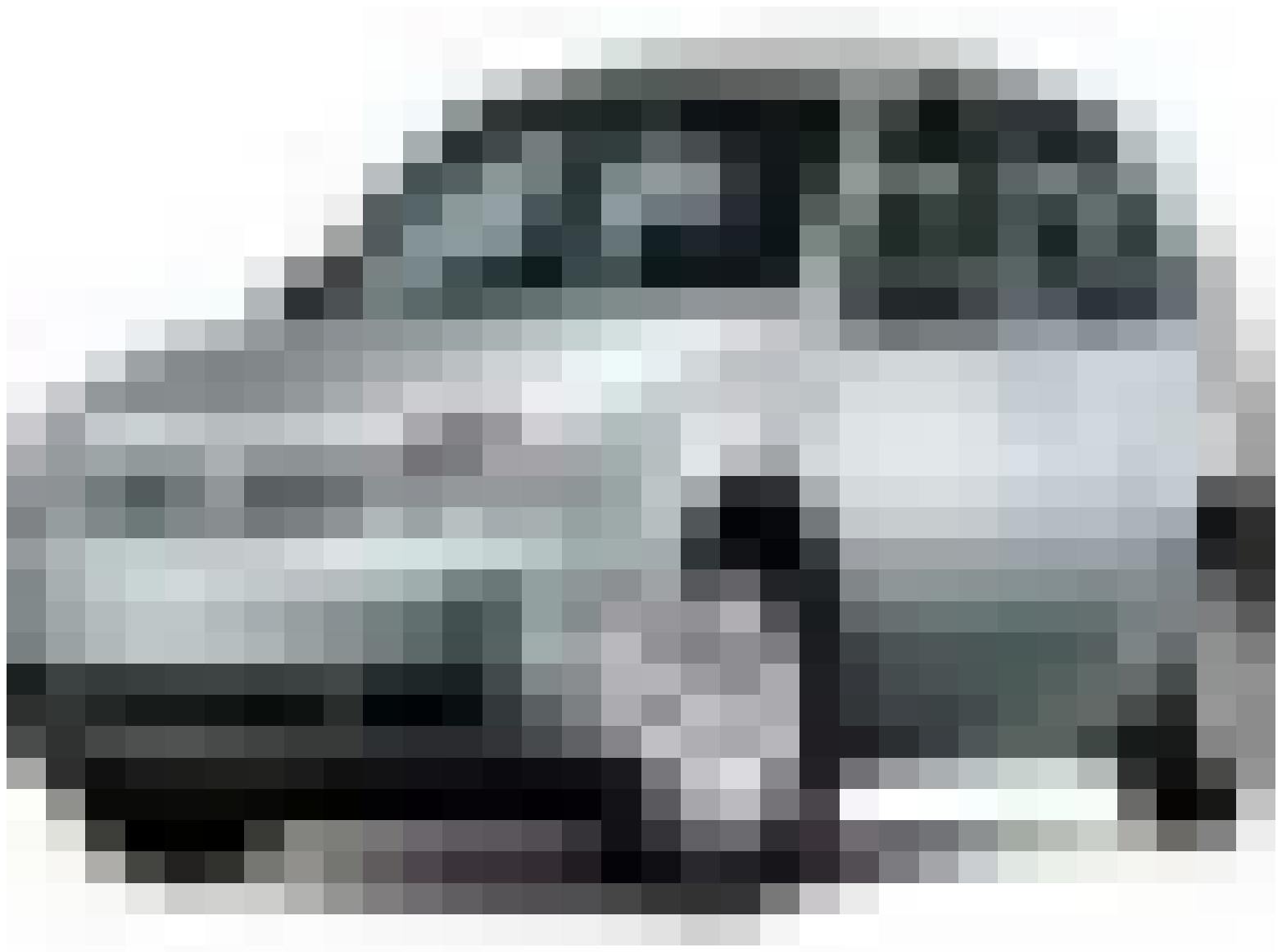}&\includegraphics[width=.1\linewidth]{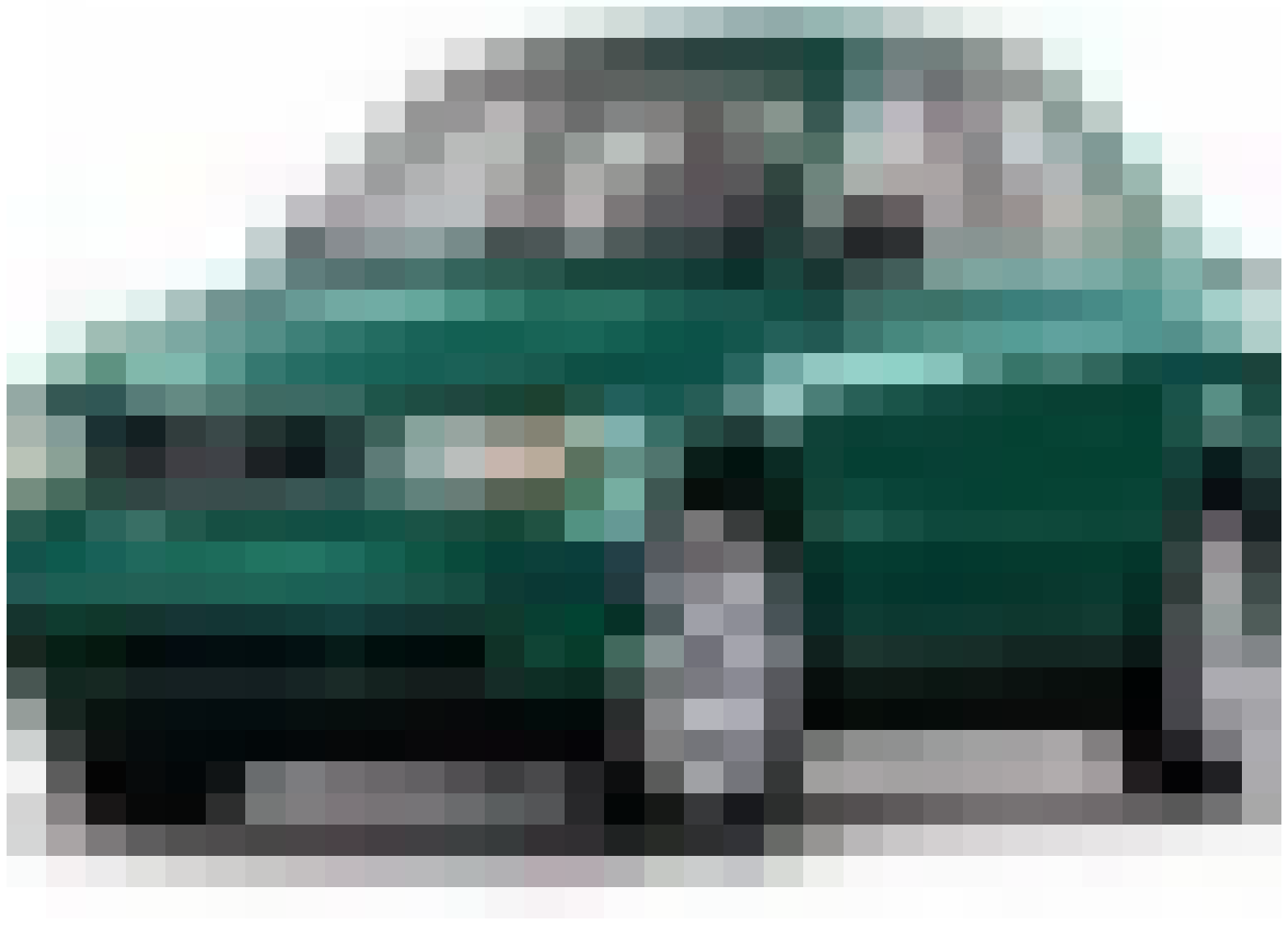}&\includegraphics[width=.1\linewidth]{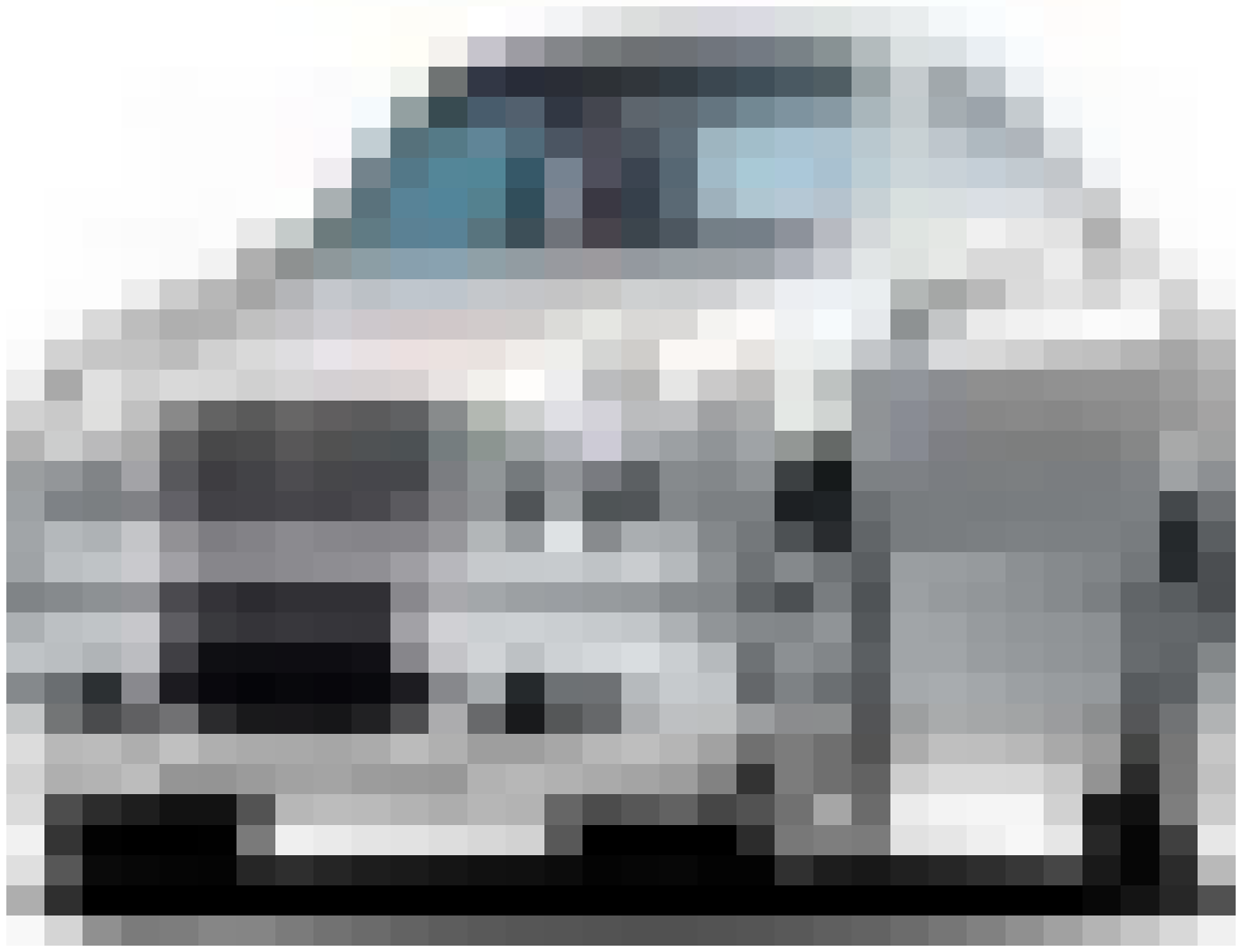}\\
    \includegraphics[width=.1\linewidth]{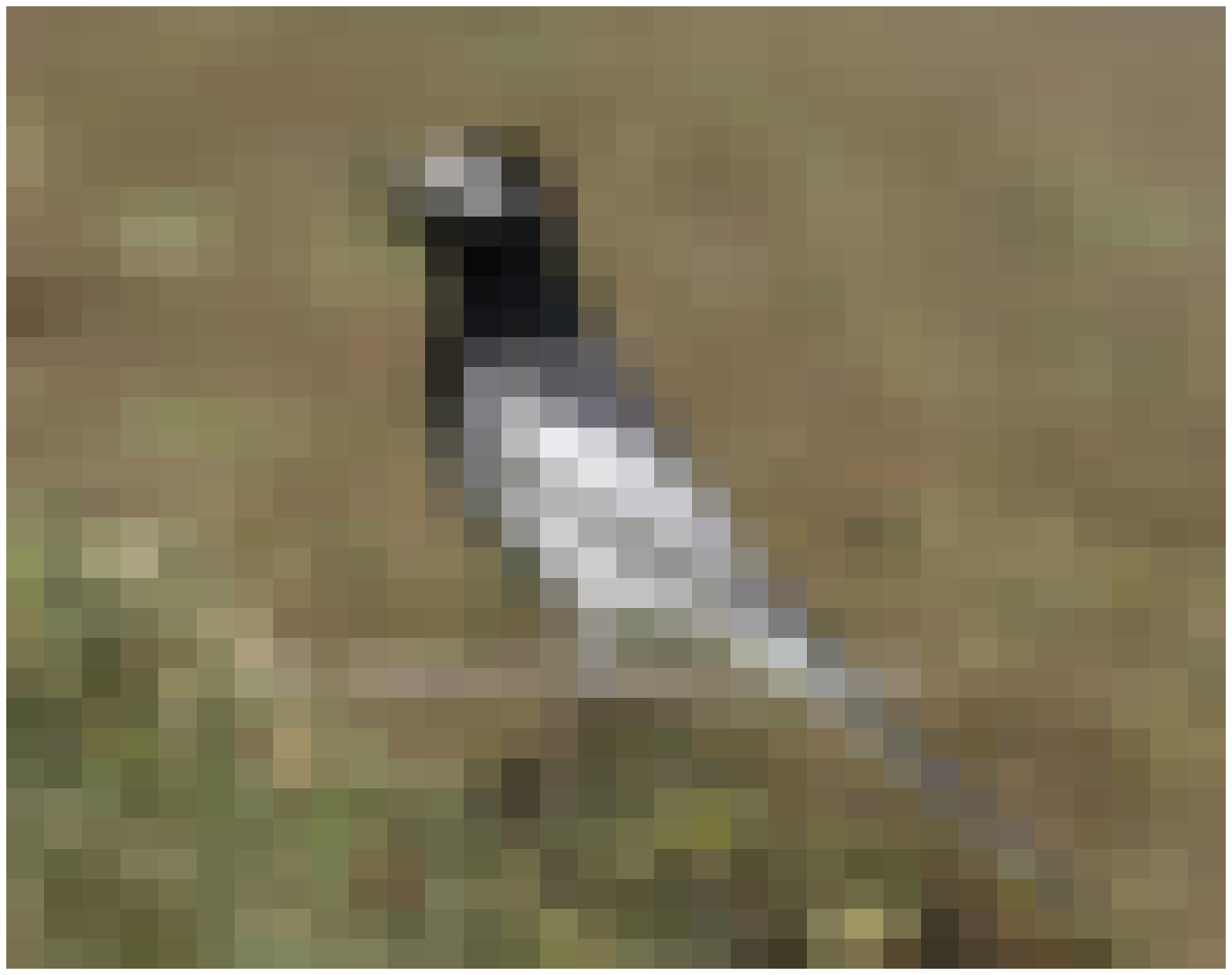}&\includegraphics[width=.1\linewidth]{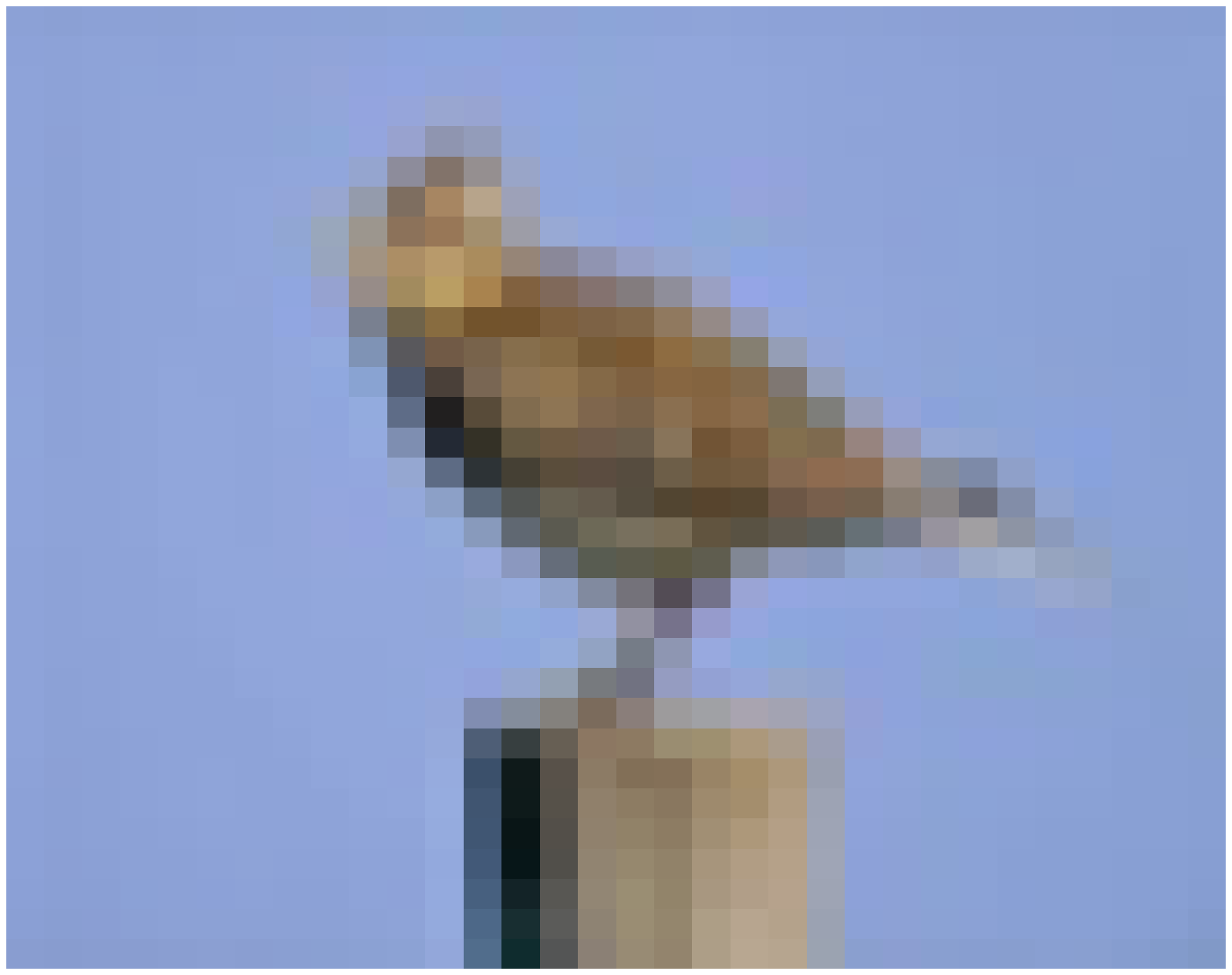}&\includegraphics[width=.1\linewidth]{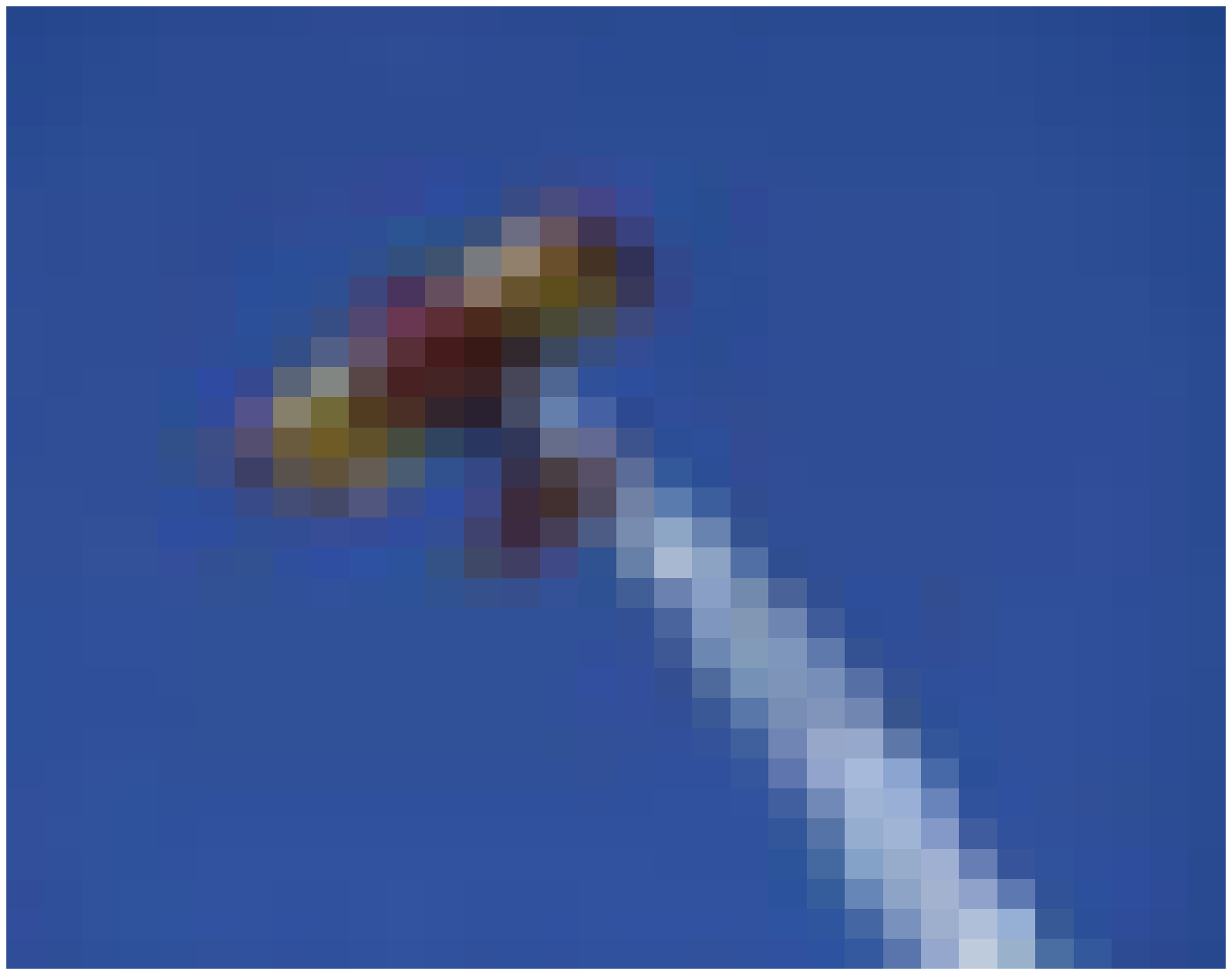}&\includegraphics[width=.1\linewidth]{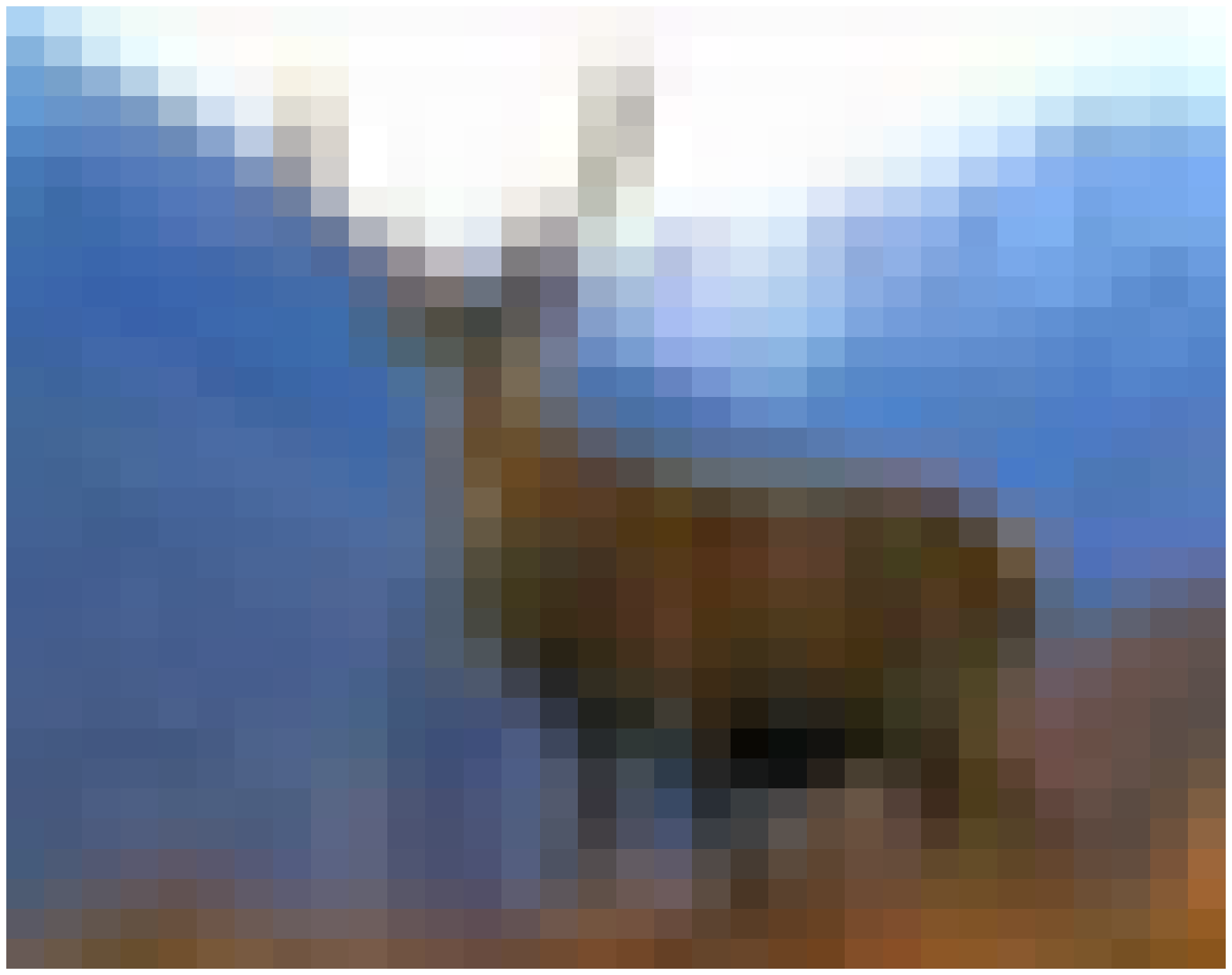}&\includegraphics[width=.1\linewidth]{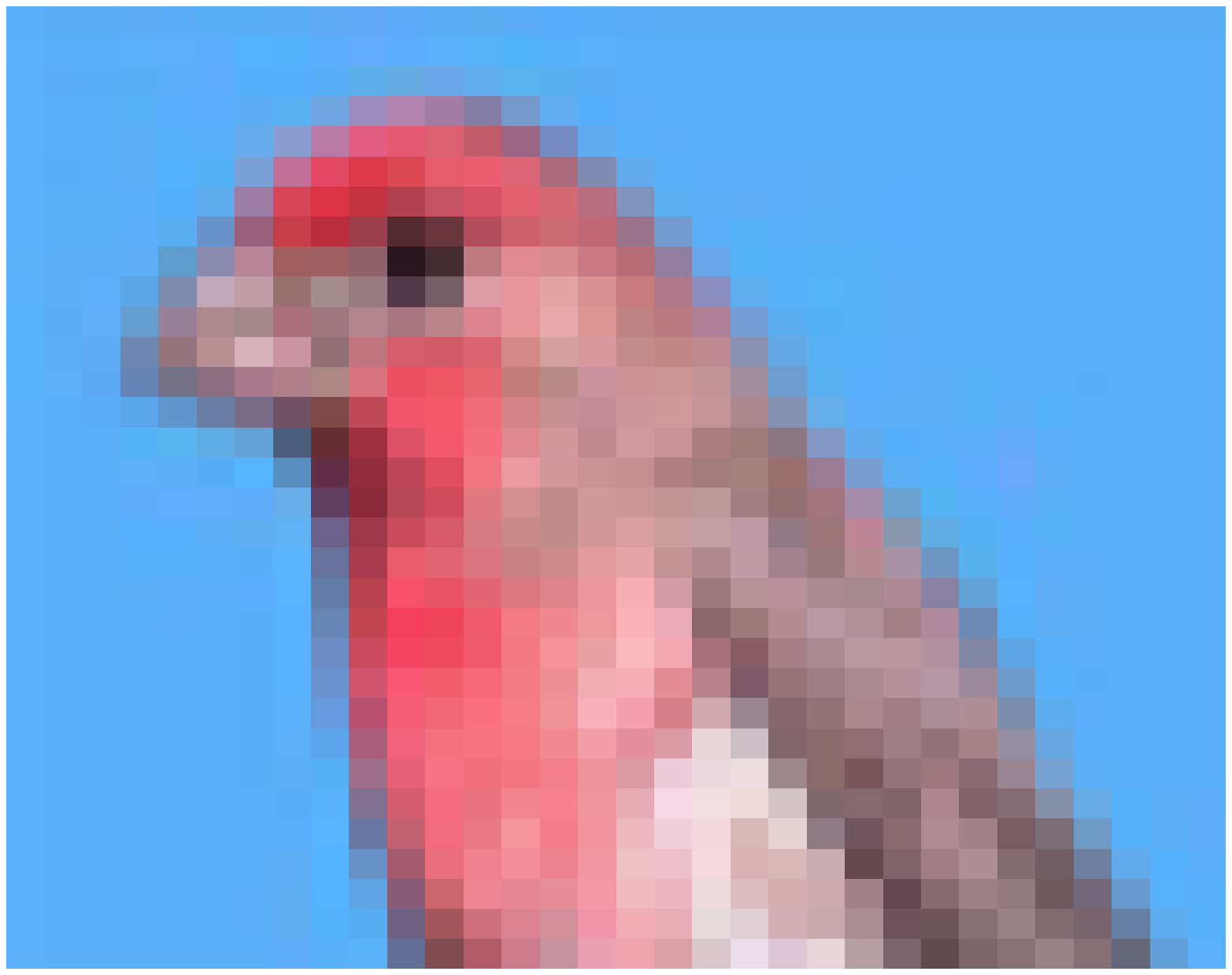}&\includegraphics[width=.1\linewidth]{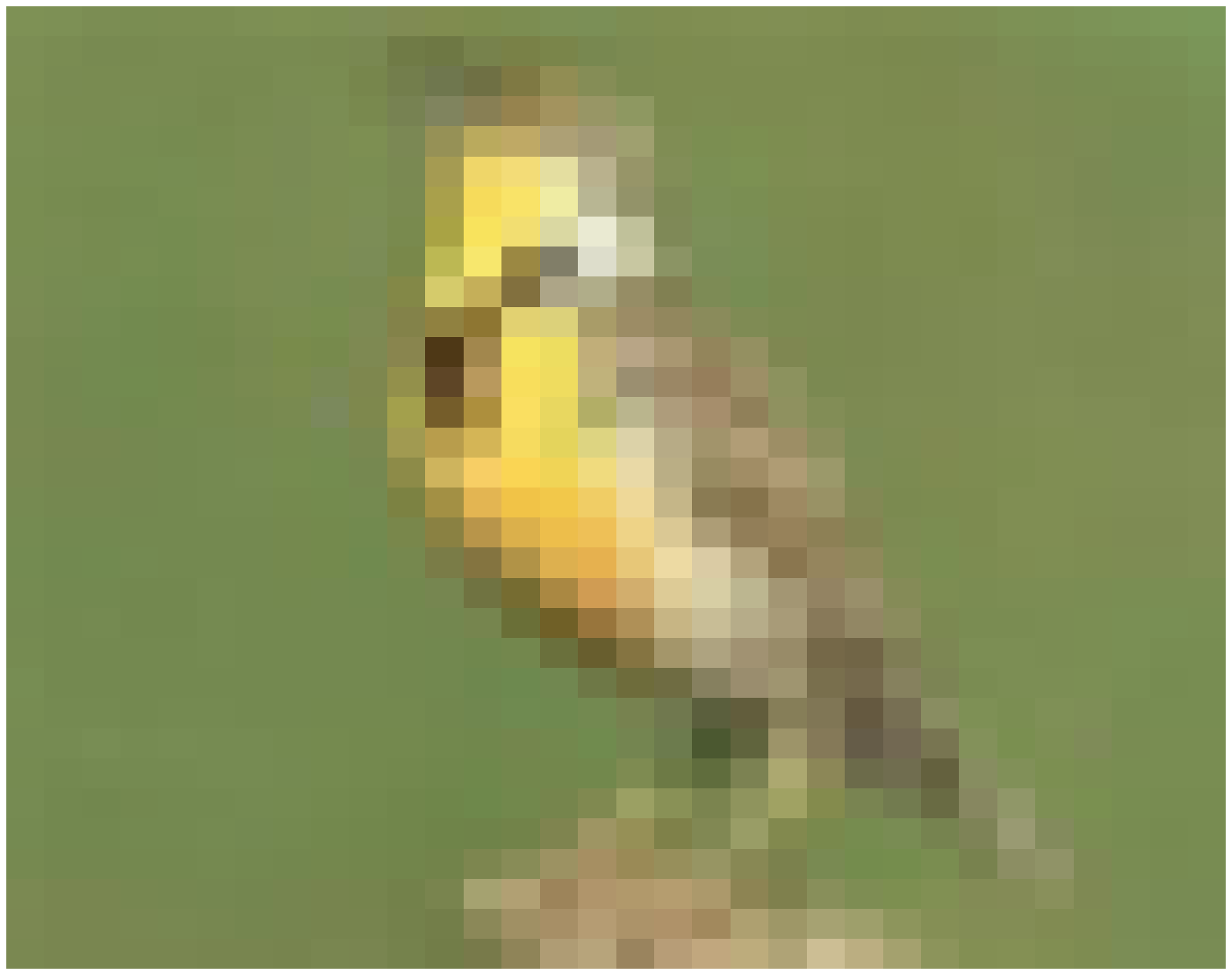}&\includegraphics[width=.1\linewidth]{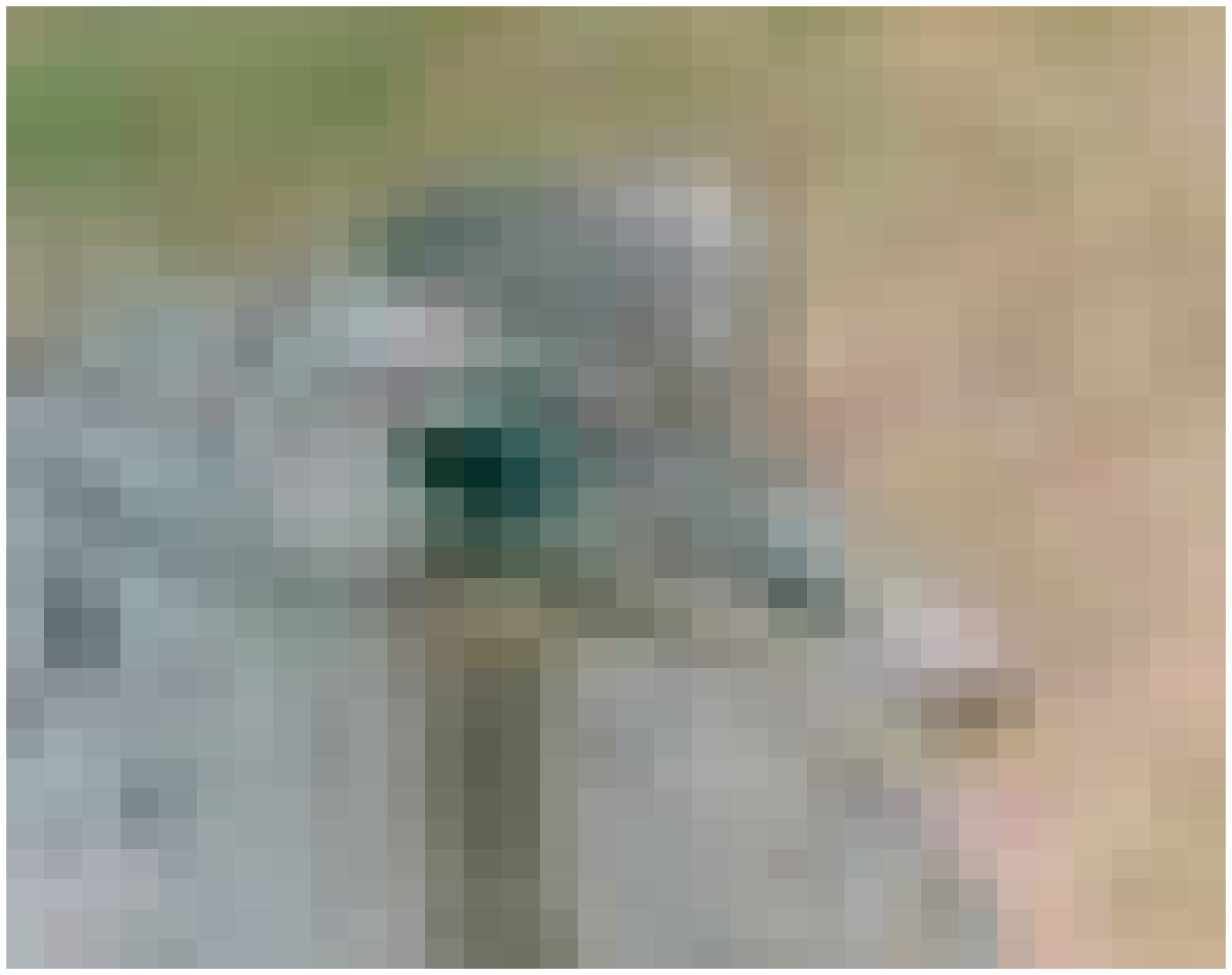}&\includegraphics[width=.1\linewidth]{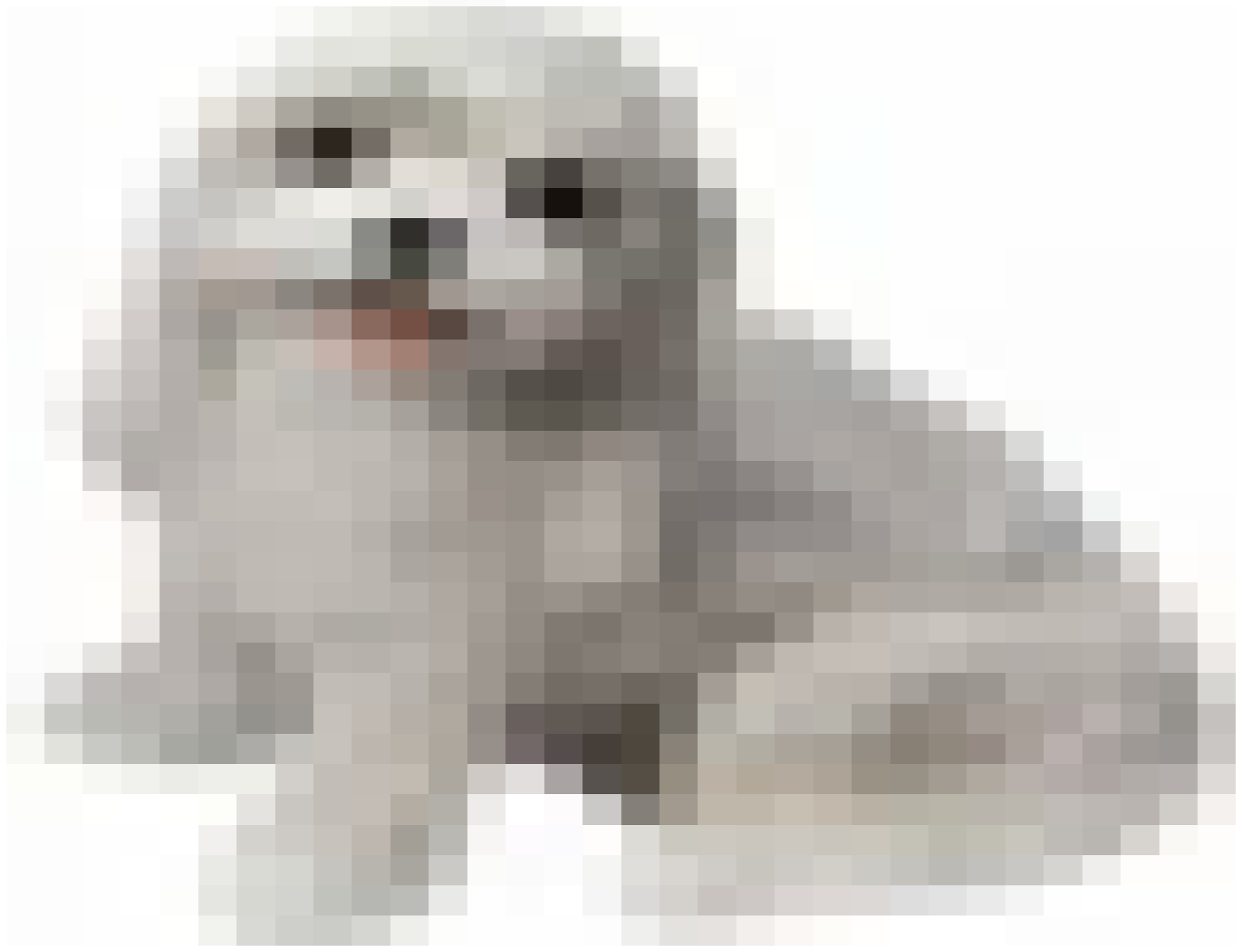}&\includegraphics[width=.1\linewidth]{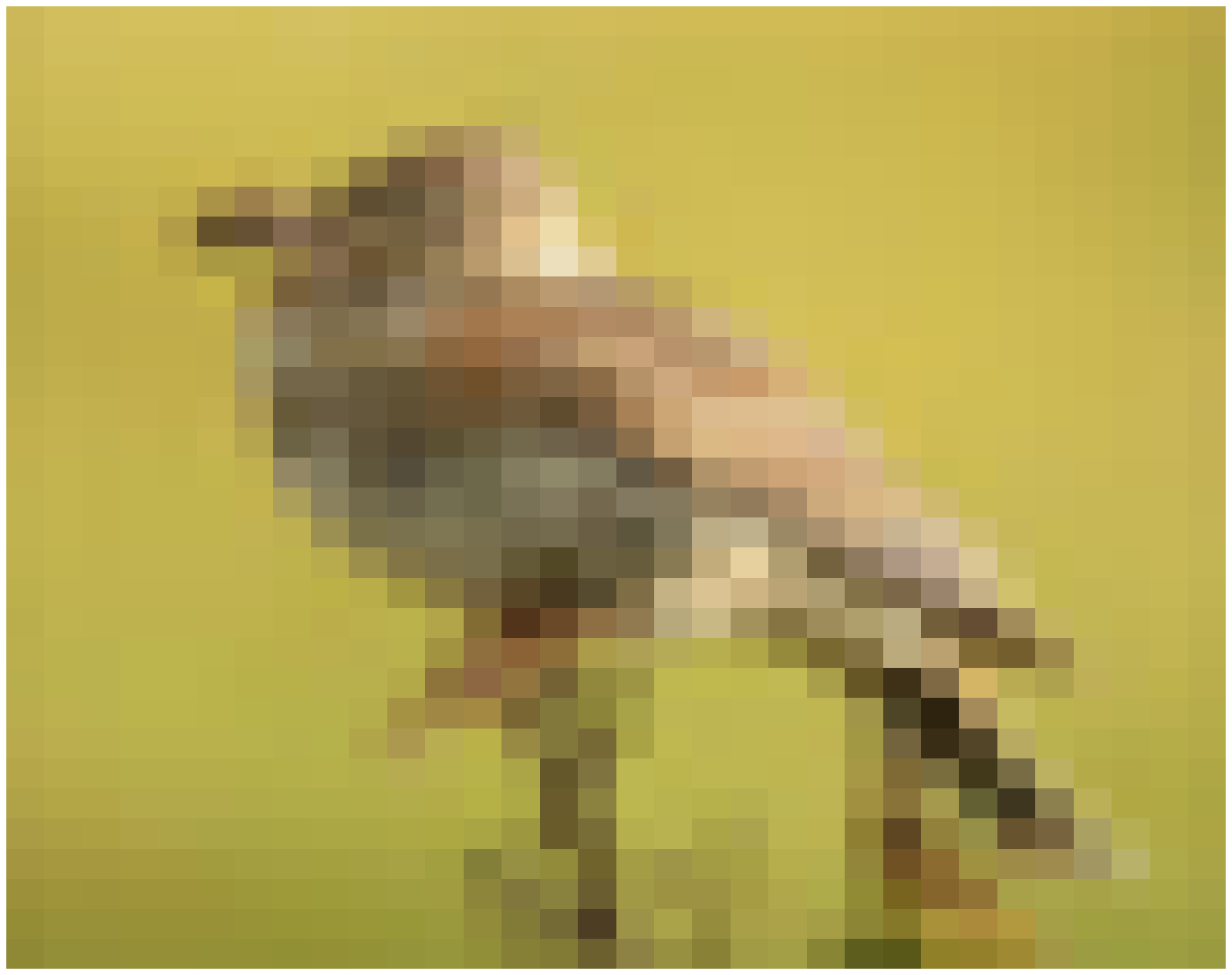}&\includegraphics[width=.1\linewidth]{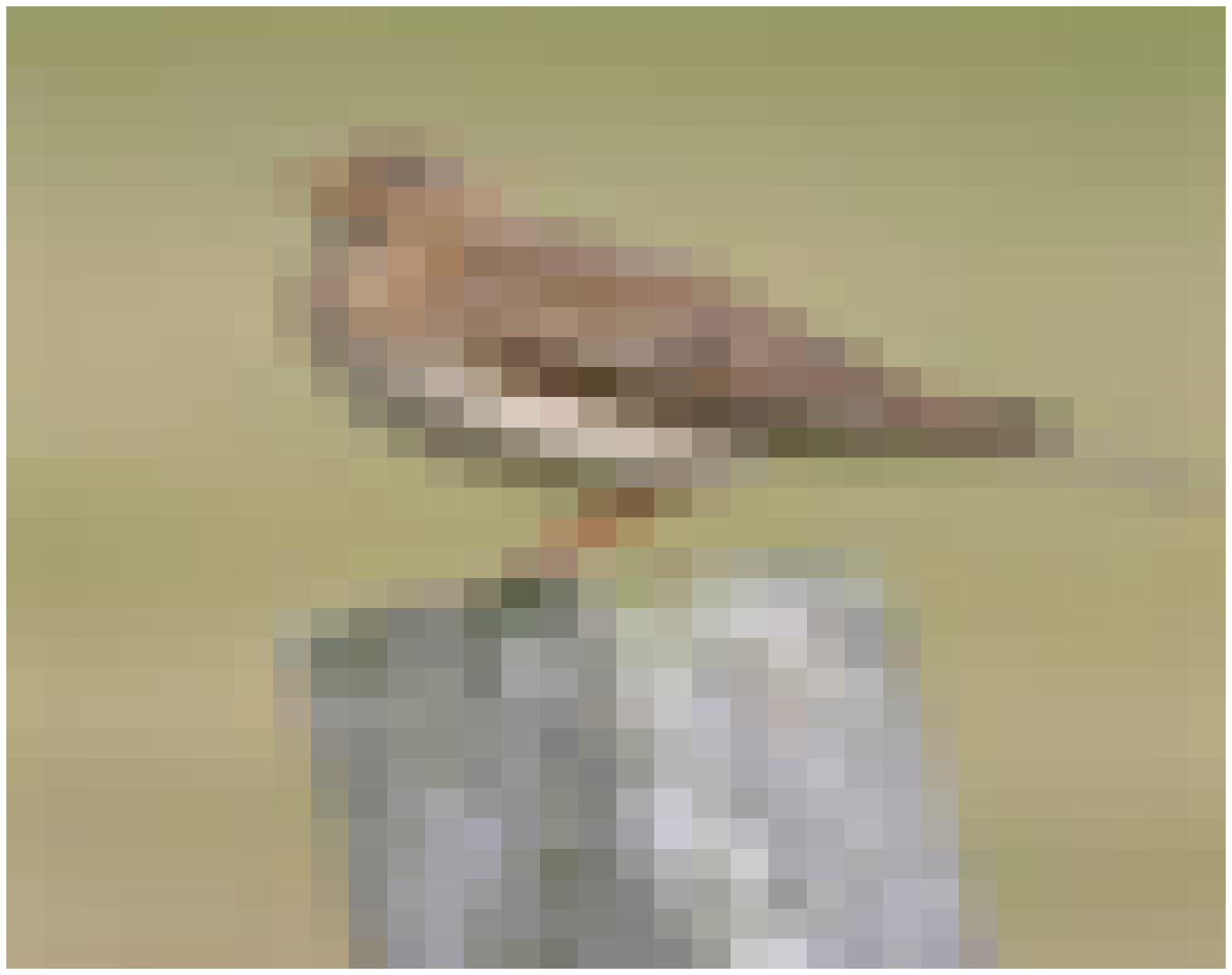}\\
    \includegraphics[width=.1\linewidth]{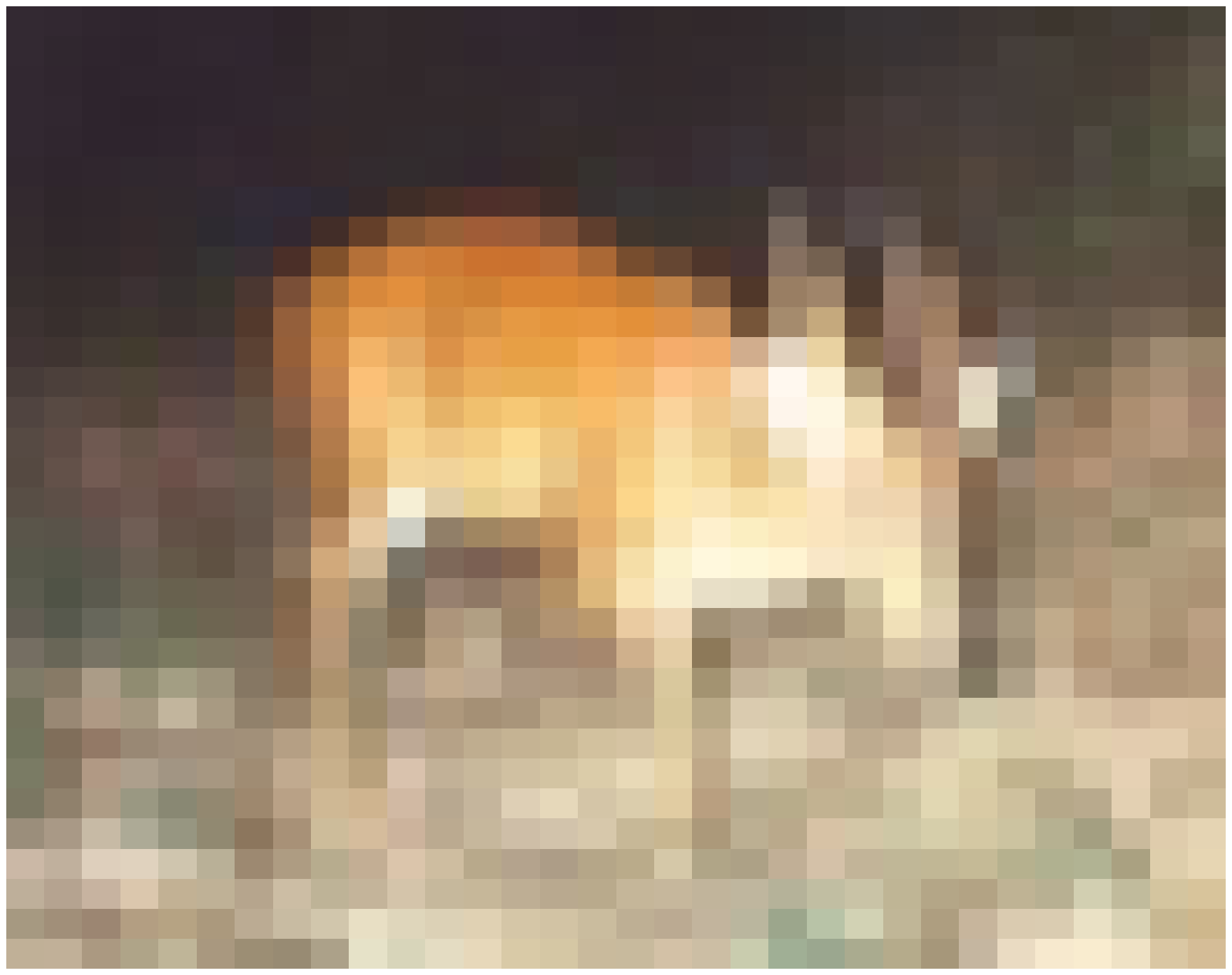}&\includegraphics[width=.1\linewidth]{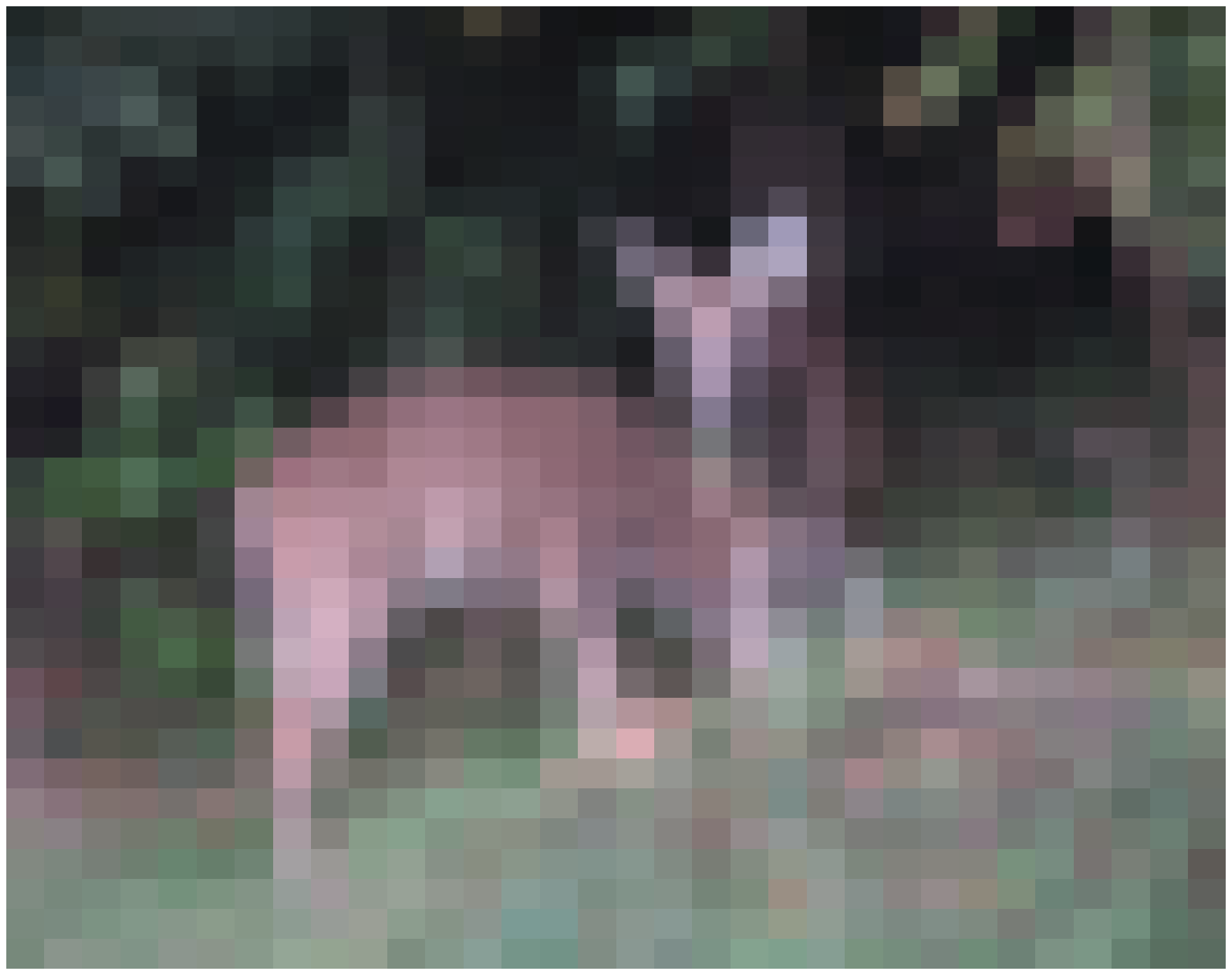}&\includegraphics[width=.1\linewidth]{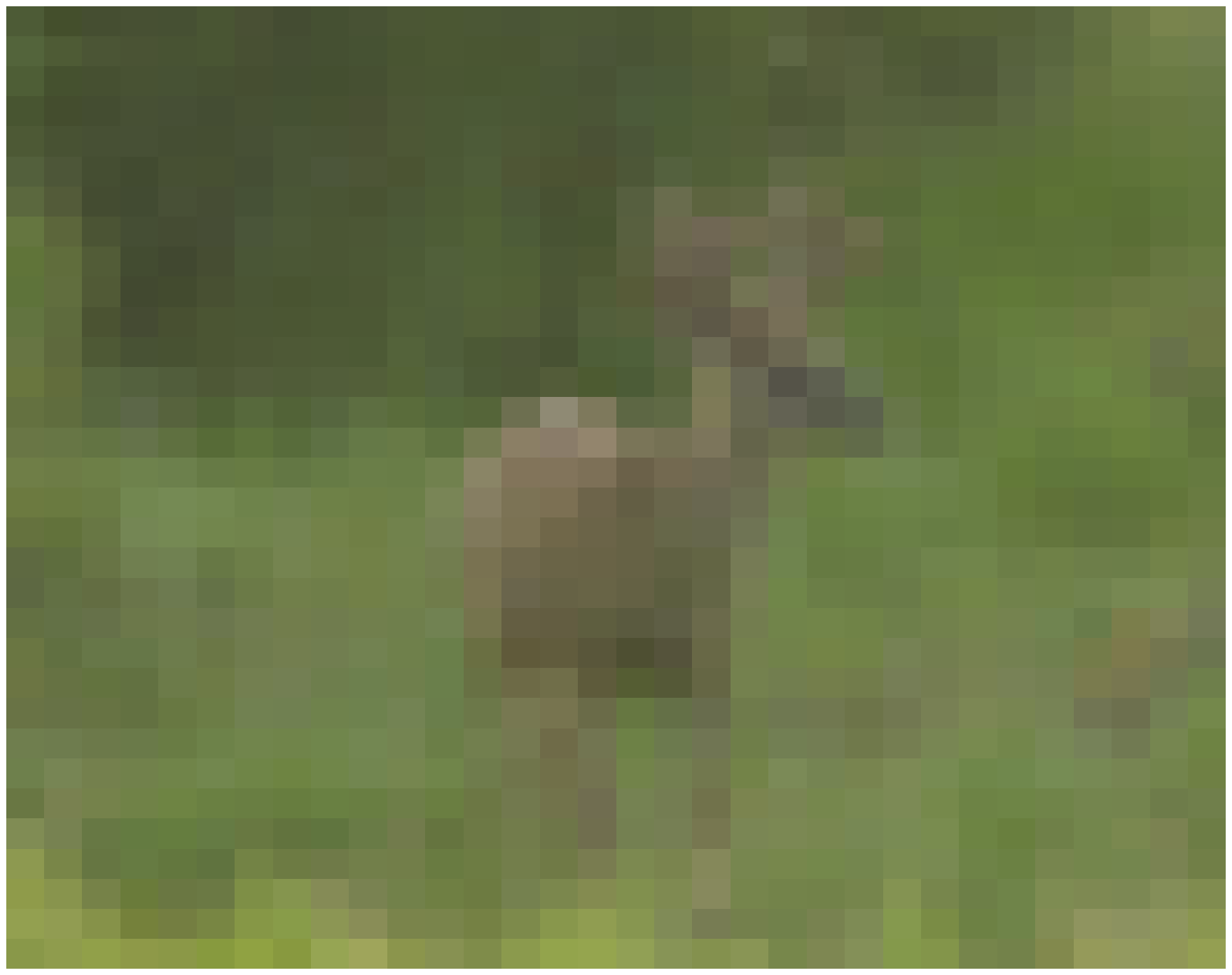}&\includegraphics[width=.1\linewidth]{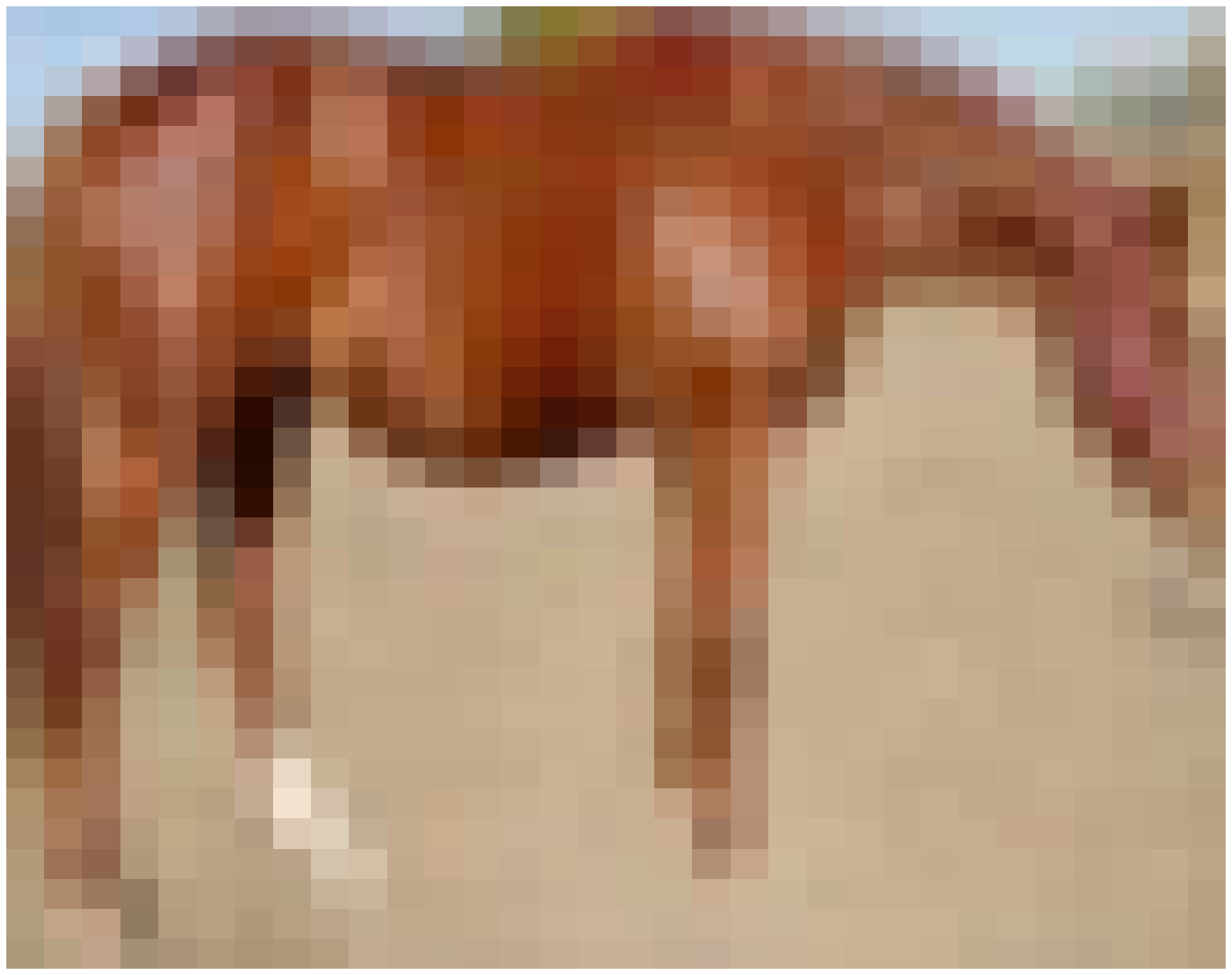}&\includegraphics[width=.1\linewidth]{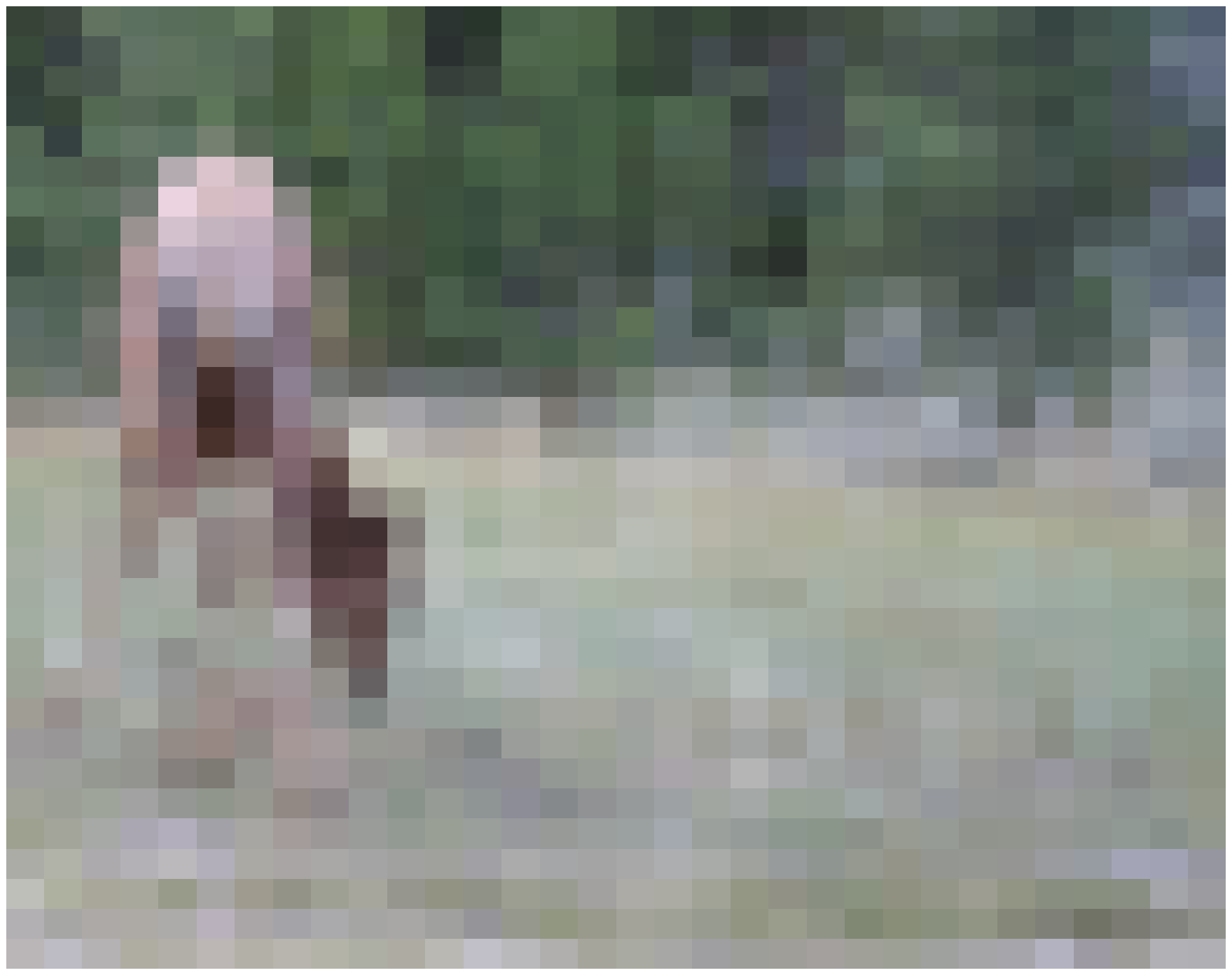}&\includegraphics[width=.1\linewidth]{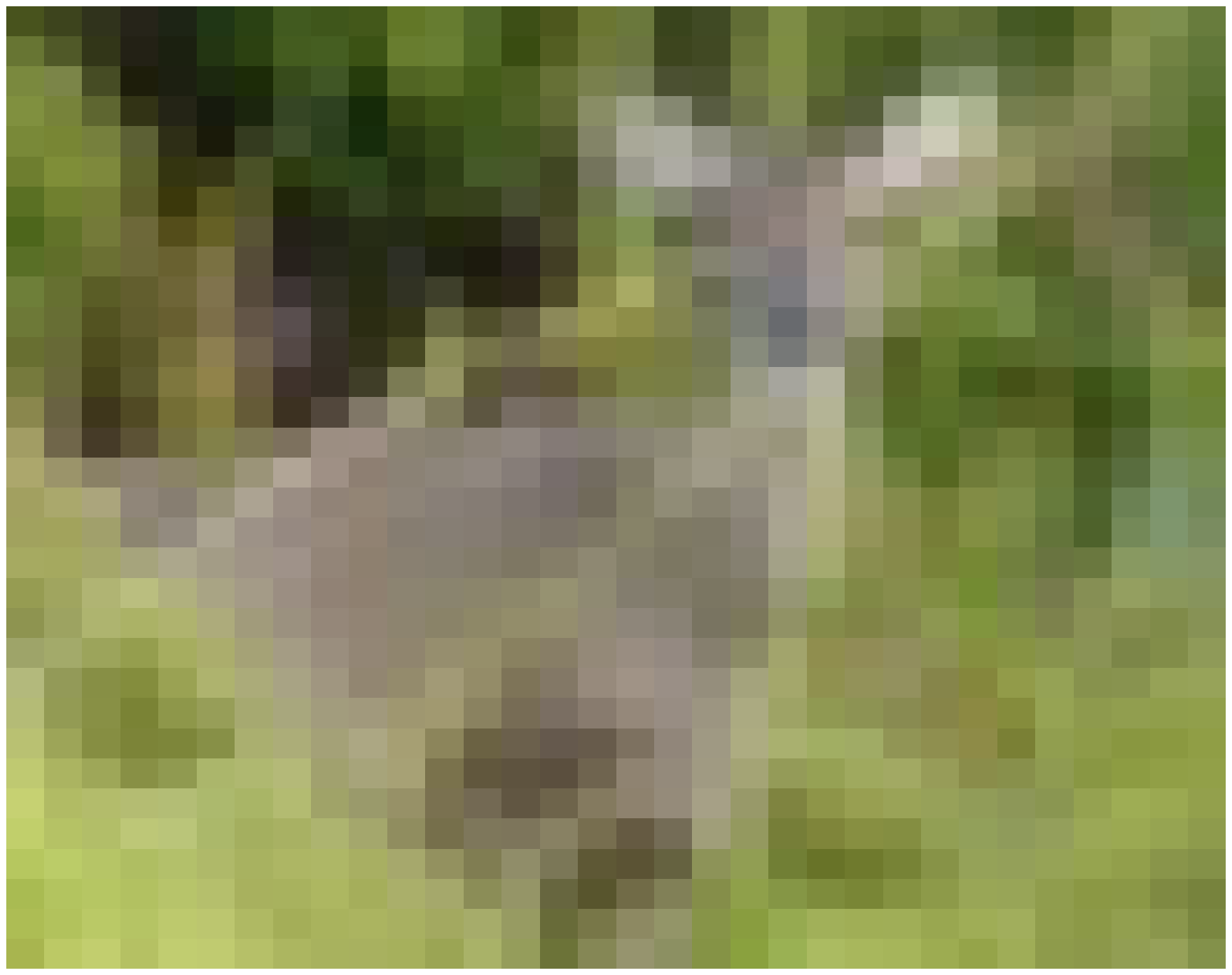}&\includegraphics[width=.1\linewidth]{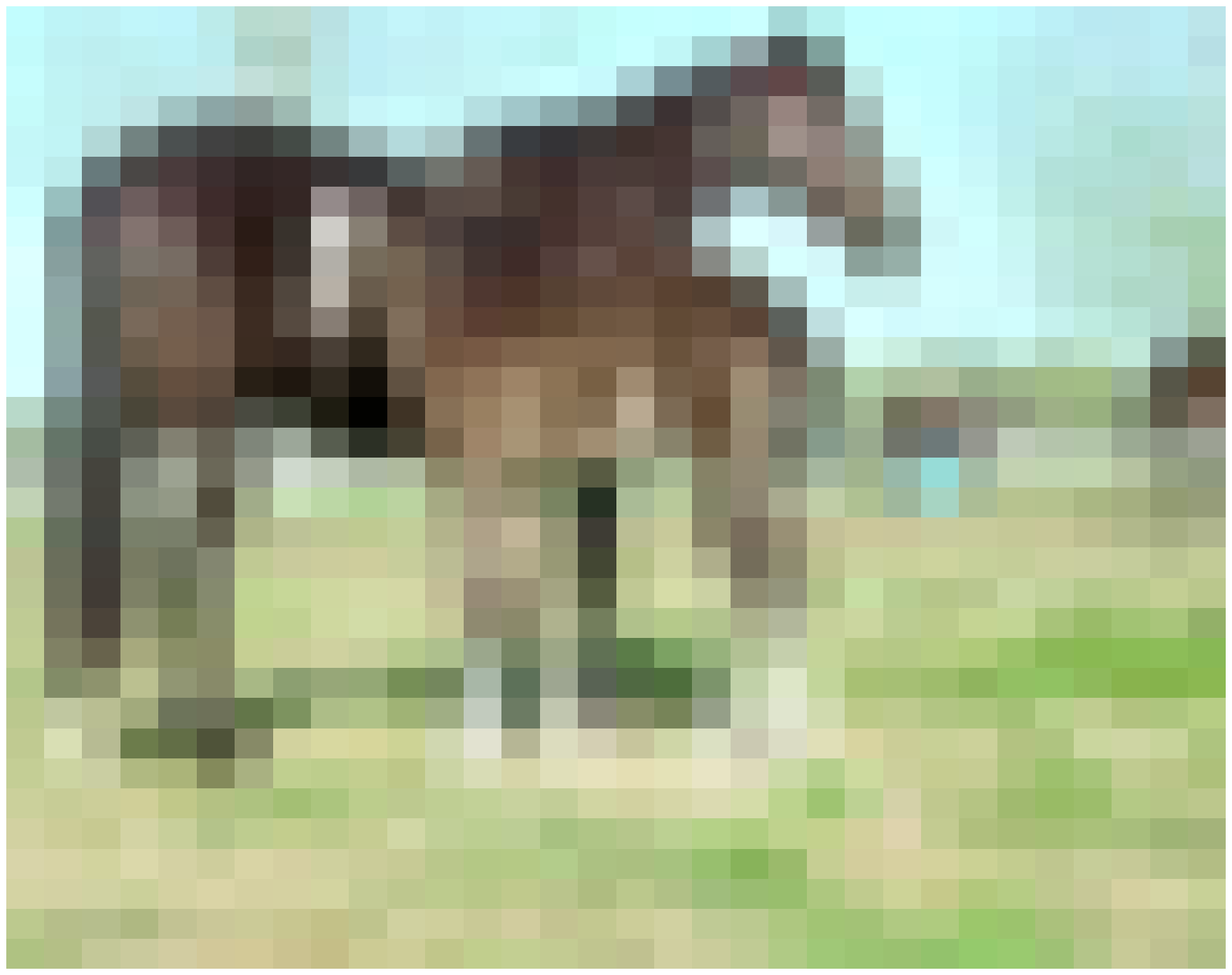}&\includegraphics[width=.1\linewidth]{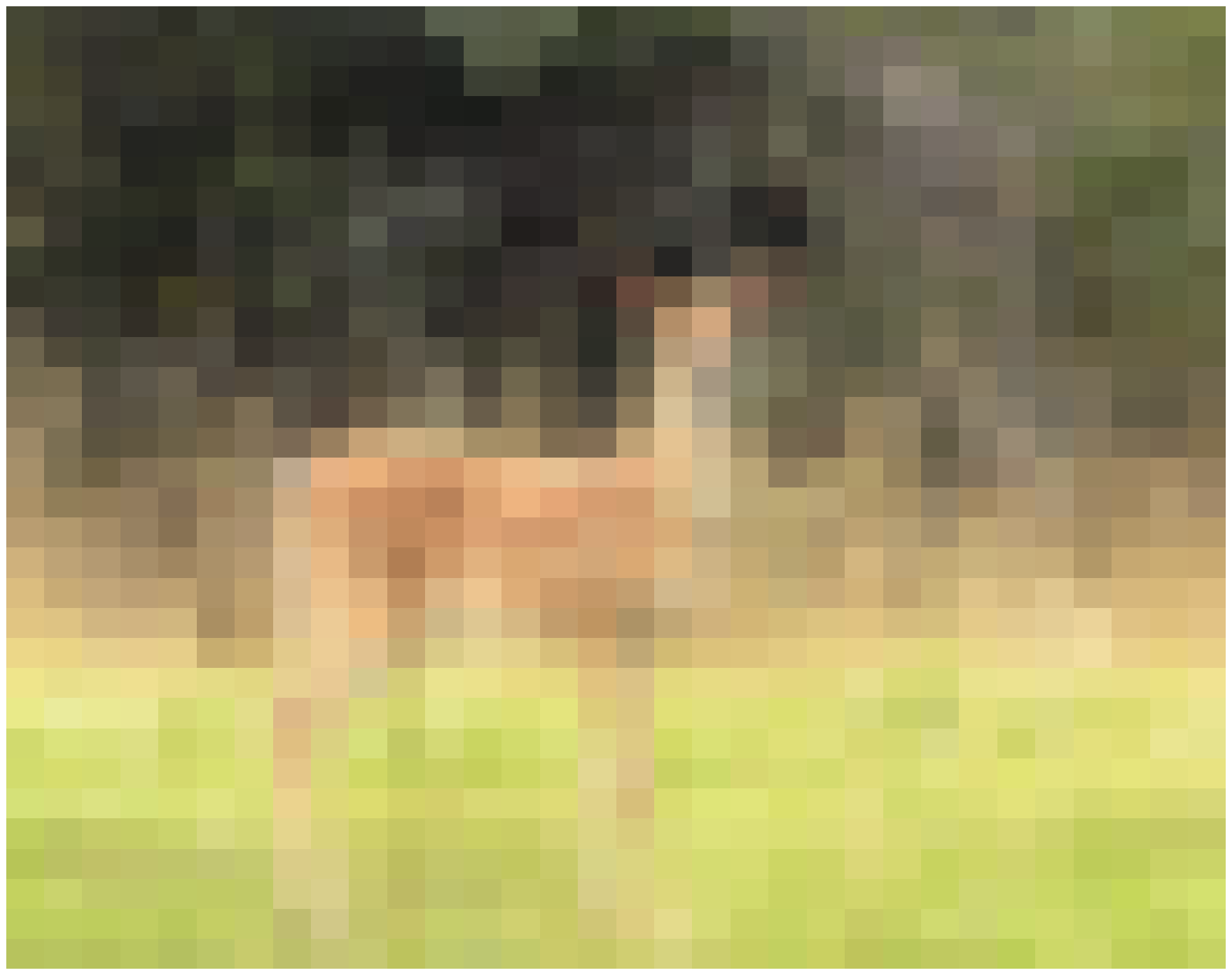}&\includegraphics[width=.1\linewidth]{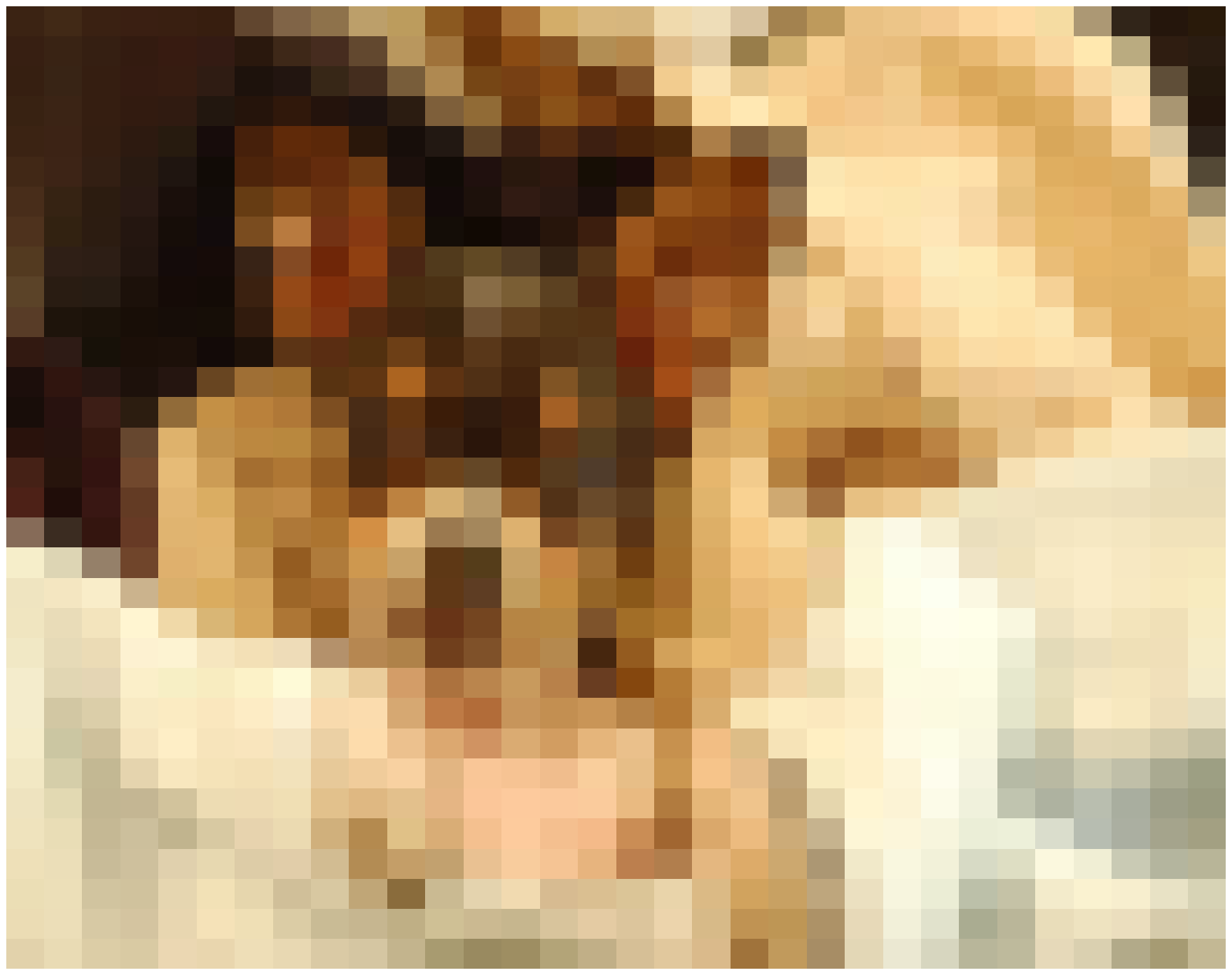}&\includegraphics[width=.1\linewidth]{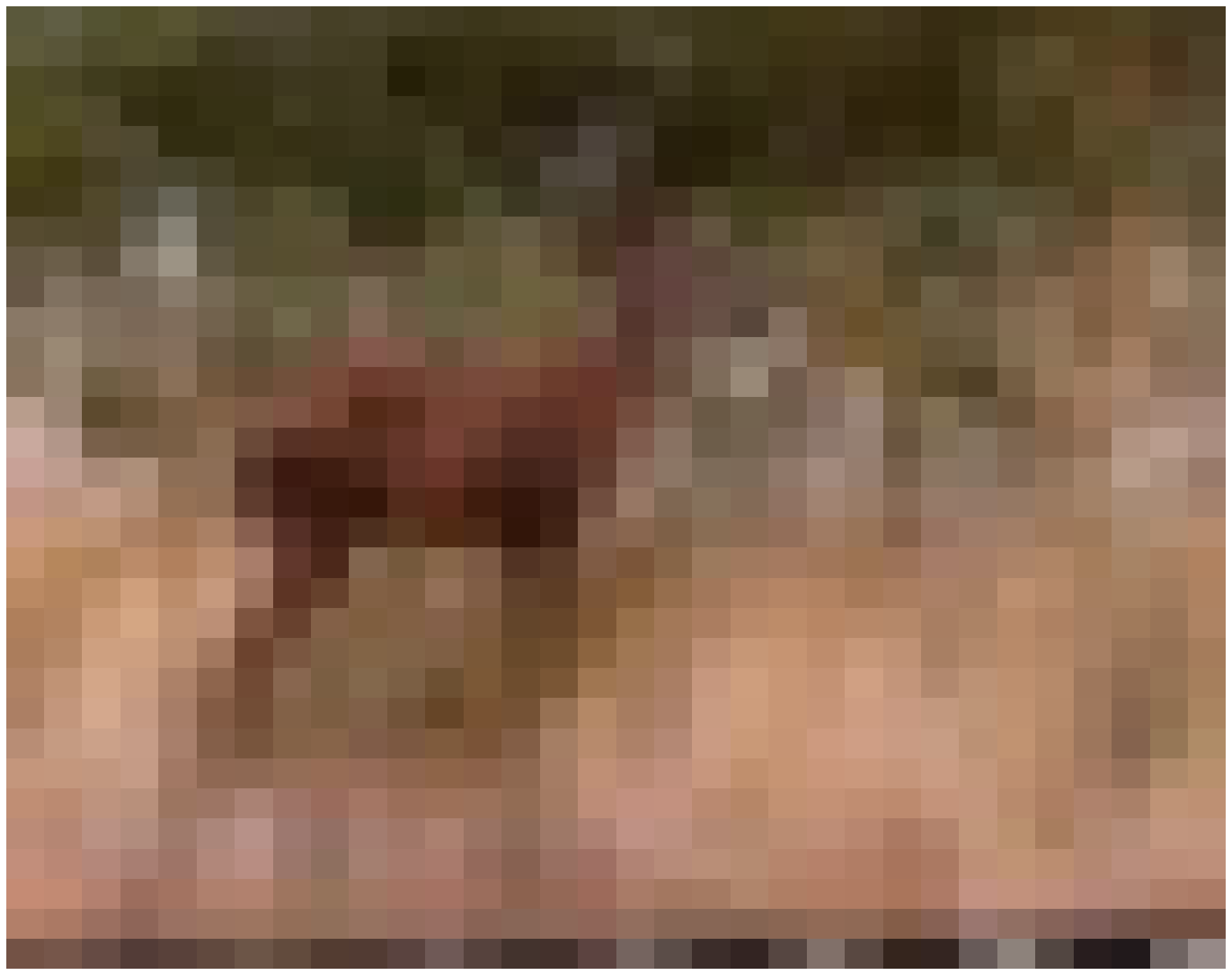}\\
    \includegraphics[width=.1\linewidth]{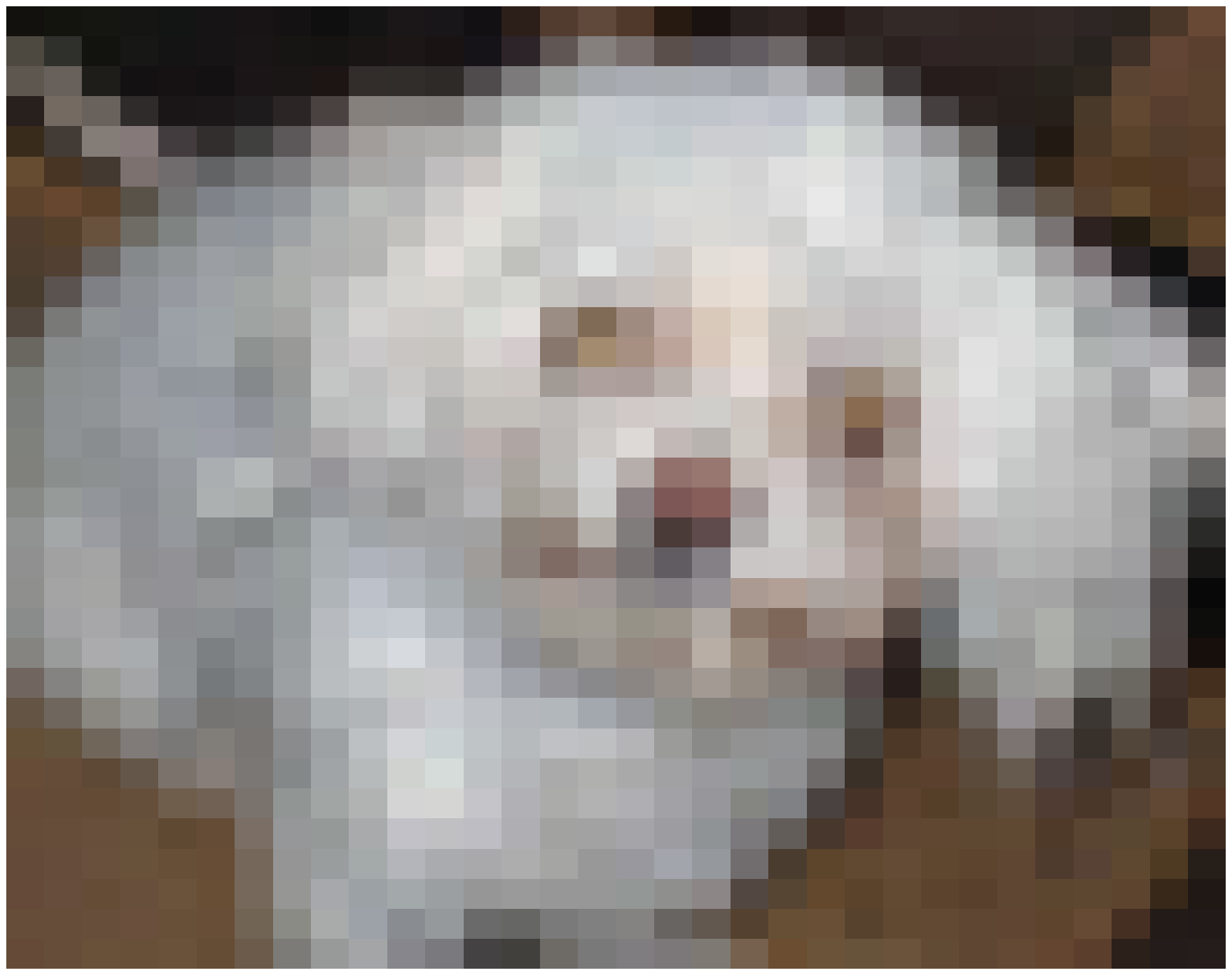}&\includegraphics[width=.1\linewidth]{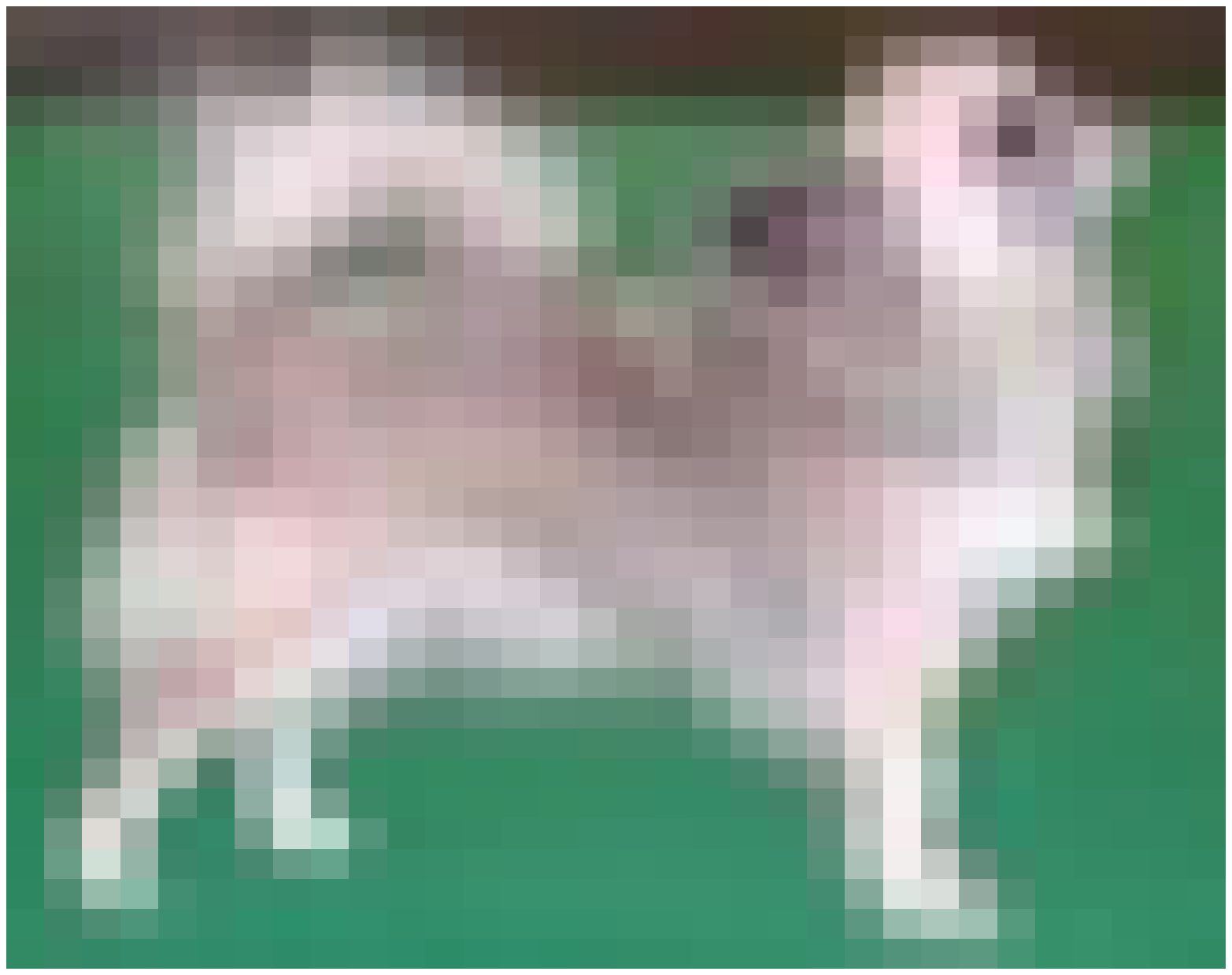}&\includegraphics[width=.1\linewidth]{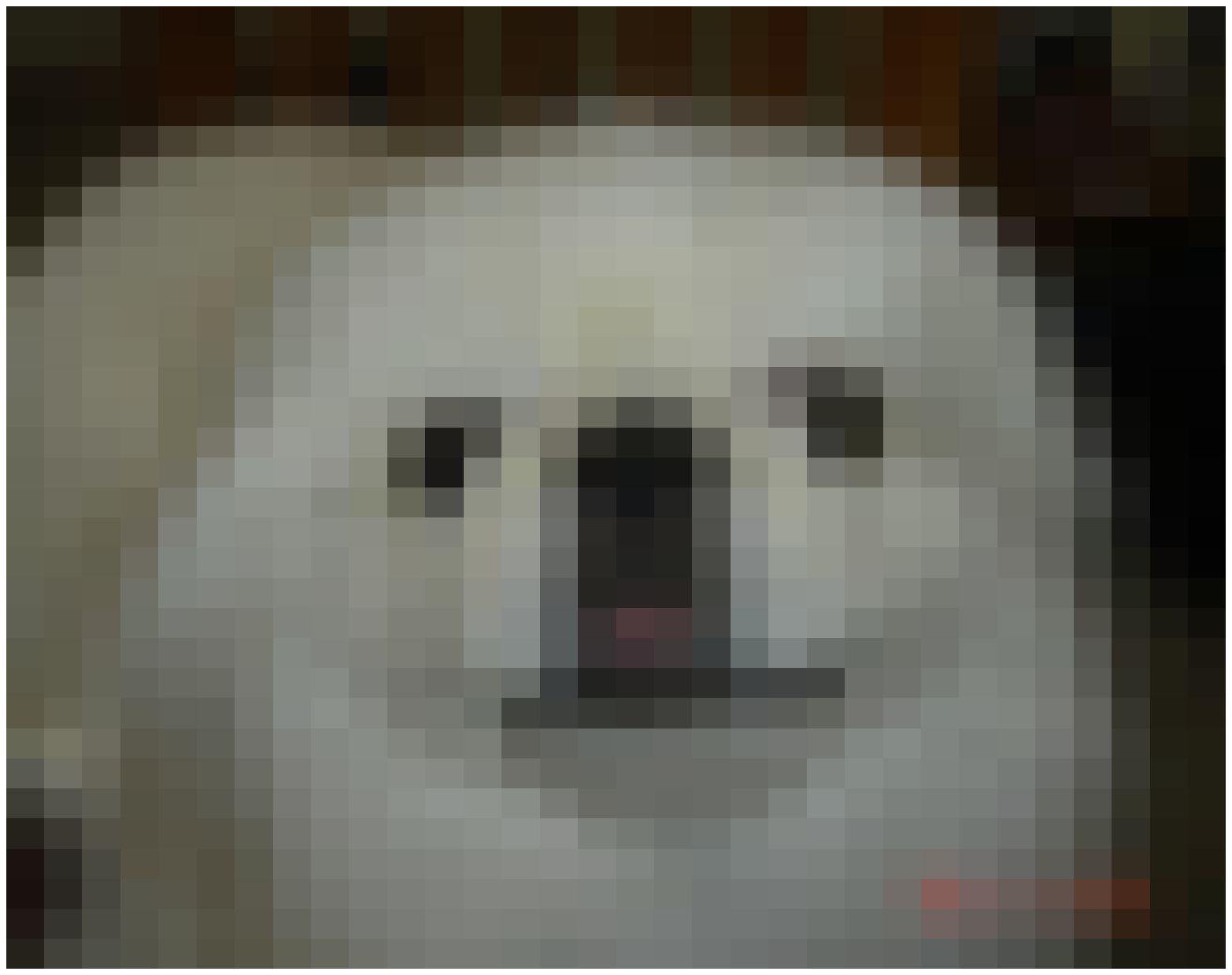}&\includegraphics[width=.1\linewidth]{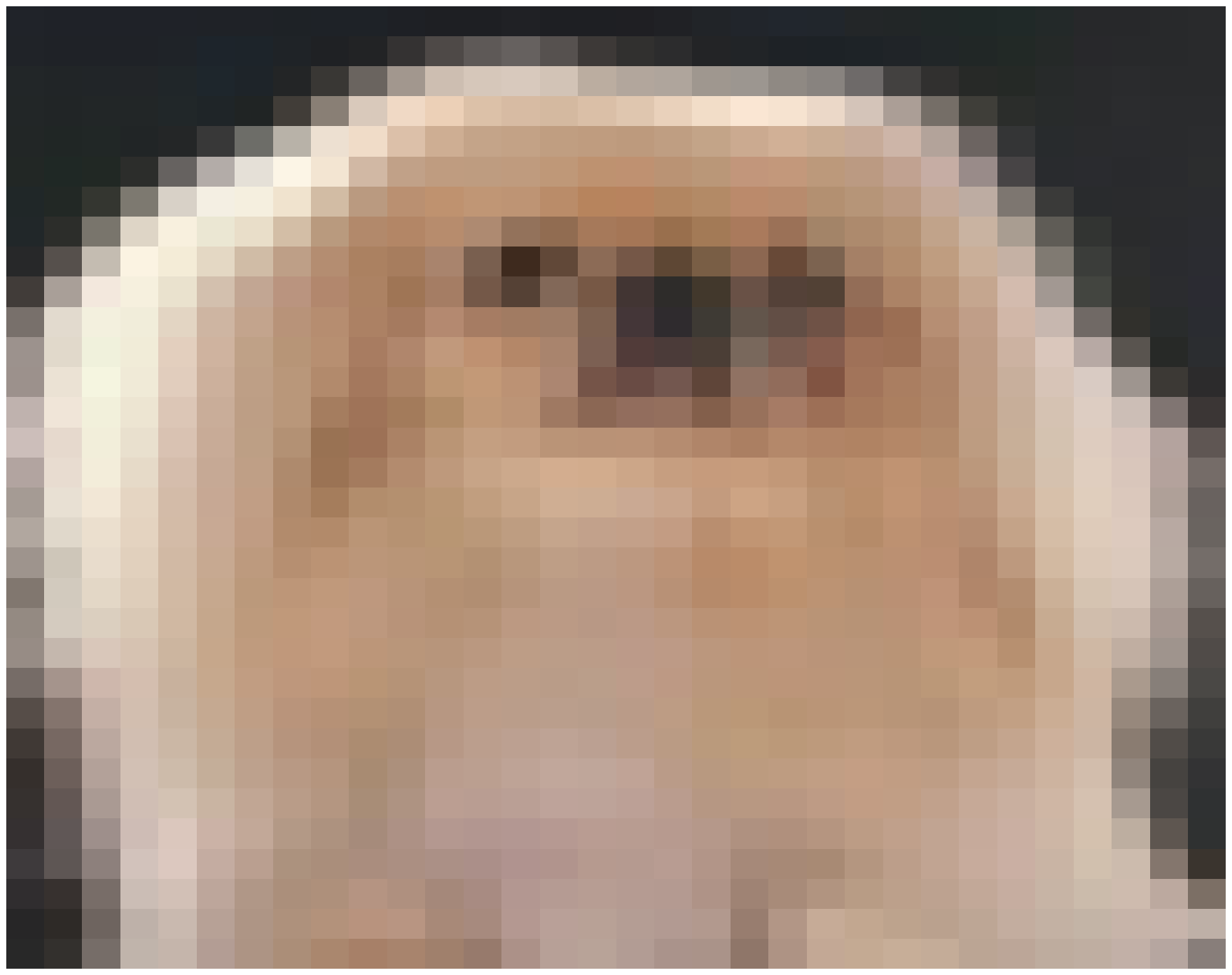}&\includegraphics[width=.1\linewidth]{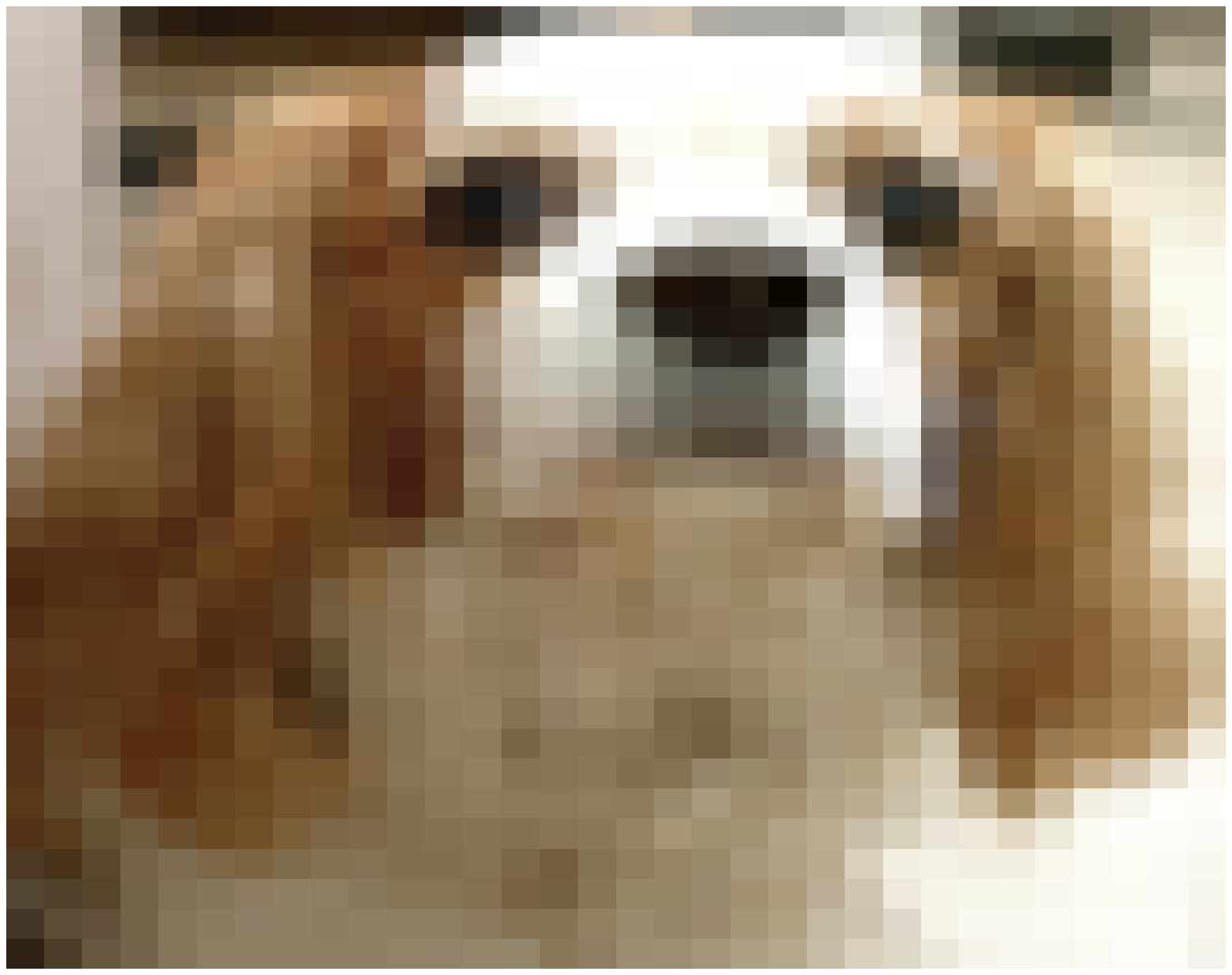}&\includegraphics[width=.1\linewidth]{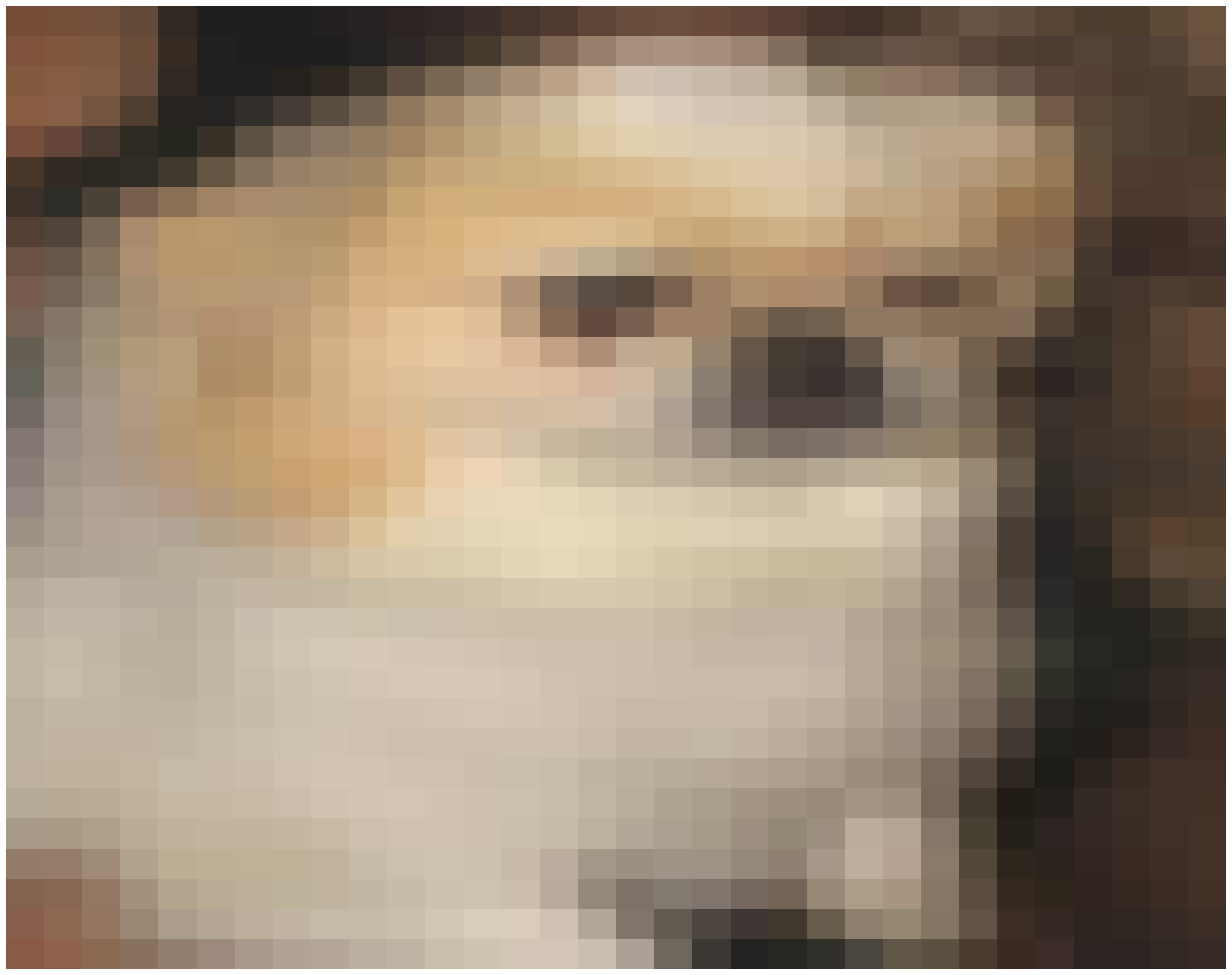}&\includegraphics[width=.1\linewidth]{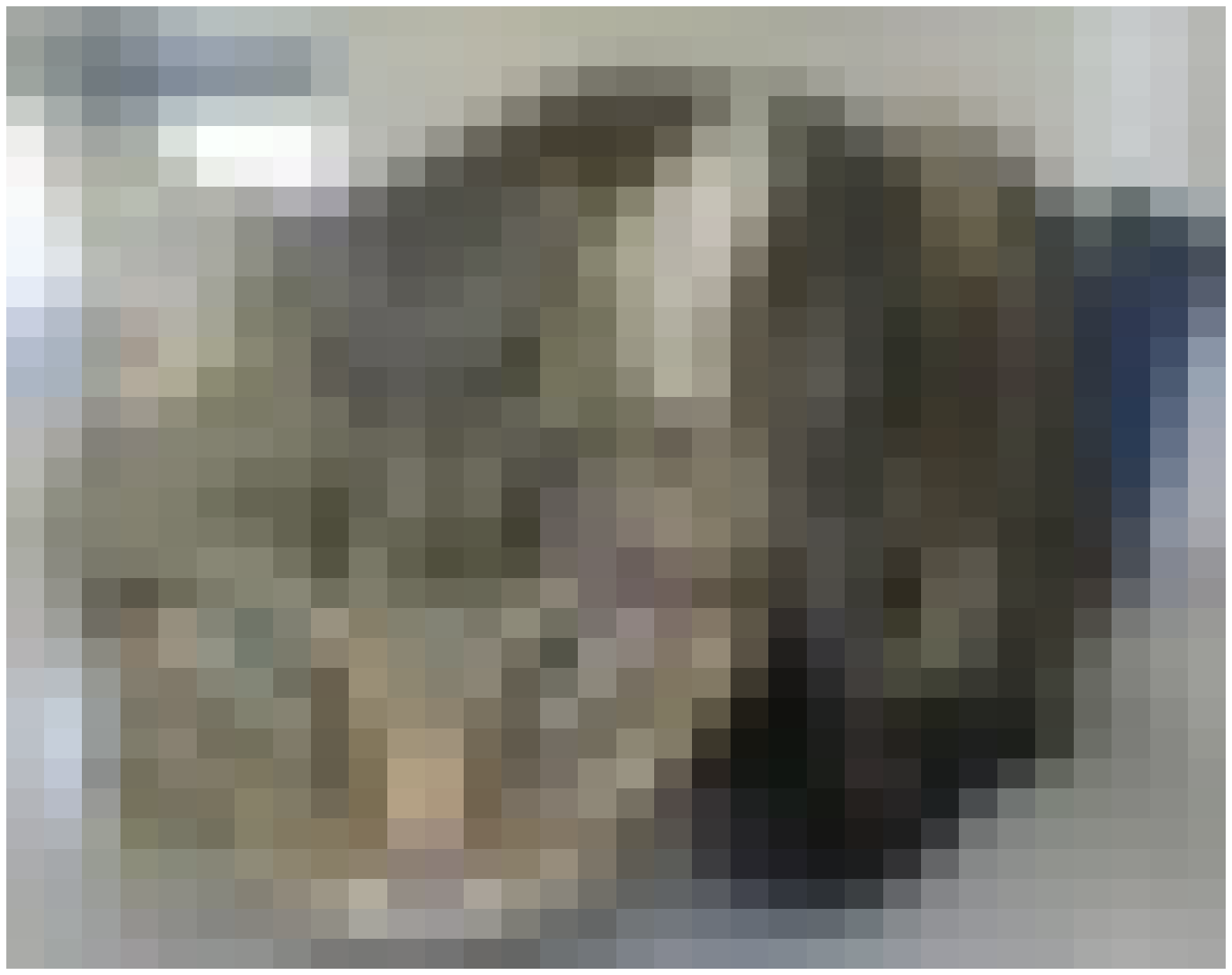}&\includegraphics[width=.1\linewidth]{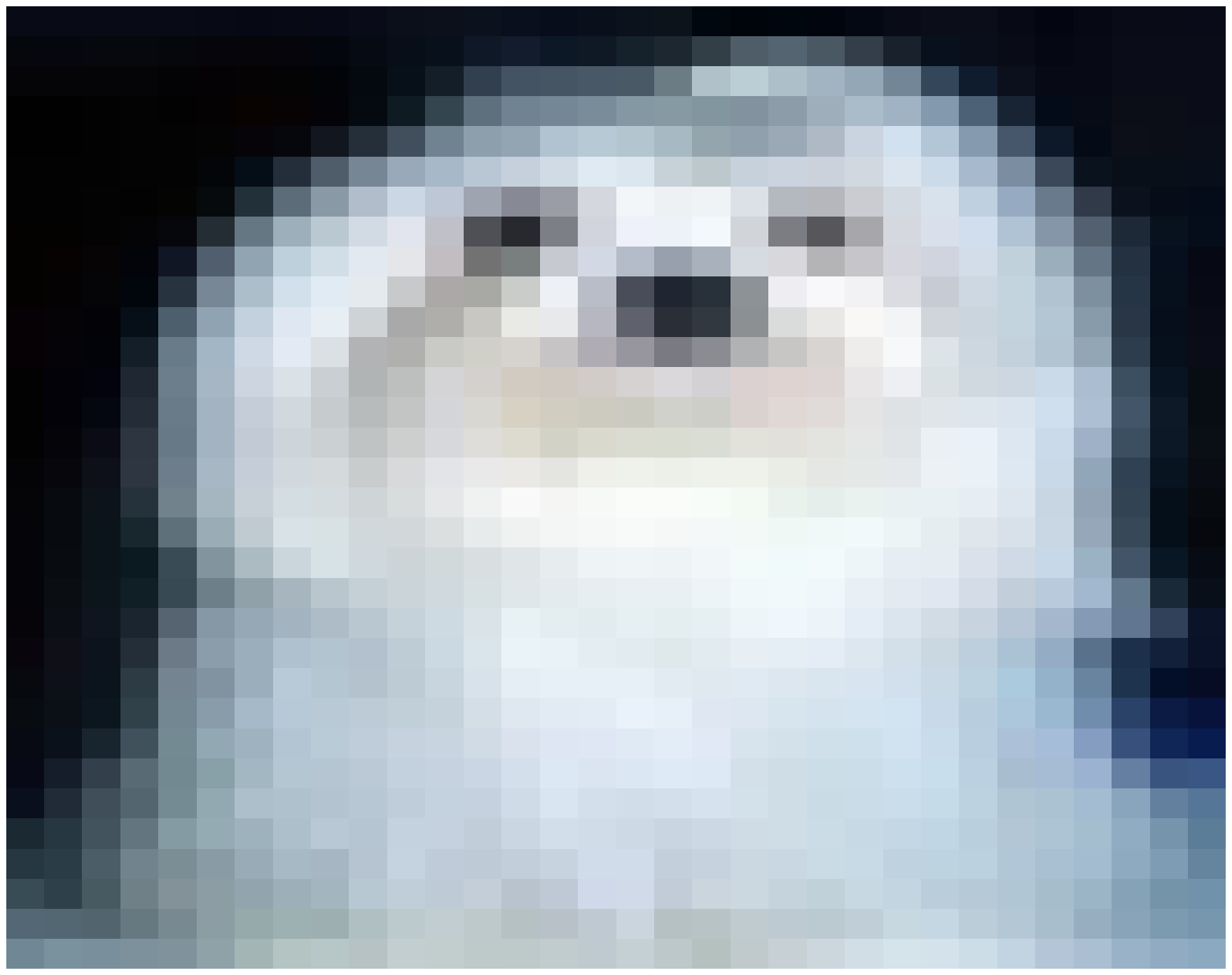}&\includegraphics[width=.1\linewidth]{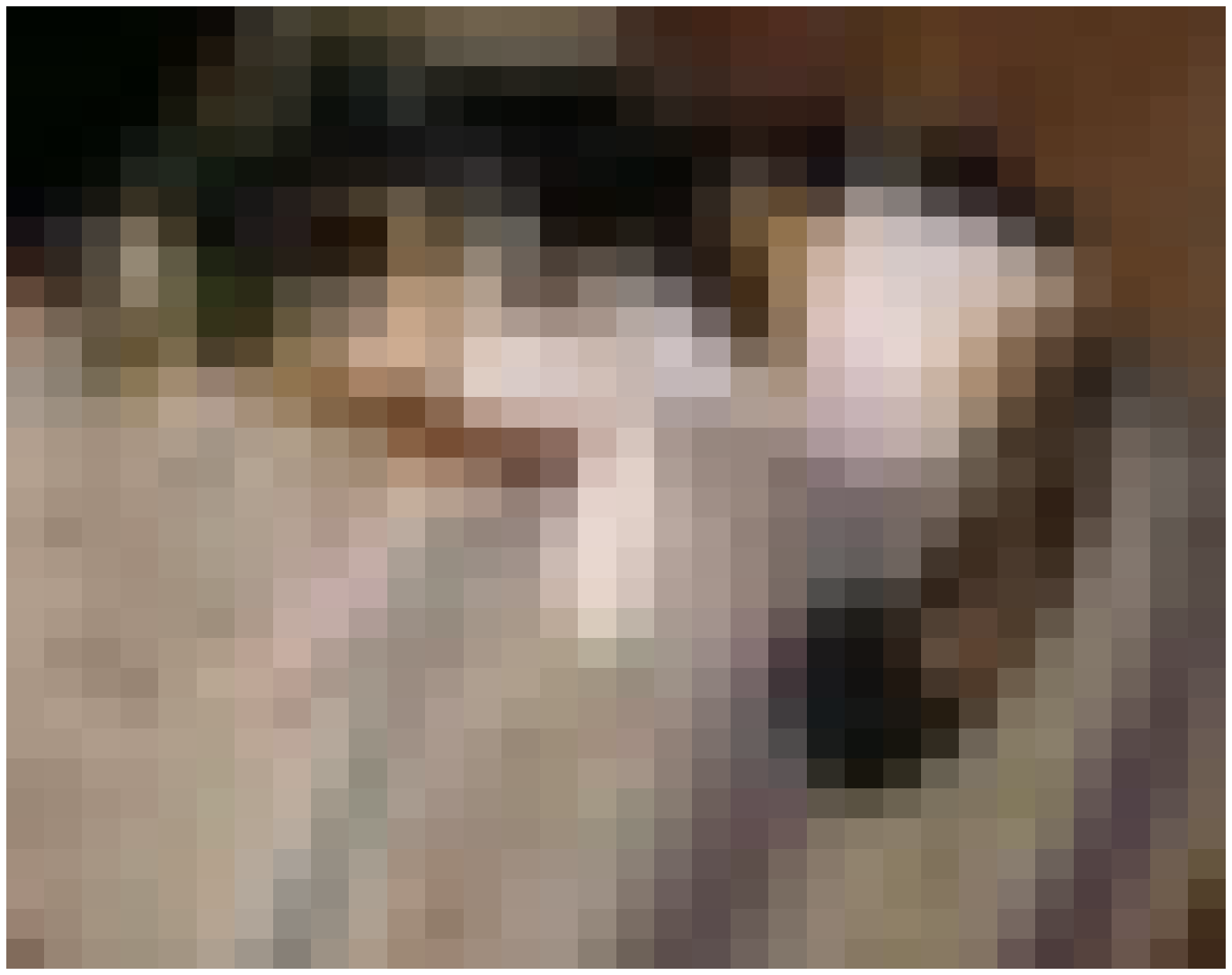}&\includegraphics[width=.1\linewidth]{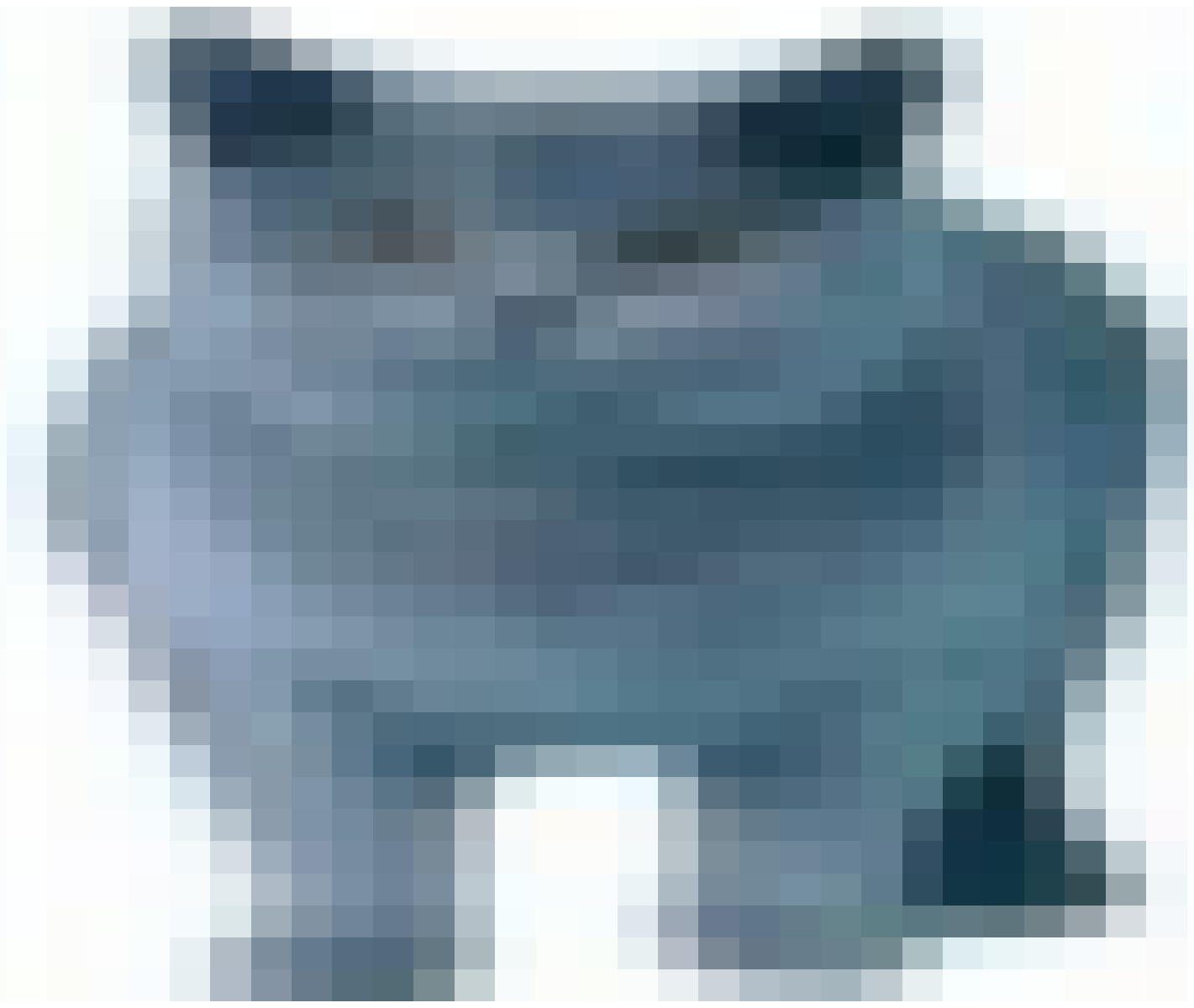}\\
  \end{tabular}
  \caption{Top retrieved images for sample test images from the CIFAR dataset using $L=32$ bits.}
  \label{f:CIFAR_sample}
\end{figure}

\begin{figure}[t]
  \centering
  \psfrag{precision}[][t]{precision}
  \psfrag{recall}[][b]{recall}
  \psfrag{bits}{$L$}
  \begin{tabular}{@{}c@{\hspace{0\linewidth}}c@{\hspace{0\linewidth}}c@{\hspace{0\linewidth}}c@{}}
    $K=50$ neighbors & Hamming distance $\le 1$ & Hamming distance $\le 2$ & Hamming distance $\le 3$ \\
    \includegraphics[width=0.265\linewidth]{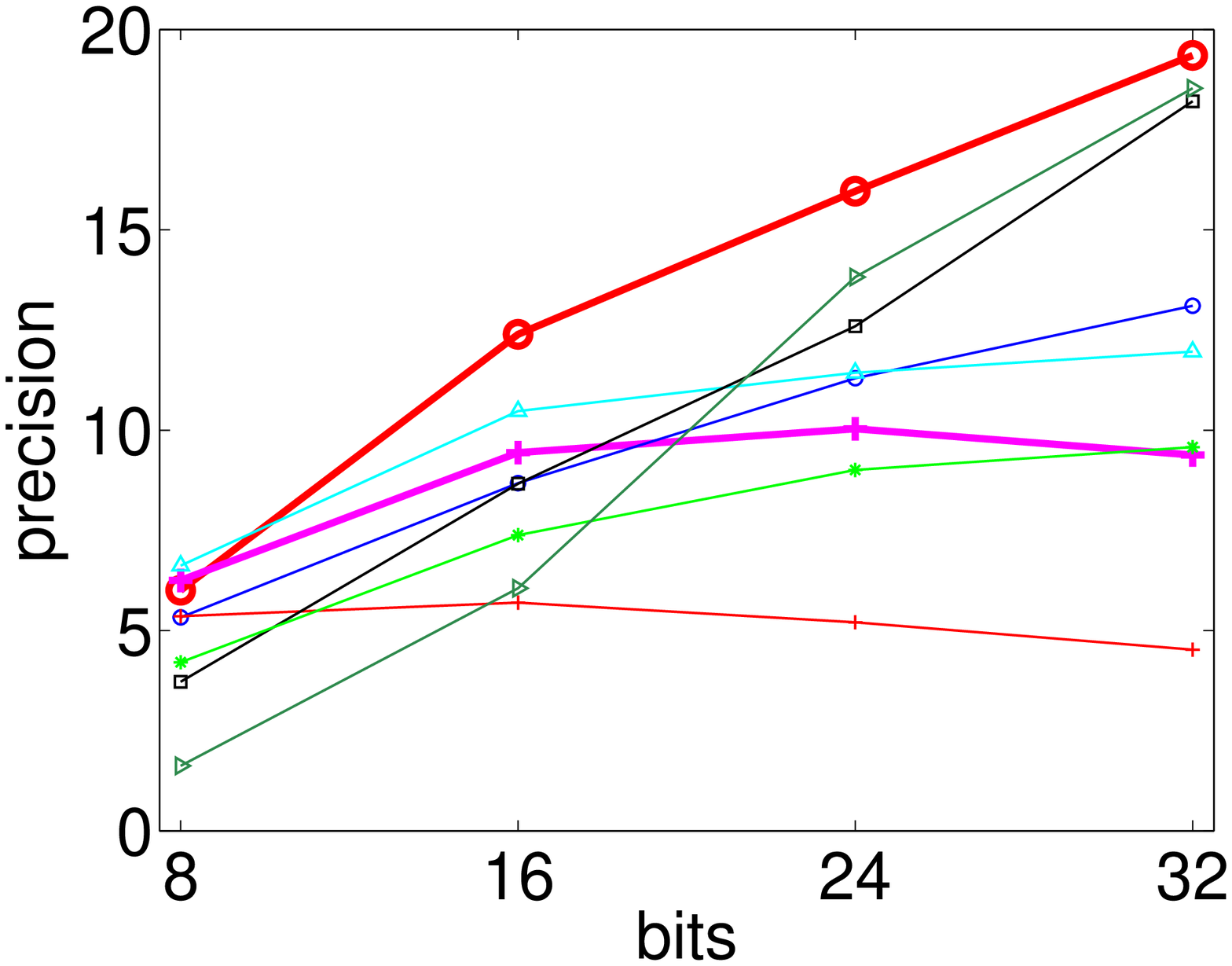} &
    \includegraphics[width=0.245\linewidth,height=0.21\linewidth]{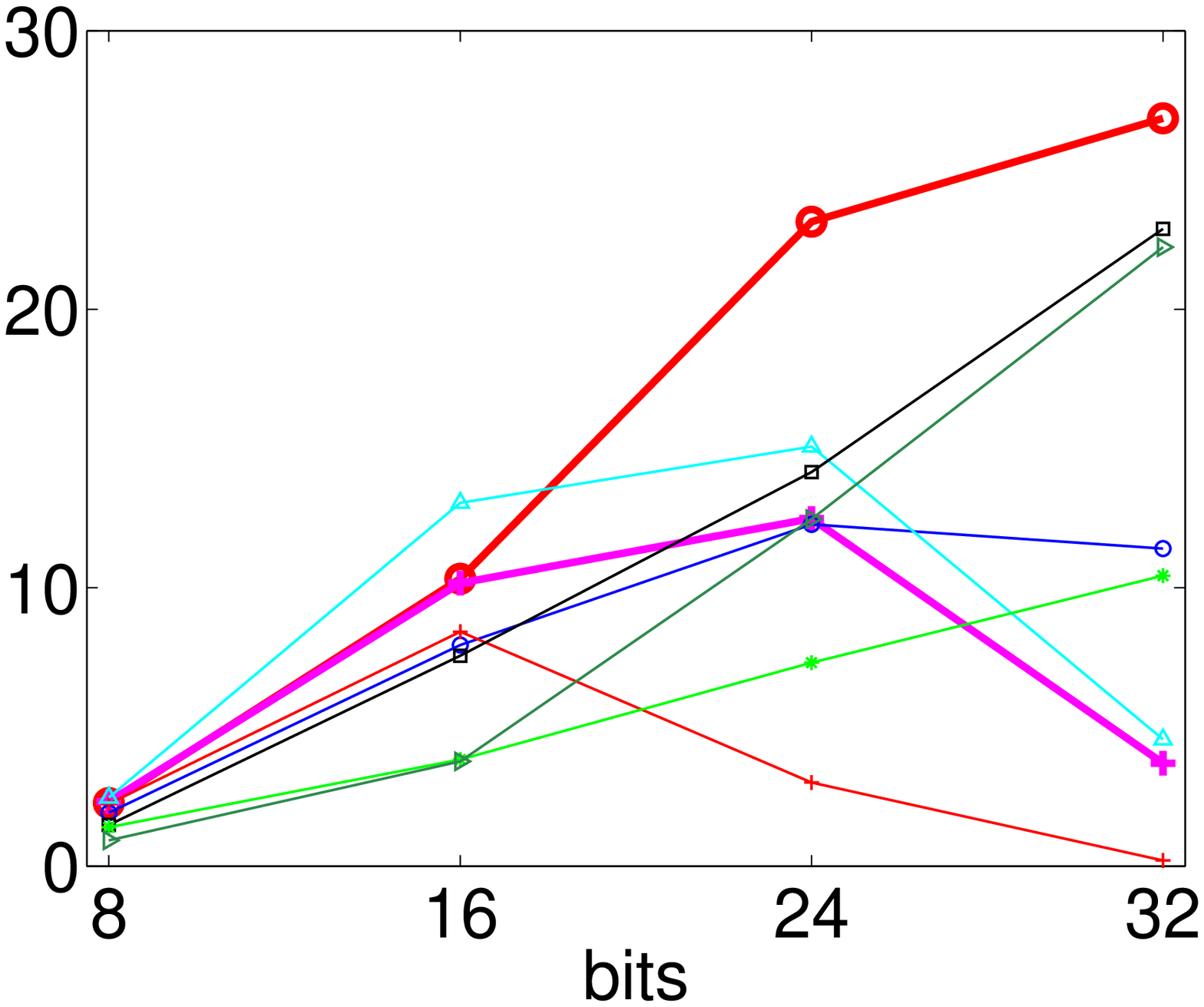} &
    \includegraphics[width=0.245\linewidth,height=0.21\linewidth]{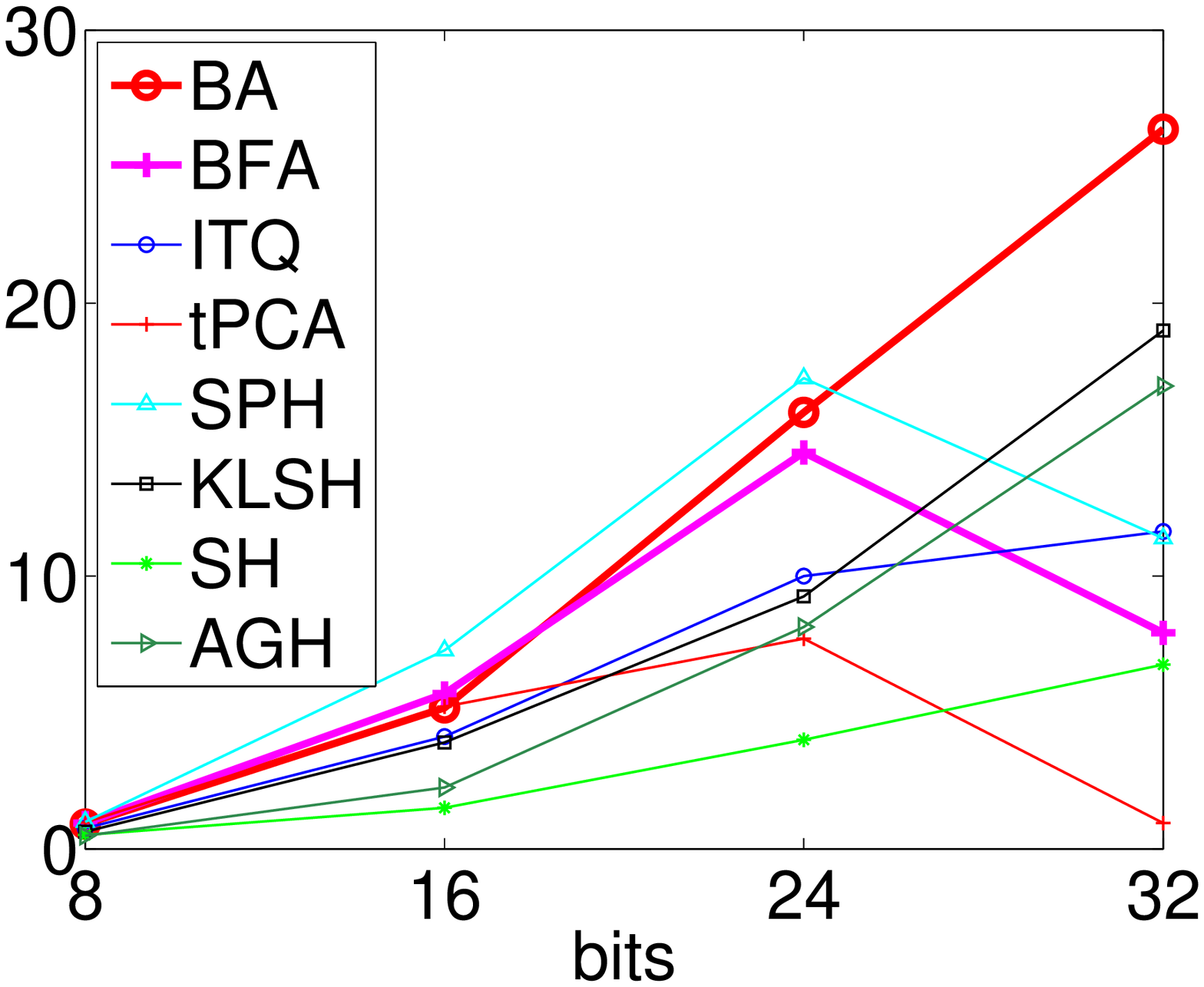} &
    \includegraphics[width=0.245\linewidth,height=0.21\linewidth]{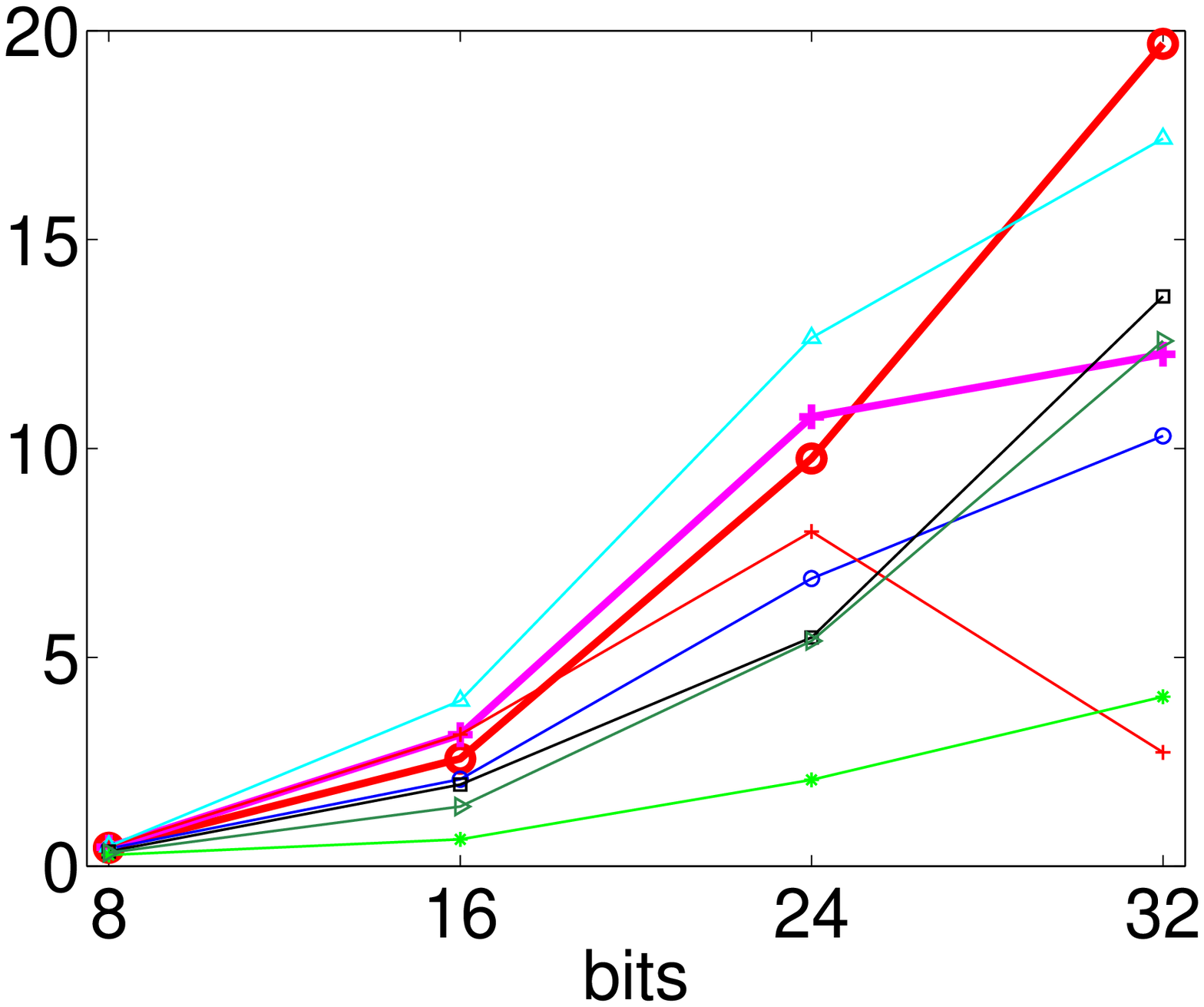} \\
    $L=8$ & $L=16$ & $L=24$ & $L=32$ \\
    \includegraphics[width=.265\linewidth]{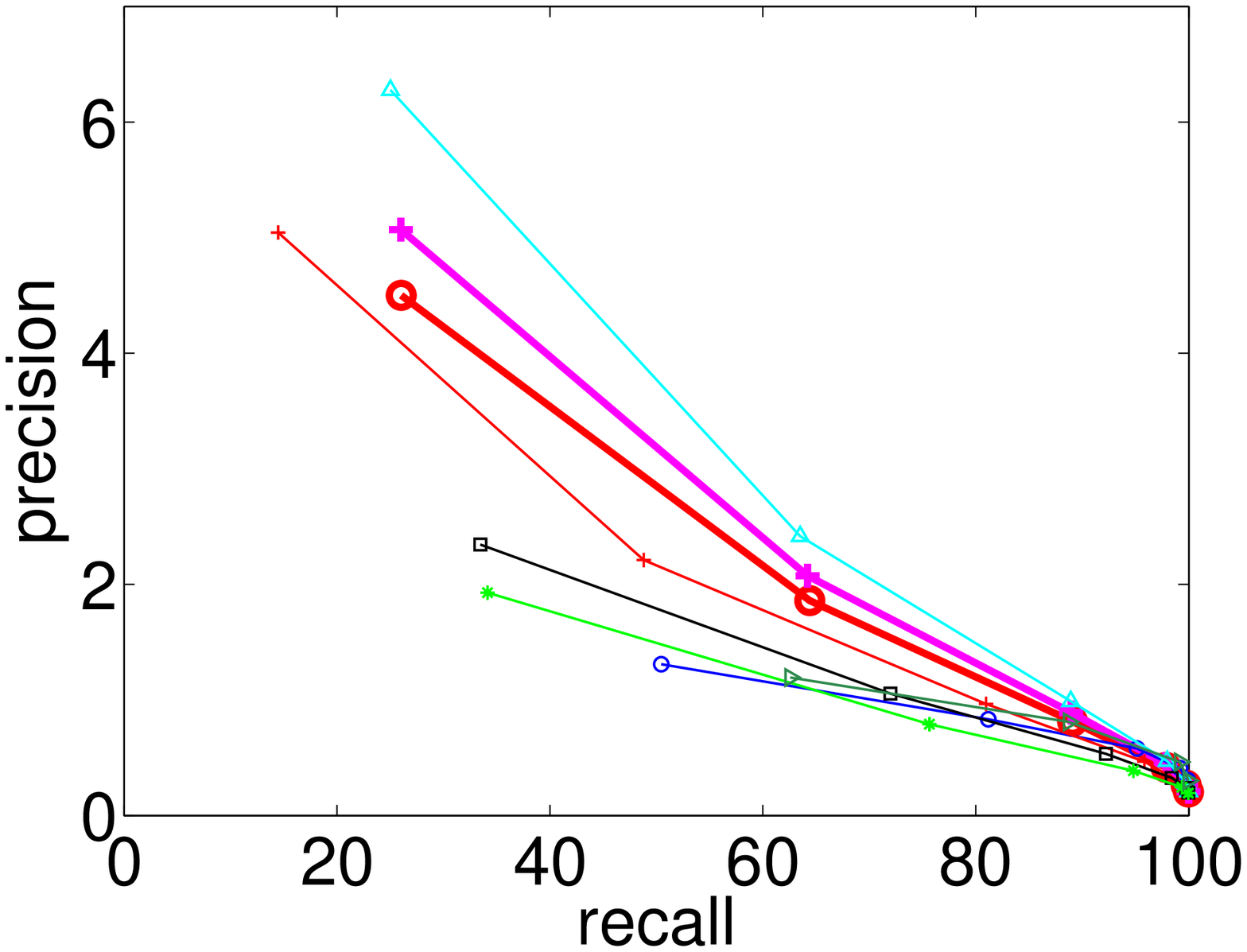} &
    \includegraphics[width=.245\linewidth,height=0.205\linewidth]{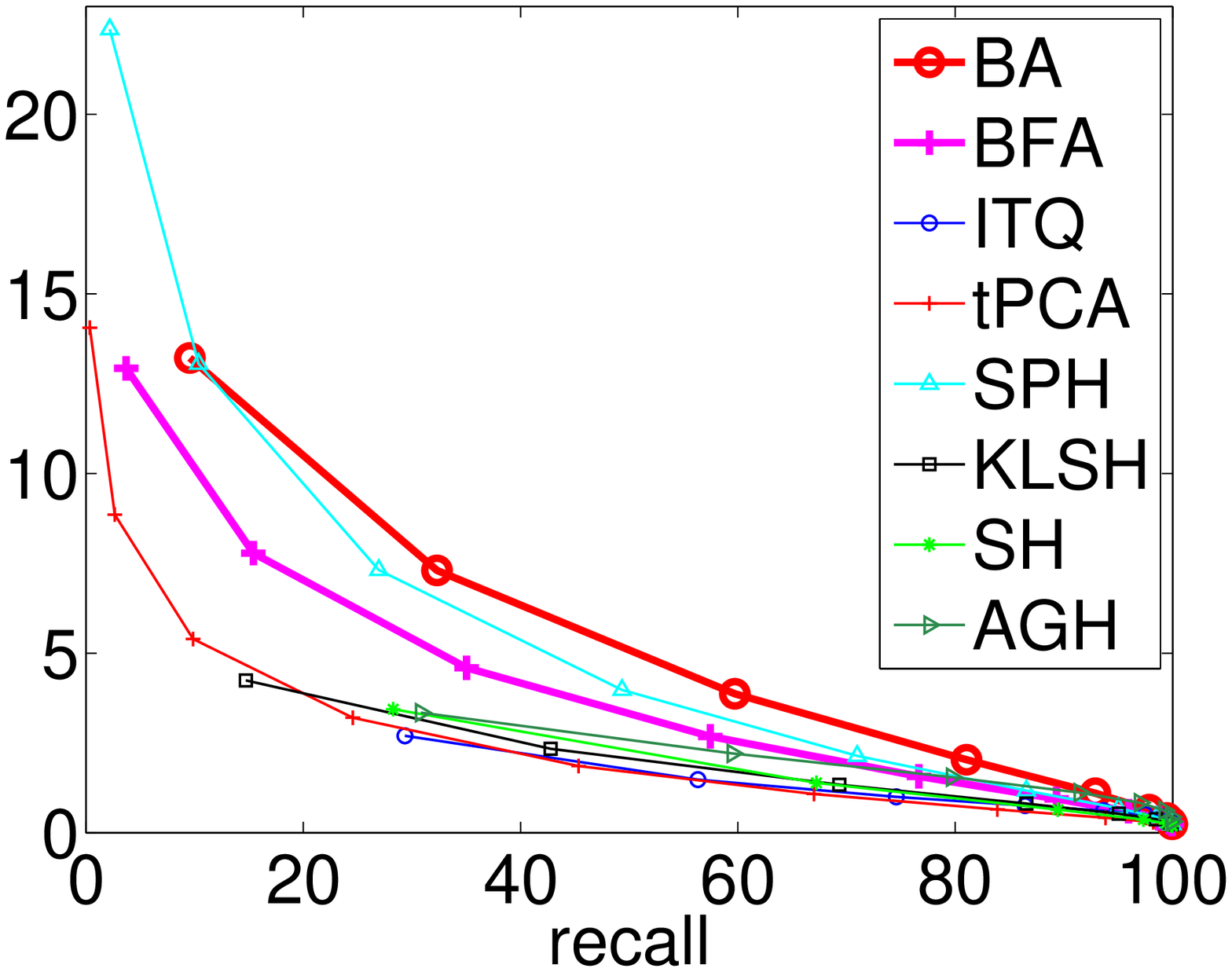} &
    \includegraphics[width=.245\linewidth,height=0.205\linewidth]{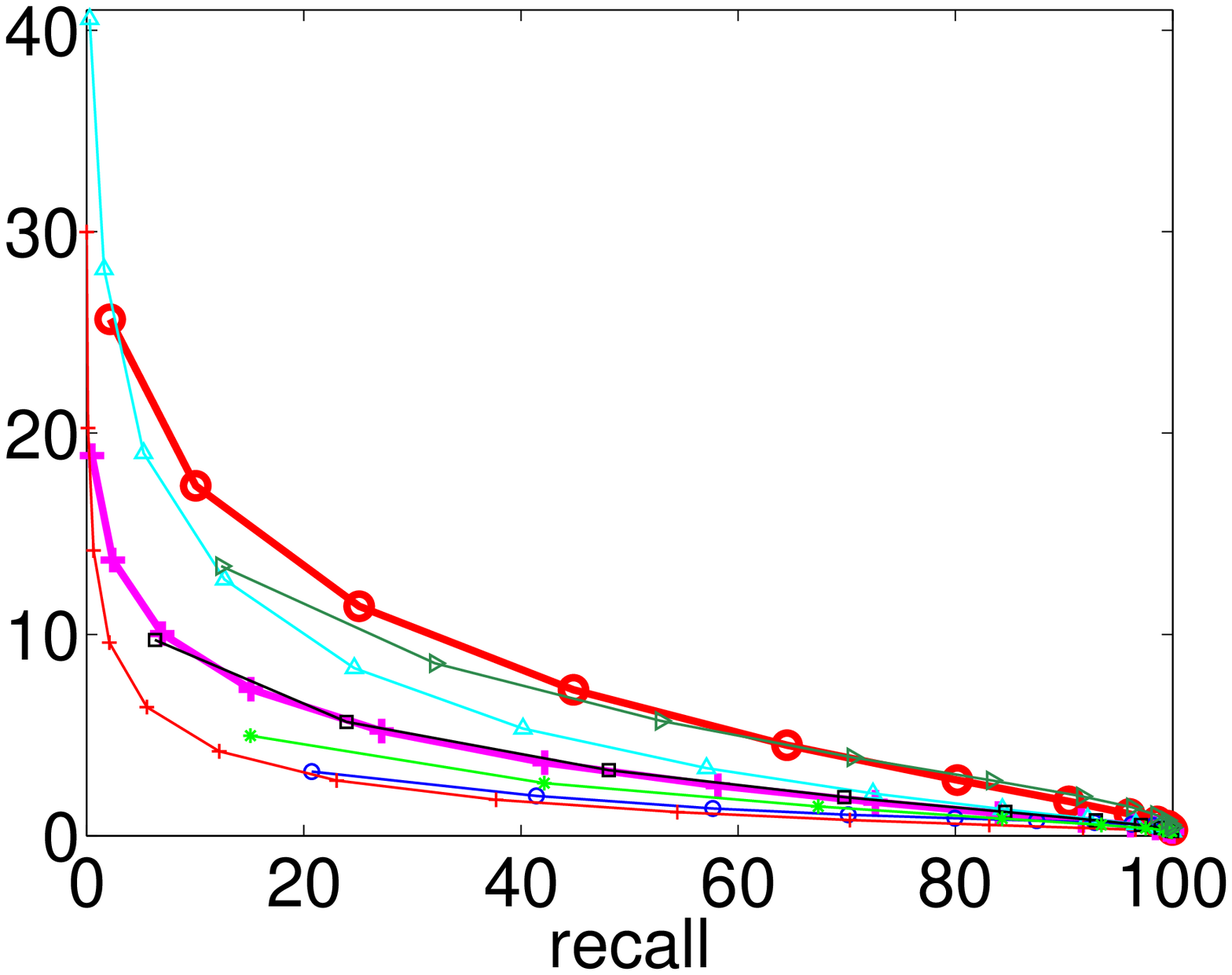} &
    \includegraphics[width=.245\linewidth,height=0.205\linewidth]{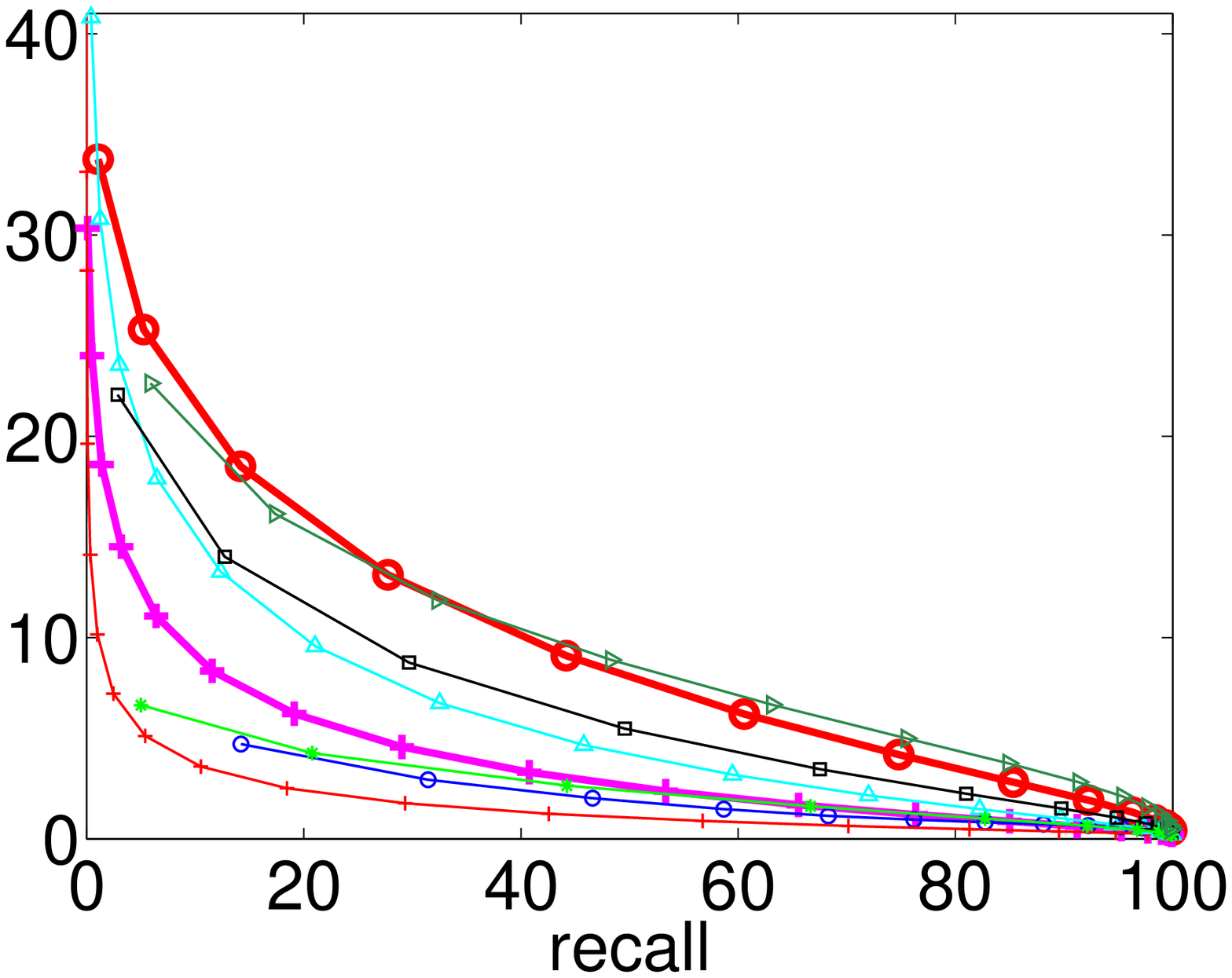}
  \end{tabular}
  \caption{Precision and precision/recall in NUS-WIDE-LITE dataset (plotted as in fig.~\ref{f:NUS-WIDE}). Ground truth: $K=50$ nearest images to the query image in the training set. \emph{Top}: precision depending on the retrieved set ($k=50$ nearest neighbors in Hamming distance, or images at Hamming distance $\le 1$ or $2$ or $3$). \emph{Bottom}: precision/recall curves for a retrieved set of images at Hamming distance $\le r$, using $L = 8$ to $32$ bits.}
  \label{f:NUS-WIDE-LITE}
\end{figure}

\begin{figure}[t]
  \centering
  \psfrag{bins}[t][]{\# used binary codes}
  \psfrag{counts}[][t]{\# vectors per code}
  \psfrag{bits}[][b]{$L$}
  \begin{tabular}{@{}c@{\hspace{0.03\linewidth}}c@{\hspace{0\linewidth}}c@{}}
    \multicolumn{2}{c}{\dotfill NUS-WIDE\dotfill} & \dotfill ANNSIFT-1M\dotfill \\
    \raisebox{0.5ex}{\includegraphics[width=0.31\linewidth]{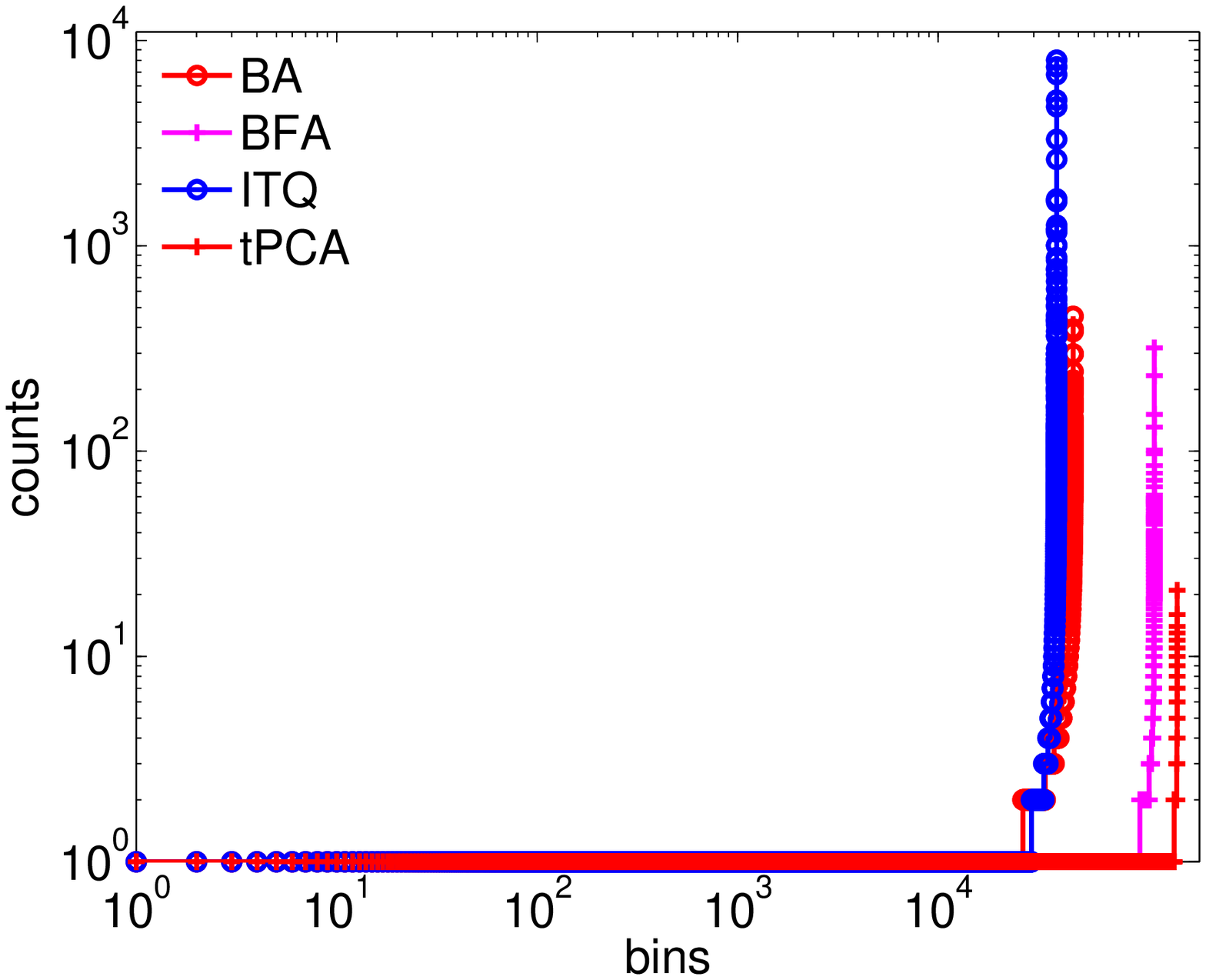}} &
    \psfrag{entropy}[][t]{entropy $L_{\text{eff}}$}
    \includegraphics[width=0.33\linewidth]{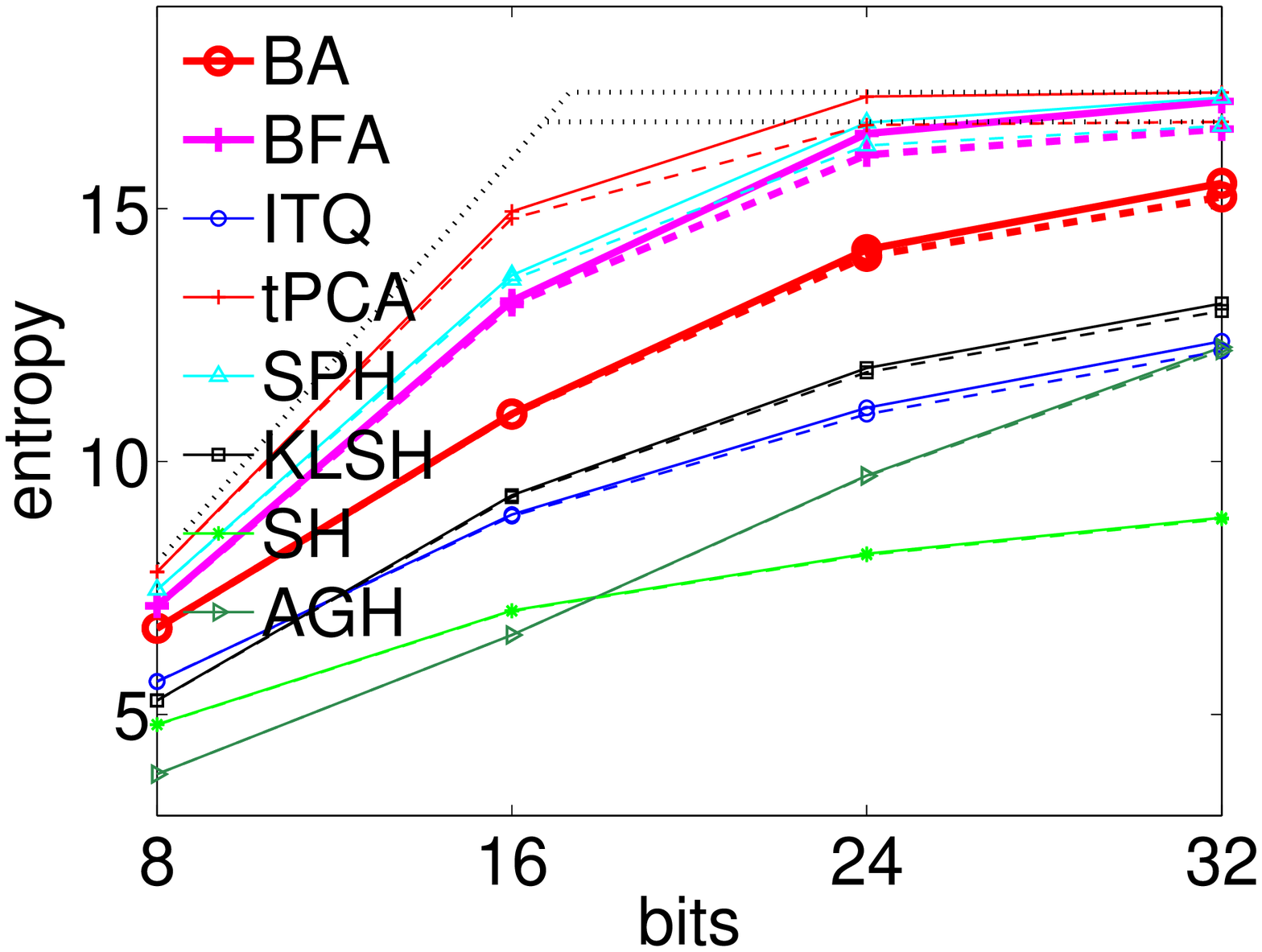} &
    \psfrag{entropy}{}
    \includegraphics[width=0.33\linewidth]{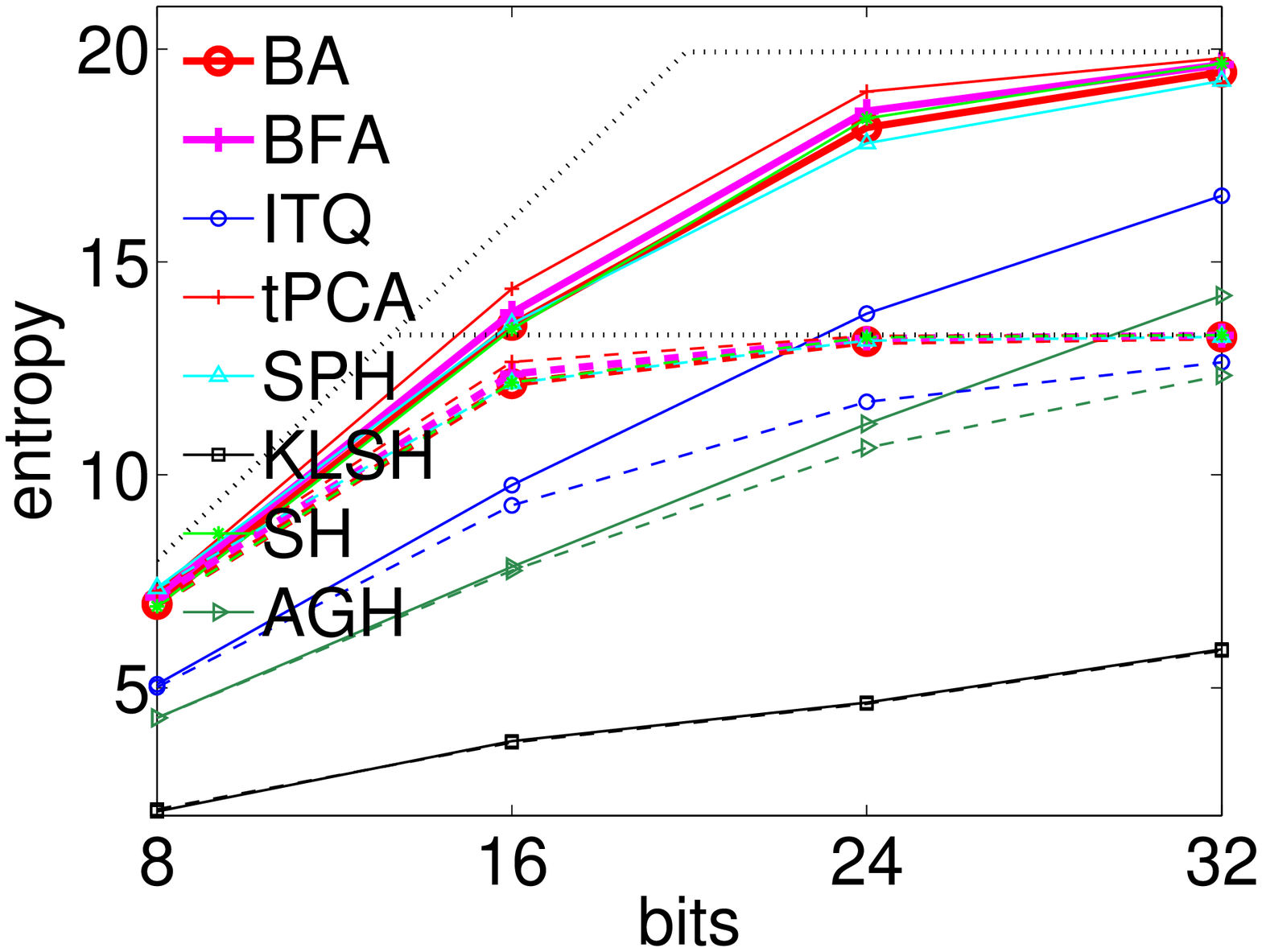}
  \end{tabular}
  \caption{Code utilization in effective number of bits $L_{\text{eff}}$ (entropy of code distribution) of different hashing algorithms, using $L = 8$ to $32$ bits, for the NUS-WIDE (\emph{left}) and ANNSIFT-1M (\emph{right}) datasets. The plots correspond to the codes obtained by the algorithms in figures~\ref{f:NUS-WIDE} and~\ref{f:ANNSIFT-1M} (top row, initialized from AGH), respectively. The hashing algorithms are color-coded, with solid lines for the training set and dashed lines for the test set. The two diagonal-horizontal black dotted lines give the upper bound (maximal code utilization) on $L_{\text{eff}}$ of any algorithm for the training and test sets. The leftmost plot, for NUS-WIDE, shows the histogram of high-dimensional vector counts per binary code, sorted by count, for selected methods. Only codes with at least one vector mapping to them are shown. The entropy of this distribution gives the $L_{\text{eff}}$ value in the middle plot.}
  \label{f:entropy}
\end{figure}

\end{document}